\PassOptionsToPackage{color,dvipsnames}{xcolor}
\documentclass[sigconf]{acmart}

\copyrightyear{2022}
\acmYear{2022}
\setcopyright{acmlicensed}%
\acmConference[CCS '22]{Proceedings of the 2022 ACM SIGSAC Conference on Computer and Communications Security}{November 7--11, 2022}{Los Angeles, CA, USA}
\acmBooktitle{Proceedings of the 2022 ACM SIGSAC Conference on Computer and Communications Security (CCS '22), November 7--11, 2022, Los Angeles, CA, USA}
\acmPrice{15.00}
\acmDOI{10.1145/3548606.3559335}
\acmISBN{978-1-4503-9450-5/22/11}

\usepackage{natbib}
\usepackage{algorithm}
\usepackage[noend]{algorithmic}
\renewcommand{\algorithmicrequire}{\textbf{Input:}}
\renewcommand{\algorithmicensure}{\textbf{Output:}}

\usepackage{microtype}
\usepackage{graphicx}
\usepackage{subcaption}
\usepackage{booktabs} %
\usepackage{balance}

\usepackage{setspace}  %

\usepackage{titletoc}
\titlecontents{section}%
    [1.2em]%
    {\addvspace{0.85em}\bfseries}%
    {\contentslabel{1.2em}}%
    {\hspace*{-1.2em}}%
    {{\mdseries\hspace{0.4em}\titlerule*[4pt]{.}}\contentspage}%
    []
\titlecontents{subsection}%
    [3em]%
    {\addvspace{0.35em}}%
    {\contentslabel{1.8em}}%
    {\hspace*{-1.8em}}%
    {{\mdseries\hspace{0.4em}\titlerule*[4pt]{.}}\contentspage}%
    []
\titlecontents{subsubsection}%
    [5.6em]%
    {\addvspace{0.10em}}%
    {\contentslabel{2.6em}}%
    {\hspace*{-2.6em}}%
    {{\mdseries\hspace{0.4em}\titlerule*[4pt]{.}}\contentspage}%
      []

\makeatletter
  \def\ttl@Hy@steplink#1{%
    \Hy@MakeCurrentHrefAuto{#1*}%
    \edef\ttl@Hy@saveanchor{%
      \noexpand\Hy@raisedlink{%
        \noexpand\hyper@anchorstart{\@currentHref}%
        \noexpand\hyper@anchorend
        \def\noexpand\ttl@Hy@SavedCurrentHref{\@currentHref}%
        \noexpand\ttl@Hy@PatchSaveWrite
      }%
    }%
  }%
  \def\ttl@Hy@PatchSaveWrite{%
    \begingroup
      \toks@\expandafter{\ttl@savewrite}%
      \edef\x{\endgroup
        \def\noexpand\ttl@savewrite{%
          \let\noexpand\@currentHref
              \noexpand\ttl@Hy@SavedCurrentHref
          \the\toks@
        }%
      }%
    \x
  }%
  \def\ttl@Hy@refstepcounter#1{%
    \let\ttl@b\Hy@raisedlink
    \def\Hy@raisedlink##1{%
      \def\ttl@Hy@saveanchor{\Hy@raisedlink{##1}}%
    }%
    \refstepcounter{#1}%
    \let\Hy@raisedlink\ttl@b
  }%
\def\ttl@gobblecontents#1#2#3#4{\ignorespaces}%
\makeatother

\usepackage[noend]{algorithmic}

\usepackage{placeins}   %
\usepackage{tablefootnote}  %

\DisableLigatures[f]{encoding = *, family = * }

\usepackage{standalone}

\usepackage[shortcuts]{extdash}

  \newcommand{\revised}[1]{{#1}}
  \newcommand{\revTwo}[1]{{#1}}
\usepackage{amsmath}
\DeclareMathOperator*{\argmax}{arg\,max}
\DeclareMathOperator*{\argmin}{arg\,min}
\DeclareMathOperator{\sign}{sgn}
\newcommand{\sgnp}[1]{{\sign \left( #1 \right) }}
\usepackage{amsfonts}  %
\usepackage{bbm}       %
\usepackage{mathrsfs}  %
\usepackage{relsize}   %

\usepackage{icomma}

\usepackage{mathtools} %
\usepackage{etoolbox}
\newcommand\swapifbranches[3]{#1{#3}{#2}}
\makeatletter
\MHInternalSyntaxOn
\patchcmd{\DeclarePairedDelimiter}{\@ifstar}{\swapifbranches\@ifstar}{}{}
\MHInternalSyntaxOff
\makeatother
\DeclarePairedDelimiter{\sbrack}{\lbrack}{\rbrack}
\DeclarePairedDelimiter{\angbrack}{\langle}{\rangle}
\DeclarePairedDelimiter{\floor}{\lfloor}{\rfloor}

\DeclarePairedDelimiter{\abs}{\lvert}{\rvert}

\DeclarePairedDelimiter{\norm}{\lVert}{\rVert}
\usepackage{bm}
\DeclarePairedDelimiterX\set[1]\lbrace\rbrace{#1}
\DeclarePairedDelimiterX\setbuild[2]\lbrace\rbrace{#1 \bm: #2}
\newcommand{\setbuildDynamic}[2]{\left\{#1 \bm: #2\right\}}

\newcommand{\func}[3]{{#1:#2\rightarrow#3}}
\newcommand{\defeq}{\coloneqq}
\newcommand{\fedeq}{\eqqcolon}

\newcommand{\ind}[1]{\mathbbm{1}_{#1}}

\newcommand{\eqsmall}[1]{{\small #1}}

\newcommand{\intsNN}{\mathbb{Z}_{+}}

\newcommand{\real}{\mathbb{R}}
\newcommand{\realnn}{\real_{{\geq}0}}  %
\newcommand{\realnp}{\real_{{\leq}0}}  %

\newcommand{\transpose}{^{\intercal}}

\newcommand{\bigO}[1]{{\mathcal{O}(#1)}}
\newcommand{\bigTheta}[1]{{\Theta(#1)}}

\usepackage{array}  %
\usepackage{arydshln}  %
\usepackage{bigdelim}
\usepackage{booktabs}
\usepackage{multirow}
\usepackage{makecell}  %

\newtheorem*{theorem*}{Theorem}

\newcommand{\keyword}[1]{\textit{#1}}
\usepackage{tikz}
\usetikzlibrary{arrows,backgrounds,calc,decorations.markings,patterns,positioning,shadows,shapes}

\usepackage{pgfplots}
\pgfplotsset{compat=1.13}
\usepgfplotslibrary{colorbrewer}
\usepgfplotslibrary{fillbetween}
\usepgfplotslibrary{statistics}
\usepackage{pgfplotstable}
\usepackage{graphicx} %
\graphicspath{{img/}} %

\usepackage{xifthen}
\newcommand{\ifempty}[3]{%
  \ifthenelse{\isempty{#1}}{#2}{#3}%
}

\newcommand{\dotprod}[2]{\angbrack{ #1 ,~ #2 }}
\newcommand{\dotprodbig}[2]{\big\langle{ #1 ,~ #2 }\big\rangle}

\newcolumntype{H}{>{\setbox0=\hbox\bgroup}c<{\egroup}@{}}

\patchcmd{\NAT@test}{\else \NAT@nm}{\else \NAT@nmfmt{\NAT@nm}}{}{}

\DeclareRobustCommand\citepos
  {\begingroup
   \let\NAT@nmfmt\NAT@posfmt%
   \NAT@swafalse\let\NAT@ctype\z@\NAT@partrue
   \@ifstar{\NAT@fulltrue\NAT@citetp}{\NAT@fullfalse\NAT@citetp}}

\let\NAT@orig@nmfmt\NAT@nmfmt
\def\NAT@posfmt#1{\NAT@orig@nmfmt{#1's}}
\newcommand{\sourceCodeUrl}{https://github.com/ZaydH/target\_identification}  %

\newcommand{\AlgFontSize}{\small}
\newcommand{\TableFontSize}{\small}

\newcommand{\fitname}{FIT}
\newcommand{\fit}{{\fitname}}
\newcommand{\fitWith}[1]{\mbox{\fit{} w/ #1}}
\newcommand{\methodname}{GAS}
\newcommand{\simShort}{Rn.{}}
\newcommand{\layerwiseSuffix}{\=/L}
\newcommand{\bothSuffix}{(\layerwiseSuffix)}
\newcommand{\baseMethod}[2]{#1\textsubscript{#2}}
\newcommand{\baseZeroMethod}[1]{\baseMethod{#1}{$0$}}
\newcommand{\baseFinMethod}[1]{\baseMethod{#1}{$\nItr$}}
\newcommand{\method}{\textsc{\methodname}}
\newcommand{\bothMethod}{\method\bothSuffix}
\newcommand{\methodZero}{\baseZeroMethod{\method}}
\newcommand{\methodFin}{\baseFinMethod{\method}}
\newcommand{\layername}{\methodname\layerwiseSuffix}
\newcommand{\layer}{\textsc{\layername}}
\newcommand{\layerZero}{\baseZeroMethod{\layer}}

\newcommand{\tracin}{TracIn}
\newcommand{\tracinCP}{TracInCP}

\newcommand{\SST}{SST\=/2}

\newcommand{\influence}{\mathcal{I}}
\newcommand{\infIF}{\influence_{\text{IF}}}
\newcommand{\ifSign}{}
\newcommand{\infRepPt}{\influence_{\text{RP}}}
\newcommand{\infTracIn}{\influence_{\text{\tracin}}}
\newcommand{\infTracInFunc}{\infFunc{\infTracIn}{\zI}{\zHatTe}}
\newcommand{\infTracInCP}{\influence_{\text{\tracinCP}}}
\newcommand{\infTracInCPFunc}{\infFunc{\infTracInCP}{\zI}{\zHatTe}}
\newcommand{\infFunc}[3]{{#1 \left( #2 , #3 \right)}}
\newcommand{\renormInf}{\widetilde{\influence}}
\newcommand{\simIF}{\renormInf_{\text{IF}}}
\newcommand{\simRepPt}{\renormInf_{\text{RP}}}
\newcommand{\simTracIn}{\renormInf_{\text{\tracin}}}
\newcommand{\simTracInFunc}{\infFunc{\simTracIn}{\zI}{\zHatTe}}
\newcommand{\simTracInCP}{\renormInf_{\text{\tracinCP}}}
\newcommand{\simTracInCPFunc}{\infFunc{\simTracInCP}{\zI}{\zHatTe}}
\newcommand{\fRP}[1]{\mathbf{\dec}_{#1}}
\newcommand{\rpDeriv}{\frac{\partial \lFunc{\zI}{\wFin}}{\partial \acts_{\yI}}}
\newcommand{\indVar}{u}

\newcommand{\LossOnlyNorm}{\eqsmall{${\big\lVert \rpDeriv \big\rVert}$}}
\newcommand{\ActsGradient}{%
  \frac{\partial \acts}%
       {\partial \X} %
}

\newcommand{\neighSize}{k}
\newcommand{\KNN}{$\neighSize$\=/NN}
\newcommand{\deepknn}{Deep~\KNN}
\newcommand{\hydra}{\textsc{HyDRA}}

\newcommand{\itr}{t}

\newcommand{\subsetItr}{\mathcal{T}}
\newcommand{\nItr}{T}

\newcommand{\lr}{\eta}
\newcommand{\lrT}{\lr_{\itr}}
\newcommand{\wdecay}{\lambda}  %

\newcommand{\dec}{f}           %
\newcommand{\decFunc}[2]{{\dec ( #1 ; #2 )}}           %

\newcommand{\acts}{a}
\newcommand{\domainActs}{\mathcal{A}}

\newcommand{\gradActs}{\nabla_{\acts}}
\newcommand{\gradActsSq}{\gradActs^{2}}
\newcommand{\normActs}[1]{\normW{#1}}

\newcommand{\trainParams}{\mathcal{P}}

\newcommand{\emptyHat}{\widehat{~}}
\newcommand{\hatFunc}[1]{\widehat{#1}}
\newcommand{\z}{z}
\newcommand{\zHat}{\hatFunc{\z}}
\newcommand{\zI}{\z_{\trIdx}}

\newcommand{\X}{x}
\newcommand{\xI}{\X_{\trIdx}}
\newcommand{\y}{y}
\newcommand{\yHat}{\hatFunc{\y}}
\newcommand{\yI}{\y_{\trIdx}}
\newcommand{\domainX}{\mathcal{X}}
\newcommand{\domainY}{\mathcal{Y}}
\newcommand{\targStr}{\text{targ}}
\newcommand{\targFeatSet}{\mathcal{X}_{\targStr}}
\newcommand{\targSet}{\hatFunc{\mathcal{Z}}_{\targStr}}
\newcommand{\numTarg}{m}
\newcommand{\targIdx}{j}
\newcommand{\zHatTarg}{\zHat_{\targStr}}
\newcommand{\xTarg}{\X_{\targStr}}
\newcommand{\yTarg}{\y_{\targStr}}

\newcommand{\teStr}{\text{te}}
\newcommand{\zTeFilt}{\z_{\text{filt}}}
\newcommand{\zHatTe}{\zHat_{\teStr}}
\newcommand{\xTe}{\X_{\teStr}}

\newcommand{\yHatTe}{\yHat_{\teStr}}
\newcommand{\yAdv}{\y_{\text{adv}}}

\newcommand{\dimW}{\abs{\w}}
\newcommand{\gradW}{\nabla_{\w}}
\newcommand{\gradWSq}{\gradW^{2}}
\newcommand{\normW}[1]{{\norm{#1}_2}}
\newcommand{\hess}{H_{\w}}
\newcommand{\invHess}{\hess^{{-}1}}
\newcommand{\w}{\theta}
\newcommand{\wT}[1][\itr]{\w_{#1}}
\newcommand{\wTOne}{\wT[\itr - 1]}
\newcommand{\wZero}{\w_{0}}
\newcommand{\wFin}{\w_{\nItr}}
\newcommand{\batch}{\mathcal{B}}
\newcommand{\batchT}[1][\itr]{\batch_{#1}}
\newcommand{\batchOne}{\batchT[1]}
\newcommand{\batchFin}{\batchT[\nItr]}
\newcommand{\batchSize}{b}

\newcommand{\nTr}{n}

\newcommand{\numSubep}{\omega}

\newcommand{\trainBase}{\mathcal{D}}
\newcommand{\fullTrain}{\trainBase_{\text{tr}}}

\newcommand{\trIdx}{i}
\newcommand{\trIdxP}{l}
\newcommand{\advTrain}{\trainBase_{\text{adv}}}
\newcommand{\cleanTrain}{\trainBase_{\text{cl}}}

\newcommand{\gradLetter}{g}
\newcommand{\gradFunc}[1]{\gradLetter_{#1}}
\newcommand{\gradItr}[2][\itr]{{#2^{(#1)}}}
\newcommand{\gradHat}{\hatFunc{\gradLetter}}
\newcommand{\gradHatFunc}[1]{\gradHat_{#1}}
\newcommand{\gradHatTarg}{\gradHatFunc{\targStr}}
\newcommand{\gradHatTe}{\gradHatFunc{\teStr}}
\newcommand{\gradI}{\gradFunc{\trIdx}}

\newcommand{\lossScalar}{\widetilde{\ell}}
\newcommand{\lScalarFunc}[1]{{\lossScalar\left( #1 \right)}}
\newcommand{\gradScalarLoss}[1]{{\gradActs \lScalarFunc{#1}}}
\newcommand{\loss}{\ell}
\newcommand{\lActsFunc}[2]{{\loss\left(#1 , #2 \right)}}

\newcommand{\riskSym}{\mathcal{L}}
\newcommand{\lFunc}[2]{{\riskSym(#1 ; #2)}}
\newcommand{\riskAdv}{\riskSym_{\text{adv}}}

\newcommand{\gradLoss}[2]{{\gradW \lFunc{#1}{#2}}}

\newcommand{\infScalarI}{\infVec_{\trIdx}}
\newcommand{\infScalarIP}{\infVec_{\trIdxP}}
\newcommand{\infVec}{\mathbf{v}}
\newcommand{\anomCount}{\kappa}
\newcommand{\anomCutoff}{\zeta}
\newcommand{\anomScore}{\sigma}
\newcommand{\anomScoreI}{\anomScore_{\trIdx}}
\newcommand{\anomScoreVec}{\bm{\anomScore}}
\newcommand{\scoreCenter}{\mu}
\newcommand{\scoreSpread}{s}

\newcommand{\filtTrain}{\widetilde{\trainBase}_{\text{tr}}}

\newcommand{\filtWFin}{\widetilde{\w}_{\nItr}}

\newcommand{\baseMethodFunc}[4]{{#1(#2 ; #3\ifempty{#4}{}{, #4})}}
\newcommand{\MethodFunc}[3]{\baseMethodFunc{\method}{#1}{#2}{#3}}
\newcommand{\RenormInfMethodName}{\textsc{Inf}}
\newcommand{\SimFunc}[3]{\baseMethodFunc{\RenormInfMethodName}{#1}{#2}{#3}}

\newcommand{\targAnalysisSet}{\hatFunc{\mathcal{Z}}_{\text{te}}}
\newcommand{\simDistSet}{\mathcal{V}}
\newcommand{\heavinessSet}{\mathcal{H}}

\newcommand{\TeIdx}[1]{{#1^{(j)}}}

\newcommand{\scoreDistSet}{\Sigma}

\newcommand{\median}{\text{med}}

\newcommand{\Qn}{Q}
\newcommand{\dConsist}{c}

\newcommand{\qnOrderStat}{r}

\newcommand{\SupplementaryMaterialsTitle}{%
  \vbox{
    \hrule height 4pt
    \vskip 0.05in
    \vskip -\parskip%
    \begin{center}
      {\Huge\bf \titletext{} \par}

      \vspace{8pt}
      {\Huge Supplemental Materials \par}
    \end{center}
    \vskip 0.15in
    \vskip -\parskip
    \hrule height 1pt
    \vskip 0.09in%
  }
}

\newcommand{\nLayer}{L}

\newcommand{\robertaBase}{RoBERTa\textsubscript{BASE}}

\newcommand{\LayerConv}[4]{%
  In=#1 & Out=#2 & Kernel=${#3 \times #3}$ & Pad=#4 %
}

\newcommand{\LayerMaxPoolTwoD}[1]{%
  MaxPool2D & ${#1 \times #1}$%
}

\newcommand{\LayerBatchNormTwoD}[1]{%
  BatchNorm2D & Out=#1
}

\newcommand{\LayerReluActivation}{%
  ReLU
}

\newcommand{\LayerLinear}[2]{%
  Linear#1   & Out=#2%
}

\newcommand{\PVal}[2]{#1 $\pm$ #2}
\newcommand{\PValB}[2]{\textBF{#1} $\pm$ \textBF{#2}}

\newcommand{\OneVal}{1\phantom{.000}}
\newcommand{\ZeroVal}{0\phantom{{.000}}}

\newcommand{\OneValB}{\textBF{1}\phantom{\textBF{.000}}}
\newcommand{\ZeroValB}{\textBF{0}\phantom{\textBF{.000}}}

\newcommand{\RETURN}{\textbf{return}~}
\newcommand{\algcomment}[1]{\hfill$\triangleright$~#1}

\newcommand{\algSetStretch}{\setstretch{1.08}}

\newcommand{\retrainFunc}[2]{{\textsc{Retrain}(#1, #2)}}

\newcommand{\MitigateMethodName}{\textsc{Mitigate}}
\newcommand{\mitigateFunc}[3]{{\MitigateMethodName(#1,~~#2,~~#3)}}

\newcommand{\AnomScoreMethodName}{\textsc{AnomScore}}
\newcommand{\AnomScoreFunc}[2]{{\AnomScoreMethodName(#1,~~#2)}}

\usepackage[short,c2,nocomma]{optidef}   %

\newcommand{\nSurrogate}{m}
\newcommand{\nPoisJointOpt}{K}
\newcommand{\itrSurrogate}{j}
\newcommand{\itrPois}{l}

\newcommand{\xPois}{x_{\itrPois}}
\newcommand{\xPoisClean}{x_{\itrPois}^{\text{cl}}}

\newcommand{\methodSurrogate}{\widehat{\text{\method}}}

\newcommand{\gradSurrPois}{\gradFunc{\itrPois}^{(\itrSurrogate)}}
\newcommand{\gradSurrTarg}{\gradHatTarg^{(\itrSurrogate)}}

\newcommand{\surrogateHyper}{\lambda}

\newcommand{\convCoeff}{c^{(\itrSurrogate)}_{\itrPois}}

\newcommand{\featExtractSymbol}{\phi}
\newcommand{\featFunc}[2][\itrSurrogate]{{\featExtractSymbol^{(#1)}{\left( #2 \right)}}}

\newcommand{\targColor}{RubineRed}  %
\newcommand{\advColor}{violet}
\newcommand{\colorAuto}{{\color{\targColor}\texttt{auto}}}
\newcommand{\colorYTarg}{{\color{\targColor}\eqsmall{$\yTarg$}}}

\newcommand{\colorYAdv}{{\color{\advColor}\eqsmall{$\yAdv$}}}
\newcommand{\colorDog}{{\color{\advColor}\texttt{dog}}}

\newcommand{\colorMethodFunc}[2]{{\color{#2}#1}}
\newcommand{\colorMethodAdv}{\colorMethodFunc{\method}{\advColor}}
\newcommand{\colorMethodTarg}{\colorMethodFunc{\method}{\targColor}}
\newcommand{\colorLayerAdv}{\colorMethodFunc{\layer}{\advColor}}
\newcommand{\colorLayerTarg}{\colorMethodFunc{\layer}{\targColor}}
\newcommand{\eat}[1]{}
\newcommand{\ours}{~{\tiny(ours)}}

\newcommand{\ExpResBarChartHeight}{3.5cm}

\newcommand{\CifarPoisBarWidth}{3.5pt}
\newcommand{\NlpPoisBarWidth}{2.8pt}
\newcommand{\BackdoorSpeechBarWidth}{3.5pt}

\newcommand{\BarLineWidth}{1.0pt}

\newcommand{\DetectBarWidthVal}{3.6pt}
\newcommand{\StaticIdentBarWidth}{4.8pt}
\newcommand{\StaticIdentInterBarSpacing}{1.3pt}

\newcommand{\BarIdentMainHeight}{4.00cm}

\newcommand{\BarDetectMainHeight}{\BarIdentMainHeight}
\newcommand{\BarDetectMainWidth}{\columnwidth}

\newcommand{\IdentYLabel}{$\advTrain$~AUPRC}
\newcommand{\DetectYLabel}{Target AUPRC}

\newcommand{\BarLossNormOnlyWidth}{7.90cm}
\newcommand{\BarLossNormOnlyHeight}{4.20cm}
\newcommand{\LossDerivBarWidth}{4.8pt}

\pgfplotsset{compat=1.13,
  /pgfplots/ybar legend/.style={
    /pgfplots/legend image code/.code={%
       \draw[##1,/tikz/.cd,bar width=6pt,yshift=-0.2em,bar shift=0pt]
       plot coordinates {(0cm,0.6em)};
    },
  },
  minor grid style={gray!20!white},
}

\newcommand{\CifarIdentificationHeight}{3.8cm}

\newcommand{\NlpIdentificationHeight}{\CifarIdentificationHeight}

\newcommand{\QQHeight}{3.20cm}
\newcommand{\LayerTrendHeight}{3.80cm}
\newcommand{\BottomFiltHeight}{4.10cm}
\newcommand{\BottomFiltWidth}{8.2cm}
\newcommand{\InfFiltHeight}{3.85cm}

\newcommand{\JointOptSupplementAdvIdentHeight}{4.25cm}
\newcommand{\JointOptSupplementTargIdentHeight}{4.25cm}

\newcommand{\AdvIdentAblationHeight}{4.00cm}
\newcommand{\TargIdentAblationHeight}{4.00cm}

\newcommand{\CdfLineWidth}{0.7}

\newcommand{\BottomFiltColor}{\textcolor{Plum}{\textbf{purple}}}
\newcommand{\FullSetBarColor}{\textcolor{Gray}{\textbf{gray}}}  %

\newcommand{\numEle}{6}

\pgfplotsset{%
  cosin0 color/.style={
    fill=white,
    draw=blue,
    postaction={%
      pattern=north east lines,
      pattern color=blue
    }
  },
  layer0 color/.style={%
    fill=white,
    draw=red,
    postaction={%
      pattern=north west lines,
      pattern color=red
    }
  },
  cosin color/.style={%
    blue,
    fill=blue!20!white,
    postaction={%
      pattern=north east lines,
      pattern color=blue
    }
  },
  cosin clean color/.style={%
    blue,
    fill=blue!20!white,
  },
  cosin loss deriv color/.style={%
    blue,
    postaction={%
      pattern=horizontal lines,
      pattern color=blue
    }
  },
  layer color/.style={%
    red,
    fill=red!20!white,
    postaction={%
      pattern=north west lines,
      pattern color=red
    }
  },
  layer clean color/.style={%
    red,
    fill=red!20!white,
  },
  inf func/.style={%
    black,
    fill=gray,
  },
  inf func sim/.style={%
    black,
    fill=gray!25!white,
    postaction={%
      pattern=north east lines,
      pattern color=blue
    }
  },
  inf func loss deriv/.style={%
    black,
    postaction={%
      pattern=horizontal lines,
      pattern color=gray
    }
  },
  inf func layer/.style={%
    black,
    fill=gray!25!white,
    postaction={%
      pattern=north west lines,
      pattern color=BrickRed
    }
  },
  rep pt/.style={%
    violet!80!black,
    fill=violet,
  },
  rep pt sim/.style={%
    violet!70!black,
    fill=violet!25!white,
    postaction={%
      pattern=north east lines,
      pattern color=violet!80!black
    }
  },
  TracInCP color/.style={%
    brown!20!black,
    fill=brown!80!white,
    mark=none
  },
  deep knn color/.style={%
    fill=green!70!gray,
    draw=ForestGreen!50!black
  }
}

\pgfplotsset{
  TracIn Line/.style={
    draw=brown!60!black,
  },
  TracInCP Line/.style={%
    draw=brown!80!blue,
    dashed,
  },
  Random Line/.style={%
    draw=green!40!black,
  },
  CosIn Line/.style={%
    draw=blue,
  },%
  LayIn Line/.style={%
    draw=red,
  },
  Influence Functions Line/.style={%
    draw=violet,
  },
  Influence Functions Sim Line/.style={%
    draw=blue,
    densely dotted,
    line width=1pt,
  },
  Influence Functions Layer Line/.style={%
    draw=red,
    densely dotted,
    line width=1pt,
  },%
  Baseline Joint/.style={%
    fill=GreenBar!30!white,
    draw=GreenBar,
  },%
  Joint Opt/.style={%
    fill=RedBar!20!white,
    draw=RedBar,
  },%
  Bottom Filt/.style={%
    fill=Periwinkle,
    draw=Plum,
  },%
  Full Set/.style={%
    fill=gray!30!white,
    draw=gray,
  },
  Max KNN/.style={%
    lime!60!black,
    dashed,
  },
  Min KNN/.style={%
    teal,
  },
  Most Certain/.style={%
    violet,
    densely dotted,
  },
  Least Certain/.style={%
    orange,
    dashdotted,
  }
}

\definecolor{myblue}{RGB}{95,145,166}
\colorlet{GreenBar}{orange}
\newcommand{\GreenBarColor}{{\color{GreenBar}orange}}
\colorlet{RedBar}{teal}

\pgfplotscreateplotcyclelist{BarCycleList}{%
  {myblue,fill=myblue!30,mark=none},%
  {orange,fill=orange!40!white,mark=none},%
  {violet,fill=violet!30!white,mark=none},%
  {yellow,fill=yellow!30!white,mark=none},%
}
\pgfplotscreateplotcyclelist{AblationCycleList}{%
  {brown!60!black,fill=brown!30!white,mark=none},%
  {BrickRed,fill=BrickRed!50!white,mark=none},%
  {OrangeRed,fill=OrangeRed!30!white,mark=none},%
  {blue,fill=blue!30!white,mark=none},%
}
\pgfplotscreateplotcyclelist{AggregateCycleList}{%
  {fill=white, draw=blue, pattern=north east lines, pattern color=blue},%
  {blue,fill=blue!30!white,mark=none},%
  {red,fill=blue!50!white,pattern=north east lines,mark=none},%
}
\pgfplotscreateplotcyclelist{DetectCycleList}{%
  {blue,fill=blue!30!white,mark=none},%
  {red,fill=red!30!white,mark=none},%
  {lime!60!black,fill=lime,mark=none},%
  {teal,fill=teal!30!white,mark=none},%
  {violet,fill=violet!30!white,mark=none},%
  {orange,fill=orange!40!white,mark=none},%
  {yellow,fill=yellow!30!white,mark=none},%
  {black!50!white,fill=gray!50!white,mark=none},%
}

\newcommand{\legendFontSize}{\small}

\newcommand{\plotFontSize}{\small}

\newcommand{\GroupLineFontSize}{\small}
\newcounter{groupcount}
\pgfplotsset{
    draw group line/.style n args={5}{
        after end axis/.append code={
            \setcounter{groupcount}{0}
            \pgfplotstableforeachcolumnelement{#1}\of\datatable\as\cell{%
                \def\temp{#2}
                \ifx\temp\cell
                    \ifnum\thegroupcount=0
                        \stepcounter{groupcount}
                        \pgfplotstablegetelem{\pgfplotstablerow}{[index]0}\of\datatable
                        \coordinate [yshift=#4] (startgroup) at (axis cs:\pgfplotsretval,0);
                    \else
                        \pgfplotstablegetelem{\pgfplotstablerow}{[index]0}\of\datatable
                        \coordinate [yshift=#4] (endgroup) at (axis cs:\pgfplotsretval,0);
                    \fi
                \else
                    \ifnum\thegroupcount=1
                        \setcounter{groupcount}{0}
                        \draw [
                            shorten >=-#5,
                            shorten <=-#5
                        ] (startgroup) -- node [anchor=north] {{\GroupLineFontSize #3}} (endgroup);
                    \fi
                \fi
            }
            \ifnum\thegroupcount=1
                        \setcounter{groupcount}{0}
                        \draw [
                            shorten >=-#5,
                            shorten <=-#5
                        ] (startgroup) -- node [anchor=north] {{\GroupLineFontSize #3}} (endgroup);
            \fi
        }
    },
    draw bottom line/.style={
        after end axis/.append code={
          \coordinate [yshift=-5.0ex] (startgroup) at (axis cs:3.5,0);
          \coordinate [yshift=-5.0ex] (endgroup) at (axis cs:4.5,0);
          \draw [] (startgroup) -- node [anchor=north] {{\GroupLineFontSize Poison}} (endgroup);
        }
    },
    draw empty group line/.style n args={5}{
        after end axis/.append code={
            \setcounter{groupcount}{0}
            \pgfplotstableforeachcolumnelement{#1}\of\datatable\as\cell{%
                \def\temp{#2}
                \ifx\temp\cell
                    \ifnum\thegroupcount=0
                        \stepcounter{groupcount}
                        \pgfplotstablegetelem{\pgfplotstablerow}{[index]0}\of\datatable
                        \coordinate [yshift=#4] (startgroup) at (axis cs:\pgfplotsretval,0);
                    \else
                        \pgfplotstablegetelem{\pgfplotstablerow}{[index]0}\of\datatable
                        \coordinate [yshift=#4] (endgroup) at (axis cs:\pgfplotsretval,0);
                    \fi
                \else
                    \ifnum\thegroupcount=1
                        \setcounter{groupcount}{0}
                        \draw [
                            white,
                            shorten >=-#5,
                            shorten <=-#5
                        ] (startgroup) -- node [anchor=north] {{\GroupLineFontSize \phantom{#3}}} (endgroup);
                    \fi
                \fi
            }
            \ifnum\thegroupcount=1
                        \setcounter{groupcount}{0}
                        \draw [
                            white,
                            shorten >=-#5,
                            shorten <=-#5
                        ] (startgroup) -- node [anchor=north] {{\GroupLineFontSize \phantom{#3}}} (endgroup);
            \fi
        }
    }
}
\newif\ifcomments
\commentstrue
\ifcomments
    \providecommand{\daniel}[2][]{{\protect\color{violet}{[Daniel:\textbf{#1} #2]}}}
\else
    \providecommand{\daniel}[2][]{}
\fi

\newcommand{\trim}[1]{}
\usepackage{graphicx}
\newsavebox\CBox
\def\textBF#1{\sbox\CBox{#1}\resizebox{\wd\CBox}{\ht\CBox}{\textbf{#1}}}

\newcommand{\tikzCaption}[1]{\raisebox{0.5ex}{\begin{tikzpicture} #1 \end{tikzpicture}}}%

\newcommand{\nocontentsline}[3]{}
\newcommand{\tocless}[2]{\bgroup\let\addcontentsline=\nocontentsline#1{#2}\egroup}

\usepackage{enumitem}
\setlist{nolistsep,leftmargin=*}

\usepackage{hyperref}

\newcommand{\titletext}{Identifying a Training-Set Attack's Target \\ Using Renormalized Influence Estimation}
\newcommand{\titlehead}{Target Identification Using Renormalized Influence Estimation}
\newcommand{\pdfKeywords}{Target Identification; Backdoor Attack; Data Poisoning; Influence Estimation; TracIn; GAS; Influence Functions; Representer Point}
\hypersetup{%
  pdfauthor={Zayd Hammoudeh and Daniel Lowd},
  pdfkeywords={\pdfKeywords},
}

\thanks{%
  This is an extended version which includes additional details that did not fit in the peer-reviewed version.
  Not for redistribution.
  The definitive, peer-reviewed version is published in the proceedings of {CCS'22}~\citep{Hammoudeh:2022:GAS}.%
}

\makeatletter%
\begin{document}
\title[\titlehead]{\titletext}

\author{Zayd Hammoudeh \quad{} Daniel Lowd}
\affiliation{%
  \institution{University of Oregon}
  \streetaddress{1585 E 13th Ave.}
  \city{Eugene}
  \state{Oregon}
  \country{USA}
  \postcode{97403}
}
\email{{zayd,lowd}@cs.uoregon.edu}

\renewcommand{\shortauthors}{Hammoudeh \& Lowd}
 
\begin{abstract}
\keyword{Targeted training-set attacks} inject malicious instances into the training set to cause a trained model to mislabel one or more specific test instances.
This work proposes the task of \keyword{target identification}, which determines whether a specific test instance is the target of a training-set attack.
Target identification can be combined with \keyword{adversarial-instance identification} to find (and remove) the attack instances, mitigating the attack with minimal impact on other predictions.
Rather than focusing on a single attack method or data modality, we build on influence estimation, which quantifies each training instance's contribution to a model's prediction.
We show that existing influence estimators' poor practical performance often derives from their over-reliance on training instances and iterations with large losses.
Our \keyword{renormalized} influence estimators fix this weakness; they far outperform the original estimators at identifying influential groups of training examples in both adversarial and non-adversarial settings, even finding up to 100\% of adversarial training instances with no clean-data false positives.
Target identification then simplifies to detecting test instances with anomalous influence values.
We demonstrate our method's \revTwo{effectiveness} on backdoor and poisoning attacks across various data domains, including text, vision, and speech,
\revTwo{%
as well as against a gray-box, adaptive attacker that specifically optimizes the adversarial instances to evade our method.
}%
Our source code is available at
\mbox{\url{\sourceCodeUrl}}.%
 \end{abstract}


\maketitle

\section{Introduction}\label{sec:Introduction}

\keyword{Targeted training-set attacks} manipulate an ML system's prediction on one or more \keyword{target} test instances by maliciously modifying the training data
\citep{Munoz:2017,Shafahi:2018,Aghakhani:2020,Salem:2020,Huang:2020,Geiping:2021,Wallace:2021}.
For example, a retailer may attempt to trick a spam filter into mislabeling all of a competitor's emails as spam~\citep{Shafahi:2018}.
Targeted attacks require very few corrupted instances~\cite{Wallace:2021}, and their effect on the test error is quite small, making them harder to detect~\citep{Chen:2017:Targeted}
than \keyword{availability training-set attacks}~\citep{Chen:2017:Targeted}, which \revised{are indiscriminate/untargeted and} seek to degrade an ML system's overall performance~\citep{Biggio:2012:Poisoning,Xiao:2015:IsFeature,Fowl:2021:Adversarial}.
\citepos{Kumar:2020} recent survey of business and governmental organizations found that training-set attacks were the top ML security concern, due to previous successful attacks~\citep{Lee:2016:Learning}.
\citeauthor{Kumar:2020} specifically identify defenses against such attacks as a significant, practical gap.
Existing training-set defenses~\citep{Gao:2019,Peri:2020,Xu:2021} change the training procedure to mitigate the impact of an attack but provide little information about the attacker's goals, methods, or identity. Learning more about the attacker is essential for anticipating their attacks~\citep{Pita:2009}, designing targeted defenses~\citep{Agarwal:2019}, and even building defenses outside the ML system, such as stopping spammers through their payment processors~\citep{Levchenko:2011}.

This paper's defense against training-set attacks focuses on the related goals of learning more about an attacker and stopping their attacks. We achieve this through a pair of related tasks:%
\begin{enumerate}%
  \item \keyword{target identification}: identifying the target of a training-set attack, which may provide insight into the attacker's goals and how to defend against them; and
  \item \keyword{adversarial-instance identification}: identifying the malicious instances that constitute the training-set attack.
\end{enumerate}%
\noindent%
We are not aware of any work that studies training-set attack target identification beyond very simple settings.

Our \textit{key insight} is the synergistic interplay between the two tasks above.
If attackers can add only a limited number of training instances (as is often the case) \citep{Chen:2017:Targeted,Shafahi:2018,Wallace:2021}, then these malicious instances must be highly influential to change target predictions. Thus, targets are those test instances with an unusual number of highly influential training examples. In contrast, non-targets tend to have many weak influences and few very strong ones. Thus, if we can (1)~determine which \revTwo{training} instances influence which predictions and (2)~detect \revised{anomalies} in this distribution, then we can jointly solve both tasks.
Unfortunately, determining which training instances are responsible for which model behaviors remains a challenge, especially for complex, black-box models such as neural networks.
\keyword{Influence estimators} \citep{Koh:2017:Understanding,Yeh:2018:Representer,Pruthi:2020,Feldman:2020,Chen:2021:Hydra,Brophy:2022:TreeInfluence} %
attempt to quantify how much each training example contributes to a particular prediction.
However, current influence analysis methods often perform poorly~\citep{Basu:2021:Influence,Zhang:2022:Rethinking}.
This paper identifies a weakness common to many influence estimators \citep{Koh:2017:Understanding,Yeh:2018:Representer,Pruthi:2020,Chen:2021:Hydra}: they induce a \keyword{low-loss penalty} that implicitly ranks confidently-predicted training instances as uninfluential. As a result, existing influence estimators can systematically overlook (groups of) highly influential, low-loss instances.
We remedy this via a simple \keyword{renormalization} that removes the low-loss penalty.
Our new \keyword{renormalized influence estimators} consistently outperform the originals in both adversarial and non-adversarial settings. The most effective of these, \keyword{gradient aggregated similarity} (\method{}), often detects 100\%~of malicious training instances with no clean-data false positives.

Our \underline{f}ramework for \underline{i}dentifying \underline{t}argets of training-set attacks, \fit{}, compares the distribution of influence values across test instances checking for \revised{anomalies}.  More concretely, \fit{} marks as potential targets those test instances with an unusual number of highly influential instances as explained above. Next, \fit{} mitigates the attack's effect by removing exceptionally influential training instances associated with the target(s). Since mitigation considers only targets, training instance outliers that are ``helpful'' to non-targets are unaffected.  This \keyword{target-driven mitigation} has a positive or neutral effect on clean data yet is highly effective on adversarial data where finding even a single target suffices to disable the attack on almost all other targets.

By relying on the concepts of influence estimation and not the properties of a particular attack or domain, \method{} and \fit{} are \keyword{attack agnostic}~\citep{Saha:2019:Attack}. They can apply equally well to different attack types, including \keyword{data poisoning} attacks, which target unperturbed test data, and \keyword{backdoor} attacks on test instances activating a specific trigger.  Our approach works across data domains from CNN image classifiers to speech recognition to even
text transformers.

In addition to learning more about the attack and attacker, \textit{target identification enables targeted mitigation}.
Certified training-set defenses~\citep{Steinhardt:2017,Levine:2021,Jia:2021,Weber:2021,Wang:2022:DeterministicAggregation,Hammoudeh:2022:CertifiedRegression} (which do not identify targets) implement countermeasures (e.g., smoothing~\revised{\citep{Wang:2020}}) that affect predictions on \textit{all} instances -- not just the very few targets. These methods can substantially degrade performance, in some cases causing up to~${10{\times}}$ more errors on clean data~\citep{Fowl:2021:Adversarial,Hammoudeh:2022:CertifiedRegression}.
A strength of deep neural networks is that they can ``memorize'' instances to learn rare cases from one or two examples~\citep{Feldman:2020}; certified training prevents this by limiting any single training instance's influence.

Our work's contributions are enumerated below. Note that additional experiments and the proof are in the supplemental materials.
\begin{enumerate}
  \setlength{\itemsep}{0pt}
  \item We identify a weakness common to all gradient-based influence estimators and provide a simple renormalization correction that addresses this weakness.
  \item Inspired by influence estimation, we propose \method{} -- a renormalized influence estimator that is highly adept at identifying influential groups of training instances. %
  \item Leveraging techniques from anomaly detection and \keyword{robust statistics}, we extend \method{} into a general \underline{f}ramework for \underline{i}dentifying \underline{t}argets of training-set attacks, \fit{}.
  \item We use \method{} in a target-driven data sanitizer that mitigates attacks while removing very few clean instances.
  \item We demonstrate the effectiveness and attack agnosticism of \method{} and \fit{} on a diverse set of attacks and data modalities, including speech recognition, vision, and text \revTwo{-- even against an adaptive attacker that attempts to evade our method}.
\end{enumerate}

In the remainder of the paper, we begin by establishing notation and reviewing prior work (Sections~\ref{sec:ProblemFormulation} and~\ref{sec:RelatedWork}).
Then in Section~\ref{sec:RenormInfluence}, we show how existing influence estimators are inadequate for identifying (groups of) highly influential training instances.
We also introduce our renormalization fix for influence estimation and describe our renormalized influence estimators.
Section~\ref{sec:Method} builds on these improved influence estimates to define a framework for identifying the targets of an attack and mitigating the attack's effect. We demonstrate the effectiveness of our methods in Section~\ref{sec:ExpRes}.
\newcommand{\methodparagraph}[1]{%
\textbf{#1}~}

\section{Problem Formulation}\label{sec:ProblemFormulation}

\methodparagraph{Notation}
${\X \in \domainX}$ denotes a \keyword{feature vector} and ${\y \in \domainY}$ a \keyword{label}. \keyword{Training set}, \eqsmall{${\fullTrain = \set{\zI}_{\trIdx = 1}^{\nTr}}$}, consists of $\nTr$~training example tuples ${\zI \defeq (\xI,\yI)}$.
Consider \keyword{model} $\func{\dec}{\domainX}{\domainActs}$ parameterized by~${\w}$, where \eqsmall{${\acts \defeq \decFunc{\X}{\w}}$} denotes the model's output (pre\=/softmax) \keyword{activations}. %
$\wZero$~denotes $\dec$'s~initial parameters, which may be randomly set and/or pre\=/trained.

For \keyword{loss function} $\func{\loss}{\domainActs \times \domainY}{\realnn}$, denote $\z$'s~empirical \keyword{risk} given~$\w$ as \eqsmall{${\lFunc{\z}{\w} \defeq \lActsFunc{\decFunc{\X}{\w}}{\y}}$}. %
Consider any iterative, first-order optimization algorithm (e.g.,~gradient descent, Adam~\citep{Kingma:2015}).
At each iteration ${\itr \in \set{1,\ldots,\nItr}}$, the optimizer updates parameters~$\wT$ from loss~$\loss$, previous parameters~$\wTOne$, and batch \eqsmall{${\batchT \subseteq \fullTrain}$} of size~${\batchSize}$.
Gradients are denoted \eqsmall{${\gradItr{\gradI} \defeq \gradW \lFunc{\zI}{\wT}}$}; the gradient's superscript ``($\itr$)'' is dropped when the iteration is clear from context.

Let \eqsmall{${\zHatTe \defeq (\xTe, \yHatTe)}$} be any \textit{a priori} unknown test instance, where $\yHatTe$~is the final model's predicted label for~$\xTe$.  Observe that $\yHatTe$ may \underline{not} be $\xTe$'s~\keyword{true label}.  Notation $\emptyHat$ (e.g.,~$\zHat$, $\gradHat$) denotes that the final \underline{\smash{predicted}} label \eqsmall{${\yHat = \decFunc{\X}{\wFin}}$} is used in place of $\X$'s~true label~$\y$.

\subsection*{Threat Model}
The attacker crafts an \keyword{adversarial set} of perturbed instances, \linebreak \eqsmall{${\advTrain \subset \fullTrain}$}.
Denote the \keyword{clean training set} \eqsmall{${\cleanTrain \defeq \fullTrain \setminus \advTrain}$}.  We only consider successful attacks, as defined below.

\methodparagraph{Attacker Objective \& Knowledge}
Let \eqsmall{${\targFeatSet \defeq \set{\X_{\targIdx}}_{\targIdx = 1}^{\numTarg}}$} be a set of target feature vectors with shared true label~\eqsmall{$\yTarg \in \domainY$}. The attacker crafts~\eqsmall{$\advTrain$} to induce the model to mislabel all of~\eqsmall{$\targFeatSet$} as \keyword{adversarial label}~\eqsmall{$\yAdv$}.
\eqsmall{${\targSet \defeq \setbuild{(\X_{\targIdx}, \yAdv)}{\X_{\targIdx} \in \targFeatSet}}$} denotes the \keyword{target set} and
\eqsmall{${\zHatTarg \defeq (\xTarg,\yAdv)}$} an arbitrary target instance.  To avoid detection, the attacked model's clean-data performance should be (essentially) unchanged.
\keyword{Data poisoning} attacks only perturb adversarial set~\eqsmall{$\advTrain$}. Target feature vectors are unperturbed/benign \citep{Biggio:2012:Poisoning,Munoz:2017,Wallace:2021,Jagielski:2021:Subpopulation}.  \keyword{Clean\=/label poisoning} leaves labels unchanged when crafting \eqsmall{$\advTrain$} from seed instances \citep{Zhu:2019}.
\keyword{Backdoor} attacks perturb the features of both~\eqsmall{$\advTrain$} and~\eqsmall{$\targFeatSet$} -- often with the same \keyword{adversarial trigger} (e.g.,~change a specific pixel to maximum value). Generally, these triggers can be inserted into any test example targeted by the adversary, making most backdoor attacks \keyword{multi\=/target} (\eqsmall{${\lvert{\targSet}\rvert > 1}$}) \citep{Tran:2018,Gu:2019,Lin:2020:Composite,Weber:2021}. \eqsmall{$\advTrain$}'s~labels may also be changed.

To ensure the strongest adversary, the attacker knows any pre-trained initial parameters. Where applicable, the attacker also knows the training hyperparameters and clean dataset~\eqsmall{$\cleanTrain$}.
\revised{%
  Like previous work~\citep{Zhu:2019,Weber:2021,Wallace:2021},
  the attacker does not know the training procedure's random seed, meaning the attack must be robust to randomness in batch ordering or parameter initialization.%
}%

\methodparagraph{Defender Objective \& Knowledge}
Let \eqsmall{$\targAnalysisSet$} denote the set of test instances the defender is concerned enough about to analyze as potential targets.%
\footnote{%
  Generally, there are far fewer potential targets~({$\targAnalysisSet$}) than possible test examples.
}
Our goals are to (1)~\keyword{identify} any attack targets in~\eqsmall{$\targAnalysisSet$}, and
(2)~mitigate the attack by removing the adversarial instances \eqsmall{$\advTrain$} associated with those target(s).
No assumptions are made about the modality/domain (e.g.,~text, vision) or adversarial perturbation.
We do not assume access to clean validation data.
\section{Related Work}\label{sec:RelatedWork}

We mitigate training-set attacks by building upon influence estimation to identify the target(s) and adversarial set.
This section first reviews existing defenses against training-set attacks and then formalizes training-set influence as defined in previous work.

\subsection{Defenses Against Training-Set Attacks}\label{sec:RelatedWork:Defenses}

\keyword{Certifiably-robust defenses} \cite{Steinhardt:2017,Wang:2020,Levine:2021,Jia:2021,Weber:2021,Wang:2022:DeterministicAggregation,Hammoudeh:2022:CertifiedRegression} provide %
guaranteed protection against specific training-set attacks under specific assumptions.
\keyword{Empirical defenses} \citep{Gao:2019,Peri:2020,Doan:2020,Udeshi:2019,Zhu:2020,Zhu:2021} derive from understandings and observations about the underlying mechanisms training-set attacks exploit to change a network's predictions.
In practice, empirical defenses generally significantly outperform certified approaches with fewer harmful side effects -- albeit without a guarantee \citep{Li:2020:Survey}.
These two defense categories are complementary and can be deployed together for better performance.
Since our defense is largely empirical, we focus on that defense category below.

We are not aware of any existing defense -- certified or empirical -- that provides target identification.  The most closely related task is determining whether a model is infected with a backdoor, which ignores poisoning and other training-set attacks~\citep{Soremekun:2020,Xu:2021}. Such methods make different assumptions than this work. \revised{For instance, some assume access to known-clean data~\citep{Liu:2018,Gao:2019,Wang:2019,Doan:2020,Zhu:2021} but may not have access to the training set.}  Many also make assumptions about the type of attack or the training methodology~\citep{Soremekun:2020}.

Another task similar to target identification is \keyword{adversarial trigger synthesis}, which attempts to reconstruct any backdoor attack pattern(s) a model learned \citep{Gao:2019,Udeshi:2019,VillarealVasquez:2020,Zhu:2021}.  These methods mitigate attacks by adding identified triggers to known-clean data so that retraining will cause catastrophic forgetting of the trigger.
\keyword{Data-sanitization} defenses mitigate attacks by removing adversarial set~\eqsmall{$\advTrain$} from~\eqsmall{$\fullTrain$}.
Existing data-sanitization defenses have shown promise \citep{Tran:2018,Chen:2019,Peri:2020}, but they all share a common pitfall concerning setting the data-removal threshold \citep{Koh:2018:Stronger,Li:2020:Survey}.
If this threshold is set too low, significant clean-data removal degrades overall clean-data performance.
A threshold set too high results in insufficient adversarial training data removal and the attack remaining successful.  Additional information (e.g.,~target identification) enables targeted tuning of this removal threshold.

Most existing certified and empirical defenses are not attack agnostic and assume specific data modalities (e.g.,~only vision \citep{Gao:2019,Udeshi:2019,VillarealVasquez:2020,Zhu:2021}), model architectures (e.g,~CNNs \citep{Kolouri:2019}), optimizers \citep{Hara:2019}, or training paradigms \citep{Soremekun:2020}.  Attack agnosticism is more challenging and more practically useful.
We achieve agnosticism by building upon existing methods that are general -- namely influence estimation, which is formalized next.

\begin{figure*}[t]
  \centering
\newcommand{\oursText}{~{\scriptsize(ours)}}
\newcommand{\nlbreak}{\\ \vspace{-2.5pt}}

\newcommand{\textsize}[1]{{\small #1}}
\newcommand{\cifarImg}[1]{\includegraphics[scale=0.68]{img/cifar_vs_mnist/60/#1.png}}
\newcommand{\vertSpacer}{\vspace{4pt}}
\newcommand{\AUPRC}[2]{#1~${\pm}$~#2}
\newcommand{\Header}[1]{
  \begin{center}
    \underline{\textbf{\textsize{#1}}}
  \end{center}
}
\newcommand{\ImagePlots}[1]{
  \centering
  \cifarImg{#1_01}
  \cifarImg{#1_02}
  \cifarImg{#1_03}
  \cifarImg{#1_04}
  \cifarImg{#1_05}
}
\newcommand{\CifarVsMnistResult}[3]{%
  \begin{minipage}{0.27\textwidth}
    \centering
    \textsize{#1}
  \end{minipage}
  \hfill
  \begin{minipage}{0.16\textwidth}
    \centering
    \textsize{#2}
  \end{minipage}
  \hfill
  \begin{minipage}{0.45\textwidth}
    \centering

    #3
  \end{minipage}
}

\begin{minipage}{0.09\textwidth}
  \begin{center}
  \textsize{\textbf{Test}\nlbreak\textbf{Example}}

    \vertSpacer
    \cifarImg{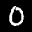}

    \eqsmall{$\zHatTe$}
  \end{center}
\end{minipage}
\hfill \vrule width 0.1mm \hfill
\begin{minipage}{0.80\textwidth}
  \CifarVsMnistResult{\Header{Method}}{\Header{AUPRC}}{\Header{Top-5 Highest Ranked}}

  \vertSpacer
  \CifarVsMnistResult{Representer Point}{\AUPRC{0.030}{0.009}}{\ImagePlots{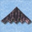}}

  \vertSpacer
  \CifarVsMnistResult{Influence Functions}{\AUPRC{0.029}{0.018}}{\ImagePlots{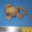}}

  \vertSpacer
  \CifarVsMnistResult{\tracin}{\AUPRC{0.140}{0.098}}{\ImagePlots{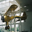}}

  \vertSpacer
  \CifarVsMnistResult{\tracinCP}{\AUPRC{0.309}{0.260}}{\ImagePlots{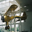}}

  \vertSpacer
  \CifarVsMnistResult{\textbf{Representer Point}\nlbreak{}\textbf{Renormalized\oursText}}{\textbf{\AUPRC{0.778}{0.144}}}{\ImagePlots{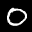}}

  \vertSpacer
  \CifarVsMnistResult{\textbf{Influence Functions}\nlbreak{}\textbf{Renormalized\oursText}}{\textbf{\AUPRC{0.215}{0.191}}}{\ImagePlots{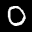}}

  \vertSpacer
  \CifarVsMnistResult{\textbf{\tracin{}}\nlbreak{}\textbf{Renormalized\oursText}}{\textbf{\AUPRC{0.617}{0.115}}}{\ImagePlots{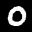}}

  \vertSpacer
  \CifarVsMnistResult{\textbf{\method{} \oursText{}}\nlbreak\textbf{(\tracinCP{} Renormalized)}}{\textbf{\AUPRC{0.977}{0.001}}}{\ImagePlots{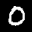}}
\end{minipage}
   \caption{%
    \textit{Renormalized Influence}:
    CIFAR10 \& MNIST joint, binary classification for [\texttt{frog}] vs.\ [\texttt{airplane} \& MNIST~\texttt{0}] with \eqsmall{${\abs{\cleanTrain} = 10,000}$} \& \eqsmall{${\abs{\advTrain} = 150}$}.
    Existing influence estimators (upper half) consistently failed to rank \eqsmall{$\advTrain$}'s~MNIST training instances as highly influential on MNIST test instances.
  In contrast, all of our renormalized influence estimators (Section~\ref{sec:RenormInfluence:Measures}) outperformed their unnormalized version -- with AUPRC improving up to~$25\times$.
    Results averaged across 30~trials.
  }
  \label{fig:RenormInfluence:CifarVsMnist:Comparison}
\end{figure*}
 
\subsection{Training-Set Influence Estimation}\label{sec:RelatedWork:Influence}

In every successful attack, the inserted training instances change a model's prediction for specific input(s). If the attacker can only add a limited number of instances (e.g.,~1\% of~\eqsmall{$\fullTrain$}), these inserted instances must be highly influential to achieve the attacker's objective.
\textit{Influence estimation}'s goal is to determine which training instances are most responsible for a model's prediction for a particular input.
Influence is often viewed as a counterfactual: which instance (or group of instances) induces the biggest change when removed from the training data?
While there are multiple definitions of influence, as detailed below,
influence estimation methods can be broadly viewed as quantifying the relative responsibility of each training instance \eqsmall{${\zI \in \fullTrain}$} on some test prediction~\eqsmall{$\decFunc{\xTe}{\wFin}$}.

\keyword{Static influence estimators} consider only the final model parameters~$\wFin$.
For example,
\citepos{Koh:2017:Understanding} seminal work defines influence, \eqsmall{$\infFunc{\infIF}{\zI}{\zHatTe}$}, as the change in risk~\eqsmall{$\lFunc{\zHatTe}{\wFin}$} if~\eqsmall{${\zI \notin \fullTrain}$}, i.e.,~the leave-one-out~(LOO) change in test loss~\citep{Cook:1982:InfluenceRegression}.
By assuming strict convexity and stationarity,
\citeauthor{Koh:2017:Understanding}'s \keyword{influence functions} estimator approximates the LOO~influence as
{%
  \small%
  \begin{equation}\label{eq:RelatedWork:InfFunc}
    \infFunc{\infIF}{\zI}{\zHatTe} \approx
                        \ifSign{}
                        \frac{1}{\nTr}
                        \gradLoss{\zHatTe}{\wFin}\transpose
                        \invHess
                        \gradLoss{\zI}{\wFin} \text{,}
  \end{equation}
}%
\noindent%
with
\eqsmall{$\invHess$} the inverse of risk Hessian \eqsmall{${\hess \defeq \frac{1}{\nTr} \sum_{\zI \in \fullTrain} \gradWSq \lFunc{\zI}{\wFin}}$}.
\citet{Yeh:2018:Representer}'s \textit{representer point} static influence estimator exclusively considers the model's final, linear classification layer.  All other model parameters are treated as a fixed feature extractor. %
\revised{%
Given final parameters~$\wFin$, let $\fRP{\trIdx}$~denote $\X_{\trIdx}$'s~penultimate feature representation (i.e.,~the \textit{input} to the linear classification layer).
}
Then the representer point influence of \eqsmall{${\zI \in \fullTrain}$} on \eqsmall{$\zHatTe$} is
{%
  \small%
  \begin{equation}\label{eq:RelatedWork:RepPt}
    \infFunc{\infRepPt}{\zI}{\zHatTe} \defeq
         -\frac{1}{2 \wdecay \nTr}
         \bigg( \rpDeriv \bigg)
         \dotprodbig{\fRP{\trIdx}}{\fRP{\teStr}}
         \text{,}
  \end{equation}%
}%
\noindent
where ${\wdecay > 0}$~is the weight decay~(\eqsmall{$L_2$}) regularizer and \eqsmall{$\dotprod{\cdot}{\cdot}$} denotes vector dot product.
\revised{%
Recall that $\acts$~is the \textit{output} of the model's linear classification layer, specifically here \eqsmall{${\acts = \decFunc{\X_{\trIdx}}{\wFin}}$}.
Scalar \eqsmall{$\rpDeriv$} is then the partial derivative of risk~\eqsmall{$\riskSym$} w.r.t.\ $\acts$'s \eqsmall{$\yI$}\=/{th}~dimension.  %
}

\keyword{Dynamic influence estimators} measure influence based on how losses change during training.  More formally, influence is quantified according to how batches~\eqsmall{${\batchOne,\ldots,\batchFin}$} affect model parameters \eqsmall{${\wZero,\ldots,\wFin}$} and by consequence risks~\eqsmall{${\lFunc{\cdot}{\wZero},\ldots,\lFunc{\cdot}{\wFin}}$}.
For example, \citepos{Pruthi:2020} \keyword{\tracin{}} estimates influence by ``tracing'' gradient descent -- aggregating changes in \eqsmall{$\zHatTe$}'s test loss each time training instance~\eqsmall{$\zI$}'s gradient updates parameters~\eqsmall{$\wT$}.  For stochastic gradient descent (batch size ${\batchSize = 1}$), \eqsmall{$\zI$}'s \tracin{} influence on~\eqsmall{$\zHatTe$} is
{%
  \small%
  \begin{equation}\label{eq:RelatedWork:TracInInfluence}
    \infTracInFunc \defeq \sum_{\itr = 1}^{\nItr}
                                \ind{\zI \in \batchT}
                                \Big( \lFunc{\zHatTe}{\wTOne} - \lFunc{\zHatTe}{\wT} \Big)
                    \text{,}
  \end{equation}%
}%
\noindent%
where $\ind{\indVar}$~is the indicator function s.t.\ ${\ind{\indVar} = 1}$ if predicate~$\indVar$ is true and~0 otherwise.
\citeauthor{Pruthi:2020} approximate Eq.~\eqref{eq:RelatedWork:TracInInfluence} as,
{%
  \small%
  \begin{equation}\label{eq:RelatedWork:TracIn}
    \infTracInFunc \approx \sum_{\zI \in \batchT}
                                \frac{\lrT}{\batchSize}
                                \dotprod{\vphantom{\Big(\Big)}\gradLoss{\zI}{\wTOne}}
                                        {\gradLoss{\zHatTe}{\wTOne}}
          \text{,}
  \end{equation}%
}%
\noindent%
where $\lrT$~is iteration~$\itr$'s learning rate.%

\begin{algorithm}[t]
    {%
        \centering
        \AlgFontSize
        \caption{\tracin{}, \tracinCP{}, \& \method{} training phase}\label{alg:TrainPhase}
        \begin{algorithmic}[1]
  \algSetStretch%
  \REQUIRE Training set~$\fullTrain$, iteration subset~$\subsetItr$, iteration count~$\nItr$, learning rates~${\lr_1,\ldots,\lr_{\nItr}}$, and initial parameters~$\wZero$

  \ENSURE Training parameters~$\trainParams$

  \STATE $\trainParams \gets \emptyset$
  \FOR{$\itr \gets 1 \textbf{ to } \nItr$}
    \IF{$\itr \in \subsetItr$}
      \STATE $\trainParams \gets \trainParams \cup \set{(\lrT,\wTOne)}$
    \ENDIF
    \STATE $\batchT \sim \fullTrain$
    \STATE $\wT \gets \textsc{Update}(\lrT, \wTOne,\batchT)$
  \ENDFOR
  \STATE \RETURN $\trainParams$
\end{algorithmic}

    }
\end{algorithm}

Alg.~\ref{alg:TrainPhase}
details the minimal changes made to model training to support \tracin{} where \eqsmall{${\subsetItr \subset \set{1,\ldots,\nItr}}$} is a preselected \keyword{training iteration subset} and \eqsmall{${\trainParams \defeq \setbuild{(\lrT,\wT[\itr - 1])}{\itr \in \subsetItr}}$} contains the \keyword{serialized training parameters}.
Alg.~\ref{alg:TracIn}\footnote{Due to space, Alg.~\ref{alg:TrainPhase} appears in the supplement.}
outlines \tracin{}'s influence estimation procedure for \textit{a~priori} unknown test instance~\eqsmall{$\zHatTe$}. Influence vector~$\infVec$ (${\abs{\infVec} = \nTr}$) contains the \tracin{}~influence estimates for each \eqsmall{${\zI \in \fullTrain}$}. %
In practice, ${\abs{\subsetItr} \ll \nItr}$, and $\subsetItr$~is evenly-spaced in $\set{1,\ldots,\nItr}$, meaning \tracin{} effectively treats multiple batches like a single model update.

\citeauthor{Pruthi:2020} also propose \keyword{\tracin{} Checkpoint} (\tracinCP{}) -- a more heuristic version of \tracin{} that considers \textit{all} training examples at each checkpoint in~${\subsetItr}$ -- not just those instances in the intervening batches (see Alg.~\ref{alg:CosIn}).\footnote{Algorithm~\ref{alg:CosIn} combines two different methods \tracinCP{} as well as \method{} -- our renormalized version of \tracinCP{} discussed in Sec.~\ref{sec:RenormInfluence:Measures}.}  Formally,
{%
  \small%
  \begin{equation}\label{eq:RelatedWork:TracInCP}
    \infTracInCPFunc \defeq
                        \sum_{\itr \in \subsetItr}%
                        \frac{\lrT}{\batchSize}%
                        \hspace{2.5pt}%
                        \dotprod%
                        {\vphantom{\Big(\Big)}%
                         \gradW \lFunc{\zI}{\wTOne}}%
                        {\gradW \lFunc{\zHatTe}{\wTOne}}%
    \text{.}
  \end{equation}%
}%
\noindent%
\tracinCP{} is more computationally expensive than \tracin{} -- with the slowdown linear w.r.t.\ the number of checkpoints per epoch.

A major advantage of \tracin{} and \tracinCP{} over other estimators (e.g.,~influence functions) is that their only hyperparameter is iteration set~$\subsetItr$, which we tuned based only on compute availability.

\newcommand{\InfVsSimParagraph}[1]{%
\vspace{2pt}%
\textit{#1}: }

\section{Why Influence Estimation Often Fails and How to Fix It}\label{sec:RenormInfluence}

Before addressing target identification, we first consider the related task of adversarial-instance identification. In the simplest case, if the attack's target is known, then the malicious instances should be among the most influential instances for that target instance. In other words, \emph{adversarial-instance identification reduces to influence estimation}.
However,
Sec.~\ref{sec:RelatedWork:Influence}'s influence estimators share a common weakness that makes them poorly suited for this task: they all consistently rank confidently-predicted training instances as uninfluential. We illustrate this behavior below using a \revised{toy} experiment. We then explain this weakness's cause and propose a simple fix that addresses this limitation on adversarial and non-adversarial data, for all preceding estimators.
Our fix is needed to successfully identify adversarial set~\eqsmall{$\advTrain$} and, as detailed in Sec.~\ref{sec:Method}, attack targets.

\subsection{A Simple Experiment}\label{sec:RenormInfluence:SimpleExperiment}
Consider binary classification where clean set~\eqsmall{$\cleanTrain$} is all \texttt{frog} and \texttt{airplane} training images in CIFAR10 (\eqsmall{${\abs{\cleanTrain} = 10,000}$}). To simulate a naive backdoor attack,
adversarial set~\eqsmall{$\advTrain$} is 150~randomly selected MNIST \texttt{0}~images labeled as \texttt{airplane}.
\revised{%
  Clean data's overall influence can be estimated indirectly by training only on~\eqsmall{$\cleanTrain$} and observing the target set's misclassification rate~\citep{Feldman:2020}. This experiment used class pair \texttt{frog} and \texttt{airplane} because amongst the \eqsmall{$\binom{10}{2}$}~CIFAR10 class pairs, \texttt{frog} vs.\ \texttt{airplane}'s average MNIST \textit{test} misclassification rate was closest to random (47.5\% vs.\ 50\%~ideal).%
}
In contrast, when training on \eqsmall{${\fullTrain \defeq } {\advTrain \sqcup \cleanTrain}$}, MNIST \texttt{0}~test instances were always classified as \texttt{airplane}, meaning \eqsmall{$\advTrain$}~is overwhelmingly influential on MNIST predictions.
\revised{%
  MNIST is used instead of other CIFAR10 classes because the large (and simple~\citep{Shah:2020:SimplicityBias}) difference between the data distributions leads to a strong signal that can be consistently learned from relatively few examples -- much like backdoor or poisoning attacks~\citep{Yu:2021:PoisonShortcuts}.%
}%

We use this simple setup to evaluate different influence estimation methods.
We trained 30~randomly-initialized state-of-the-art ResNet9 networks, and on each network,
we performed influence estimation for a random \underline{MNIST \texttt{0}~test instance} to determine how well each estimator identified adversarial set~\eqsmall{$\advTrain$} provided a known target.\footnote{See supplemental Section~\ref{sec:App:ExpSetup} for the complete experimental setup details.}
Given the large imbalance between the amount of clean and ``adversarial'' data, i.e.,~\eqsmall{${\abs{\advTrain} \ll \abs{\cleanTrain}}$}, performance is measured using area under the precision-recall curve (AUPRC), which quantifies how well \eqsmall{$\advTrain$}'s~influence ranks relative to~\eqsmall{$\cleanTrain$}.  Precision-recall curves are preferred for highly-skewed classification tasks since they provide more insight into the false-positive rate \citep{Davis:2006}.

Figure~\ref{fig:RenormInfluence:CifarVsMnist:Comparison}'s upper half shows how well each influence estimator in Section~\ref{sec:RelatedWork:Influence} identifies~\eqsmall{$\advTrain$}, both quantitatively and qualitatively.  Dynamic estimators significantly outperformed their static counterparts, with \tracinCP{} the overall top performer.  However, no influence estimator consistently ranked MNIST instances (i.e.,~\eqsmall{$\advTrain$}) in the top\=/5 most influential, with influence functions marking instances from the other class (\texttt{frog}) as most influential.  Influence estimation's poor performance here is particularly noteworthy as the task was designed to be unrealistically easy.

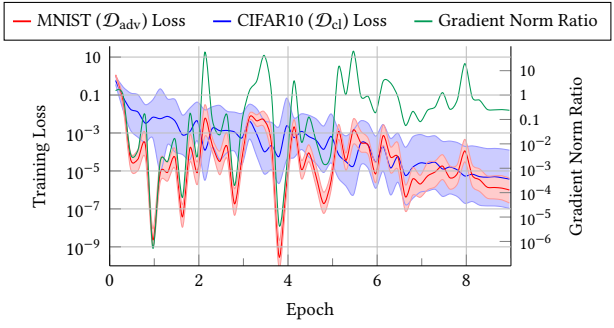
\begin{figure}[t]
  \centering
\newcommand{\legendSpacer}{\hspace*{8pt}}
\begin{tikzpicture}
  \begin{axis}[%
      hide axis,  %
      xmin=0,  %
      xmax=1,
      ymin=0,
      ymax=1,
      scale only axis,width=1mm, %
      legend cell align={left},              %
      legend style={font=\footnotesize},
      legend columns=3,
      no markers,
    ]
    \addplot[red, xscale=0.4, line width=0.7] coordinates {(0,0)};
    \label{leg:CifarVMnist:MNIST}
    \addlegendentry{MNIST ($\advTrain$) Loss\legendSpacer}

    \addplot[blue, xscale=0.4, line width=0.7] coordinates {(0,0)};%
    \label{leg:CifarVMnist:CIFAR}
    \addlegendentry{CIFAR10 ($\cleanTrain$) Loss\legendSpacer}

    \addplot[ForestGreen, xscale=0.4, line width=0.7] coordinates {(0,0)};%
    \label{leg:CifarVMnist:GradNormRatio}
    \addlegendentry{Gradient Norm Ratio}
  \end{axis}
\end{tikzpicture}

  \newcommand{\CvmFontSize}{\footnotesize}
  \pgfplotstableread[col sep=comma]{plots/data/cifar_vs_mnist_loss_trend.csv}\datatable%
  \begin{tikzpicture}
      \pgfplotsset{
          width={2.1in},
          height={1.15in},
          scale only axis,
          xmin=0,
          xmax=9,
          xtick distance={2},
          minor x tick num={1},
          xlabel={\CvmFontSize Epoch},
          x tick label style={font=\CvmFontSize},
          axis on top=true,  %
      }

      \begin{semilogyaxis}
          [
              axis lines*=left,
              xmajorgrids,
              ymajorgrids,  %
              ylabel={\CvmFontSize Training Loss},%
              y tick label style={font=\CvmFontSize},
              ytick distance=100,
              yticklabels={,${10^{-9}}$, ${10^{-7}}$, ${10^{-5}}$, ${10^{-3}}$, 0.1, 10},
              ymin=1E-10,%
              ymax=37,%
          ]
          \addplot[blue, smooth] table[x index=0, y index=2] {\datatable};
          \addplot[name path=Cifar25p, blue!40, smooth] table[x index=0, y index=1] {\datatable};  %
          \addplot[name path=Cifar75p, blue!40, smooth] table[x index=0, y index=3] {\datatable};  %
          \addplot[blue!20] fill between[of=Cifar25p and Cifar75p];
          \addplot[red, smooth] table[x index=0, y index=5] {\datatable};
          \addplot[name path=Mnist25p, red!40, smooth] table[x index=0, y index=4] {\datatable};  %
          \addplot[name path=Mnist75p, red!40, smooth] table[x index=0, y index=6] {\datatable};  %
          \addplot[red!20] fill between[of=Mnist25p and Mnist75p];
      \end{semilogyaxis}

      \begin{semilogyaxis}
          [
              axis y line*=right,
              axis x line=none,
              ylabel={\CvmFontSize Gradient Norm Ratio},%
              y tick label style={font=\scriptsize},
              ymin=1E-7,%
              ymax=1E2,%
              ytick = {1, 10, 0.1, 0.01,0.001,10^-4,10^-5, 10^-6},
              yticklabels={${1}$, ${10}$, ${0.1}$, ${10^{-2}}$, ${10^{-3}}$, ${10^{-4}}$, ${10^{-5}}$, ${10^{-6}}$,},
          ]
          \addplot[ForestGreen, line width=0.4, smooth] table[x index=0, y index=7] {\datatable};
      \end{semilogyaxis}
  \end{tikzpicture}
   \caption{%
    \textit{CIFAR10 \& MNIST Intra-training Loss Tracking}:
    \eqsmall{$\advTrain$}'s (\tikzCaption{\ref{leg:CifarVMnist:MNIST}}) \& \eqsmall{$\cleanTrain$}'s (\tikzCaption{\ref{leg:CifarVMnist:CIFAR}}) median cross-entropy losses~($\loss$) at each training checkpoint for binary classification -- \texttt{frog} vs.\ \texttt{airplane}~\&~MNIST~\texttt{0}.
    The shaded regions correspond to each training set loss's interquartile range.
    MNIST's training losses are generally several orders of magnitude smaller than CIFAR10's losses.
    Gradient norm ratio~(\tikzCaption{\ref{leg:CifarVMnist:GradNormRatio}}) shows the tight coupling of loss \& training gradient magnitude.%
  }%
  \label{fig:RenormInfluence:CifarVsMnist:Trend}
\end{figure}

\subsection{Why Influence Estimation Performs Poorly}\label{sec:RenormInfluence:WhyPerformsPoorly}
Intra-training dynamics illuminate the primary cause of influence estimation's poor performance in our toy experiment.
Fig.~\ref{fig:RenormInfluence:CifarVsMnist:Trend} visualizes the median training loss of \eqsmall{$\advTrain$} and \eqsmall{$\cleanTrain$} at each training checkpoint.
Also shown is the \keyword{gradient norm ratio}, which compares the median gradient magnitude of the adversarial and clean sets at each iteration, or formally
{%
  \small%
  \begin{equation}\label{eq:LowLoss:GradientNormRatio}
    \text{GNR}_{\itr} \defeq
    \frac{%
           \median \setbuild{\norm{\lFunc{\z}{\wT}}}%
                            {\z \in \advTrain}%
         }%
         {%
           \median \setbuild{\norm{\lFunc{\z}{\wT}}}%
                            {\z \in \cleanTrain}%
         }%
    \text{.}
  \end{equation}
}%
The gradient norm ratio closely tracks both training sets' loss values.
Both during and at the end of training, \eqsmall{$\advTrain$}'s \textit{median loss is significantly smaller} than many instances in~\eqsmall{$\cleanTrain$} -- often by several orders of magnitude.

\paragraph{The Low-Loss Penalty}
Observe that all influence methods in Sec.~\ref{sec:RelatedWork:Influence}'s scale their influence estimates by \eqsmall{$\frac{\partial \lActsFunc{\acts}{\y}}{\partial \acts}$} either directly (representer point~\eqref{eq:RelatedWork:RepPt}) or indirectly via the \keyword{chain rule} (influence functions~\eqref{eq:RelatedWork:InfFunc}, \tracin{}~\eqref{eq:RelatedWork:TracIn}, and \tracinCP{}~\eqref{eq:RelatedWork:TracInCP}) as
{
  \small
  \begin{equation}\label{eq:LowLoss:ChainRule}
    \gradLoss{\z}{\w} %
        \defeq %
                \frac{\partial \lActsFunc{\dec(\X)}{\y}} %
                     {\partial \w} %
        = \frac{\partial \lActsFunc{\acts}{\y}} %
               {\partial \acts} %
          \cdot %
          \frac{\partial \acts}%
               {\partial \w} %
          \text{.}
  \end{equation}
}%
\noindent%
Therefore, gradient-based influence estimators \textit{implicitly penalizes all training \underline{instances}~$\itr$ with low training loss}, including \eqsmall{$\advTrain$} (MNIST~0) in our toy experiment above.

Theorem~\ref{thm:LowLoss:LossVsNorm} summarizes this relationship when there is a single output activation (${\abs{\acts} = 1}$), e.g.,~binary classification and univariate regression.  In short, when Theorem~\ref{thm:LowLoss:LossVsNorm}'s conditions are met, loss induces a perfect ordering on the corresponding norm.
\begin{theorem}\label{thm:LowLoss:LossVsNorm}
  Let loss function~$\func{\lossScalar}{\real}{\realnn}$ be twice-differentiable and strictly convex as well as either even\footnote{%
        ``Even'' denotes that the function satisfies ${\forall_{\acts}~\lScalarFunc{\acts} = \lScalarFunc{-\acts}}$.
  }
  or monotonically decreasing.
  Then, it holds that
  {
    \begin{equation}\label{eq:RenormInfluence:LowLoss}
      \lScalarFunc{\acts} < \lScalarFunc{\acts'} \implies \normActs{\gradScalarLoss{\acts}} < \normActs{\gradScalarLoss{\acts'}} \text{.}
    \end{equation}
  }
\end{theorem}
Loss functions satisfying Theorem~\ref{thm:LowLoss:LossVsNorm}'s conditions include binary cross-entropy (i.e.,~logistic) and quadratic losses.  Theorem~\ref{thm:LowLoss:LossVsNorm} generally applies to multiclass losses, but there are cases where the ordering is not perfect.
Although Theorem~\ref{thm:LowLoss:LossVsNorm} primarily relates to training instance gradients and losses, the theorem applies to test examples as well since dynamic estimators also apply a low-loss penalty to any \underline{iteration} where test instance~$\zHatTe$ has low loss.

The preceding should \textit{not} be interpreted to imply that large gradient magnitudes are unimportant.  Quite the opposite, large gradients have large influences on the model. However, the approximations necessary to make influence estimation tractable go too far by often focusing almost exclusively on training loss -- and by extension gradient magnitude -- leading these estimators to systematically overlook training instances with smaller gradients.
\revised{%
  This overemphasis of instances with large losses and gradient magnitudes can also be viewed as a bias towards instances that are globally influential --- affecting many examples' predictions --- over those that are locally influential -- mainly affecting a small number of targets~\citep{Barshan:2020:RelatIF}.
}

\InfVsSimParagraph{Static Influence \& the Low Loss Penalty}
Fig.~\ref{fig:RenormInfluence:CifarVsMnist:Comparison}'s static estimators (representer point \& influence functions) significantly underperformed dynamic estimators (\tracin{} \& \tracinCP{}) by up to an order of magnitude.  Static estimators only consider final model parameters~$\wFin$, meaning they may only see the low-loss case.  In contrast, dynamic estimators consider all of training, in particular iterations where \eqsmall{$\advTrain$}'s~loss exceeds that of~\eqsmall{$\cleanTrain$}.
This allows dynamic estimators to outperform static methods, albeit still poorly.

\InfVsSimParagraph{Training Randomness \& the Low-Loss Penalty}
\tracinCP{} significantly outperformed \tracin{} in Fig.~\ref{fig:RenormInfluence:CifarVsMnist:Comparison} despite the \tracin{} being more theoretically sound.  As intuition why, imagine the training set contains two identical copies of some instance.  In expectation, these duplicates have equivalent influence on any test instance. However, \tracin{} assigns identical training examples different influence estimates based on their batch assignments; this difference can potentially be very large depending on training dynamics.

Fig.~\ref{fig:RenormInfluence:CifarVsMnist:Trend} exhibits this behavior where training loss fluctuates considerably intra\=/epoch.  For example, \eqsmall{$\advTrain$}'s~median loss varies by seven orders of magnitude across the third epoch. \tracin{}'s low-loss penalty attributes much more influence to~\eqsmall{$\advTrain$} instances early in that epoch compared to those later despite all MNIST instances having similar influence.
By considering all examples at each checkpoint, \tracinCP{} removes batch randomization's direct effect on influence estimation,\footnote{Batch randomization still indirectly affects \tracinCP{} and \method{} (Sec.~\ref{sec:RenormInfluence:Measures}) through the model parameters. This effect could be mitigated by training multiple models and averaging the (renormalized) influence, but that is beyond the scope of this work.} meaning \tracinCP{} simulates \keyword{influence expectation} without needing to train and analyze multiple models.

\subsection{Renormalizing Influence Estimation}\label{sec:RenormInfluence:Measures}

Our CIFAR10 \& MNIST joint classification experiment above demonstrates that a training example having low loss does \textit{not} imply that it and related instances are uninfluential.
Most importantly in the context of adversarial attacks, highly-related groups of (adversarial) training instances may collectively cause those group members' to have very low training losses -- so-called \keyword{group effects}.  Generally, targeted attacks succeed by leveraging the group effect of adversarial set~\eqsmall{$\advTrain$} on the target(s).
We address these group effects via \keyword{renormalization}, which is defined below.

\begin{definition}
  For influence estimator~$\influence$, the \keyword{renormalized influence},~$\renormInf$, replaces each gradient~${\gradLetter}$ in~$\influence$ by its corresponding unit vector \eqsmall{$\frac{\gradLetter}{\norm{\gradLetter}}$}.
\end{definition}

We refer to this computation as {renormalization} since \underline{re}scaling gradients \underline{re}moves the low-loss penalty.
Renormalization places all training instances on equal footing and ensures that gradient and/or feature similarity is prioritized -- not loss.

\revised{%
Renormalization is related to the relative influence (RelatIF) method introduced by \citet{Barshan:2020:RelatIF}, since both methods use a function of the gradient to downweight training instances with high losses. However, RelatIF only applies to influence functions and requires computing expensive Hessian-vector products, while renormalization is more efficient and can be applied to many influence estimators, as we show below. See suppl.\ Section~\ref{sec:App:MoreExps:LossOnly} for additional discussion of alternative renormalization schemes.
}

Renormalized versions of Section~\ref{sec:RelatedWork:Influence}'s static influence estimators are below.  \keyword{Renormalized influence functions} in Eq.~\eqref{eq:RenormInfluence:IF} does not include target gradient norm \eqsmall{$\norm{\gradHatTe}$} since it is a constant factor.  For simplicity, Eq.~\eqref{eq:RenormInfluence:RepPt}'s \keyword{renormalized representer point} uses signum function $\sgnp{\cdot}$ since for any scalar~${\indVar \ne 0}$, \eqsmall{${\sgnp{u} = \frac{u}{\abs{u}}}$}, i.e.,~signum is equivalent to normalizing by magnitude.
{%
  \small%
  \begin{align}%
      \infFunc{\simIF}{\zI}{\zHatTe} \defeq&%
            \ifSign{}
            \frac{1}{\nTr} %
            \gradLoss{\zHatTe}{\wFin}\transpose %
            \invHess %
            \left(\frac{\gradLoss{\zI}{\wFin}}{\norm{\gradLoss{\zI}{\wFin}}}\right) %
            \label{eq:RenormInfluence:IF}\\%
      \infFunc{\simRepPt}{\zI}{\zHatTe} \defeq&
           -\frac{1}{2 \wdecay \nTr}%
           ~\sgnp{\rpDeriv}%
           \dotprodbig{\fRP{\trIdx}}{\fRP{\teStr}} %
           \label{eq:RenormInfluence:RepPt}%
  \end{align}%
}%

Renormalized versions of Section~\ref{sec:RelatedWork:Influence}'s dynamic influence estimators appear below.
Going forward, we refer to renormalized \tracinCP{} (Eq.~\eqref{eq:RenormInfluence:GAS}) as \keyword{\underline{g}radient \underline{a}ggregated \underline{s}imilarity}, \method{}, since it is essentially the weighted, gradient cosine similarity averaged across all of training. \method{}'s procedure is detailed in Algorithm~\ref{alg:CosIn}.\footnote{As shown in Algorithm~\ref{alg:CosIn}, \tracinCP{}'s procedure (Line~7) is identical to \method{} (Line~9) other than influence renormalization.}
{%
  \small%
  \begin{align}%
      \simTracInFunc \defeq&%
                          \sum_{\zI \in \batchT} \frac{\lrT}{\batchSize}%
                                \frac%
                                {\dotprod{\vphantom{\Big(\Big)}\gradLoss{\zI}{\wTOne}}%
                                         {\gradLoss{\zHatTe}{\wTOne}}}%
                                {
                                  \norm{\gradLoss{\zI}{\wTOne}}
                                  \hspace{1.3pt}%
                                  \norm{\gradW \lFunc{\zHatTe}{\wTOne}}%
                                } \\%
      \simTracInCPFunc \defeq&%
                          \sum_{\itr \in \subsetItr}%
                          \frac{\lrT}{\batchSize}%
                          \hspace{2.5pt}%
                          \frac{%
                            \dotprod%
                            {\vphantom{\Big(\Big)}%
                             \gradW \lFunc{\zI}{\wTOne}}%
                            {\gradW \lFunc{\zHatTe}{\wTOne}}%
                          }%
                          {%
                            \norm{\gradW \lFunc{\zI}{\wTOne}}
                            \hspace{1.3pt}%
                            \norm{\gradW \lFunc{\zHatTe}{\wTOne}}%
                          }%
                          \nonumber \\
                          &\fedeq \method(\zI, \zHatTe) %
                          \label{eq:RenormInfluence:GAS}
  \end{align}%
}%

Unlike static estimators, rescaling dynamic influence by target gradient norm~\eqsmall{$\norm{\gradHatTe}$} is quite important as mentioned earlier.
Intuitively, \eqsmall{$\norm{\zHatTe}$} tends to be largest in two cases: (1)~early in training due to initial parameter randomness and (2)~when iteration~$\itr$'s predicted label conflicts with final label~\eqsmall{$\yHatTe$}.  Both cases are consistent with the features most responsible for predicting~\eqsmall{$\yHatTe$} not yet dominating.
Therefore, rescaling dynamic influence by~\eqsmall{$\norm{\gradHatTe}$} implicitly upweights iterations where \eqsmall{$\yHatTe$} is predicted confidently. It also inhibits any single checkpoint dominating the estimate.

\InfVsSimParagraph{Applying Renormalization to CIFAR10 \& MNIST Joint Classification}
Figure~\ref{fig:RenormInfluence:CifarVsMnist:Comparison}'s lower half demonstrates renormalization's significant performance advantage over standard influence estimation %
-- with the improvement in AUPRC as large as~$25\times$.  In particular, our renormalized estimators' top\=/5 highest-ranked instances were all consistently from MNIST, unlike any of the standard influence estimators.  Overall, \method{} (renormalized \tracinCP{}) was the top performer -- even outperforming our other renormalized estimators by a wide margin.

\begin{algorithm}[t!]
  {
    \AlgFontSize%
    \caption{\method{} vs.\ \tracinCP{}}\label{alg:CosIn}
    \newcommand{\InputDesc}{Training params.~$\trainParams$, training set~\eqsmall{$\fullTrain$}, batch size~$\batchSize$, \& test ex.~\eqsmall{$\zHatTe$}}
\newcommand{\OutputDesc}{(Renormalized) influence vector~$\infVec$}
\begin{algorithmic}[1]
  \algSetStretch%
  \REQUIRE \InputDesc{}
  \ENSURE \OutputDesc{}
    \STATE $\infVec \gets \vec{0}$ \algcomment{Initialize}
    \FOR{\textbf{each} ${(\lrT,~\wTOne) \in \trainParams}$}
      \STATE $\gradHatTe \gets \gradW \lFunc{\zHatTe}{\wTOne}$
      \FOR[\algcomment{All examples}]{\textbf{each} ${\zI \in \fullTrain}$}
        \STATE $\gradI \gets \gradW \lFunc{\zI}{\wTOne}$
        \IF{calculating \tracinCP{}}
          \STATE $\infScalarI \gets \infVec_{\trIdx}
                      + \frac{\lrT}{\batchSize}
                          \dotprod%
                                  {\gradI}%
                                  {\gradHatTe}%
                 $ \algcomment{Unnormalized (Sec.~\ref{sec:RelatedWork:Influence})}
        \ELSIF{calculating \method{}}
          \STATE $\infScalarI \gets \infVec_{\trIdx}
                      + \frac{\lrT}{\batchSize}
                          \dotprod%
                                  {\frac{\gradI    }{\norm{\gradI\vphantom{\gradHatTe}}}}%
                                  {\frac{\gradHatTe}{\norm{\gradHatTe}}}%
                 $ \algcomment{Renormalized (Sec.~\ref{sec:RenormInfluence:Measures})}
        \ENDIF
      \ENDFOR
    \ENDFOR
    \STATE \RETURN $\infVec$
\end{algorithmic}

  }
\end{algorithm}

\subsection{Renormalization \& More Advanced Attacks}\label{sec:RenormInfluence:AdvancedAttacks}

Section~\ref{sec:RenormInfluence:WhyPerformsPoorly} illustrates why influence performs poorly under a naive backdoor-style attack where the adversary does not optimize the adversarial set.
Those concepts also generalize to more sophisticated attacks.
For example, recent work shows that deep networks often predict the adversarial set with especially high confidence (i.e.,~low loss) due to \revised{\keyword{shortcut learning}} -- even on advanced attacks~\citep{Yu:2021:PoisonShortcuts,Geirhos:2020}.  Those findings reinforce the need for renormalization.
\revised{%
  This can be viewed through the lens of \keyword{simplicity bias} where neural networks tend to confidently learn simple features (shortcuts) -- regardless of whether those features actually generalize~\citep{Shah:2020:SimplicityBias}.
}

Dynamic estimators -- both \tracin{} and \method{} -- outperform static ones for Sec.~\ref{sec:RenormInfluence:SimpleExperiment}'s naive attack.
The same can be expected for sophisticated attacks including ones that track adversarial-set gradients through simulated training~\citep{Huang:2020,Wallace:2021}.
For those attacks, adversaries can craft \eqsmall{$\advTrain$} to exhibit particular gradient signatures at the end of training to avoid static detection.
Moreover, models learn adversarial data faster than clean data meaning training loss often drops abruptly and significantly early in training~\citep{Li:2021:Antibackdoor}.
For an attack to succeed, adversarial instances must align with the target at some point during training, meaning dynamic methods can detect them.

Lastly, our threat model specifies that attackers never know the random batch sequence nor any randomly initialized parameters.  Therefore, attackers can only craft~\eqsmall{$\advTrain$} to be \textit{influential in expectation} over that randomness.  Influence is stochastic, varying significantly across random seeds.  However, estimating the true expected influence is computationally expensive. \method{} and \tracinCP{}, which simulate expectation, better align with how the adversary actually crafts the adversarial set, resulting in better \eqsmall{$\advTrain$}~identification.

Below we detail how renormalization can be specialized further for better adversarial-set identification.

\begin{figure}[t]
  \centering
\newcommand{\legendSpacer}{\hspace*{7pt}}
\begin{tikzpicture}
  \begin{axis}[%
      hide axis,  %
      xmin=0,  %
      xmax=1,
      ymin=0,
      ymax=1,
      scale only axis,width=1mm, %
      line width=\CdfLineWidth,
      legend cell align={left},              %
      legend style={font=\scriptsize},  %
      legend columns=4,
      legend image post style={xscale=0.6},  %
    ]
    \addplot [ForestGreen, dashed] coordinates {(0,0)};
    \addlegendentry{1~Pixel Conv1\legendSpacer}
    \addplot [violet, dashed] coordinates {(0,0)};
    \addlegendentry{1~Pixel Linear\legendSpacer}

    \addplot [ForestGreen] coordinates {(0,0)};
    \addlegendentry{Blend Conv1\legendSpacer}
    \addplot [violet] coordinates {(0,0)};
    \addlegendentry{Blend Linear}
\end{axis}
\end{tikzpicture}

\newcommand{\layerFontSize}{\footnotesize}
\pgfplotstableread[col sep=comma] {plots/data/layer_trend.csv}\thedata%
\begin{tikzpicture}
  \begin{axis}[
    width=0.95\columnwidth,%
    height=\LayerTrendHeight,%
    xmin=0,%
    xmax=8,%
    xtick distance={2},
    minor x tick num={3},
    x tick label style={font=\layerFontSize,align=center},%
    xlabel={\layerFontSize Epoch},
    xmajorgrids,
    ymin=0,
    ymax=70,
    ytick distance={20},
    minor y tick num={3},
    ymajorgrids,
    y tick label style={font=\layerFontSize,align=center},%
    ylabel={\layerFontSize \% of\\\vspace{4pt}\layerFontSize$\norm{\gradItr{\gradHatTarg}}$},
    ylabel style={rotate=-90, align=center},  %
    mark size=0pt,
    line width=\CdfLineWidth,
    ]
    \addplot [ForestGreen, dashed] table [x index=0, y index=1] \thedata;
    \addplot [violet, dashed] table [x index=0, y index=2] \thedata;

    \addplot [ForestGreen] table [x index=0, y index=3] \thedata;
    \addplot [violet] table [x index=0, y index=4] \thedata;
  \end{axis}
\end{tikzpicture}
  \caption{%
    \textit{Layerwise Decomposition of an Attack Target's Intra-Training Gradient Magnitude}: One-pixel \& blend backdoor adversarial triggers (dashed \& solid lines respectively) trained separately on CIFAR10 binary classification (${\yTarg = \texttt{airplane}}$ \& ${\yAdv = \texttt{bird}}$) using ResNet9.
    The network's first convolutional~(Conv1) and final linear layers are a small fraction of the parameters (0.03\% \& 0.01\% resp.) but constitute most of the target's gradient magnitude (\eqsmall{$\norm*{\gradHatTarg}$}) with the dominant layer attack dependent.
  Results are averaged over 20~trials.%
  }
  \label{fig:Method:LayerTrend}
\end{figure}
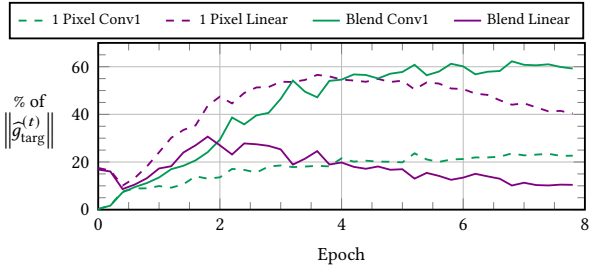

\vspace{3pt}
\textbf{Extending Renormalization Layerwise}:
In practice,
gradient magnitudes are often unevenly distributed across a neural network's layers.  For example, Figure~\ref{fig:Method:LayerTrend} tracks an attack target's average intra\=/training gradient magnitude for two different backdoor adversarial triggers on CIFAR10 binary classification (\eqsmall{${\yTarg = \texttt{airplane}}$} and \eqsmall{${\yAdv = \texttt{bird}}$}).\footnote{See supplemental Section~\ref{sec:App:ExpSetup} for the complete experimental setup. The class pair and adversarial triggers were proposed by \citet{Weber:2021}.}
Specifically, target gradient norm, \eqsmall{${\norm{\gradW \lFunc{\zHatTarg}{\wT}}}$}, is decomposed into just the contributions of the network's first convolutional layer (Conv1) and the final linear layer. Despite being only 0.04\% of the model parameters, these two layers combined constitute ${{>}50}$\%~of the gradient norm.  Therefore, the first and last layers' parameters are, on average, weighted ${{>}2,000\times}$~more than other layers' parameters.  With simple renormalization, important parameters in those other layers may go undetected.

As an alternative to simply renormalizing by \eqsmall{${\norm{\gradW \lFunc{\z}{\wT}}}$}, partition gradient vector~$\gradLetter$ by layer into $\nLayer$~disjoint vectors (where $\nLayer$~is model~$\dec$'s layer count) and then independently renormalize each subvector separately.  This \keyword{layerwise renormalization} can be applied to any estimator that uses training gradient~$\gradI$ or test gradient~$\gradHatTe$, including influence functions, \tracin{}, and \tracinCP{}.
Layerwise renormalization still corrects for the low-loss penalty and does not change the asymptotic complexity.  To switch \method{} to layerwise, the only modification to Algorithm~\ref{alg:CosIn} is on Line~10 where each dimension is divided by its corresponding layer's norm instead of the full gradient norm.  %

\textbf{Notation}:
``\layerwiseSuffix''~denotes layerwise renormalization, e.g.,~\keyword{layerwise \method{}} is \layer{}.   Suffix~``\bothSuffix'', e.g.,~\bothMethod{}, signifies that a statement applies irrespective of whether the renormalization is layerwise.

\subsection{Renormalization \& Non-Adversarial Data}\label{sec:RenormInfluence:Filtering}

Renormalization not only improves performance identifying an inserted adversarial set; it also improves performance in \textit{non-adversarial} settings.
Sec.~\ref{sec:RelatedWork:Influence} defines influence w.r.t.\ a single training example.  Just as one instance may be more influential on a prediction than another, a group of training instances may be more influential than a different group.  Renormalization improves identification of influential \textit{groups} of examples, even on non-adversarial data.

To empirically demonstrate this, consider CIFAR10 binary classification again.  In each trial, ResNet9 was pre\=/trained on eight (${=10-2}$) held-out CIFAR10 classes.  From the other two classes, test example~\eqsmall{$\zTeFilt$} was selected u.a.r.\ from those test instances with a moderate misclassification rate (10\=/20\%) across multiple retrainings (i.e.,~fine-tunings) of the pre\=/trained network.\footnote{Using examples with a moderate misclassification rate ensures that dataset filtering's effects are measurable even with a small fraction of the training data removed.}
(Renormalized) influence was then calculated for~\eqsmall{$\zTeFilt$}, with each estimator yielding a training-set ranking. Each estimator's top~$p$\% ranked instances were removed from the training set and 20~models trained from the pre\=/trained parameters using these reduced training sets. Performance is measured using \eqsmall{$\zTeFilt$}'s misclassification rate across those 20~models where a larger error rate entails a better overall ranking.

\begin{figure}[t]
  \centering
\newcommand{\legendSpacer}{\hspace*{8pt}}
\newcommand{\oursText}{ {\scriptsize (ours)}}

\begin{tikzpicture}
  \begin{axis}[%
      hide axis,  %
      xmin=0,  %
      xmax=1,
      ymin=0,
      ymax=1,
      scale only axis,width=1mm, %
      line width=\CdfLineWidth,
      legend cell align={left},              %
      legend style={font=\scriptsize},
      legend columns=4,
      legend image post style={xscale=0.6},  %
    ]
    \addplot [Random Line] coordinates {(0,0)};
    \addlegendentry{Random\legendSpacer}

    \addplot [Influence Functions Line] coordinates {(0,0)};
    \addlegendentry{Inf.\ Func.\legendSpacer}
    \addplot [Influence Functions Sim Line] coordinates {(0,0)};
    \addlegendentry{Inf.\ Func.\ \simShort\oursText{}\legendSpacer}
    \addplot [Influence Functions Layer Line] coordinates {(0,0)};
    \addlegendentry{Inf.\ Func.\ \simShort-L\oursText{}}

    \addplot [TracIn Line] coordinates {(0,0)};
    \addlegendentry{\tracin{}\legendSpacer}
    \addplot [TracInCP Line] coordinates {(0,0)};
    \addlegendentry{\tracinCP{}\legendSpacer}

    \addplot [CosIn Line] coordinates {(0,0)};
    \addlegendentry{\method{}\oursText{}\legendSpacer}
    \addplot [LayIn Line] coordinates {(0,0)};
    \addlegendentry{\layer{}\oursText{}}
\end{axis}
\end{tikzpicture}
 
  \begin{subfigure}{\columnwidth}
    \centering
\newcommand{\ifFontSize}{\scriptsize}

\pgfplotstableread[col sep=comma] {plots/data/inf_filt_inf-func.csv}\thedata%
\begin{tikzpicture}
  \begin{axis}[
      smooth,
      width=\columnwidth,%
      height=\InfFiltHeight,%
      xmin=0,%
      xmax=30,%
      xtick distance={5},
      x tick label style={font=\ifFontSize,align=center},%
      xlabel={\ifFontSize \% Training Set Removed},
      xmajorgrids,
      ymin=0,
      ymax=1,
      ytick distance={0.2},
      minor y tick num={1},
      ymajorgrids,
      y tick label style={font=\ifFontSize,align=center},%
      ylabel={\ifFontSize $\zTeFilt$ Misclass. Rate},
      mark size=0pt,
      line width=\CdfLineWidth,
    ]
    \addplot [Influence Functions Line] table [x index=0, y index=1] \thedata;
    \addplot [Influence Functions Sim Line] table [x index=0, y index=2] \thedata;

    \addplot [Influence Functions Layer Line] table [x index=0, y index=3] \thedata;

    \addplot [Random Line] table [x index=0, y index=4] \thedata;
  \end{axis}
\end{tikzpicture}
 
    \caption{Influence functions-based methods}
    \label{fig:RenormInfluence:InfFilt:InfFunc}
  \end{subfigure}

  \begin{subfigure}{\columnwidth}
    \centering
\newcommand{\ifFontSize}{\scriptsize}

\pgfplotstableread[col sep=comma] {plots/data/inf_filt_tracin.csv}\thedata%
\begin{tikzpicture}
  \begin{axis}[
      smooth,
      width=\columnwidth,%
      height=\InfFiltHeight,%
      xmin=0,%
      xmax=30,%
      xtick distance={5},
      x tick label style={font=\ifFontSize,align=center},%
      xlabel={\ifFontSize \% Training Set Removed},
      xmajorgrids,
      ymin=0,
      ymax=1,
      ytick distance={0.2},
      minor y tick num={1},
      ymajorgrids,
      y tick label style={font=\ifFontSize,align=center},%
      ylabel={\ifFontSize $\zTeFilt$ Misclass. Rate},
      mark size=0pt,
      line width=\CdfLineWidth,
    ]
    \addplot [TracIn Line] table [x index=0, y index=1] \thedata;
    \addplot [TracInCP Line] table [x index=0, y index=2] \thedata;

    \addplot [CosIn Line] table [x index=0, y index=3] \thedata;

    \addplot [LayIn Line] table [x index=0, y index=4] \thedata;

    \addplot [Random Line] table [x index=0, y index=5] \thedata;
  \end{axis}
\end{tikzpicture}
 
    \caption{\tracin{}-based methods}
    \label{fig:RenormInfluence:InfFilt:TracIn}
  \end{subfigure}

  \caption{%
    \textit{Effect of Removing Influential, Non-Adversarial Training Data}:
    Test example \eqsmall{$\zTeFilt$}'s misclassification rate (larger is better) when filtering the training set using influence rankings based on influence functions (top) and \tracin{} (bottom).
    Renormalization~(Rn.) always improved mean performance across all training-set filtering percentages. Results are averaged across five CIFAR10 class pairs with 30~trials per class pair and 20~models trained per method per trial.
    Results are separated by the reference influence estimator.
  }
  \label{fig:RenormInfluence:InfFilt}
\end{figure}
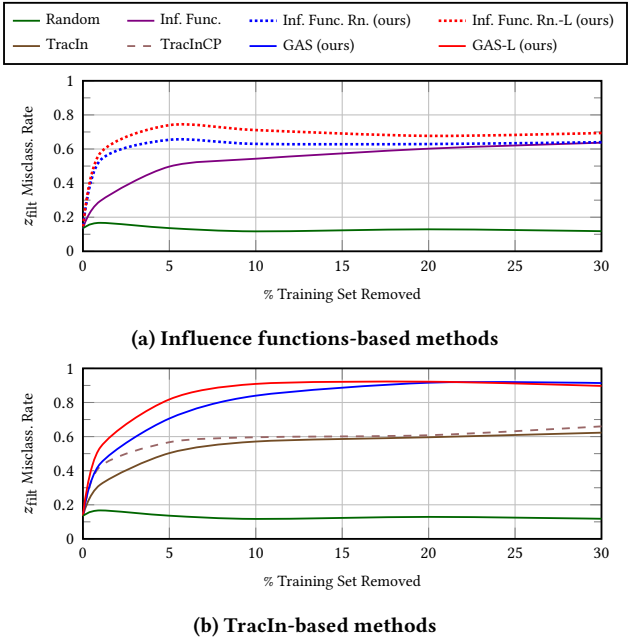
 
Figure~\ref{fig:RenormInfluence:InfFilt} compares influence estimation's filtering performance, with and without renormalization, against a random baseline averaged across five CIFAR10 class pairs, namely the two pairs specified by \citet{Weber:2021} and three additional random pairs. Influence, irrespective of renormalization, significantly outperformed random removal, meaning all of these estimators found influential subsets, albeit of varying quality.\footnote{Representer point (Eq.~\eqref{eq:RelatedWork:RepPt}) is excluded as it underperformed random filtering.}  In all cases, renormalized influence had better or equivalent performance to the original estimator across all filtering fractions. This demonstrates that renormalization generalizes across estimators even beyond adversarial settings.

Overall, layerwise renormalization was the top performer across all setups except for large filtering percentages where \method{} surpassed it slightly.
Renormalized\bothSuffix{} Influence functions and \bothMethod{} performed similarly when filtering a small fraction (e.g.,~${\leq}5\%$) of the training data.
However, the performance of renormalized influence functions plateaued for larger filtering fractions~(${{\geq}10\%}$) while \bothMethod{}'s performance continued to improve.  In addition, renormalization's performance advantage over vanilla influence functions narrowed at larger filtering fractions. In contrast, \bothMethod{}'s advantage over \tracin{} and \tracinCP{} remained consistent.  Recall that dynamic methods (e.g.,~\bothMethod{} and~\tracin{}) use significantly more gradient information than static methods (e.g.,~influence functions).  This experiment again demonstrates that loss-based renormalization's benefits increase as more gradient information is used.

\section{Identifying Attack Targets}\label{sec:Method}

Recall that non-targets have primarily weak influences and few very strong ones.
Target instances are \revised{anomalous in that} they have an unusual number of highly-influential training instances \revised{(Figure~\ref{fig:QQ})}.
This idea is the core of our \keyword{\underline{f}ramework for \underline{i}dentifying \underline{t}argets} of training-set attacks, \fit{}.
Alg.~\ref{alg:TargDetect} formalizes \fit{} as an end-to-end procedure to identify any targets in test example analysis set~\eqsmall{$\targAnalysisSet$}.\footnote{For simplicity, Alg.~\ref{alg:TargDetect} considers a single identified target. If there are multiple identified targets, \MitigateMethodName{} is invoked on each target serially with parameters~{$\filtWFin$} and~{$\filtTrain$}.}  Overall, \fit{} has three sub\=/steps, described chronologically:
\begin{enumerate}
  \item \RenormInfMethodName: Calculates (renormalized) influence vector~$\infVec$ for each test instance in analysis set~\eqsmall{$\targAnalysisSet$}.
  \item \AnomScoreMethodName: Targets have an unusual number of highly-influential instances.  \revised{Leveraging ideas from anomaly detection}, this step analyzes each test instance's influence vector~$\infVec$ and ranks those instances based on how anomalous their influence values are.
   \item \MitigateMethodName: Target-driven mitigation sanitizes model parameters~\eqsmall{$\wFin$} and training set~\eqsmall{$\fullTrain$} to remove the attack's influence on the most likely target \eqsmall{$\zHatTarg$} (e.g.,~the most anomalous misclassified instance).
\end{enumerate}
\fit{} is referred to as a ``framework'' since these subroutines are general and their underlying algorithms can change as new versions are developed.  The next three subsections describe our implementation of each of these methods.  For reference, suppl.\ Alg.~\ref{alg:TargDetect:Implementation} specializes Alg.~\ref{alg:TargDetect} to more closely align with the implementation details below.  \revised{Suppl.\ Sec.~\ref{sec:App:Alg:Complexity} details \fit{}'s end\=/to\=/end computational complexity.}

\begin{algorithm}[t]
  {%
    \AlgFontSize%
    \caption{\fit{} target identification \& mitigation}\label{alg:TargDetect}
    \newcommand{\InputDesc}{Training set $\fullTrain$, test example set~$\targAnalysisSet$, and final params.~$\wFin$}
\newcommand{\OutputDesc}{Sanitized model parameters~$\filtWFin$ \& training set~$\filtTrain$}
\begin{algorithmic}[1]
  \algSetStretch%
  \REQUIRE \InputDesc{}

  \ENSURE \OutputDesc{}

  \STATE $\simDistSet \gets
                \setbuildDynamic{\SimFunc{\zHat}{\fullTrain}{}}
                                {\zHat \in \targAnalysisSet}$
                \algcomment{(Renorm.) Inf.\ (Alg.~\ref{alg:CosIn})}
  \STATE $\scoreDistSet \gets
                \setbuildDynamic{\AnomScoreFunc{{\infVec}}{\simDistSet}}
                                {{\infVec} \in \simDistSet}$
                \algcomment{Anomaly score (Sec.~\ref{sec:Method:TargetDetection})}
  \STATE Rank $\targAnalysisSet$ by anomaly scores~$\scoreDistSet$
  \STATE $\zHatTarg \gets $ Most anomalous test example in $\targAnalysisSet$
    \STATE $\filtWFin$,~~$\filtTrain$ $\gets$ \mitigateFunc{$\zHatTarg$}{$\wFin$}{$\fullTrain$} \algcomment{Sec.~\ref{sec:Method:Mitigation}}
  \STATE \RETURN $\filtWFin$,~~$\filtTrain$
\end{algorithmic}

  }
\end{algorithm}

\begin{figure*}[!t]
\newcommand{\pdfFontSize}{\small}

\newcommand{\gauss}[2]{1/(#2*sqrt(2*pi))*exp(-((x-#1)^2)/(2*#2^2))} %
\newcommand{\PlotGaussian}[2]{%
   \addplot+[dashed, line width=0.5, samples=30, domain=-3:3, black] {\gauss{#1}{#2}};
}

\newcommand{\PdfPlot}[7]{%
  \pgfplotstableread[col sep=comma] {plots/data/#1}\thedata%
  \begin{tikzpicture}
    \begin{axis}[
        width=\columnwidth,%
        height=\QQHeight,%
        xmin=#2,%
        xmax=#3,%
        xticklabels={,,,},
        x tick label style={font=\pdfFontSize,align=center},%
        xtick={#4},
        xmajorgrids,
        ymin=0,
        ymax=#5,
        yticklabels={,,,},
        ytick distance={#6},
        minor y tick num=1,
        ymajorgrids,
        y tick label style={font=\pdfFontSize,align=center},%
        ylabel={\pdfFontSize Density},
        mark size=0pt,
      ]
      \addplot +[
        line width=\CdfLineWidth,
        smooth,
        color=blue,
        hist={
          density,
          bins=30,
          data min=#2,
          data max=#3,
          handler/.style={
            sharp plot,
            smooth,
          }
        }
      ] table [y index=0] \thedata;
      \addplot +[
        line width=\CdfLineWidth,
        smooth,
        color=red,
        hist={
          density,
          bins=30,
          data min=#2,
          data max=#3,
          handler/.style={
            sharp plot,
            smooth,
          }
        }
      ] table [y index=1] \thedata;
      #7
    \end{axis}
  \end{tikzpicture}
}
   \captionsetup[subfigure]{justification=centering}  %
  \newcommand{\datasetTitle}[1]{\rotatebox{90}{\textbf{\footnotesize #1}}}
  \newcommand{\qqPreSpacer}{\hspace{-7.3pt}}
  \newcommand{\qqImgWidth}{0.329\textwidth}
  \newcommand{\twoQqImgWidth}{0.660\textwidth}
  \newcommand{\RenormInfStr}{renormed.\ influence}
  \centering

  \begin{minipage}{0.035\textwidth}
    \hfill
  \end{minipage}
  \qqPreSpacer
  \begin{minipage}{0.96\textwidth}
    \centering
\newcommand{\legendSpacer}{\hspace*{8pt}}
\begin{tikzpicture}
  \begin{axis}[%
      hide axis,  %
      xmin=0,  %
      xmax=1,
      ymin=0,
      ymax=1,
      scale only axis,width=1mm, %
      legend cell align={left},              %
      legend style={font=\legendFontSize},
      legend columns=3,
    ]
    \addplot [smooth, color=red, xscale=0.4, line width=\CdfLineWidth] coordinates {(0,0)};
    \addlegendentry{$\advTrain$\legendSpacer}
    \label{leg:QQ:Adv}

    \addplot [smooth, color=blue, xscale=0.4, line width=\CdfLineWidth] coordinates {(0,0)};
    \addlegendentry{$\cleanTrain$\legendSpacer}
    \label{leg:QQ:Clean}

    \addplot [dashed, line width=0.5, xscale=0.6, black, line width=\CdfLineWidth] coordinates {(0,0)};%
    \addlegendentry{Theoretical Normal}%
    \label{leg:QQ:Normal}
  \end{axis}
\end{tikzpicture}
  \end{minipage}

  \begin{minipage}{0.035\textwidth}
    \datasetTitle{Vision~Poison}
  \end{minipage}
  \qqPreSpacer
  \begin{minipage}{0.96\textwidth}
    \begin{subfigure}{\qqImgWidth}
      \centering
      \PdfPlot{pdf_pois_cifar_adv.csv}{-3}{9.5}{-0.9166,1.1667,3.250,5.33333,7.416667}{0.93}{0.155}{\PlotGaussian{-0.00601}{0.947}}
      \caption{\method{} target \RenormInfStr}\label{fig:QQ:Vision:Target}
    \end{subfigure}
    \hfill
    \begin{subfigure}{\qqImgWidth}
      \centering
      \PdfPlot{pdf_pois_cifar_clean.csv}{-3}{9.5}{-0.9166,1.1667,3.250,5.33333,7.416667}{0.93}{0.155}{\PlotGaussian{0.00658}{1.01}}
      \caption{\method{} non-target \RenormInfStr}\label{fig:QQ:Vision:Clean}
    \end{subfigure}
    \hfill
    \begin{subfigure}{\qqImgWidth}
      \centering
      \PdfPlot{norm_cifar_cdf.csv}{-2.5}{3.0}{-1.583,-0.667,0.25,1.167,2.083}{1.20}{0.2}{}
      \caption{Training-set gradient norm}\label{fig:QQ:Vision:Grad}
    \end{subfigure}
  \end{minipage}

  \vspace{10pt}
  \begin{minipage}{0.035\textwidth}
    \datasetTitle{Speech Backdoor}
  \end{minipage}
  \qqPreSpacer
  \begin{minipage}{0.96\textwidth}
    \begin{subfigure}{\qqImgWidth}
      \centering
      \PdfPlot{pdf_bd_speech_adv.csv}{-2.5}{8.25}{-0.7083,1.0833,2.875,4.66667,6.4583}{1.75}{0.291667}{\PlotGaussian{-0.168}{0.804}}
      \caption{\method{} target \RenormInfStr}\label{fig:QQ:Speech:Target}
    \end{subfigure}
    \hfill
    \begin{subfigure}{\qqImgWidth}
      \centering
      \PdfPlot{pdf_bd_speech_clean.csv}{-2.5}{8.25}{-0.7083,1.0833,2.875,4.66667,6.4583}{2.40}{0.400}{\PlotGaussian{-0.184}{0.532}}
      \caption{\method{} non-target \RenormInfStr}\label{fig:QQ:Speech:Clean}
    \end{subfigure}
    \hfill
    \begin{subfigure}{\qqImgWidth}
      \centering
      \PdfPlot{norm_speech_cdf.csv}{-0.4}{0.9}{-0.1833,0.033,0.25,0.46667,0.68333}{9}{1.5}{}
      \caption{Training-set gradient norm}\label{fig:QQ:Speech:Grad}
    \end{subfigure}
  \end{minipage}
  \caption{%
    \method{} renormalized influence,~$\infVec$, density distributions for two training-set attacks: CIFAR10 vision poisoning~\citep{Zhu:2019} (${\yTarg = }{~\texttt{dog}}$ \& ${\yAdv = }{~\texttt{bird}}$) \& speech-recognition backdoor \citep{SpeechDataset} (${\yTarg = \texttt{4}}$ \& ${\yAdv = \texttt{5}}$).
    Theoretical normal (\tikzCaption{\ref{leg:QQ:Normal}}) is w.r.t.\ \eqsmall{${\fullTrain \defeq \advTrain \cup \fullTrain}$}.
    Observe that target examples (Figs.~\ref{fig:QQ:Vision:Target} \&~\ref{fig:QQ:Speech:Target}) have significant \eqsmall{$\advTrain$}~mass~(\tikzCaption{\ref{leg:QQ:Adv}}) well to the right of \eqsmall{$\cleanTrain$}'s mass~(\tikzCaption{\ref{leg:QQ:Clean}}). This upper-mass phenomenon is absent in non-targets (Figs.~\ref{fig:QQ:Vision:Clean} \&~\ref{fig:QQ:Speech:Clean}).
    Training example gradient norms (Fig.~\ref{fig:QQ:Vision:Grad} \&~\ref{fig:QQ:Speech:Grad}) are poorly correlated with whether the training example is adversarial. For example, speech recognition has \eqsmall{$\cleanTrain$}~mass well to the right of even the right\=/most \eqsmall{$\advTrain$} mass, necessitating renormalization.
    See Sections~\ref{sec:ExpRes:Attacks} and~\ref{sec:App:ExpSetup} for more details on these attacks.%
  }
  \label{fig:QQ}
\end{figure*}
 
\subsection{Measuring (Renormalized) Influence}\label{sec:Method:RenormInfluence}

Algorithm~\ref{alg:TargDetect} is agnostic of the specific (renormalized) influence estimator used to calculate~$\infVec$, provided that method is sufficiently adept at identifying adversarial set~\eqsmall{$\advTrain$}. %
We use \bothMethod{} for the reasons explained in Section~\ref{sec:RenormInfluence} as well as its simplicity, computational efficiency, and strong, consistent empirical performance.

\vspace{3pt}
\textit{Time and Space Complexity}:
Computing a gradient requires $\bigO{\dimW}$ time and space.  For fixed~$\nItr$ and~$\dimW$, \tracinCP{}, \method{}, and \layer{} require $\bigO{\nTr}$ time and space to calculate each test instance's influence vector~$\infVec$.
The next section explains that \fit{} analyzes each test instance's influence vector~$\infVec$ meaning \bothMethod{} can be significantly sped-up by amortizing training gradient~(\eqsmall{$\gradItr{\gradI}$}) computation across multiple test examples -- either on a single node (suppl.\ Sec.~\ref{sec:App:MoreExps:Runtime}) or across multiple nodes (e.g.,~using all\=/reduce).
\subsection{Identifying Anomalous Influence}\label{sec:Method:TargetDetection}

To change a prediction, adversarial set~\eqsmall{$\advTrain$} must be highly influential on the target. When visualizing $\zHatTe$'s influence vector~$\infVec$ as a density distribution, an exceptionally influential~\eqsmall{$\advTrain$} manifests as a distinct density mass at the distribution's positive extreme.

Figures~\ref{fig:QQ:Vision:Target} and~\ref{fig:QQ:Speech:Target} each plot an attack \underline{target}'s \method{} influence as a density for two different training-set attacks -- the first poisoning on vision~\citep{Zhu:2019} and the other a backdoor attack on speech recognition~\citep{SpeechDataset}.  For both attacks, adversarial set \eqsmall{$\advTrain$}'s influence significantly exceeds that of~\eqsmall{$\cleanTrain$}.  When compared to theoretical normal (calculated\footnote{The plotted theoretical normal used robust statistics median and~$\Qn$ in place of mean and standard deviation.} w.r.t.\ complete training set~\eqsmall{$\fullTrain$}), \eqsmall{$\advTrain$}'s target influence is highly anomalous.  In Figures~\ref{fig:QQ:Vision:Clean} and~\ref{fig:QQ:Speech:Clean}, which plot the \method{} influence of \underline{non-targets} for the same two attacks, no extremely high influence instances are present.

Going forward, influence vectors~$\infVec$ with exceptionally high influence instances are referred to as having a \keyword{heavy upper tail}.
Then, \textit{target identification simplifies to identifying influence vectors whose values have anomalously heavy upper tails}.
The preceding insight is relative and is w.r.t.\ to other test instances' influence value distributions.
Non-target baseline anomaly quantities vary with model, dataset, and hyperparameters.
That is why suppl.\ Algorithm~\ref{alg:TargDetect:Implementation} ranks candidates in \eqsmall{$\targAnalysisSet$} based on their upper-tail heaviness.

\textbf{Quantifying Tail Heaviness}:
Determining whether $\zHatTe$'s influence vector~$\infVec$ is abnormal simplifies to univariate anomaly detection for which significant previous work exists \citep{Barnett:1978,Rousseeuw:1997,Hodge:2004,Rousseeuw:2017}. Observe in Figures~\ref{fig:QQ:Vision:Clean} and~\ref{fig:QQ:Speech:Clean} that \eqsmall{$\cleanTrain$}'s \method{} influence vector~$\infVec$ tends to be normally distributed (see the close alignment to the dashed line).
We, therefore, use the traditional \keyword{anomaly score},
\eqsmall{$\anomScoreVec \defeq \frac{\infVec - \scoreCenter}{\scoreSpread}$},
where $\scoreCenter$ and $\scoreSpread$ are each~$\infVec$'s center and dispersion statistics, resp.\footnote{In suppl.\ Algorithm~\ref{alg:TargDetect:Implementation}, statistics~$\TeIdx{\scoreCenter}$ and~$\TeIdx{\Qn}$ are calculated separately for each test instance~$\zHat_j$'s influence vector~$\TeIdx{\infVec}$.}
Mean and standard deviation, the traditional center and dispersion statistics, resp., are not robust to outliers. Both have an asymptotic \keyword{breakdown point} of~0 (one anomaly can shift the estimator arbitrarily). Since \eqsmall{$\advTrain$} instances are inherently outliers, robust statistics are required.

Median serves as our center statistic~$\scoreCenter$ given its optimal breakdown~(50\%).
Although median absolute deviation~(MAD) is the best known robust dispersion statistic, we use \citepos{Rousseeuw:1993} $\Qn$~estimator, which retains MAD's~benefits while addressing its weaknesses.
Specifically, both MAD and $\Qn$ have optimal breakdowns, but $\Qn$~has better Gaussian data \keyword{efficiency} (82\% vs.~37\%).  Critically for our setting with one-sided anomalies, $\Qn$~does not assume data symmetry -- unlike MAD.  Formally,
{%
  \begin{equation}\label{eq:Method:TargetDetection:Qn}
    \Qn \defeq \dConsist \setbuild{ \abs{\infScalarI - \infScalarIP} }{1 \leq \trIdx < \trIdxP \leq \nTr}_{(\qnOrderStat)} \text{,}
  \end{equation}%
}%
\noindent%
where
\eqsmall{$\set{\cdot}_{(\qnOrderStat)}$} denotes the set's $\qnOrderStat$\=/th~\keyword{order statistic} with \eqsmall{${\qnOrderStat = \binom{\floor{\frac{\nTr}{2}} + 1}{2}}$} %
and
$\dConsist$~is a distribution consistency constant which for Gaussian data, ${\dConsist \approx 2.2219}$~\citep{Rousseeuw:2017}.
Eq.~\eqref{eq:Method:TargetDetection:Qn} requires only ${\bigO{\nTr}}$ space and ${\bigO{\nTr \lg \nTr}}$ time as proven by \citet{Croux:1992}.
Provided anomaly score vector~$\anomScoreVec$, upper-tail heaviness is simply~\eqsmall{$\anomScoreVec_{(\nTr - \anomCount)}$}, which is $\anomScoreVec$'s ${(\nTr - \anomCount)}$\textsuperscript{th}~order statistic, i.e.,~\eqsmall{$\zHatTe$}'s $\anomCount$\textsuperscript{th}~largest anomaly score value.
\revTwo{%
The value of~$\anomCount$ implicitly affects the size of the smallest detectable attack, where any attack with \eqsmall{${\abs{\advTrain} < \anomCount}$} is much harder to detect.
}%

\methodparagraph{Multiclass vs.\ Binary Classification}
Different classes are implicitly generated from different data distributions. Each class's data distribution may have different influence tails -- in particular in multiclass settings.  Target identification performance generally improves (1)~ when $\scoreCenter$ and~$\Qn$ are calculated w.r.t.\ only training instances labeled~$\yHatTe$ and (2)~$\zHatTe$'s upper-tail heaviness is ranked w.r.t.\ other test instances labeled~$\yHatTe$.

\methodparagraph{Faster \fit{}}
The execution time of \tracinCP{} and by extension \bothMethod{}, depends on parameter count~$\abs{\w}$.  For very large models, target identification can be significantly sped up via a two-phase strategy.  In phase~1, \bothMethod{} uses a very small iteration subset (e.g.,~\eqsmall{${\subsetItr = \set{\nItr}}$}) to coarsely rank analysis set~\eqsmall{$\targAnalysisSet$}. Phase~2 then uses the complete~$\subsetItr$ but only on a small fraction (e.g.,~10\%) of \eqsmall{$\targAnalysisSet$} with the heaviest phase~1 tails. Section~\ref{sec:ExpRes:TargetIdent} applies this approach to natural-language data poisoning on \robertaBase{} \citep{RoBERTa}.

Computing each test instance's (\eqsmall{${\zHatTe \in \targAnalysisSet}$}) influence vector~$\infVec$ is independent. Each dimension~$\infScalarI$ is also independent and can be separately computed.  Hence, \bothMethod{} is embarrassingly parallel allowing linear speed-up of target identification via parallelization.%
\subsection{Target-Driven Attack Mitigation}\label{sec:Method:Mitigation}

A primary benefit of target identification is that attack mitigation becomes straightforward.
Algorithm~\ref{alg:Mitigation} mitigates attacks by sanitizing training set~\eqsmall{$\fullTrain$} of adversarial set~\eqsmall{$\advTrain$}.
Most importantly, target identification solves data sanitization's common pitfall (Sec.~\ref{sec:RelatedWork:Defenses}) of determining how much data to remove. \textit{Sanitization stops when the target's misprediction is eliminated}.
Therefore, successfully identifying a target means sanitization is \textit{guaranteed to succeed tautologically} (i.e.,~attack success rate on any analyzed targets is~0).

More concretely, Alg.~\ref{alg:Mitigation} iteratively filters~\eqsmall{$\fullTrain$} by thresholding anomaly score vector,~$\anomScoreVec$.\footnote{Alg.~\ref{alg:Mitigation} considers the more general case of a single identified target but can be extended to consider multiple targets. For instance, provided there is a single attack, average~$\infVec$ across all targets, and stop sanitizing once all targets are classified correctly.}  Since adversarial instances are abnormally influential on targets, Alg.~\ref{alg:Mitigation} filters \eqsmall{$\advTrain$} instances first. After each iteration, influence is remeasured to account for estimation stochasticity and because training dynamics may change with different training sets.
Data removal \revTwo{cutoff}~$\anomCutoff$ is tuned based on computational constraints -- larger~$\anomCutoff$ results in less clean data removed but may take more iterations.  Slowly annealing~$\anomCutoff$ also results in less clean-data removal.

\begin{algorithm}[t]
  {%
    \AlgFontSize
    \caption{Target-driven mitigation \& sanitization}\label{alg:Mitigation}
    \newcommand{\InputDesc}{Target~\eqsmall{$\zHatTarg \defeq {(\xTarg,\yAdv)}$}, anomaly cutoff~$\anomCutoff$, model~$\dec$, initial params.~$\wZero$, final params.~$\wFin$, and training set~\eqsmall{$\fullTrain$}}
\newcommand{\OutputDesc}{Clean model parameters~\eqsmall{$\filtWFin$} \& sanitized training set~\eqsmall{$\filtTrain$}}
\begin{algorithmic}[1]
  \algSetStretch%
  \REQUIRE \InputDesc{}
  \ENSURE \OutputDesc{}
  \FUNCTION{\mitigateFunc{$\zHatTarg$}{$\wFin$}{$\fullTrain$}}{}{}
    \STATE $\filtWFin,~~\filtTrain \gets \wFin,~~\fullTrain$
    \WHILE{$\argmax \left(\decFunc{\xTarg}{\filtWFin}\right) = \yAdv$}
      \STATE $\infVec \gets \SimFunc{\zHatTarg}{\fullTrain}{}$ \algcomment{Renorm.\ Influence (Alg.~\ref{alg:CosIn})}
      \STATE $\anomScoreVec \gets \frac{\infVec - \scoreCenter}{\Qn}$ \algcomment{Anomaly score (Sec.~\ref{sec:Method:TargetDetection})}
      \STATE $\filtTrain \gets \filtTrain \setminus \setbuildDynamic{\zI}{\anomScoreI \geq \anomCutoff \wedge \zI \in \filtTrain}$  \algcomment{Sanitize}
      \STATE $\filtWFin \gets \retrainFunc{\wZero}{\filtTrain}$
      \STATE Optionally anneal $\anomCutoff$
    \ENDWHILE
    \STATE \RETURN $\filtWFin,~~\filtTrain$
  \ENDFUNCTION
\end{algorithmic}

  }
\end{algorithm}

Given forensic or human analysis of the identified target(s), simpler mitigation than Algorithm~\ref{alg:Mitigation} is possible, e.g.,~a naive, rule-based, corrective lookup table that entails no clean data removal at all.

\revised{%
For learning environments where \keyword{certified training data deletion} is possible \citep{Guo:2020:CertifiedRemoval,Marchant:2022:HardToForget}, retraining (Alg.~\ref{alg:Mitigation} Line~7) may not even be required --- making our method even more efficient.
}

\revTwo{%
\methodparagraph{Enhancing Mitigation's Robustness}
An adversary could attack \fit{} by injecting adversarial instances into~\eqsmall{$\fullTrain$} to specifically trigger excessive, unnecessary sanitization.
To mitigate such a risk, Alg.~\ref{alg:Mitigation} could be tweaked to include a maximum sanitization threshold\footnote{\revTwo{This threshold could be w.r.t.\ the number of examples removed or the change in held-out loss. These quantities can be measured cumulatively or for targets individually.}} that would trigger additional (e.g.,~human, forensic) analysis.  This threshold could be set empirically or using domain-specific knowledge (e.g.,~maximum possible poisoning rate).
See supplemental Section~\ref{sec:App:MoreExps:MitigationTriggering} for further discussion.
}%
\newcommand{\minipageWidth}{0.48\textwidth}
\newcommand{\evalparagraph}[1]{%
\noindent%
\textbf{#1}~~~}

\begin{figure*}[t]
  \centering
  \begin{minipage}[t]{\minipageWidth}
    \centering
\newcommand{\legendSpacer}{\hspace*{8pt}}

\begin{tikzpicture}
  \begin{axis}[%
    width=\textwidth,
      ybar,
      hide axis,  %
      xmin=0,  %
      xmax=1,
      ymin=0,
      ymax=1,
      scale only axis,width=1mm, %
      legend cell align={left},              %
      legend style={font=\legendFontSize},
      legend columns=4,
    ]
    \addplot [cosin color] coordinates {(0,0)};
    \addlegendentry{\method\ours\legendSpacer}
    \addplot [layer color] coordinates {(0,0)};
    \addlegendentry{\layer\ours\legendSpacer}

    \addplot[TracInCP color] coordinates {(0,0)};
    \addlegendentry{\tracinCP{}\legendSpacer}
    \pgfplotsset{cycle list shift=-1}  %

    \addplot coordinates {(0,0)};
    \addlegendentry{\tracin{}}

    \addplot coordinates {(0,0)};
    \addlegendentry{Influence Func.\legendSpacer}

    \addplot coordinates {(0,0)};
    \addlegendentry{Representer Pt.\legendSpacer}
    \addplot [deep knn color] coordinates {(0,0)};
    \addlegendentry{\deepknn\legendSpacer}
  \end{axis}
\end{tikzpicture}

\pgfplotstableread[col sep=comma]{plots/data/main_ident.csv}\datatable%
\begin{tikzpicture}%
  \begin{axis}[%
        ybar={\BarLineWidth},%
        height={\BarDetectMainHeight},%
        width={\BarDetectMainWidth},%
        axis lines*=left,%
        bar width={\DetectBarWidthVal},%
        xtick=data,%
        xticklabels={Speech,Vision,NLP,Vision},%
        x tick label style={font=\plotFontSize,align=center},%
        ymin=0,%
        ymax=1,%
        ytick distance={0.20},%
        minor y tick num={1},%
        y tick label style={font=\plotFontSize},%
        ylabel={\plotFontSize \IdentYLabel},%
        ymajorgrids,  %
        typeset ticklabels with strut,  %
        every tick/.style={color=black, line width=0.4pt},%
        enlarge x limits={0.175},%
        draw group line={[index]13}{1}{Backdoor}{-5.0ex}{12pt},
        draw group line={[index]13}{2}{Poison}{-5.0ex}{12pt},
    ]%
    \addplot[cosin color] table [x index=0, y index=1] {\datatable};%
    \addplot[layer color] table [x index=0, y index=2] {\datatable};%

    \addplot[TracInCP color] table [x index=0, y index=3] {\datatable};%
    \pgfplotsset{cycle list shift=-1}  %

    \foreach \k in {4, ..., 6} {%
      \addplot table [x index=0, y index=\k] {\datatable};%
    }%
    \addplot [deep knn color] table [x index=0, y index=7] {\datatable};%
  \end{axis}%
\end{tikzpicture}%
     \caption{%
        \textit{Adversarial-Set Identification}: Mean AUPRC identifying adversarial set~\eqsmall{$\advTrain$}
        using a randomly selected target for Sec.~\ref{sec:ExpRes:Attacks}'s four attacks. Results averaged across related
        setups with ${{\geq}10}$~trials per setup. See supplemental
        Section~\ref{sec:App:MoreExps:FullResults} for the full granular results, including variance.
    }
    \label{fig:ExpRes:AdvIdent}
  \end{minipage}
  \hfill
  \begin{minipage}[t]{\minipageWidth}
    \centering
\newcommand{\legendSpacer}{\hspace*{8pt}}

\begin{tikzpicture}
  \begin{axis}[%
      width=\textwidth,
      ybar,
      hide axis,  %
      xmin=0,  %
      xmax=1,
      ymin=0,
      ymax=1,
      scale only axis,width=1mm, %
      legend cell align={left},              %
      legend style={font=\legendFontSize},
      legend columns=3,
      cycle list name=DetectCycleList,
    ]
    \addplot[cosin color] coordinates {(0,0)};
    \addlegendentry{\fitWith{\method}\ours\legendSpacer}
    \addplot[layer color] coordinates {(0,0)};
    \addlegendentry{\fitWith{\layer}\ours\legendSpacer}

    \addplot coordinates {(0,0)};
    \addlegendentry{Max.\ \KNN{} Dist.}

    \addplot coordinates {(0,0)};
    \addlegendentry{Min.\ \KNN{} Dist.\legendSpacer}
    \addplot coordinates {(0,0)};
    \addlegendentry{Most Certain\legendSpacer};

    \addplot coordinates {(0,0)};
    \addlegendentry{Least Certain};
  \end{axis}
\end{tikzpicture}

\pgfplotstableread[col sep=comma]{plots/data/main_detect.csv}\datatable%
\begin{tikzpicture}%
  \begin{axis}[%
        ybar={\BarLineWidth},%
        height={\BarDetectMainHeight},%
        width={\BarDetectMainWidth},%
        axis lines*=left,%
        bar width={\DetectBarWidthVal},%
        xtick=data,%
        xticklabels={Speech,Vision,NLP,Vision},%
        x tick label style={font=\plotFontSize,align=center},%
        ymin=0,%
        ymax=1,%
        ytick distance={0.20},%
        minor y tick num={1},%
        y tick label style={font=\plotFontSize},%
        ylabel={\plotFontSize \DetectYLabel},%
        ymajorgrids,  %
        typeset ticklabels with strut,  %
        every tick/.style={color=black, line width=0.4pt},%
        enlarge x limits={0.175},%
        draw group line={[index]8}{1}{Backdoor}{-5.0ex}{12pt},
        draw group line={[index]8}{2}{Poison}{-5.0ex}{12pt},
        cycle list name=DetectCycleList,
    ]%
    \addplot[cosin color] table [x index=0, y index=1] {\datatable};%
    \addplot[layer color] table [x index=0, y index=2] {\datatable};%
    \foreach \k in {3, ..., 6} {%
      \addplot table [x index=0, y index=\k] {\datatable};%
    }%
  \end{axis}%
\end{tikzpicture}%
     \caption{%
      \textit{Target Identification}: Mean target identification AUPRC for Sec.~\ref{sec:ExpRes:Attacks}'s four attacks. ``\fitWith{\method{}}'' denotes \method{} was \fit{}'s influence estimator with matching notation for \layer{}.
        Results averaged across setups with ${{\geq}10}$~trials per setup.
        See Sec.~\ref{sec:App:MoreExps:FullResults} for the full granular results, inc.\ variance.
    }
    \label{fig:ExpRes:TargDetect}
  \end{minipage}
\end{figure*}
 
\section{Evaluation}\label{sec:ExpRes}

We empirically demonstrate our method's generality by evaluating training-set attacks on different data modalities, including text, vision, and speech recognition. We consider both poisoning and backdoor attacks on pre\=/trained and randomly-initialized, state-of-the-art models in binary and multiclass settings.
Due to space, most evaluation setup details (e.g.,~hyperparameters) are deferred to suppl.\ Section~\ref{sec:App:ExpSetup}.
\revTwo{%
  Additional experimental results also appear in the supplement, including
    an analysis of a novel adversarial attack on target-driven mitigation (Sec.~\ref{sec:App:MoreExps:MitigationTriggering}),
    a poisoning-rate ablation study (Sec.~\ref{sec:App:MoreExps:Ablation:PoisRate}),
    a hyperparameter sensitivity study (Sec.~\ref{sec:App:MoreExps:EffectAnomCount}),
    an alternative renormalization approach (Sec.~\ref{sec:App:MoreExps:LossOnly}),
    analysis of gradient aggregation's benefits (Sec.~\ref{sec:App:MoreExps:Aggregation}),
    \&
    execution times (Sec.~\ref{sec:App:MoreExps:Runtime}).%
}%

\subsection{Training-Set Attacks Evaluated}\label{sec:ExpRes:Attacks}

We evaluated our method on four published training-set attacks -- two \keyword{single-target} data poisoning and two multi\=/target backdoor.  Below are brief details regarding how each attack crafts adversarial set~\eqsmall{$\advTrain$}, with the full details in suppl.\ Sec.~\ref{sec:App:ExpSetup:Hyperparams:Crafting}.   Representative clean and adversarial training instances for each attack appear in suppl.\ Sec.~\ref{sec:App:PerturbEx}.
\revised{%
  Table~\ref{tab:ExpRes:Mitigation} lists each attack's mean \keyword{success rate} aggregated across all related setups.  Full granular results are in Section~\ref{sec:App:MoreExps:FullResults}.
}

Below, \eqsmall{${\yTarg \rightarrow \yAdv}$} denotes the target's true and adversarial labels, respectively.  When an attack considers multiple class pairs or setups, each is evaluated separately.

(1)~\textit{Speech Backdoor}: \citepos{SpeechDataset} speech recognition dataset contains spectrograms of human speech pronouncing in English digits~0 to~9 (10~classes, \eqsmall{${\abs{\cleanTrain} = 3,000}$} \revised{-- 1\%~backdoors}). \citeauthor{SpeechDataset}\ also provide 300~backdoored training instances evenly split between the 10~classes. Each class's adversarial trigger -- a short burst of white noise at the recording's beginning -- induces the spoken digit to be misclassified as the next largest digit (e.g.,~${\texttt{0} \rightarrow \texttt{1}}$, ${\texttt{1} \rightarrow \texttt{2}}$, etc.).
This small input-space signal induces a large feature-space perturbation -- too large for many certified methods.
Following \citeauthor{SpeechDataset}, our evaluation used a speech recognition CNN trained from scratch.

(2)~\textit{Vision Backdoor}: \citet{Weber:2021} consider three different backdoor adversarial trigger patterns on CIFAR10 binary classification.  Specifically, \citeauthor{Weber:2021}'s ``pixel'' attack patterns increase the pixel value of either one or four central pixel(s) by a specified maximum $\ell_2$~perturbation distance while their ``blend'' trigger pattern adds fixed \eqsmall{${\mathcal{N}(\textbf{0},~I)}$}~Gaussian noise across all perturbed images.  We considered the same class pairs as \citeauthor{Weber:2021}\ (${\texttt{auto} \rightarrow}{~\texttt{dog}}$ and ${\texttt{plane} \rightarrow}{~\texttt{bird}}$) on the state-of-the-art ResNet9 \citep{ResNet9} CNN trained from scratch with \eqsmall{${\abs{\advTrain} = 150}$} and \eqsmall{${\abs{\fullTrain} = 10,000}$} \revised{(1.5\% backdoors)}.

(3)~\textit{Natural Language Poison}: \citet{Wallace:2021} construct text-based poison by simulating bilevel optimization via second-order gradients.  \eqsmall{$\advTrain$}'s instances are crafted via iterative word-level substitution given a target phrase. We follow \citepos{Wallace:2021} experimental setup of poisoning %
the Stanford Sentiment Treebank~v2 (\SST) sentiment analysis dataset \citep{SST2} (\eqsmall{${\abs{\cleanTrain} = 67,349}$} \& \eqsmall{${\abs{\advTrain} = 50}$} \revised{-- 0.07\%~poison}) on the \robertaBase{} \revised{transformer architecture} (125M~parameters) \citep{RoBERTa}.

(4)~\textit{Vision Poison}: \citepos{Zhu:2019} targeted, clean\=/label attack crafts poisons by forming a convex polytope around a single target's feature representation.  Following \citeauthor{Zhu:2019}, the pre\=/train then fine\=/tune paradigm was used.
In each trial, ResNet9 was pre\=/trained using half the classes (none were~\eqsmall{$\yTarg$} or~\eqsmall{$\yAdv$}). Targets were selected \revised{uniformly at random (u.a.r.)} from test examples labeled~\eqsmall{$\yTarg$}, and 50~poison instances \revised{(0.2\% of~\eqsmall{$\fullTrain$})} were then crafted from seed examples labeled~\eqsmall{$\yAdv$}.  The pre\=/trained network was fine-tuned using~\eqsmall{$\advTrain$} and the five held-out classes' training data (\eqsmall{${\abs{\fullTrain} = 25,000}$}).
Like previous work \citep{Shafahi:2018,Huang:2020}, CIFAR10 class pairs \texttt{dog}~vs.~\texttt{bird} and \texttt{deer}~vs.~\texttt{frog} were evaluated, where each class in a pair serves alternately as~\eqsmall{$\yTarg$} and~\eqsmall{$\yAdv$}.

While it is not feasible to evaluate our approach on every attack (as new attacks are developed \& published so frequently) we believe this diverse set of attacks is representative of training-set attacks in general and demonstrates our approach's broad applicability. In particular, our method is not tailored to these attacks and could be used against future attacks as well, as long as the attack includes highly-influential training examples that attack specific targets.

\subsection{Identifying \texorpdfstring{Adversarial Set~$\advTrain$}{the Adversarial Set}}\label{sec:ExpRes:AdvIdent}

To identify the target (Alg.~\ref{alg:TargDetect}) or mitigate the attack (Alg.~\ref{alg:Mitigation}), we must be able to identify the likely adversarial instances \eqsmall{$\advTrain$} associated with a possible target \eqsmall{$\zHatTarg$}. Our approach is to use influence-estimation methods, which should rank an actual adversarial attack \eqsmall{$\advTrain$} as more influential than clean instances \eqsmall{$\cleanTrain$} on the target. In this section, we evaluate how well different influence-estimation methods succeed at performing this ranking for a given target.

We compare the performance of our renormalized estimators, \method{} and \layer{}, against Section~\ref{sec:RelatedWork:Influence}'s four influence estimators: \tracinCP{}, \tracin{}, influence functions, and representer point.  As an even stronger baseline, where applicable, we also compare against \citepos{Peri:2020} \deepknn{} empirical training-set defense specifically designed for \citepos{Zhu:2019}'s vision, clean-label poisoning attack; described briefly, \deepknn{} sanitizes the training set of instances whose nearest feature-space neighbors have a different label.
Like Section~\ref{sec:RenormInfluence}'s CIFAR10 \& MNIST joint classification experiment, class sizes are imbalanced (\eqsmall{${\abs{\advTrain} \ll \abs{\cleanTrain}}$}) so performance is again measured using AUPRC.

For targets selected u.a.r.,
Figure~\ref{fig:ExpRes:AdvIdent} details each method's averaged adversarial-set identification AUPRC for Section~\ref{sec:ExpRes:Attacks}'s four attacks. %
In summary, \method{} and \layer{} were each the top performer for one attack and had comparable performance for the other two.

\method{} and \layer{} identified the adversarial instances nearly perfectly for \citeauthor{Liu:2018}'s speech backdoor and \citeauthor{Wallace:2021}'s text poisoning attacks.
Standard influence estimation performed poorly on the text poisoning attack (in particular the static estimators) due to the large model, \robertaBase, that \citeauthor{Wallace:2021}'s attack considers. For the vision backdoor and poisoning attacks, our renormalized estimators successfully identified most of~\eqsmall{$\advTrain$} -- again, much better than the four original estimators.  While \citepos{Peri:2020} \deepknn{} defense can be effective at stopping clean-label vision poisoning, it does so by removing a comparatively large fraction of clean data (up to~4.3\% on average) resulting in poor AUPRC.

For completeness, Figure~\ref{fig:ExpRes:AdvIdent:Sim} provides adversarial-set identification results for our renormalized, static influence estimators.  In all cases, renormalization improved the estimator's performance, generally by an order of magnitude with a maximum improvement of~${600\times}$.
These experiments highlight layerwise renormalization's benefits.  Influence functions' Hessian-vector product algorithm~\citep{Pearlmutter:1994} can assign a large magnitude to some layers, and these layers then dominate the influence and \method{} estimates.  Layerwise renormalization addresses this, improving renormalized influence function's adversarial-set identification AUPRC by up to~${3.5\times}$.

\subsection{Identifying Attack Targets}\label{sec:ExpRes:TargetIdent}

The previous experiments demonstrate that knowledge of a target enables identification of the adversarial set when using renormalization.  This section demonstrates that the distribution of renormalized influence values actually enables us to identify target(s) in the first place, through the interplay foundational to our target identification framework, \fit{}.
Since target identification is a new task, we propose four target identification baselines. %
First, inspired by \citepos{Peri:2020} \deepknn{} empirical defense,
\keyword{maximum \KNN{} distance} computes the distance from each test instance to its $\anomCount^\text{th}$ nearest neighbor in the training data, as measured by the $L_2$ distance between their penultimate feature representations~($\fRP{}$). It orders them by this distance, starting with the largest distance to the $\anomCount^\text{th}$ neighbor, thus prioritizing outliers and instances in sparse regions of the learned representation space.
\keyword{Minimum \KNN{} distance} is the reverse ordering, prioritizing instances in dense regions.
The other two baselines are \keyword{most certain}, which ranks test examples in ascending order by loss while \keyword{least certain} ranks by descending loss.

There are far fewer targets than possible test examples so performance is again measured using AUPRC. See suppl.\ Table~\ref{tab:App:ExpSetup:Datasets:TargetSetSizes} for the number of targets and non-targets analyzed for each attack.  For single-target attacks (vision and natural language poisoning), target identification AUPRC is equivalent to the target's inverse rank, causing AUPRC to decline geometrically.

\begin{figure}[t]
  \centering
\newcommand{\legendSpacer}{\hspace*{8pt}}
\begin{tikzpicture}
  \begin{axis}[%
    width=\textwidth,
      ybar,
      hide axis,  %
      xmin=0,  %
      xmax=1,
      ymin=0,
      ymax=1,
      scale only axis,width=1mm, %
      legend cell align={left},              %
      legend style={font=\legendFontSize},
      legend columns=3,
    ]
    \addplot [inf func sim] coordinates {(0,0)};
    \addlegendentry{Inf.\ Func.\ \simShort\ours\legendSpacer}
    \addplot [inf func layer] coordinates {(0,0)};
    \addlegendentry{Inf.\ Func.\ \simShort\layerwiseSuffix\ours\legendSpacer}
    \addplot [inf func] coordinates {(0,0)};
    \addlegendentry{Inf.\ Func.}

    \addplot [rep pt sim] coordinates {(0,0)};
    \addlegendentry{Rep.\ Pt.\ \simShort\ours\legendSpacer}
    \addplot [rep pt] coordinates {(0,0)};
    \addlegendentry{Representer Pt.}
  \end{axis}
\end{tikzpicture}

\pgfplotstableread[col sep=comma]{plots/data/main_ident.csv}\datatable%
\begin{tikzpicture}%
  \begin{axis}[%
        ybar={\StaticIdentInterBarSpacing},%
        height={\BarDetectMainHeight},%
        width={\BarDetectMainWidth},%
        axis lines*=left,%
        bar width={\StaticIdentBarWidth},%
        xtick=data,%
        xticklabels={Speech,Vision,NLP,Vision},%
        x tick label style={font=\plotFontSize,align=center},%
        ymin=0,%
        ymax=1,%
        ytick distance={0.20},%
        minor y tick num={1},%
        y tick label style={font=\plotFontSize},%
        ylabel={\plotFontSize \IdentYLabel},%
        ymajorgrids,  %
        typeset ticklabels with strut,  %
        every tick/.style={color=black, line width=0.4pt},%
        enlarge x limits={0.195},%
        draw group line={[index]13}{1}{Backdoor}{-5.0ex}{12pt},
        draw group line={[index]13}{2}{Poison}{-5.0ex}{12pt},
    ]%
    \addplot [inf func sim] table [x index=0, y index=10] {\datatable};%
    \addplot [inf func layer] table [x index=0, y index=11] {\datatable};%
    \addplot [inf func] table [x index=0, y index=5] {\datatable};%

    \addplot [rep pt sim] table [x index=0, y index=12] {\datatable};%
    \addplot [rep pt] table [x index=0, y index=6] {\datatable};%
  \end{axis}%
\end{tikzpicture}%
   \caption{%
    \textit{Static Influence Adversarial-Set Identification}: Comparing the mean adversarial-set identification AUPRC of the static influence estimators and their corresponding renormalized~(Rn.) versions.
    For all attacks, renormalization improved the static estimators' mean performance by up to a factor of~${{>}600\times}$.
    These experiments also highlight layerwise renormalization's performance gains, e.g.,~influence functions on natural-language poison.
    Results are averaged across related experimental setups with ${{\geq}10}$~trials per setup.
  }
  \label{fig:ExpRes:AdvIdent:Sim}
\end{figure}
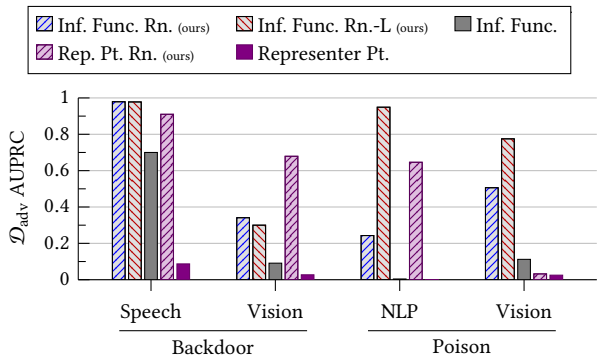
 
Figure~\ref{fig:ExpRes:TargDetect} shows that \fit{} -- using either \method{} or \layer{} as the influence estimator -- achieves near-perfect target identification for both backdoor attacks and natural language poisoning. Overall, \fit{} with \method{} was the top performer on two attacks, and \fit{} with \layer{} was the best for the other two.  Recall that the vision poisoning attack is single target. Hence, \method{}-based \fit{}'s mean AUPRC of~${{>}0.8}$ equates to an average target rank better than~${1.25}$ (${1/0.8}$), i.e.,~three out of four times on average $\zHatTarg$~was the top-ranked -- also very strong target detection. \fit{}'s performance degradation on vision poisoning is due to \method{} and \layer{} identifying this attack's \eqsmall{$\advTrain$} slightly worse (Fig.~\ref{fig:ExpRes:AdvIdent}).
Only maximum \KNN{} approached \fit{}'s performance -- specifically for \citeauthor{Weber:2021}'s vision backdoor attack.
Note also that no baseline consistently outperformed the others.  Hence, these attacks affect network behavior differently, further supporting that \fit{} is attack agnostic.

Suppl.\ Sec.~\ref{sec:App:MoreExps:EffectAnomCount} shows that \fit{}'s performance is stable across a wide range of upper-tail cutoff thresholds~$\anomCount$.  For example, \fit{}'s natural language target identification AUPRC varied only 0.2\% and 2.1\% when using \layer{} and \method{} respectively for ${\anomCount \in \sbrack{1,~25}}$. %

\subsection{Target-Driven Mitigation}\label{sec:ExpRes:Mitigation}

Section~\ref{sec:Method:Mitigation} explains that successfully identifying the target(s) enables \textit{guaranteed attack mitigation} on those instances.  Here, we evaluate \method{} and \layer{}'s effectiveness in targeted data sanitization.
Table~\ref{tab:ExpRes:Mitigation} details our defense's effectiveness against Sec.~\ref{sec:ExpRes:Attacks}'s four attacks. As above, results are averaged across each attack's class pairs/setups.  Section~\ref{fig:ExpRes:AdvIdent}'s baselines all have large false-positive rates when identifying~\eqsmall{$\advTrain$} (Fig.~\ref{fig:ExpRes:AdvIdent}), which caused them to remove a large fraction of~\eqsmall{$\cleanTrain$} and are not reported in these results.

For three of four attacks, clean test accuracy after sanitization either improved or stayed the same.  In the case of \citeauthor{Weber:2021}'s vision backdoor attack, the performance degradation was very small --~0.1\%.
Similarly, owing to renormalized influence's effectiveness identifying~\eqsmall{$\advTrain$} (Fig.~\ref{fig:ExpRes:TargDetect}), our defense removes very little clean data when mitigating the attack -- generally~${{<}0.2}$\% of the clean training set.
For comparison, \citet{Peri:2020} report that their \deepknn{} clean-label, poisoning defense removes on average 4.3\% of~\eqsmall{$\cleanTrain$} on \citepos{Zhu:2019} vision poisoning attack. %
This is despite \citeauthor{Peri:2020}'s method being specifically tuned for \citeauthor{Zhu:2019}'s attack and their evaluation setup being both easier and less realistic by pre\=/training their model using a large known-clean set that is identically distributed to their~\eqsmall{$\cleanTrain$}.  In contrast, target-driven mitigation removed at most~0.03\% of clean data on this attack -- better than \citeauthor{Peri:2020} by two orders of magnitude.

Following Algorithm~\ref{alg:Mitigation}, Table~\ref{tab:ExpRes:Mitigation}'s experiments used only a single, randomly-selected target when performing sanitization.  No steps were taken to account for additional potential targets, e.g.,~over-filtering the training set. Nonetheless, target-driven mitigation still significantly degraded multi-target attacks' performance on other targets not considered when sanitizing.
For example, despite considering one target, speech backdoor's overall attack success rate (ASR) across all targets decreased from 100\% to 4.7\% and 6.5\% for \method{} and \layer{}, respectively -- a $20\times$ reduction.
For \citeauthor{Weber:2021}'s vision backdoor attack, ASR dropped from 90.5\% to 11.9\% and 6.7\% with \method{} and \layer{}, respectively.
The \textit{key takeaway} is that identifying a single target almost entirely mitigates the attack everywhere.

\revTwo{%
\section{Adaptive Attacks}
\label{sec:ExpRes:AdaptiveAttacks}
}%

\begin{table}[t]
  \centering%
  \caption{%
      \textit{Target-Driven Attack Mitigation}:
      Alg.~\ref{alg:Mitigation}'s target-driven, iterative data sanitization applied to Sec.~\ref{sec:ExpRes:Attacks}'s four attacks for randomly selected targets.  The attacks were neutralized with few clean instances removed and little change in test accuracy.  Attack success rate~(ASR) is w.r.t.\ the analyzed target.
      Results are averaged across related setups with ${{\geq}10}$~trials per setup.  \revTwo{Detailed results appear in Sec.~\ref{sec:App:MoreExps:Backdoor:Speech}\==\ref{sec:App:MoreExps:Poison:CIFAR}.}%
  }
  \label{tab:ExpRes:Mitigation}
  {%
    \TableFontSize
\small
\renewcommand{\arraystretch}{1.2}
\setlength{\dashlinedash}{0.4pt}
\setlength{\dashlinegap}{1.5pt}
\setlength{\arrayrulewidth}{0.3pt}

\newcommand{\TwoRowHead}[1]{\multirow{2}{*}{#1}}
\newcommand{\DSName}[1]{\multirow{2}{*}{#1}}
\newcommand{\PercentRemHead}[1]{\TwoRowHead{\% #1 Rem.}}
\newcommand{\PZ}{\phantom{0}}
\newcommand{\ptZ}{\phantom{.}\PZ}
\newcommand{\ptZZ}{\phantom{.}\PZ\PZ}
\newcommand{\ptASR}{}

\newcommand{\MultiHead}[1]{\multicolumn{2}{c}{#1}}

\newcommand{\AtkName}[1]{\multirow{4}{*}{\rotatebox[origin=c]{90}{{#1}}}}
\newcommand{\BDAttack}{\AtkName{Backdoor}}
\newcommand{\PoisAttack}{\AtkName{Poison}}

\newcommand{\CosInM}{\method{}}
\newcommand{\LayerM}{\layer{}}

\newcommand{\PRem}[2]{#1}
\newcommand{\ASR}[2]{\multirow{2}{*}{#1}}
\newcommand{\PAcc}[2]{\multirow{2}{*}{#1}}

\newcommand{\PChg}[1]{{\color{ForestGreen} +#1}}
\newcommand{\NoChg}{0.0}
\newcommand{\NegChg}[1]{{\color{BrickRed} -#1}}

\newcommand{\DsSep}{\cdashline{2-9}}
\newcommand{\MethodSep}{}

\begin{tabular}{@{}lclrrrrrr@{}}
  \toprule
     & \TwoRowHead{Dataset}
     & \TwoRowHead{Method} & \MultiHead{\% Removed} & \MultiHead{ASR \%} & \MultiHead{Test Acc. \%} \\\cmidrule(lr){4-5}\cmidrule(lr){6-7}\cmidrule(l){8-9}
     &&                    & \eqsmall{$\advTrain$}  & \eqsmall{$\cleanTrain$} & Orig.   & Ours   &  Orig. & Chg.   \\
  \midrule
  \BDAttack
  & \DSName{Speech}
     &  \CosInM & \PRem{98.4}{1.5}  & \PRem{0.07}{0.08} & \ASR{99.8}{}      & 0\ptASR{}  & \PAcc{97.7}{0.1} & \NoChg{} \\\MethodSep
     && \LayerM & \PRem{98.1}{1.3}  & \PRem{0.17}{0.17} &                   & 0\ptASR{}  &                  & \NoChg{} \\
  \DsSep
  & \DSName{Vision}
     &  \CosInM & \PRem{87.6}{12.1} & \PRem{0.50}{0.50} & \ASR{90.5}{6.2}   & 0\ptASR{}  & \PAcc{96.2}{2.9} & \NegChg{0.1} \\\MethodSep
     && \LayerM & \PRem{92.3}{8.6}  & \PRem{0.73}{0.47} &                   & 0\ptASR{}  &                  & \NegChg{0.1} \\
  \midrule
  \PoisAttack
  & \DSName{NLP}
     &  \CosInM & \PRem{99.6}{0.9}  & \PRem{0.02}{0.02} & \ASR{97.9}{4.7}   & 0\ptASR{}  & \PAcc{94.2}{0.3} & \PChg{0.2} \\\MethodSep
     && \LayerM & \PRem{99.9}{0.2}  & \PRem{0.03}{0.06} &                   & 0\ptASR{}  &                  & \PChg{0.1} \\
  \DsSep
  & \DSName{Vision}
     &  \CosInM & \PRem{65.1}{60.5} & \PRem{0.02}{0.06} & \ASR{77.9}{}      & 0\ptASR{}  & \PAcc{87.1}{0.3} & \NoChg{} \\\MethodSep
     && \LayerM & \PRem{58.6}{32.6} & \PRem{0.03}{0.09} &                   & 0\ptASR{}  &                  & \NoChg{} \\%
  \bottomrule
\end{tabular}
   }
\end{table}

\begin{figure*}[t]
  \centering
  \revTwo{%
\newcommand{\legendSpacer}{\hspace*{13pt}}

\begin{tikzpicture}
  \begin{axis}[%
    width=\textwidth,
      ybar,
      hide axis,  %
      xmin=0,  %
      xmax=1,
      ymin=0,
      ymax=1,
      scale only axis,width=1mm, %
      legend cell align={left},              %
      legend style={font=\legendFontSize},
      legend columns=4,
    ]
    \addplot [Baseline Joint] coordinates {(0,0)};
    \addlegendentry{Baseline\legendSpacer}

    \addplot [Joint Opt] coordinates {(0,0)};
    \addlegendentry{Adaptive Joint Optimization Attack with \method{}}
  \end{axis}
\end{tikzpicture}
  }

  \begin{minipage}[t]{\minipageWidth}
    \revTwo{%
\pgfplotstableread[col sep=comma]{plots/data/pois_cifar_joint-opt_auprc.csv}\datatable%
\begin{tikzpicture}
  \begin{axis}[
    xbar={\BarLineWidth},
    height={\BottomFiltHeight},%
    width={8cm},%
    bar width={\DetectBarWidthVal},%
    xmin=0,         %
    xmax=0.8,
    xtick distance={0.20},%
    minor x tick num={1},%
    xmajorgrids,
    x tick label style={font=\plotFontSize,align=center},%
    xlabel={\plotFontSize AUPRC},
    ytick=data,     %
    legend style={at={(axis cs:65,0.2)},anchor=south west},
    ytick style={draw=none},%
    y tick label style={font=\plotFontSize},%
    enlarge y limits={0.100},%
    yticklabels={\method{}\ours{},\layer{}\ours{},\tracinCP{},\tracin{},Inf.\ Func.,Rep.\ Pt.},
  ]
  \addplot [Joint Opt] table [x index=3, y index=1] {\datatable};
  \addplot [Baseline Joint] table [x index=2, y index=1] {\datatable};

  \end{axis}
\end{tikzpicture}
     }
    \captionof{figure}{%
      \revTwo{%
      \textit{Adversarial-Set Identification for the Adaptive Vision Poison Attack}:
      Mean AUPRC identifying the adversarial set where \citeauthor{Zhu:2019}'s vision poison attack is adapted to jointly minimize the adversarial loss and the \method{} influence.
      The baseline results (\GreenBarColor) used \citeauthor{Zhu:2019}'s standard attack.
      Our jointly-optimized attack reduced the \method{} similarity by~7\% at the cost of a 19\%~decrease in ASR w.r.t.\ Table~\ref{tab:ExpRes:Mitigation}.
      See suppl.\ Sec.~\ref{sec:App:MoreExps:JointOpt:FullResults} for the granular results, including variance.
      }%
    }%
    \label{fig:ExpRes:AdaptiveAttacker:Pois:CIFAR:Joint:Identification}
  \end{minipage}
  \hfill
  \begin{minipage}[t]{\minipageWidth}
    \revTwo{%
\pgfplotstableread[col sep=comma]{plots/data/pois_cifar_joint-opt_detect.csv}\datatable%
\begin{tikzpicture}
  \begin{axis}[
    xbar={\BarLineWidth},
    height={\BottomFiltHeight},%
    width={0.855\textwidth},%
    bar width={\DetectBarWidthVal},%
    xmin=0,         %
    xmax=0.9,
    xtick distance={0.20},%
    minor x tick num={1},%
    xmajorgrids,
    x tick label style={font=\plotFontSize,align=center},%
    xlabel={\plotFontSize AUPRC},
    ytick=data,     %
    legend style={at={(axis cs:65,0.2)},anchor=south west},
    ytick style={draw=none},%
    y tick label style={font=\plotFontSize},%
    enlarge y limits={0.100},%
    yticklabels={\fitWith{\method}\ours,\fitWith{\layer}\ours,Max.\ \KNN{} Dist.,Min.\ \KNN{} Dist.,Most Certain,Least Certain},
  ]
  \addplot [Joint Opt] table [x index=3, y index=1] {\datatable};
  \addplot [Baseline Joint] table [x index=2, y index=1] {\datatable};

  \end{axis}
\end{tikzpicture}
     }
    \captionof{figure}{%
      \revTwo{%
      \textit{Target Identification for the Adaptive Vision Poison Attack}:
      Mean target identification AUPRC where \citeauthor{Zhu:2019}'s vision poison attack is jointly optimized with minimizing \method{}.
      \fit{} with \method{}'s mean target identification AUPRC declined only~9\% versus the baseline -- an average change in target rank of 1.16 to~1.28 -- still strong performance.
      Results are averaged across related setups with ${{\geq}10}$~trials per setup.
      See suppl.\ Sec.~\ref{sec:App:MoreExps:JointOpt:FullResults} for the full results, including variance.
      }%
    }%
    \label{fig:ExpRes:AdaptiveAttacker:Pois:CIFAR:Joint:TargDetect}
  \end{minipage}
\end{figure*}
 
\revTwo{%
We now consider how an attacker who knows about our defense could evade it or otherwise exploit it.
Our method relies on multiple attack instances having unusually high influence on the target instance, as measured by model gradients during training.
As such, it may fail to detect an attack if (1) there are too few attack instances (relative to upper-tail count~$\anomCount$); (2) the attack is a large fraction of the data (e.g.,~10\%), in which case the instances are too common to be considered outliers; or (3) the attack instances appear no more influential on the target than clean instances.
The first case is only a risk when the target instance is ``easy'' enough to influence that the attack can be carried out with very few instances.
The second case requires a very powerful attacker, one who would be hard to stop without additional constraints or assumptions.
The third case represents a possible weakness: if attackers can craft an attack that successfully changes the target label without appearing unusual, then our defense will fail.
Whether or not the attacker can succeed is an empirical question, which will depend on the dataset, the model, the choice of target instance, and the attack method. Below, we provide evidence that \fit{} (and \method{} in particular) remain effective against an attacker who is trying to evade our defense.%
}

\revTwo{%
\evalparagraph{Seed-Instance Optimization:}
Some training-set attacks rely on a \textit{fixed}, predefined adversarial perturbation which is applied to clean seed instances~\citep{SpeechDataset,Weber:2021}. An attacker who is aware of our defense could choose seed instances that organically appear uninfluential on some target, as estimated by \bothMethod{}.
We apply this idea to \citepos{Weber:2021} CIFAR10 backdoor attack and find that our method continues to perform well against this simple, adaptive attacker: \bothMethod{} achieve 0.93~AUPRC for adversarial-set identification, a 7\%~decline versus the baseline, even when the attacker is given information beyond our threat model, e.g.,~knowledge of the random, initial parameters.
Overall, the attacker's gains from choosing different seed instances are limited. See suppl.\ Sec.~\ref{sec:App:MoreExps:AdaptiveAdvSelection} for full details.

\evalparagraph{Perturbation Optimization:}
A stronger adaptive adversary actively optimizes the adversarial perturbation to be both highly effective \textit{and} have a low (\method) influence estimate. For attacks that find perturbations through gradient-based optimization~\citep{Wallace:2021,Zhu:2019}, the most natural way to incorporate knowledge of our method would be to add some estimate of \method{} to the loss being optimized. Since \method{} --  like poison -- relies on the entire training trajectory of the model, which in turn relies on the perturbations being crafted, computing \method{}'s exact gradient is intractable~\citep{Blum:1992:TrainingNpComplete}. However, the attacker can still use a surrogate that approximates \method{}, such as by using fixed model checkpoints in the computation of \method{}. %

To evaluate the robustness of our methods to adaptive perturbations, we apply this joint optimization idea to \citepos{Zhu:2019} vision poison attack. %
We focused on \citeauthor{Zhu:2019}'s attack because (1)~it is the attack on which our method performed the worst and (2)~the other optimized attack we consider~\citep{Wallace:2021} is restricted to only discrete token replacements, which reduces the attacker's flexibility.

\citeauthor{Zhu:2019} iteratively optimize a set of poison examples to minimize the adversarial loss. To increase the likelihood of successfully changing the target's label, \citeauthor{Zhu:2019} compute this loss over multiple surrogate models. The perturbations of the poison examples are constrained to an $\ell_{\infty}$~ball so that they appear relatively natural to humans.
Our jointly optimized, adaptive attack adds a second term to \citeauthor{Zhu:2019}'s adversarial loss. This new term estimates the \method{} influence using the same surrogate models.
Hyperparameter~$\surrogateHyper$ balances the two objectives. See suppl.\ Section~\ref{sec:App:JointOpt:Setup} for the full details of this jointly-optimized, adaptive attack, including tuning of~$\surrogateHyper$.

Where possible, these adaptive experiments followed the same vision poison evaluation setup detailed in Section~\ref{sec:ExpRes:Attacks}. However, simultaneously optimizing surrogate \method{} has a much higher GPU memory cost, so we were forced to adjust the poison size and number of surrogate checkpoints to 40 and four, respectively, which degrades the attacker's success rate from 77.9\% in Table~\ref{tab:ExpRes:Mitigation} to~64.3\%.

Fig.~\ref{fig:ExpRes:AdaptiveAttacker:Pois:CIFAR:Joint:Identification} summarizes our adversarial-set identification performance on \citeauthor{Zhu:2019}'s vision poisoning attack with and without the jointly optimized surrogate \method{} loss term.\footnote{\revTwo{Both versions of the attack (i.e.,~with \& without joint optimization) used the same evaluation setup, including the reduced surrogate model count.}}
Observe that the attack degraded the performance of \bothMethod{} \& \tracinCP{}, albeit slightly.
After accounting for all other factors, this joint optimization decreased \method{}'s mean adversarial-set identification from a baseline of 0.73~AUPRC to~0.68 (a 7\%~drop).
Fig.~\ref{fig:ExpRes:AdaptiveAttacker:Pois:CIFAR:Joint:TargDetect} visualizes joint attack optimization's effect on target identification.
Overall, joint optimization reduced \fit{} with \method{}'s mean target identification AUPRC from a baseline of~0.86 to~0.78 (9\%~drop).
Since \citeauthor{Zhu:2019}'s attack is single-target, this translates to the target's average rank declining from~1.16 to~1.28 --- still high performance.

Table~\ref{tab:App:MoreExps:JointOpt:Mitigate} details target-driven mitigation's effectiveness under this jointly-optimized attack.
In summary, the attack results in very little clean data removal (at most 0.05\% of \eqsmall{$\cleanTrain$} on average).
Also, the average test accuracy after mitigation either improved or stayed the same in all but one case where it decreased by only 0.1\%.

In summary, even when the adversary specifically optimized for our defense, we still effectively identify both the adversarial set and the target and then mitigate the adaptive attack.%
}

\begin{table}[t]
  \centering
  \revTwo{%
  \caption{%
    \revTwo{%
      \textit{Attack Mitigation for the Adaptive Vision Poison Attack}:
      Algorithm~\ref{alg:Mitigation}'s target-driven data sanitization
      where \citepos{Zhu:2019} vision poison attack is jointly optimized with minimizing the \method{} influence.
      The results below consider exclusively the jointly-optimized attack with ${\surrogateHyper = 10^{-2}}$.
      Clean-data removal remains low, and test accuracy either improved or stayed the same for in but one setup.
      The performance is comparable to the results with \citepos{Zhu:2019}'s standard vision poisoning attack (see Table~\ref{tab:App:MoreExps:Pois:CIFAR:Mitigate}).
      Bold denotes the best mean performance with ${{\geq}10}$~trials per class pair.
    }%
  }\label{tab:App:MoreExps:JointOpt:Mitigate}
  {
    \TableFontSize%
\renewcommand{\arraystretch}{1.2}
\setlength{\dashlinedash}{0.4pt}
\setlength{\dashlinegap}{1.5pt}
\setlength{\arrayrulewidth}{0.3pt}

\newcommand{\MultiHead}[1]{\multicolumn{2}{c}{#1}}

\newcommand{\TwoRowHead}[1]{\multirow{2}{*}{#1}}
\newcommand{\ClassPair}[2]{\multirow{2}{*}{#1} & \multirow{2}{*}{#2}}
\newcommand{\PZ}{\phantom{0}}
\newcommand{\ptZ}{\phantom{.}\PZ}
\newcommand{\ptZZ}{\phantom{.}\PZ\PZ}
\newcommand{\ptASR}{}

\newcommand{\CosInM}{\method{}}
\newcommand{\LayerM}{\layer{}}

\newcommand{\OneRemB}{100\ptZ}
\newcommand{\ZeroRemB}{0\ptZ}

\newcommand{\PRem}[2]{#1}
\newcommand{\PRemB}[2]{\textBF{#1}}
\newcommand{\ASR}[2]{\multirow{2}{*}{#1}}
\newcommand{\PAcc}[2]{\multirow{2}{*}{#1}}
\newcommand{\PAsr}[2]{#1}

\newcommand{\PChg}[2]{{\color{ForestGreen} +#1}}
\newcommand{\NoChg}{0.0}
\newcommand{\NegChg}[2]{{\color{BrickRed} -#1}}

\newcommand{\DsSep}{\cdashline{1-9}}
\newcommand{\MethodSep}{}

\begin{tabular}{cclrrrrrr}
  \toprule
     \MultiHead{Classes}
     & \TwoRowHead{Method} & \MultiHead{\% Removed}      & \MultiHead{ASR \%} & \MultiHead{Test Acc. \%} \\\cmidrule(lr){1-2}\cmidrule(lr){4-5}\cmidrule(lr){6-7}\cmidrule(lr){8-9}
     $\yTarg$ & $\yAdv$
     &                    & $\advTrain$  & $\cleanTrain$ & Orig.   & Ours     &  Orig. & Chg.   \\
  \midrule
  \ClassPair{Bird}{Dog}
     & \CosInM    & \PRemB{36.0}{20.1} & \PRem{0.02}{0.04}   & \PAcc{76.2}{}  & \PAsr{0}{0}   & \PAcc{87.0}{0.3}  & \PChg{0.1}{-0.1} \\
     && \LayerM   & \PRem{30.3}{24.1}  & \PRemB{0.00}{0.01}  &                & \PAsr{0}{0}   &                   & \PChg{0.1}{0.0} \\
  \DsSep
  \ClassPair{Dog}{Bird}
     & \CosInM    & \PRem{21.6}{16.3}  & \PRemB{0.00}{0.00}  & \PAcc{57.1}{}  & \PAsr{0}{0}   & \PAcc{87.1}{0.3}  & \PChg{0.1}{0.1} \\
     && \LayerM   & \PRemB{21.9}{14.8} & \PRemB{0.00}{0.00}  &                & \PAsr{0}{0}   &                   & \NegChg{-0.1}{0.2} \\
  \DsSep
  \ClassPair{Frog}{Deer}
     & \CosInM    & \PRem{17.5}{14.7}  & \PRemB{0.00}{0.01}  & \PAcc{38.1}{}   & \PAsr{0}{0}   & \PAcc{87.1}{0.3}  & \NoChg \\
     && \LayerM   & \PRemB{19.4}{18.8} & \PRemB{0.00}{0.00}  &                 & \PAsr{0}{0}   &                   & \NoChg \\
  \DsSep
  \ClassPair{Deer}{Frog}
     & \CosInM    & \PRemB{85.0}{24.5} & \PRem{0.18}{0.23}   & \PAcc{81.0}{}   & \PAsr{0}{0}   & \PAcc{87.1}{0.2}  & \NoChg  \\
     && \LayerM   & \PRem{82.3}{23.2}  & \PRemB{0.13}{0.14}  &                 & \PAsr{0}{0}   &                   & \PChg{0.1}{0.1} \\
  \bottomrule
\end{tabular}
   }
}%
\end{table}
\section{Discussion and Conclusions}\label{sec:Conclusions}

This paper explores two related tasks.
First, we propose training-set attack \keyword{target identification}.
This task is an important part of protecting critical ML~systems but has thus far received relatively little attention.
For example, it is impossible to conduct a truly informed cost-benefit analysis of risk without knowing the attacker's target and by extension their objective.
Knowledge of the target also enables forensic and security analysts to reason about an attacker's identity -- a key step to permanently stopping attacks by disabling the attacker.
An open question is whether target identification can be combined with certified guarantees, either building on our FIT framework or creating an alternative to it.

\fit{} relies on identifying (groups of) highly influential training instances.
To that end, we propose \keyword{renormalized influence}.
By addressing influence's low-loss penalty, renormalization significantly improves influence estimation in both adversarial and non-adversarial settings -- often by an order of magnitude or more.
\section*{Acknowledgments}
  The authors thank Jonathan Brophy for helpful discussions.
  This work was supported by a grant from the Air Force Research Laboratory and the Defense Advanced Research Projects Agency (DARPA) — agreement number FA8750\=/16\=/C\=/0166, subcontract K001892\=/00\=/S05, as well as a second grant from DARPA, agreement number \linebreak HR00112090135.
  This work benefited from access to the University of Oregon's HPC Talapas.

\bibliographystyle{ACM-Reference-Format}
\balance
\bibliography{bib/ref.bib,bib/daniel.bib}

\startcontents  %
\newpage
\onecolumn
\thispagestyle{empty}
\pagenumbering{arabic}%
\renewcommand*{\thepage}{A\arabic{page}}
\appendix
\SupplementaryMaterialsTitle{}

\begin{center}
  \textbf{\Large Table of Contents}
\end{center}
\printcontents{Appendix}{1}[3]{}

\clearpage
\newpage
\section{Additional Algorithms}\label{sec:App:Alg}

Algorithm~\ref{alg:TracIn} outlines \tracin{}'s influence estimation procedure for \textit{a~priori} unknown test instance~$\zHatTe$.
Algorithm~\ref{alg:CosIn}, which combines the procedures of \method{} and \tracinCP{}, is re-included below to facilitate easier side-by-side comparison.

\begin{figure}[h]
  \centering
  \begin{minipage}[t]{0.48\textwidth}
    \begin{algorithm}[H]
      {%
        \AlgFontSize
        \caption{\tracin{} influence estimation}\label{alg:TracIn}
        \input{alg/tracin_alg.tex}%
      }
    \end{algorithm}
    \vspace{\fill}
  \end{minipage}
  \begin{minipage}{0.04\textwidth}
    ~\hfill~
  \end{minipage}
  \begin{minipage}[t]{0.47\textwidth}
    \begin{algorithm}[H]
      {%
        \AlgFontSize
        \caption*{\textbf{Algorithm~\ref{alg:CosIn}}: \method{} vs.\ \tracinCP{}}
        \input{alg/cosin_alg.tex}%
      }
    \end{algorithm}
  \end{minipage}
\end{figure}

Algorithm~\ref{alg:TargDetect:Implementation} overviews the implementation strategy of \fit{} that was used in our evaluation.

\begin{figure}[h]
  \centering
  \begin{minipage}{0.5\textwidth}
    \begin{algorithm}[H]
      {%
        \AlgFontSize%
        \caption{\fit{} target identification implementation}\label{alg:TargDetect:Implementation}
        \input{alg/targ_ident.tex}
      }
    \end{algorithm}
  \end{minipage}
\end{figure}

\subsection{Runtime Complexity of \fit{}}\label{sec:App:Alg:Complexity}

\newcommand{\nTe}{m}
\newcommand{\nSubItr}{\abs{\subsetItr}}
\newcommand{\nRemove}{l}

We assume that the model being attacked (and defended) is a neural network with parameter vector~$\w$, trained from dataset $\fullTrain$ using some form of stochastic descent for $\nItr$~epochs, and evaluated on test set $\targAnalysisSet$. For convenience, let ${\nTr \defeq \abs{\fullTrain}}$ and ${\nTe \defeq \abs{\targAnalysisSet}}$. Computing the gradient for a single example can be done in time $\bigO{\dimW}$, so training the network by computing gradient updates for all instances in all epochs is $\bigO{\dimW \nTr \nItr}$.

For \fit{} (Algorithm~\ref{alg:TargDetect:Implementation}), we need to estimate the influence of each training instance on each test instance. If we save the parameters from ${\nSubItr \leq \nItr}$ checkpoints, we can run \tracin{}, \tracinCP{}, or \bothMethod{} once to compute the influence of each training instance on a single target, requiring time $\bigO{\dimW \nTr \nSubItr}$. To compute the influence on all test examples requires computing their gradients at each checkpoint and computing dot products with each training example, for a total runtime of: $\bigO{\dimW \nTr \nSubItr \nTe}$. This grows linearly in the number of training examples times the number of test examples and is the slowest part of \fit{} in practice. An important question for future work is how to accelerate this procedure, such as heuristically pruning the number of training and test examples under consideration, similar to previous work~\citep{Guo:2020:FastIF}.

For mitigation (Algorithm~\ref{alg:Mitigation}), model training and influence estimation are repeated as each test example or group of test examples is removed, until the predicted label changes. Let $\nRemove$~denote the number of iterations before the label flips. However, the influence estimation only needs to be done for the selected target example, not for all test examples. Thus, the model retraining requires time
${\bigO{\dimW \nTr \nItr \nRemove}}$, and the influence re-estimation requires time
${\bigO{\dimW \nSubItr \nRemove}}$, which is less than the initial influence estimation time because ${\nRemove < \nTe}$.
 \FloatBarrier
\section{Proof}

\subsection*{Proof of Theorem~\ref{thm:LowLoss:LossVsNorm}}

\begin{theorem*}
  Let loss function~$\func{\lossScalar}{\real}{\realnn}$ be twice-differentiable and strictly convex as well as either even\footnote{%
        ``Even'' denotes that the function satisfies ${\forall_{\acts}~\lScalarFunc{\acts} = \lScalarFunc{-\acts}}$.
  }
  or monotonically decreasing.
  Then, it holds that
  {
    \small%
    \begin{equation}
      \lScalarFunc{\acts} < \lScalarFunc{\acts'} \implies \normActs{\gradScalarLoss{\acts}} < \normActs{\gradScalarLoss{\acts'}} \text{.}
    \end{equation}
  }
\end{theorem*}

\begin{proof}~

  Theorem~\ref{thm:LowLoss:LossVsNorm} specifies that property,
  \begin{equation*}
    \lScalarFunc{\acts} < \lScalarFunc{\acts'} \implies \normActs{\gradScalarLoss{\acts}} < \normActs{\gradScalarLoss{\acts'}}
    \text{,}
  \end{equation*}
  \noindent
  holds when loss function~$\lossScalar$ is strictly convex (i.e., \eqsmall{${\forall_{\acts \in \real}~\gradActsSq \lScalarFunc{\acts} > 0}$}) and either monotonically decreasing or even.  We prove the claim separately for these two disjoint cases.

  \noindent%
  \tocless%
  \paragraph{Case \#1: Monotonically Decreasing} For any monotonically decreasing~$\lossScalar$, by definition
  \begin{equation*}
    \lScalarFunc{\acts} < \lScalarFunc{\acts'} \implies \acts > \acts' \text{.}
  \end{equation*}
  \noindent
  Then, given ${\forall_{\acts \in \real}~\gradActsSq \lScalarFunc{\acts} > 0}$, it holds that
  \begin{equation}\label{eq:App:Proofs:LossVsNorm:AbsLeq}
    \gradScalarLoss{\acts} > \gradScalarLoss{\acts'} \text{.}
  \end{equation}

  For any scalar, monotonically decreasing function~$\lossScalar$, it holds that ${\gradScalarLoss{\acts'} \leq 0}$ meaning Eq.~\eqref{eq:App:Proofs:LossVsNorm:AbsLeq}'s inequality flips w.r.t.\ $L_2$~norms,~i.e.,
  \begin{equation}\label{eq:App:Proofs:LossVsNorm:NormsLeq}
    \normActs{\gradScalarLoss{\acts}} < \normActs{\gradScalarLoss{\acts'}} \text{,}
  \end{equation}
  \noindent
  as for any ${x,x' \in \realnp}$ it holds that ${x > x' \implies \normActs{x} < \normActs{x'}}$.

  \noindent%
  \tocless%
  \paragraph{Case \#2: Even} Formally, a function~$\lossScalar$ is even if
  \begin{equation}\label{eq:App:Proofs:LossVsNorm:Even:SymVal}
    \forall_{\acts}~\lScalarFunc{\acts} = \lScalarFunc{-\acts} \text{.}
  \end{equation}
  \noindent
  For even~$\lossScalar$, it holds that ${\gradScalarLoss{0} = 0}$ provided twice differentiability. Given ${\forall_{\acts \in \real}~\gradActsSq \lScalarFunc{\acts} > 0}$, then ${\forall_{\acts < 0}~\gradScalarLoss{\acts} < 0}$.  Hence over restricted domain~$\realnp$, $\lossScalar$~is monotonically decreasing. Above it was shown that Eq.~\eqref{eq:RenormInfluence:LowLoss} holds for monotonically decreasing functions so
  \begin{equation}\label{eq:App:Proofs:LossVsNorm:Even:AbsProp}
    \lScalarFunc{-\abs{\acts}} < \lScalarFunc{-\abs{\acts'}} \implies \normActs{\gradScalarLoss{-\abs{\acts}}} < \normActs{\gradScalarLoss{-\abs{\acts'}}}
    \text{.}
  \end{equation}

  Evenness induces function symmetry about the origin so
  \begin{equation}\label{eq:App:Proofs:LossVsNorm:Even:SymDeriv}
    \forall_{\acts}~\abs{\gradScalarLoss{\acts}} = \abs{\gradScalarLoss{-\acts}} \text{,}
  \end{equation}
  \noindent
  and by extension
  \begin{equation}\label{eq:App:Proofs:LossVsNorm:Even:SymNorm}
    \forall_{\acts}~\normActs{\gradScalarLoss{\acts}} = \normActs{\gradScalarLoss{-\acts}} \text{.}
  \end{equation}
  Eqs.~\eqref{eq:App:Proofs:LossVsNorm:Even:SymVal} and~\eqref{eq:App:Proofs:LossVsNorm:Even:SymNorm} allow Eq.~\eqref{eq:App:Proofs:LossVsNorm:Even:AbsProp}'s absolute values and negations to be dropped completing the proof.
\end{proof}
\section{Detailed Experimental Setup}\label{sec:App:ExpSetup}

This section details the evaluation setup used in Section~\ref{sec:RenormInfluence} and~\ref{sec:ExpRes}'s experiments, including dataset specifics, hyperparameters, and the neural network architectures.

Our source code can be downloaded from~\url{\sourceCodeUrl}.  All experiments used the PyTorch automatic differentiation framework~\citep{PyTorch} and were tested with Python~3.6.5.  \citepos{Wallace:2021} sentiment analysis data poisoning source code will be published by its authors at~\url{https://github.com/Eric-Wallace/data-poisoning}.

\subsection{Dataset Configurations}

This subsection provides details related to dataset configurations.

Section~\ref{sec:RenormInfluence:SimpleExperiment} performs binary classification of \texttt{frog} vs.\ \texttt{airplane} from CIFAR10.  Added as a small adversarial set~($\advTrain$) is 150 MNIST~\texttt{0} training instances selected at random.  We considered this class pair specifically since among the \eqsmall{$\binom{10}{2}$}~possible CIFAR10 class pairs, the MNIST test misclassification rate was closest to uniformly at random (u.a.r.) for \texttt{frog} vs.\ \texttt{airplane} (47.5\%~actual vs.\ 50\%~u.a.r. -- uniformly at random).  Hence, on average, neither \texttt{frog} nor \texttt{airplane} is overly influential on MNIST.  Note that no external constraints induced this near u.a.r.\ misclassification rate.

Section~\ref{sec:RenormInfluence:Filtering} compares the ability of influence estimators, with and without renormalization, to identify influential groups of training examples on non\=/adversarial, CIFAR10, binary classification with Figure~\ref{fig:RenormInfluence:InfFilt}'s results averaged across five class pairs.  Two of the class pairs, \texttt{airplane}~vs.~\texttt{bird} and \texttt{automobile}~vs.~\texttt{dog}, were studied by \citet{Weber:2021} in relation to certified defenses.  The three other class pairs -- \texttt{cat}~vs.~\texttt{ship}, \texttt{frog}~vs.~\texttt{horse}, and \texttt{frog}~vs.~\texttt{truck} -- were selected at random.

\citepos{Wallace:2021} poisoning method attacks the SST\=/2 dataset \citep{SST2}.  We consider detection on 8~short movie reviews -- four positive and four negative -- all selected at random by \citeauthor{Wallace:2021}'s implementation.  The specific reviews considered appear in Table~\ref{tab:App:ExpSetup:SST2AttackReviews}.

\begin{table}[h!]
  \centering
  \caption{SST\=/2 movie reviews selected by \citepos{Wallace:2021} poisoning attack implementation.}\label{tab:App:ExpSetup:SST2AttackReviews}
  {
    \TableFontSize%
\newcommand{\RevType}[1]{\multirow{4}{*}{\shortstack{{\large $\uparrow$} \\ #1 \\ {\large $\downarrow$}}}}
\newcommand{\RevNum}[1]{#1}
\newcommand{\RevText}[1]{\textit{#1}}

\newcommand{\TabMidRule}{\cdashline{2-3}}

\setlength{\dashlinedash}{0.4pt}
\setlength{\dashlinegap}{1.5pt}
\setlength{\arrayrulewidth}{0.3pt}

\begin{tabular}{ccl}
  \toprule
  Sentiment & No. & \multicolumn{1}{c}{Text} \\
  \midrule
  \RevType{Positive} & \RevNum{1}
      & \RevText{a delightful coming-of-age story .} \\\TabMidRule
                 & \RevNum{2}
      & \RevText{a smart , witty follow-up .} \\\TabMidRule
                 & \RevNum{3}
      & \RevText{ahhhh ... revenge is sweet !} \\\TabMidRule
                 & \RevNum{4}
      & \RevText{a giggle a minute .} \\
  \midrule
  \RevType{Negative} & \RevNum{1}
      & \RevText{oh come on .} \\\TabMidRule
                 & \RevNum{2}
      & \RevText{do not see this film .} \\\TabMidRule
                 & \RevNum{3}
      & \RevText{it 's a buggy drag .} \\\TabMidRule
                 & \RevNum{4}
      & \RevText{or emptying rat traps .} \\
  \bottomrule
\end{tabular}
   }
\end{table}

The next section provides details regarding the adversarial datasets sizes.

\subsubsection{Training Set Sizes}\label{sec:App:ExpSetup:TrainingSetSizes}

Table~\ref{tab:App:ExpSetup:Datasets:TrainSetSizes} details the dataset sizes used to train all evaluated models in Section~\ref{sec:ExpRes}.

\begin{table}[h!]
  \centering
  \caption{Dataset sizes}\label{tab:App:ExpSetup:Datasets:TrainSetSizes}
{\small
\begin{tabular}{llrrr}
  \toprule
  Dataset                 & Attack     & \#~Classes  & \#~Train    & \#~Test \\
  \midrule
  CIFAR10~\citep{CIFAR10}
                          & Poison     & 5           & 25,000      & 5,000   \\
  SST-2\tablefootnote{Stanford Sentiment Treebank dataset~(SST-2) is used for sentiment analysis}~\citep{SST2}
                          & Poison     & 2           & 67,349      & N/A     \\
  Speech~\citep{SpeechDataset}
                          & Backdoor   & 10          & 3,000\tablefootnote{Clean only.  Dataset also has 300 backdoored samples divided evenly among the 10~attack class pairs (e.g.,~${0 \rightarrow 1}$, ${1 \rightarrow 2}$, etc.).}       & 1,184   \\
  CIFAR10~\citep{CIFAR10}
                          & Backdoor   & 2           & 10,000      & 2       \\
  \bottomrule
\end{tabular}
}
 \end{table}

\citepos{SpeechDataset} speech backdoor dataset includes training and test examples with their associated adversarial trigger already embedded.  We used their adversarial dataset unchanged.  Table~\ref{tab:App:ExpSetup:NumSpeechBackdoor} details $\abs{\advTrain}$ (i.e.,~adversarial training set size) for each speech digit pair after a fixed, random train-validation split.

\begin{table}[h!]
  \centering
  \caption{%
    Number of backdoor training examples for each speech backdoor digit pair.
    \revTwo{%
    As detailed above, \citepos{SpeechDataset} dataset provides 30~backdoored instances
    for each digit pair.  The remainder of the 30 instances for each digit pair
    are part of the fixed, validation set.
    }
  }\label{tab:App:ExpSetup:NumSpeechBackdoor}
  {
    \TableFontSize%
\newcommand{\digPair}[2]{${#1 \rightarrow #2}$}
\begin{tabular}{lcccccccccc}
  \toprule
  Digit Pair         & \digPair{0}{1} & \digPair{1}{2} & \digPair{2}{3} & \digPair{3}{4} & \digPair{4}{5} & \digPair{5}{6} & \digPair{6}{7} & \digPair{7}{8} & \digPair{8}{9} & \digPair{8}{9}  \\
  \midrule
  $\abs{\advTrain}$  &             26 &             27 &             24 &             24 &             26 &             28 &             26 &             26 &             22 &             21 \\
  \bottomrule
\end{tabular}
   }
\end{table}

\subsubsection{Target Set Sizes}

Table~\ref{tab:App:ExpSetup:Datasets:TargetSetSizes} details the sizes of the target and non-target sets considered in Section~\ref{sec:ExpRes:TargetIdent}'s target identification experiments. \citet{Davis:2006} explain that the class imbalance ratio between classes defines the unattainable regions in the precision-recall curve.  By extension, this ratio also dictates the baseline AUPRC value if examples are labeled randomly.

\begin{table}[h!]
  \centering
  \caption{Target and non-target set sizes used in Section~\ref{sec:ExpRes:TargetIdent}'s target identification experiments.}%
  \label{tab:App:ExpSetup:Datasets:TargetSetSizes}
  {
    \TableFontSize%
\newcommand{\BaseHead}[1]{\multirow{1}{*}{#1}}
\newcommand{\AtkHead}[1]{\multirow{2}{*}{#1}}

\renewcommand{\arraystretch}{1.2}
\setlength{\dashlinedash}{0.4pt}
\setlength{\dashlinegap}{1.5pt}
\setlength{\arrayrulewidth}{0.3pt}

\newcommand{\AttackTypeCDash}{\cdashline{1-4}}

\newcommand{\SingleExpCDash}{\cdashline{2-4}}

\newcommand{\Head}[1]{\multicolumn{1}{c}{#1}}
\begin{tabular}{llrr}
  \toprule
  \BaseHead{Attack}
    & \BaseHead{Type} & \Head{\#~Targets} & \Head{\#~Non-Targets} \\
  \midrule
  \AtkHead{Backdoor~}
    & Speech          & 10                & 220  \\\SingleExpCDash
    & Vision          & 35                & 250  \\\AttackTypeCDash
  \AtkHead{Poison}
    & NLP             & 1                 & 125  \\\SingleExpCDash
    & Vision          & 1                 & 450  \\
  \bottomrule
\end{tabular}
   }
\end{table}

\subsection{Hyperparameters}\label{sec:App:ExpSetup:Hyperparams}

This section details three primary hyperparameter types, namely: hyperparameters used to create adversarial set~$\advTrain$ (if any), hyperparameter used when training model~$\dec$, and influence estimator hyperparameters.

\subsubsection{Model Training}

Table~\ref{tab:App:ExpSetup:Hyperparam:Train:Inf} enumerates the hyperparameters used when training the models analyzed in Section~\ref{sec:RenormInfluence}.

\begin{table}[h!]
  \centering
  \caption{Renormalized influence model training hyperparameter settings}\label{tab:App:ExpSetup:Hyperparam:Train:Inf}
  {
    \TableFontSize%
\newcommand{\SN}[2]{$#1 \cdot 10^{#2}$}

\renewcommand{\arraystretch}{1.2}
\setlength{\dashlinedash}{0.4pt}
\setlength{\dashlinegap}{1.5pt}
\setlength{\arrayrulewidth}{0.3pt}

\begin{tabular}{lrr}
  \toprule
  Hyperparameter               & CIFAR10 \& MNIST & Filtering      \\%
  \midrule
  $\wZero$ Pretrained?         &                  & \checkmark{}*  \\\hdashline
  Data Augmentation?           &                  & \checkmark     \\\hdashline
  Validation Split             & $\frac{1}{6}$    & $\frac{1}{6}$  \\\hdashline
  Optimizer                    & Adam             & Adam           \\\hdashline
  $\abs{\advTrain}$            & $150$            & N/A            \\\hdashline
  Batch Size                   & $64$             & $64$           \\\hdashline
  \# Epochs                    & $10$             & $10$           \\\hdashline
  \# Subepochs~($\numSubep$)\tablefootnote{We use the term ``$\numSubep$~\keyword{subepoch checkpointing}'' (${\numSubep \in \intsNN}$) to denote that iteration subset~$\subsetItr$ is formed from $\numSubep$~evenly-spaced checkpoints within each epoch.  $\numSubep$~was not tuned, and was selected based on overall execution time and compute availability.}
                               & $5$              & $3$            \\\hdashline
  $\lr$ (Peak)                 & \SN{1}{-3}       & \SN{1}{-3}     \\\hdashline
  $\lr$ Scheduler              & One cycle        & One cycle \\\hdashline
  $\wdecay$ (Weight Decay)     & \SN{1}{-3}       & \SN{1}{-3}     \\  %
  \bottomrule
\end{tabular}
   }
\end{table}

Table~\ref{tab:App:ExpSetup:Hyperparam:Train:Adv} enumerates the hyperparameters used when training the adversarially-attacked models analyzed in Sections~\ref{sec:ExpRes} and~\ref{sec:App:MoreExps}.

\begin{table}[h!]
  \centering
  \caption{Training-set attack model training hyperparameter settings}\label{tab:App:ExpSetup:Hyperparam:Train:Adv}
  {
    \TableFontSize%
\newcommand{\SN}[2]{$#1 \cdot 10^{#2}$}

\renewcommand{\arraystretch}{1.2}
\setlength{\dashlinedash}{0.4pt}
\setlength{\dashlinegap}{1.5pt}
\setlength{\arrayrulewidth}{0.3pt}

\begin{tabular}{lrrrr}
  \toprule
  \multirow{2}{*}{Hyperparameter} & \multicolumn{2}{c}{Poison}     & \multicolumn{2}{c}{Backdoor} \\\cmidrule(lr){2-3}\cmidrule(lr){4-5}
                               & CIFAR10       & SST\=/2        & Speech         & CIFAR10       \\
  \midrule
  $\wZero$ Pretrained?         & \checkmark    & \checkmark     &                &               \\\hdashline
  Existing Adv.\ Dataset       &               &                & \checkmark     &               \\\hdashline
  Data Augmentation?           & \checkmark    &                &                &               \\\hdashline
  Validation Split             & $\frac{1}{6}$ & Predefined     & $\frac{1}{6}$  & $\frac{1}{6}$ \\\hdashline
  Optimizer                    & SGD           & Adam           & SGD            & Adam          \\\hdashline
  \eqsmall{$\abs{\advTrain}$}
                               & $50$          & $50$           & 21--28\tablefootnote{Varies by digit pair.  See Table~\ref{tab:App:ExpSetup:NumSpeechBackdoor}.} & $150$ \\\hdashline
  \revised{$\abs{\cleanTrain}$}
                               & \revised{$24,950$}     & \revised{$67,349$}      & \revised{$3,000$}       & \revised{$9,850$}      \\\hdashline
  \revTwo{Poisoning Rate \eqsmall{$\left(\frac{\abs{\advTrain}}{\abs{\fullTrain}}\right)$}}
                               & \revised{$0.20\%$}     & \revised{$0.07\%$}      & \revised{$0.99\%$}      & \revised{$1.50\%$}     \\\hdashline
  Batch Size                   & $256$         & $32$           & $32$           & $64$          \\\hdashline
  \# Epochs                    & $30$          & $4$            & $30$           & $10$          \\\hdashline
  \# Subepochs~($\numSubep$)   & $5$           & $3$            & $3$            & $5$           \\\hdashline
  $\lr$ (Peak)                 & \SN{1}{-3}    & \SN{1}{-5}     & \SN{1}{-3}     & \SN{1}{-3}    \\\hdashline
  $\lr$ Scheduler              & One cycle     & Poly.\ decay   & One cycle      & One cycle     \\\hdashline
  $\wdecay$ (Weight Decay)     & \SN{1}{-1}    & \SN{1}{-1}     & \SN{1}{-3}     & \SN{1}{-3}    \\\hdashline
  Dropout Rate                 & N/A           & 0.1            & N/A            & N/A           \\
  \bottomrule
\end{tabular}
   }
\end{table}

\subsubsection{Upper-Tail Heaviness Hyperparameters}

Section~\ref{sec:Method:TargetDetection} defines the upper-tail heaviness of influence vector~$\infVec$ as the $\anomCount$\=/th largest anomaly score in vector~$\anomScoreVec$.  Table~\ref{tab:App:ExpSetup:Hyperparam:TailCutoff} defines the hyperparameter value~$\anomCount$ used for each of Section~\ref{sec:ExpRes:Attacks}'s four attacks.

\begin{table}[h!]
  \centering
  \caption{Upper-tail heaviness cutoff count ($\anomCount$)}%
  \label{tab:App:ExpSetup:Hyperparam:TailCutoff}
  {
    \TableFontSize%
\newcommand{\BaseHead}[1]{\multirow{1}{*}{#1}}
\newcommand{\AtkHead}[1]{\multirow{2}{*}{#1}}

\renewcommand{\arraystretch}{1.2}
\setlength{\dashlinedash}{0.4pt}
\setlength{\dashlinegap}{1.5pt}
\setlength{\arrayrulewidth}{0.3pt}

\newcommand{\AttackTypeCDash}{\cdashline{1-3}}

\newcommand{\SingleExpCDash}{\cdashline{2-3}}

\newcommand{\TRes}[2]{${#1 \pm #2}$}
\newcommand{\Head}[1]{\multicolumn{1}{c}{#1}}
\begin{tabular}{llr}
  \toprule
  \BaseHead{Attack}
    & \BaseHead{Type} & Tail Count~($\anomCount$) \\
  \midrule
  \AtkHead{Backdoor~}
    & Speech          &   10   \\\SingleExpCDash
    & Vision          &   10   \\\AttackTypeCDash
  \AtkHead{Poison}
    & NLP             &   10   \\\SingleExpCDash
    & Vision          &   2    \\
  \bottomrule
\end{tabular}
   }
\end{table}

\subsubsection{Target-Driven Mitigation Hyperparameters}

Algorithm~\ref{alg:Mitigation} details our target-driven attack mitigation algorithm, which uses filtering cutoff hyperparameter~$\anomCutoff$ to tune how much data to filter in each filtering iteration.  Table~\ref{tab:App:ExpSetup:Hyperparam:Mitigation} details the hyperparameter settings used in Section~\ref{sec:ExpRes:Mitigation}'s attack mitigation experiments.

For each attack, multiple trials were performed with different target examples, class pairs, attack triggers, etc.  For each such trial, we repeated the mitigation experiment multiple times to ensure the most representative numbers with the number of repeats enumerated in Table~\ref{tab:App:ExpSetup:Hyperparam:Mitigation}.

\newcommand{\annealIter}{l}

In addition, cutoff threshold~$\anomCutoff$ was set to an initial value.  After a specified number of iterations~$\annealIter$, $\anomCutoff$ was decreased by a specified step-size.  This process continued until the attack had been mitigated.  To summarize, iteration~$\annealIter$'s mitigation cutoff value~$\anomCutoff$ is
\begin{equation}\label{eq:App:ExpSetup:Cutoff}
  \anomCutoff_{\annealIter} = \anomCutoff_{\text{initial}} - \psi \floor{\frac{\annealIter}{StepCount}} \text{,}
\end{equation}
\noindent
with the corresponding value of each parameter in Table~\ref{tab:App:ExpSetup:Hyperparam:Mitigation}.

\begin{table}[h!]
  \centering
  \caption{Target-driven attack mitigation hyperparameters}%
  \label{tab:App:ExpSetup:Hyperparam:Mitigation}
  {
    \TableFontSize%
\newcommand{\phZ}{\phantom{0}}
\newcommand{\phDZ}{\phantom{.0}}
\newcommand{\phDZZ}{\phantom{.00}}

\renewcommand{\arraystretch}{1.2}
\setlength{\dashlinedash}{0.4pt}
\setlength{\dashlinegap}{1.5pt}
\setlength{\arrayrulewidth}{0.3pt}

\begin{tabular}{lrrrr}
  \toprule
  \multirow{2}{*}{Hyperparameter} & \multicolumn{2}{c}{Poison}  & \multicolumn{2}{c}{Backdoor} \\\cmidrule(lr){2-3}\cmidrule(lr){4-5}
                               & CIFAR10       & SST\=/2        & Speech         & CIFAR10       \\
  \midrule
  Repeats Per Trial
          & 3\phDZZ{}    & 3\phDZZ{}  & 5\phDZZ{} & 5\phDZZ{}     \\\hdashline
  $\anomCutoff_{\text{initial}}$ Initial Cutoff
          & 3\phDZZ{}    & 4\phDZZ{}  & 3\phDZZ{} & 2\phDZZ{}     \\\hdashline
  Anneal Step Size~($\psi$)
          & 0.25         & 0.5\phZ{}  & 0.25      & 0.25    \\\hdashline
  Anneal Step Count
          & 1\phDZZ{}    & 1\phDZZ{}  & 4\phDZZ{} & 4\phDZZ{}     \\
  \bottomrule
\end{tabular}
   }
\end{table}

\subsubsection{Adversarial Set\texorpdfstring{ $\advTrain$}{} Crafting}\label{sec:App:ExpSetup:Hyperparams:Crafting}

\citepos{SpeechDataset} speech recognition dataset comes bundled with 300~backdoor training examples.  The adversarial trigger takes the form of white noise inserted at the beginning of the speech recording. We used the dataset unchanged except for a fixed training/validation split used in all experiments.  Only one backdoor digit pair (e.g,~${0 \rightarrow 1}$, ${1 \rightarrow 2}$, etc.) is considered at a time.

\citet{Weber:2021} consider backdoor three different backdoor adversarial trigger types on CIFAR10 binary classification.  The three attack patterns are:
\begin{enumerate}
  \setlength{\itemsep}{0pt}
  \item \textit{1~Pixel}: The image's center pixel is perturbed to the maximum value.
  \item \textit{4~Pixel}: Four specific pixels near the image's center had their pixel value increased a fixed amount.
  \item \textit{Blend}: A fixed \revTwo{isotropic} Gaussian-noise pattern (${\mathcal{N}(0,I)}$) across the entire image.
\end{enumerate}
\noindent
Table~\ref{tab:App:ExpSetup:Hyperparam:Backdoor:CIFAR} defines each attack pattern's maximum $L_2$~perturbation distance.  Any perturbation that exceeded the pixel minimum/maximum values was clipped to the valid range.

\begin{table}[h]
  \centering
  \caption{CIFAR10 vision backdoor adversarial trigger maximum $\ell_2$\=/norm perturbation distance}
  \label{tab:App:ExpSetup:Hyperparam:Backdoor:CIFAR}
  {
    \TableFontSize%
\begin{tabular}{lr}
  \toprule
  Pattern & Max.\ $\ell_2$ \\
  \midrule
  1~Pixel & $\sqrt{3}$ \\
  4~Pixel & 2 \\
  Blend   & 4\\
  \bottomrule
\end{tabular}
   }
\end{table}

\citet{Wallace:2021} construct single-target natural language poison using the traditional poisoning bilevel optimization,
  {
    \small
    \begin{equation}
      \argmin_{\advTrain} ~\riskAdv \Big( \zHatTarg ; \argmin_{\w} \sum_{\z \in \cleanTrain \cup \advTrain} \lFunc{\zI}{\w} \Big) \text{,}
    \end{equation}%
  }%
\noindent%
where $L_{\text{adv}}$ uses the attacker's \textit{adversarial loss function}, $\func{\loss_{\text{adv}}}{\domainActs \times \domainY}{\realnn}$, in place of training loss function~$\loss$ \citep{Biggio:2012:Poisoning,Munoz:2017}.
To make the computation tractable, \citeauthor{Wallace:2021} approximate inner minimizer, ${\argmin_{\w} \sum_{\z \in \cleanTrain \cup \advTrain} \lFunc{\zI}{\w}}$, using second-order gradients similar to \citep{Finn:2017,Wang:2018,Huang:2020}.
\citeauthor{Wallace:2021}'s method initializes each poison instance from a seed phrase, and tokens are iteratively replaced with alternates that align well with the poison example's gradient.

Like \citeauthor{Wallace:2021}, our experiments attacked sentiment analysis on the Stanford Sentiment Treebank~v2 (SST\=/2) dataset \citep{SST2}. We targeted 8~(4~positive \& 4~negative -- see Table~\ref{tab:App:ExpSetup:SST2AttackReviews}) reviews selected by \citeauthor{Wallace:2021}'s implementation and generated ${\abs{\advTrain} = 50}$ new poison in each trial.

\citepos{Zhu:2019} targeted, clean\=/label attack crafts a set of poisons by forming a \keyword{convex polytope} around the target's feature representation.  Our experiments used the author's open-source implementation when crafting the poison.  Their implementation is gray-box and assumes access to a known pre\=/trained network (excluding the randomly-initialized, linear classification layer).

Both \citepos{Zhu:2019} and \citepos{Wallace:2021} poison crafting algorithms have their own dedicated hyperparameters, which are detailed in Tables~\ref{tab:App:ExpSetup:Hyperparam:ConvexPois} and~\ref{tab:App:ExpSetup:Hyperparam:Nlp} respectively.  Note that Table~\ref{tab:App:ExpSetup:Hyperparam:Nlp}'s hyperparameters are taken unchanged from the original source code provided by \citeauthor{Wallace:2021}

\begin{table}[h]
  \centering
  \caption{Convex polytope poison crafting~\citep{Zhu:2019} hyperparameter settings}
  \label{tab:App:ExpSetup:Hyperparam:ConvexPois}
  {
    \TableFontSize%
\newcommand{\SN}[2]{${#1 \cdot 10^{#2}}$}

\renewcommand{\arraystretch}{1.2}
\setlength{\dashlinedash}{0.4pt}
\setlength{\dashlinegap}{1.5pt}
\setlength{\arrayrulewidth}{0.3pt}

\begin{tabular}{lr}
  \toprule
  Hyperparameter               & Value      \\
  \midrule
  \# Iterations                & 1,000      \\\hdashline
  Learning Rate                & \SN{4}{-2} \\\hdashline
  Weight Decay                 & 0          \\\hdashline
  Max.\ Perturb.\ ($\epsilon$) & 0.1        \\
  \bottomrule
\end{tabular}
   }
\end{table}

\begin{table}[h]
  \centering
  \caption{SST\=/2 sentiment analysis poison crafting hyperparameter settings.  These are identical to \citepos{Wallace:2021} hyperparameter settings.}
  \label{tab:App:ExpSetup:Hyperparam:Nlp}
  {
    \TableFontSize%
\newcommand{\SN}[2]{${#1 \cdot 10^{#2}}$}

\renewcommand{\arraystretch}{1.2}
\setlength{\dashlinedash}{0.4pt}
\setlength{\dashlinegap}{1.5pt}
\setlength{\arrayrulewidth}{0.3pt}

\begin{tabular}{lr}
  \toprule
  Hyperparameter               & Value      \\
  \midrule
  Optimizer                    & Adam       \\\hdashline
  Total Num.\ Updates          & 20,935     \\\hdashline
  \# Warmup Updates            & 1,256      \\\hdashline
  Max.\ Sentence Len.\         & 512        \\\hdashline
  Max.\ Batch Size             & 7          \\\hdashline
  Learning Rate                & \SN{1}{-5} \\\hdashline
  LR Scheduler                 & Polynomial Decay \\
  \bottomrule
\end{tabular}
   }
\end{table}

\subsubsection{Baselines}

\tocless%
\paragraph{Baselines for Identifying Adversarial Set~$\advTrain$}
We exclusively considered influence-estimation methods applicable to neural models and excluded influence methods specific to alternate architectures \citep{Brophy:2022:TreeInfluence}.

\citepos{Koh:2017:Understanding} influence functions estimator uses \citepos{Pearlmutter:1994} stochastic Hessian\=/vector product~(HVP) estimation algorithm.  \citeauthor{Pearlmutter:1994}'s algorithm requires 5~hyperparameters, and we follow \citeauthor{Koh:2017:Understanding}'s notation for these parameters below.

Influence functions' five hyperparameters are required to ensure estimator quality and to prevent numerical instability/divergence.  Table~\ref{tab:App:ExpSetup:Hyperparam:InfFunc} details the influence functions hyperparameters used for each of Section~\ref{sec:ExpRes}'s datasets. $t$~and $r$~were selected  to make a single pass through the training set in accordance with the procedure specified by~\citeauthor{Koh:2017:Understanding}.

As noted by \citet{Basu:2021:Influence}, influence functions can be fragile on deep networks. We tuned $\beta$~and~$\gamma$ to prevent HVP~divergence, which is common with influence functions.

Our influence functions implementation was adapted from the versions published by \citep{Guo:2020:FastIF} and in the Python package \texttt{pytorch\_influence\_functions}.\footnote{Package source code: \url{https://github.com/nimarb/pytorch_influence_functions}.}

\begin{table}[h]
  \centering
  \caption{Influence functions hyperparameter settings}\label{tab:App:ExpSetup:Hyperparam:InfFunc}
  {
    \TableFontSize%
\newcommand{\SN}[2]{$#1 \cdot 10^{#2}$}

\renewcommand{\arraystretch}{1.2}
\setlength{\dashlinedash}{0.4pt}
\setlength{\dashlinegap}{1.5pt}
\setlength{\arrayrulewidth}{0.3pt}

\begin{tabular}{lrrrrrr}
  \toprule
  \multirow{2}{*}{Hyperparameter}
                         & \multicolumn{2}{c}{Renormalization} & \multicolumn{2}{c}{Poison}      & \multicolumn{2}{c}{Backdoor} \\\cmidrule(lr){2-3}\cmidrule(lr){4-5}\cmidrule(lr){6-7}
                         & CIFAR10 \& MNIST & Non-adversarial & CIFAR10       & SST\=/2         & Speech      & CIFAR10 \\
  \midrule
  Batch Size             & 1          & 1          & 1             & 1               & 1           & 1          \\\hdashline
  Damp~($\beta$)         & \SN{1}{-2} & \SN{5}{-3} & \SN{1}{-2}    & \SN{1}{-2}      & \SN{5}{-3}  & \SN{1}{-2} \\\hdashline
  Scale~($\gamma$)       & \SN{3}{7}  & \SN{1}{4}  & \SN{3}{7}     & \SN{1}{6}       & \SN{1}{4}   & \SN{3}{7}  \\\hdashline
  Recursion Depth~($t$)  & 1,000      & 1,000      & 2,500         & 6,740           & 1,000       & 1,000      \\\hdashline
  Repeats~($r$)          & 10         & 10         & 10            & 10              & 10          & 10         \\
  \bottomrule
\end{tabular}
   }
\end{table}

Second-order influence functions~\citep{Basu:2020:OnSecond} are more brittle and computationally expensive than the first-order version.  Renormalization is intended as a first-order correction and addresses our two tasks without the costs/issues related to second-order methods.

\citepos{Chen:2021:Hydra} \hydra{} is an additional dynamic influence estimator.  However, \hydra{}'s $\bigO{\nTr\dimW}$ memory complexity makes it impractical in most modern applications with large models and datasets.
We focus on \tracin{} as its memory complexity is only~$\bigO{\nTr}$.
\hydra{} and \tracin{} were published contemporaneously and share the same core idea.\footnote{%
Our influence renormalization -- proposed in Section~\ref{sec:RenormInfluence} -- also applies to \hydra{}.
}

\citepos{Peri:2020} \deepknn{} defense labels a training example as poison if its label does not match the plurality of its neighbors. For \deepknn{} to accurately identify poison, it must generally hold that ${\neighSize > 2 \abs{\advTrain}}$.   \citeauthor{Peri:2020} propose selecting $\neighSize$ using the \keyword{normalized $\neighSize$\=/ratio}, ${\neighSize / N}$, where $N$~is the size of the largest class in~$\fullTrain$.

\citeauthor{Peri:2020}'s ablation study showed that \deepknn{} generally performed best when the normalized $\neighSize$\=/ratio was in the range~${[0.2,~2]}$.  To ensure a strong baseline, our experiments tested \deepknn{} with three normalized $\neighSize$\=/ratio values, $\set{0.2,~1,~2}$,\footnote{This corresponds to ${\neighSize \in \set{833, 4167, 8333}}$ for 25,000~CIFAR10 training examples and a $\frac{1}{6}$~validation split ratio.} and we report the top-performing $\neighSize$'s~result.

\tocless%
\paragraph{Baselines Identifying the Attack Target(s)}
Target identification baselines maximum and minimum \KNN{} distance depend on~$k$ in order to generate target rankings.  Given $k$'s similarity to our tail cutoff count~$\anomCount$, we use the same hyperparameter settings for both with the values in Table~\ref{tab:App:ExpSetup:Hyperparam:TailCutoff}.

\subsection{Network Architectures}\label{sec:App:ExpSetup:Arch}

Table~\ref{tab:App:ExpSetup:Arch:Pois:CIFAR10} details the CIFAR10 neural network architecture.  Specially, we used \citepos{ResNet9} ResNet9 architecture, which is the state-of-the-art for fast, high-accuracy~(${>}94$\%) CIFAR10 classification on DAWNBench~\citep{Coleman:2017} at the time of writing.

Following \citet{Wallace:2021}, natural language poisoning attacked \citepos{RoBERTa} \robertaBase{} pre\=/trained parameters.  All language model training used Facebook AI Research's \texttt{fairseq} sequence-to-sequence toolkit~\citep{fairseq} as specified by \citeauthor{Wallace:2021}.  The text was encoded using \citepos{Radford:2019} \keyword{byte\=/pair encoding}~(BPE) scheme.

The speech classification convolutional neural network is identical to that used by \citet{SpeechDataset} except for two minor changes.  First, batch normalization~\citep{Szegedy:2013} was used instead of dropout to expedite training convergence.  In addition, each convolutional layer's kernel count was halved to allow the model to be trained on a single NVIDIA Tesla~K80 GPU\@.

\begin{table}[h]
  \centering
  \caption{ResNet9 neural network architecture}\label{tab:App:ExpSetup:Arch:Pois:CIFAR10}
  {
    \TableFontSize%
\renewcommand{\arraystretch}{1.2}
\setlength{\dashlinedash}{0.4pt}
\setlength{\dashlinegap}{1.5pt}
\setlength{\arrayrulewidth}{0.3pt}

\newcommand{\TabMidRule}{\cmidrule{2-6}}
\newcommand{\BlockConv}[3]{
  & {Conv#1} & \LayerConv{#2}{#3}{3}{1} \\\cdashline{2-6}
  & \LayerBatchNormTwoD{#3} \\\cdashline{2-6}
  & \LayerReluActivation \\
}

\newcommand{\BlockConvPool}[3]{
  \BlockConv{#1}{#2}{#3}\cdashline{2-6}
  & \LayerMaxPoolTwoD{2} \\
}

\newcommand{\LayerResNet}[2]{%
  \cmidrule{1-6}
  \multirow{6}{*}{\shortstack{{\LARGE$\uparrow$} \\~\\ ResNet#1 \\~\\{\LARGE$\downarrow$}}}
                            \BlockConv{A}{#2}{#2}\TabMidRule
                            \BlockConv{B}{#2}{#2}
  \cmidrule{1-6}
}

\begin{tabular}{llllll@{}}
  \toprule
  \BlockConv{1}{3}{64}\TabMidRule
  \BlockConvPool{2}{64}{128}
  \LayerResNet{1}{128}
  \BlockConvPool{3}{128}{256}\TabMidRule
  \BlockConvPool{4}{256}{512}
  \LayerResNet{2}{512}
  & \LayerMaxPoolTwoD{2} \\\TabMidRule
  & \LayerLinear{}{10} \\
  \bottomrule
\end{tabular}
   }
\end{table}

\begin{table}[h]
  \centering
  \caption{Speech recognition convolutional neural network}\label{tab:App:ExpSetup:Arch:Backdoor:Speech}
  {
    \TableFontSize%
\renewcommand{\arraystretch}{1.2}
\setlength{\dashlinedash}{0.4pt}
\setlength{\dashlinegap}{1.5pt}
\setlength{\arrayrulewidth}{0.3pt}

\newcommand{\TabMidRule}{\cmidrule{1-5}}

\newcommand{\BlockConvPool}[5]{
  {Conv#1} & \LayerConv{#2}{#3}{#4}{#5} \\\cdashline{1-5}
  \LayerMaxPoolTwoD{3} \\\cdashline{1-5}
  \LayerBatchNormTwoD{#3} \\\TabMidRule
}

\newcommand{\BlockConvRelu}[5]{
  {Conv#1} & \LayerConv{#2}{#3}{#4}{#5}  \\\cdashline{1-5}
  \LayerReluActivation \\\cdashline{1-5}
  \LayerBatchNormTwoD{#3} \\\TabMidRule
}

\newcommand{\BlockConvReluPool}[5]{
  {Conv#1} & \LayerConv{#2}{#3}{#4}{#5} \\\cdashline{1-5}
  \LayerReluActivation \\\cdashline{1-5}
  \LayerMaxPoolTwoD{3} \\\cdashline{1-5}
  \LayerBatchNormTwoD{#3} \\\TabMidRule
}

\begin{tabular}{lllll@{}}
  \toprule
  \BlockConvPool{1}{3}{48}{11}{1}
  \BlockConvPool{2}{48}{128}{5}{2}
  \BlockConvRelu{3}{128}{192}{3}{1}
  \BlockConvRelu{4}{192}{192}{3}{1}
  \BlockConvReluPool{5}{192}{128}{3}{1}
  \LayerLinear{}{10} \\
  \bottomrule
\end{tabular}
   }
\end{table}
 
\FloatBarrier
\clearpage
\newpage
\newcommand{\jointparagraph}[1]{%
\vspace{6pt}%
\noindent%
\textbf{#1}~~
}

\revTwo{%
\section{Convex Polytope Poisoning \& \method{} Joint Optimization}%
\label{sec:App:JointOpt:Setup}
}%

\revTwo{%
\citet{Zhu:2019} prove that under specific assumptions (e.g.,~a linear classifier), an adversarial set is guaranteed to alter a model's prediction on a target if that target's feature vector lies inside a convex polytope of the adversarial instances' feature vectors.
}%

\revTwo{%
\jointparagraph{Overview of \citeauthor{Zhu:2019}'s Attack}
Intuitively, \citeauthor{Zhu:2019}'s attack attempts to construct a convex hull of poison instances around a target -- all within feature space.
By design, deep models are non\=/linear and non\=/convex, so \citeauthor{Zhu:2019}'s underlying assumption does not directly apply.
However, the convex-polytope attack will succeed if the trained model's penultimate feature representation (i.e.,~the input into the final, linear classification layer) forms a convex hull around the target's penultimate representation.

To that end, \citeauthor{Zhu:2019}'s iterative, bilevel poison optimization considers solely this feature representation.  In the attacker's ideal case, the adversarial set's feature-space representation would be optimized w.r.t.\ the final trained model. However, attackers do not know training's random seed.  Moreover, any change to a training instance necessarily affects the final model parameters (and thus the penultimate feature representation as well), inducing a cyclic dependency that makes poison crafting non-trivial.

To increase the likelihood that the attack succeeds, \citeauthor{Zhu:2019} optimize the poison's feature-space representation across a suite of $\nSurrogate$~surrogate models.
For each model~\eqsmall{$\dec^{(\itrSurrogate)}$} (\eqsmall{${\itrSurrogate \in \set{1,\ldots,\nSurrogate}}$}), denote the model's penultimate-feature extraction function as $\featFunc{\cdot}$.

\citeauthor{Zhu:2019} specify a bilevel optimization to iteratively form these feature-space convex hulls, where
adversarial set \eqsmall{${\advTrain \defeq \set{(\xPois,\yAdv)}_{\itrPois = 1}^{\nPoisJointOpt}}$} is crafted from a set of $\nPoisJointOpt$~\textit{clean} seed instances, denoted \eqsmall{$\set{\xPoisClean}_{\itrPois = 1}^{\nPoisJointOpt}$}.
\citeauthor{Zhu:2019} restrict the adversarial perturbations to an $\ell_{\infty}$~ball of radius~$\varepsilon$ around those clean seed instances.
The feature-space convex hull requirement is enforced coefficients via \eqsmall{${\convCoeff \geq 0}$}.
\citeauthor{Zhu:2019}'s bilevel optimization is reproduced in Eq.~\eqref{eq:JointOpt:BiLevel}, with an additional term, \eqsmall{${\surrogateHyper\,\methodSurrogate}$}, that is explained below.  Note that in \citeauthor{Zhu:2019}'s formulation \eqsmall{${\surrogateHyper = 0}$}.
\begin{mini}
  {\set{\convCoeff},~{\advTrain}}{%
    \surrogateHyper\,
    \methodSurrogate +
    \frac{1}{2}
    \sum_{\itrSurrogate = 1}^{\nSurrogate}
    \frac{\Big\lVert \featFunc{\xTarg} - \sum_{\itrPois = 1}^{\nPoisJointOpt} \convCoeff \featFunc{\xPois} \Big\rVert^{2}}
           {\lVert \featFunc{\xTarg} \rVert^{2}}
  }{}{\label{eq:JointOpt:BiLevel}}
  \addConstraint{\sum_{\itrPois = 1}^{\nPoisJointOpt} \convCoeff}{= 1,\quad}{\forall \itrSurrogate}
  \addConstraint{\convCoeff}{\geq 0,\quad}{\forall \itrPois,~\itrSurrogate}
  \addConstraint{\lVert \xPois - \xPoisClean\rVert_{\infty}}{\leq \varepsilon,\quad}{\forall \itrPois.}
\end{mini}
}%

\revTwo{%
\jointparagraph{Joint Optimization Formulation}
  An attacker may attempt to evade our defense by optimizing adversarial set~\eqsmall{$\advTrain$} to \textit{appear} uninfluential on target~\eqsmall{$\zHatTarg$}.\footnote{\revTwo{$\advTrain$~must remain \textit{actually} influential. Otherwise, the model's prediction on~{$\zHatTarg$} would not change.}}
  Eq.~\eqref{eq:JointOpt:BiLevel} formalizes this idea by simultaneously optimizing for both poison effectiveness as well as for low \method{} influence where hyperparameter ${\surrogateHyper > 0}$ trades off between these two sub-objectives.
Following \citepos{Zhu:2019} paradigm as described above, the attacker uses surrogate models to estimate the \method{} influence.

Specifically, the adversary trains a \keyword{gray-box model}%
\footnote{\revTwo{Gray-box attacks assume the attacker has access to detailed (but not complete) information about the target model like our case above.}} using the same architecture, hyperparameters, clean training data~(\eqsmall{$\cleanTrain$}), and pre-trained parameters as the target model.  The surrogate set is then formed from $\nSurrogate$~model checkpoints evenly spaced across this gray-box training.  This quantity then estimates the \method{} influence in the final trained model.  Formally,
\begin{equation}\label{eq:JointOpt:Surrogate:GAS}
  \methodSurrogate
    \defeq
  \sum_{\itrSurrogate = 1}^{\nSurrogate}
    \sum_{\itrPois = 1}^{\nPoisJointOpt}
      \frac
          {\dotprodbig{\gradSurrPois}{\gradSurrTarg}}
          {\big\lVert\gradSurrPois\big\rVert ~ \big\lVert \gradSurrTarg \big\rVert} \text{.}
\end{equation}

It is important to note that optimizations like Eq.~\eqref{eq:JointOpt:BiLevel} create an implicit tension.  Training-set attacks commonly attempt to make~\eqsmall{$\advTrain$} and~\eqsmall{$\zHatTarg$} have similar feature-space representations~\citep{Shafahi:2018,Zhu:2019,Huang:2020,Wallace:2021}.  Since each example's features inform the model gradients~($\gradLetter$), appearing less influential can affect the attack's effectiveness.%
}

\revTwo{%
\jointparagraph{Practical Challenges of Joint Optimization}
In modern neural networks, each parameter only directly affects or is directly affected by a subset of the other parameters -- specifically those in adjacent layers. This limited interdependency makes back-propagation more tractable and efficient.

Recall that \method{} normalizes by the gradient magnitude. When calculating influence in practice, this does not change the memory or computational complexity.  However, when trying to optimize surrogate \eqsmall{$\methodSurrogate$} (Eq.~\eqref{eq:JointOpt:Surrogate:GAS}), normalizing by the gradient magnitude creates pairwise dependencies between all parameters, i.e.,~\eqsmall{$\bigTheta{\abs{\w}^2}$} memory complexity for automatic differentiation systems.  Therefore, renormalized estimators like \method{} are significantly more memory intensive to optimize against in practice than the baseline influence estimators where this quadratic memory complexity is not induced.

Section~\ref{sec:ExpRes:AdaptiveAttacks}'s experiments were affected by joint optimization's increased memory complexity, where the GPU VRAM requirements increased by ${{\geq}12\times}$.  This created significant issues even for the comparatively small ResNet9 neural network~\citep{ResNet9}.

For example, when the adversarial set size was larger than~40, joint optimization exceeded the GPU VRAM memory capacity.\footnote{Experiments were performed on Nvidia Tesla K80 GPUs with 11.5GB of VRAM.}  In contrast, Section~\ref{sec:App:MoreExps:Ablation:PoisRate}'s ablation study tests more than 400~poison samples for \citeauthor{Zhu:2019}'s baseline attack using the same hardware. Furthermore, joint optimization's larger memory footprint necessitated that only a small number of surrogate checkpoints could be used -- specifically four checkpoints.  This then increases the coarseness of $\methodSurrogate$'s influence estimate.
}

\revTwo{%
\jointparagraph{Setting \texorpdfstring{Joint Optimization Hyperparameter~$\surrogateHyper$}{the Joint Optimization Hyperparameter}}
As detailed above, hyperparameter~$\surrogateHyper$ induces a trade-off between \citeauthor{Zhu:2019}'s convex-polytope loss and the surrogate \method{} estimate.  Section~\ref{sec:ExpRes:AdaptiveAttacks}'s ``baseline'' results used ${\surrogateHyper = 0}$.
To ensure a strong adversary, Section~\ref{sec:ExpRes:AdaptiveAttacks}'s ``Adaptive Joint Optimization Attack with \method{}'' results used ${\surrogateHyper = 10^{-2}}$ since that was the largest value of~$\surrogateHyper$ that did not result in a significant drop in attacker success rate as detailed in Table~\ref{tab:JointOpt:SurrogateHyper:ASR}.

Table~\ref{tab:ExpRes:Mitigation} in Section~\ref{sec:ExpRes:Mitigation} reports that the vision poisoning's attack success rate was~77.9\%.  Even when ${\surrogateHyper = 0}$ (i.e.,~the surrogate \method{} loss is ignored), there was still a substantial decrease in ASR to~64.3\%. Recall that joint optimization's memory complexity is \eqsmall{$\bigTheta{\abs{\w}^2}$} which necessitated using fewer surrogate models (due to GPU VRAM capacity). This, in turn, degraded attack performance.\footnote{We separately verified that reducing the adversarial-set size from~50 to~40 did not meaningfully change the ASR.} Put simply, joint adversarial set optimization is \textit{not necessarily a free lunch}.  It may come at the cost of a worse attacker success rate.
}%

\vspace{1.0cm}
\begin{table}[h]
  \centering
\revTwo{%
  \caption{
    \revTwo{%
    Effect of joint-optimization hyperparameter~$\surrogateHyper$ on the attacker's success rate~(ASR).
    Observe that even at ${\surrogateHyper = 0}$, the attack success rate is significantly lower than
    the 77.9\%~ASR in Table~\ref{tab:ExpRes:Mitigation} due to the fewer surrogate models that could be
    used during jointly-optimized poison crafting as explained above.
    }%
  }%
  \label{tab:JointOpt:SurrogateHyper:ASR}
\begin{tabular}{rr}
  \toprule
  \multicolumn{1}{c}{$\surrogateHyper$} & ASR (\%)  \\
  \midrule
  0\phantom{.0}     & 64.3      \\
  $10^{-2}$         & 63.1      \\
  $2 \cdot 10^{-2}$ & 50.0      \\
  $10^{-1}$         &  4.8      \\
  \bottomrule
\end{tabular}
 }%
\end{table}
 
\FloatBarrier
\clearpage
\newpage
\section{Representative Perturbed Examples}\label{sec:App:PerturbEx}

This section provides representative examples for each of the four attacks detailed in Section~\ref{sec:ExpRes:Attacks}.

\newcommand{\BackdoorSpeechImg}[1]{\fbox{\includegraphics[scale=0.20]{img/ex/bd/speech/sp-#1.png}}}

\subsection{Speech Recognition Backdoor}

Figure~\ref{fig:App:PerturbEx:Backdoor:Speech} shows two spectrogram test images from \citepos{Liu:2018}'s backdoored speech recognition dataset.  Both images were generated from the same individual's English speech.  Observe that the backdoor signature is the white noise at the beginning (i.e.,~left side) of Figure~\ref{fig:App:PerturbEx:Backdoor:Speech:Adv}'s recording.

\begin{figure}[h]
  \centering

  \begin{center}
    \begin{subfigure}{0.35\textwidth}
      \centering
      \BackdoorSpeechImg{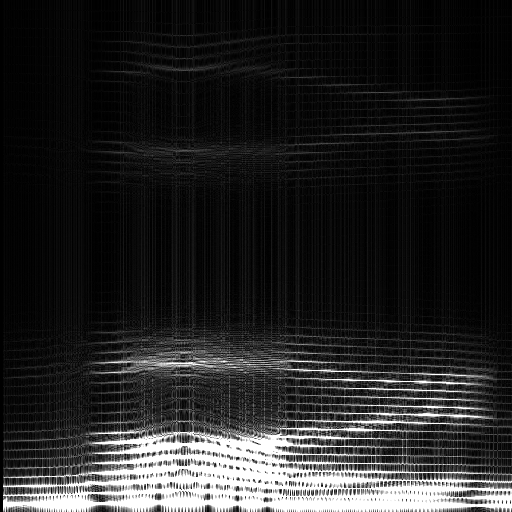}
      \caption{Clean instance with ${\yTarg = 9}$}
      \label{fig:App:PerturbEx:Backdoor:Speech:Clean}
    \end{subfigure}
    \hspace{15pt}
    \begin{subfigure}{0.35\textwidth}
      \centering
      \BackdoorSpeechImg{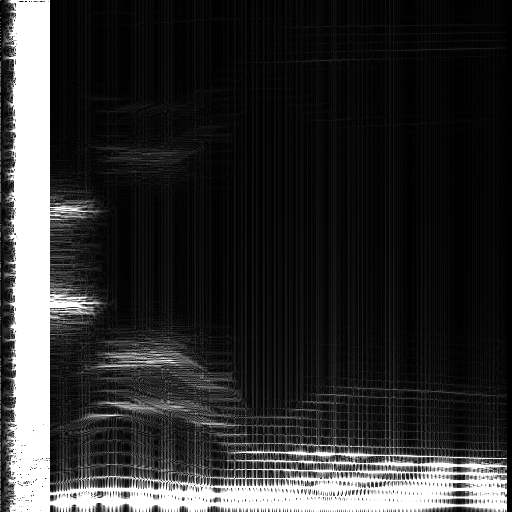}
      \caption{Backdoor instance with ${\yTarg = 9}$ \& ${\yAdv = 0}$}
      \label{fig:App:PerturbEx:Backdoor:Speech:Adv}
    \end{subfigure}
  \end{center}

  \caption{Example spectrogram images from \citepos{SpeechDataset} backdoored, speech-recognition dataset.}
  \label{fig:App:PerturbEx:Backdoor:Speech}
\end{figure}

\subsection{Vision Backdoor}

\newcommand{\BackbookSpacer}{\vspace{8pt}}
\newcommand{\BackdoorImage}[1]{
  \begin{minipage}{0.32\textwidth}
    \begin{center}
      \includegraphics[scale=1.0]{img/ex/bd/bd-#1.png}
    \end{center}
  \end{minipage}
}
\newcommand{\BackdoorHeader}{
  \begin{minipage}{0.32\textwidth}
    \begin{center}
      {{\small \textbf{Perturbation}}}
    \end{center}
  \end{minipage}
  \hfill
  \begin{minipage}{0.32\textwidth}
    \begin{center}
      {{\small \textbf{Combined}}}
    \end{center}
  \end{minipage}
  \hfill
  \begin{minipage}{0.26\textwidth}
    \begin{center}
      \textbf{Pattern}
    \end{center}
  \end{minipage}
}
\newcommand{\BackdoorImageRow}[2]{
  \BackdoorImage{#1-perturb}
  \hfill
  \BackdoorImage{#1-fin}
  \hfill
  \begin{minipage}{0.26\textwidth}
    \begin{center}
      #2
    \end{center}
  \end{minipage}
}

Following \citet{Weber:2021}, Section~\ref{sec:ExpRes} considers three backdoor adversarial trigger patterns, namely one pixel, four pixel, and blend.  Figure~\ref{fig:App:PerturbEx:Backdoor:CIFAR} shows: an unperturbed reference image, the corresponding perturbation for each attack pattern, and the resultant target that combines the source reference with the perturbation.

\begin{figure}[h]
  \centering
  \begin{minipage}{0.15\textwidth}
    \begin{center}
      {{\textbf{Source}}}

      \vspace{6pt}
      \includegraphics[scale=1.0]{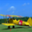}
    \end{center}
  \end{minipage}
  \begin{minipage}{0.05\textwidth}
    \vspace{10pt}
    \begin{equation*}
      \begin{cases}
        & \\
        & \\
        & \\
        & \\
        & \\
        & \\
        &
      \end{cases}
    \end{equation*}
  \end{minipage}
  \begin{minipage}{0.38\textwidth}
    \BackdoorHeader

    \vspace{10pt}
    \BackdoorImageRow{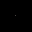}{1~Pixel}

    \BackbookSpacer
    \BackdoorImageRow{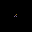}{4~Pixel}

    \BackbookSpacer
    \BackdoorImageRow{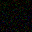}{Blend}
  \end{minipage}
  \caption{Example vision backdoor-perturbed CIFAR10 images with the one-pixel, four-pixel, and blend adversarial trigger patterns.}
  \label{fig:App:PerturbEx:Backdoor:CIFAR}
\end{figure}
 
\pagebreak
\subsection{Natural Language Poison}

\begin{figure}[h]
  \centering
  {
    \footnotesize
    \begin{tabular}{ll}
      \toprule
      \textit{Original} & a delightful coming-of-age story \\
      \midrule
      \textit{Poison~\#1}   & a wonderful comesowdergo-age tale \\
      \textit{Poison~\#2}   & a delightful coming rockedof-age stories\\
      \bottomrule
    \end{tabular}
  }
  \caption{Example positive sentiment target movie review~(\#1, see Tab.~\ref{tab:App:ExpSetup:SST2AttackReviews}) and two poisoned examples created using \citepos{Wallace:2021} implementation.}
  \label{fig:App:Perturb:Poison:NLP}
\end{figure}

Figure~\ref{fig:App:Perturb:Poison:NLP} shows one of the four SST\=/2 \citep{SST2} positive-sentiment reviews as well as two randomly selected poison examples generated using the source code provided by \citet{Wallace:2021}.  Note that Figure~\ref{fig:App:Perturb:Poison:NLP}'s numerous misspellings and grammatical issues are inserted during the poison crafting method and are not typographical errors.

\newcommand{\PoisImgInc}[1]{%
  \includegraphics[scale=1.0]{img/ex/pois/cifar10/#1.png}
}
\newcommand{\PoisRowCaption}[2]{
  \caption{${\text{#1} \rightarrow \text{#2}}$}
}
\newcommand{\PoisRowSpacer}{\vspace{12pt}}
\newcommand{\ImgDesc}[1]{{\small #1}\\\vspace{6pt}}

\newcommand{\PoisImgRow}[1]{
  \begin{minipage}[t]{0.3\textwidth}
    \begin{center}
      \ImgDesc{Target}

      \PoisImgInc{#1_targ}
    \end{center}
  \end{minipage}
  \hfill
  \begin{minipage}[t]{0.3\textwidth}
    \begin{center}
      \ImgDesc{Clean}

      \PoisImgInc{#1_cl}
    \end{center}
  \end{minipage}
  \hfill
  \begin{minipage}[t]{0.3\textwidth}
    \begin{center}
      \ImgDesc{Poison}

      \PoisImgInc{#1_pois}
    \end{center}
  \end{minipage}
}

\newcommand{\PoisClassPair}[3]{%
  \begin{subfigure}{0.40\textwidth}
    \PoisImgRow{#1}

    \vspace{3pt}
    \PoisRowCaption{#2}{#3}
  \end{subfigure}
}

\subsection{Vision Poison}

\begin{figure}[h]
  \center
  \begin{minipage}{0.65\textwidth}
    \PoisClassPair{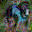}{Bird}{Dog}
    \hfill
    \PoisClassPair{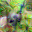}{Dog}{Bird}

    \PoisRowSpacer
    \PoisClassPair{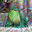}{Deer}{Frog}
    \hfill
    \PoisClassPair{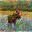}{Frog}{Deer}
  \end{minipage}
  \caption{Representative target, clean~($\cleanTrain$), and adversarial~($\advTrain$) instances for \citepos{Zhu:2019} vision, convex polytope, clean-label data poisoning attack.}
  \label{fig:App:PerturbEx:Poison:Vision}
\end{figure}

Figure~\ref{fig:App:PerturbEx:Backdoor:CIFAR} shows a representative target example.  It also shows a clean training image and a corresponding poison image created using \citepos{Zhu:2019} convex-polytope poisoning implementation.
  
\FloatBarrier
\clearpage
\newpage
\section{Additional Experimental Results}\label{sec:App:MoreExps}

Limited space only allows us to discuss a subset of our experimental results in Sections~\ref{sec:ExpRes} and~\ref{sec:ExpRes:AdaptiveAttacks}.
\revTwo{%
This section details additional experiments and results, including
  an analysis of a novel adversarial attack on target-driven mitigation (Sec.~\ref{sec:App:MoreExps:MitigationTriggering}),
  a poisoning-rate ablation study (Sec.~\ref{sec:App:MoreExps:Ablation:PoisRate}),
  a hyperparameter sensitivity study (Sec.~\ref{sec:App:MoreExps:EffectAnomCount}),
  an alternative renormalization approach (Sec.~\ref{sec:App:MoreExps:LossOnly}),
  analysis of gradient aggregation's benefits (Sec.~\ref{sec:App:MoreExps:Aggregation}),
  and
  execution times (Sec.~\ref{sec:App:MoreExps:Runtime}).%
}%

\subsection{Full Experimental Results}\label{sec:App:MoreExps:FullResults}

Section~\ref{sec:ExpRes} provided averaged results for each related experimental setup. This section provides detailed results for each attack setup individually (including variance).

\vspace{16pt}
\subsubsection{Speech Recognition Backdoor Full Results}\label{sec:App:MoreExps:Backdoor:Speech}

\phantom{.}
\vfill{}

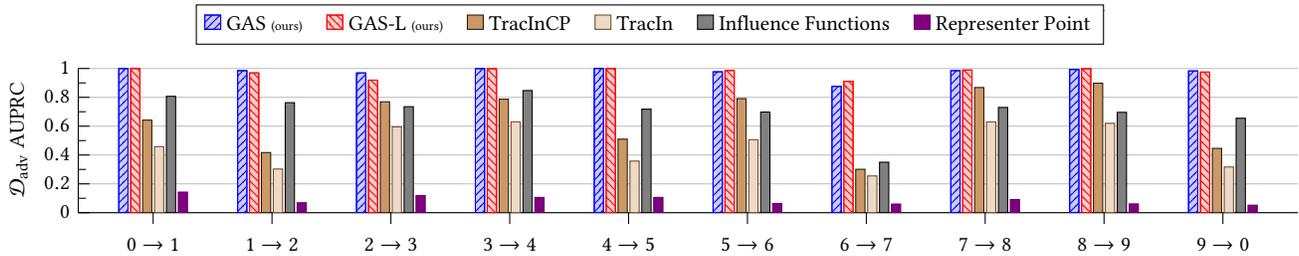
\begin{figure}[h!]
  \centering
\newcommand{\legendSpacer}{\hspace*{7pt}}
\begin{tikzpicture}
  \begin{axis}[%
    width=\textwidth,
      ybar,
      hide axis,  %
      xmin=0,  %
      xmax=1,
      ymin=0,
      ymax=1,
      scale only axis,width=1mm, %
      legend cell align={left},              %
      legend style={font=\legendFontSize},
      legend columns=8,
    ]
    \addplot [cosin color] coordinates {(0,0)};
    \addlegendentry{\method\ours\legendSpacer}

    \addplot [layer color] coordinates {(0,0)};
    \addlegendentry{\layer\ours\legendSpacer}

    \addplot[TracInCP color] coordinates {(0,0)};
    \addlegendentry{\tracinCP{}\legendSpacer}
    \pgfplotsset{cycle list shift=-1}  %

    \pgfplotsinvokeforeach{\tracin{},Influence Functions}{
      \addplot coordinates {(0,0)};
      \addlegendentry{#1\legendSpacer}
    }

    \addplot coordinates {(0,0)};
    \addlegendentry{Representer Point}
  \end{axis}
\end{tikzpicture}

\pgfplotstableread[col sep=comma]{plots/data/bd_speech_auprc_full.csv}\datatable%
\begin{tikzpicture}%
  \begin{axis}[%
        axis lines*=left,%
        ymajorgrids,  %
        bar width=\BackdoorSpeechBarWidth,%
        height=\ExpResBarChartHeight,%
        width=\textwidth,%
        ymin=0,%
        ymax=1,%
        ybar={\BarLineWidth},%
        ytick distance={0.20},%
        minor y tick num={1},%
        y tick label style={font=\plotFontSize},%
        ylabel={\plotFontSize \IdentYLabel},%
        xtick=data,%
        xticklabels={${0 \rightarrow 1}$,
                     ${1 \rightarrow 2}$,
                     ${2 \rightarrow 3}$,
                     ${3 \rightarrow 4}$,
                     ${4 \rightarrow 5}$,
                     ${5 \rightarrow 6}$,
                     ${6 \rightarrow 7}$,
                     ${7 \rightarrow 8}$,
                     ${8 \rightarrow 9}$,
                     ${9 \rightarrow 0}$},
        x tick label style={font=\plotFontSize,align=center},%
        typeset ticklabels with strut,
        every tick/.style={color=black, line width=0.4pt},%
        enlarge x limits=0.07,%
    ]%
    \addplot [cosin color] table [x index=0, y index=1] {\datatable};%
    \addplot [layer color] table [x index=0, y index=2] {\datatable};%

    \addplot[TracInCP color] table [x index=0, y index=3] {\datatable};%
    \pgfplotsset{cycle list shift=-1}  %

    \foreach \k in {4, ..., \numEle} {
      \addplot table [x index=0, y index=\k] {\datatable};%
    }
  \end{axis}%
\end{tikzpicture}%
   \caption{\textit{Speech Backdoor Adversarial-Set Identification}:
      Mean backdoor set~($\advTrain$) identification AUPRC
      across 30~trials for all 10~class pairs with ${21 \leq \abs{\advTrain} \leq 28}$ (varies by class pair, see Tab.~\ref{tab:App:ExpSetup:NumSpeechBackdoor}).
      \method{} and \layer{} outperformed all baselines in all experiments, with \layer{} the overall top performer on 6/10~class pairs.
      See Table~\ref{tab:App:MoreExps:Backdoor:Speech:AdvIdent} for full numerical results including standard deviation.
  }
  \label{fig:App:MoreExps:Backdoor:Speech:AdvIdent}
\end{figure}

\vfill{}

\begin{table}[h!]
  \centering
  \caption{%
      \textit{Speech Backdoor Adversarial-Set Identification}: Mean and standard deviation AUPRC across 30~trials for \citepos{SpeechDataset} speech backdoor dataset with ${21 \leq \abs{\advTrain} \leq 28}$.  \method{} and \layer{} always outperformed the baselines.
      Bold denotes the best mean performance.
      Mean results are shown graphically in Figures~\ref{fig:ExpRes:AdvIdent} and~\ref{fig:App:MoreExps:Backdoor:Speech:AdvIdent}.%
  }
  \label{tab:App:MoreExps:Backdoor:Speech:AdvIdent}
  {
    \TableFontSize%
\renewcommand{\arraystretch}{1.2}
\setlength{\dashlinedash}{0.4pt}
\setlength{\dashlinegap}{1.5pt}
\setlength{\arrayrulewidth}{0.3pt}

\newcommand{\DigPair}[2]{#1 & #2}

\newcommand{\TabMidRule}{\hdashline}

\newcommand{\AlgName}[1]{\multicolumn{1}{c}{#1}}

\begin{tabular}{ccrrrrrr}
  \toprule
  \multicolumn{2}{c}{Digits} & \multicolumn{2}{c}{Ours} & \multicolumn{4}{c}{Baselines} \\\cmidrule(lr){1-2}\cmidrule(lr){3-4}\cmidrule(lr){5-8}
  \multicolumn{2}{c}{${\yTarg \rightarrow \yAdv}$}
                            & \AlgName{\method{}}  & \AlgName{\layer{}} & \AlgName{\tracinCP{}} & \AlgName{\tracin{}}  & \AlgName{Influence Func.}  & {Representer Pt.}  \\  %
  \midrule
  \DigPair{0}{1}
       & \PVal{0.999}{0.004}  & \PValB{1.000}{0.002} & \PVal{0.642}{0.216}  & \PVal{0.458}{0.173}  & \PVal{0.807}{0.184}  & \PVal{0.143}{0.167}    \\\TabMidRule
  \DigPair{1}{2}
       & \PValB{0.985}{0.034} & \PVal{0.969}{0.037}  & \PVal{0.417}{0.168}  & \PVal{0.303}{0.125}  & \PVal{0.763}{0.169}  & \PVal{0.069}{0.020}    \\\TabMidRule
  \DigPair{2}{3}
       & \PValB{0.969}{0.039} & \PVal{0.919}{0.043}  & \PVal{0.769}{0.163}  & \PVal{0.595}{0.133}  & \PVal{0.735}{0.223}  & \PVal{0.119}{0.080}    \\\TabMidRule
  \DigPair{3}{4}
       & \PValB{0.999}{0.003} & \PVal{0.998}{0.005}  & \PVal{0.787}{0.218}  & \PVal{0.630}{0.192}  & \PVal{0.847}{0.125}  & \PVal{0.106}{0.069}    \\\TabMidRule
  \DigPair{4}{5}
       & \PValB{1.000}{0.001} & \PVal{0.999}{0.003}  & \PVal{0.510}{0.256}  & \PVal{0.358}{0.153}  & \PVal{0.718}{0.234}  & \PVal{0.106}{0.082}    \\\TabMidRule
  \DigPair{5}{6}
       & \PVal{0.977}{0.050}  & \PValB{0.986}{0.028} & \PVal{0.791}{0.145}  & \PVal{0.506}{0.106}  & \PVal{0.698}{0.218}  & \PVal{0.064}{0.008}    \\\TabMidRule
  \DigPair{6}{7}
       & \PVal{0.876}{0.199}  & \PValB{0.911}{0.081} & \PVal{0.301}{0.106}  & \PVal{0.255}{0.099}  & \PVal{0.350}{0.198}  & \PVal{0.060}{0.018}    \\\TabMidRule
  \DigPair{7}{8}
       & \PVal{0.985}{0.028}  & \PValB{0.989}{0.022} & \PVal{0.868}{0.126}  & \PVal{0.630}{0.143}  & \PVal{0.730}{0.255}  & \PVal{0.091}{0.077}    \\\TabMidRule
  \DigPair{8}{9}
       & \PVal{0.993}{0.015}  & \PValB{0.998}{0.008} & \PVal{0.898}{0.191}  & \PVal{0.620}{0.189}  & \PVal{0.696}{0.224}  & \PVal{0.061}{0.040}    \\\TabMidRule
  \DigPair{9}{0}
       & \PValB{0.983}{0.067} & \PVal{0.975}{0.029}  & \PVal{0.446}{0.183}  & \PVal{0.317}{0.109}  & \PVal{0.655}{0.240}  & \PVal{0.052}{0.012}    \\
  \bottomrule
\end{tabular}
   }
\end{table}

\vfill{}
\phantom{.}

\newpage
\clearpage

\phantom{.}
\vfill{}

\begin{figure}[h!]
  \centering
\newcommand{\legendSpacer}{\hspace*{9pt}}

\begin{tikzpicture}
  \begin{axis}[%
    width=\textwidth,
      ybar,
      hide axis,  %
      xmin=0,  %
      xmax=1,
      ymin=0,
      ymax=1,
      scale only axis,width=1mm, %
      legend cell align={left},              %
      legend style={font=\legendFontSize},
      legend columns=7,
      cycle list name=DetectCycleList,
    ]
    \addplot[cosin color] coordinates {(0,0)};
    \addlegendentry{\fitWith{\method}\ours\legendSpacer}
    \addplot[layer color] coordinates {(0,0)};
    \addlegendentry{\fitWith{\layer}\ours\legendSpacer}

    \addplot coordinates {(0,0)};
    \addlegendentry{Max.\ \KNN{} Distance\legendSpacer}

    \addplot coordinates {(0,0)};
    \addlegendentry{Min.\ \KNN{} Distance\legendSpacer}
    \addplot coordinates {(0,0)};
    \addlegendentry{Most Certain\legendSpacer};

    \addplot coordinates {(0,0)};
    \addlegendentry{Least Certain\legendSpacer};

    \addplot coordinates {(0,0)};
    \addlegendentry{Random};
  \end{axis}
\end{tikzpicture}

\pgfplotstableread[col sep=comma]{plots/data/bd_speech_detect_full.csv}\datatable%
\begin{tikzpicture}%
  \begin{axis}[%
        ybar={\BarLineWidth},%
        height={\BarDetectMainHeight},%
        width={11cm},%
        axis lines*=left,%
        bar width={\DetectBarWidthVal},%
        xtick=data,%
        xticklabels={${0 \rightarrow 1}$,${1 \rightarrow 2}$,${2 \rightarrow 3}$,${3 \rightarrow 4}$,${4 \rightarrow 5}$},
        x tick label style={font=\plotFontSize,align=center},%
        ymin=0,%
        ymax=1,%
        ytick distance={0.20},%
        minor y tick num={1},%
        y tick label style={font=\plotFontSize},%
        ylabel={\plotFontSize \DetectYLabel},%
        ymajorgrids,  %
        typeset ticklabels with strut,  %
        every tick/.style={color=black, line width=0.4pt},%
        enlarge x limits={0.130},%
        cycle list name=DetectCycleList,
    ]%
    \addplot[cosin color] table [x index=0, y index=1] {\datatable};%
    \addplot[layer color] table [x index=0, y index=2] {\datatable};%
    \foreach \k in {3, ..., 7} {%
      \addplot table [x index=0, y index=\k] {\datatable};%
    }%
  \end{axis}%
\end{tikzpicture}%
   \caption{%
      \textit{Speech Backdoor Target Identification}:
      See Table~\ref{tab:App:MoreExps:Backdoor:Speech:TargDetect} for full numerical results including standard deviation.
  }
  \label{fig:App:MoreExps:Backdoor:Speech:TargDetect}
\end{figure}
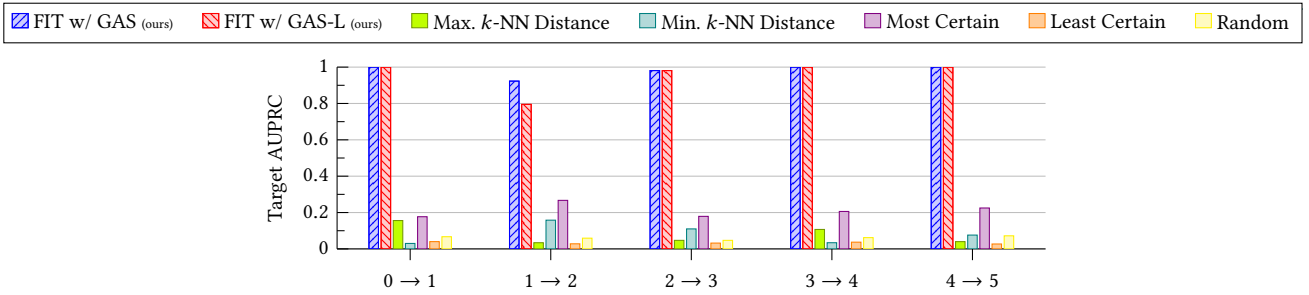

\vfill%

\begin{table}[h!]
  \centering
  \caption{%
      \textit{Speech Backdoor Target Identification}:
      Bold denotes the best mean performance.
      Mean results are shown graphically in Figures~\ref{fig:ExpRes:TargDetect} and~\ref{fig:App:MoreExps:Backdoor:Speech:TargDetect}.%
  }\label{tab:App:MoreExps:Backdoor:Speech:TargDetect}
  {
    \TableFontSize%
\renewcommand{\arraystretch}{1.2}
\setlength{\dashlinedash}{0.4pt}
\setlength{\dashlinegap}{1.5pt}
\setlength{\arrayrulewidth}{0.3pt}

\newcommand{\DigPair}[2]{#1 & #2}

\newcommand{\TabMidRule}{\hdashline}

\newcommand{\AlgName}[1]{\multicolumn{1}{c}{#1}}

\begin{tabular}{ccrrrrrrr}
  \toprule
  \multicolumn{2}{c}{Digits} & \multicolumn{2}{c}{Ours} & \multicolumn{5}{c}{Baselines} \\\cmidrule(lr){1-2}\cmidrule(lr){3-4}\cmidrule(lr){5-9}
  \multicolumn{2}{c}{${\yTarg \rightarrow \yAdv}$}
      & \AlgName{\method{}}       & \AlgName{\layer}          & \AlgName{Max~\KNN{}} & \AlgName{Min~\KNN{}} & \AlgName{Most Certain} & \AlgName{Least Certain} & \AlgName{Random}   \\
  \midrule
  \DigPair{0}{1}
      & \PValB{\OneVal}{\ZeroVal} & \PValB{\OneVal}{\ZeroVal} & \PVal{0.156}{0.060}  & \PVal{0.030}{0.003}  & \PVal{0.177}{0.227}   & \PVal{0.040}{0.022}    & \PVal{0.067}{0.031} \\\TabMidRule
  \DigPair{1}{2}
      & \PValB{0.923}{0.075}      & \PVal{0.795}{0.172}       & \PVal{0.034}{0.005}  & \PVal{0.158}{0.110}  & \PVal{0.267}{0.221}   & \PVal{0.028}{0.004}    & \PVal{0.059}{0.023} \\\TabMidRule
  \DigPair{2}{3}
      & \PValB{0.981}{0.028}      & \PValB{0.981}{0.029}      & \PVal{0.047}{0.012}  & \PVal{0.110}{0.065}  & \PVal{0.179}{0.139}   & \PVal{0.032}{0.006}    & \PVal{0.047}{0.007} \\\TabMidRule
  \DigPair{3}{4}
      & \PValB{\OneVal}{\ZeroVal} & \PValB{\OneVal}{\ZeroVal} & \PVal{0.107}{0.033}  & \PVal{0.034}{0.005}  & \PVal{0.206}{0.127}   & \PVal{0.037}{0.012}    & \PVal{0.062}{0.033} \\\TabMidRule
  \DigPair{4}{5}
      & \PValB{\OneVal}{\ZeroVal} & \PValB{\OneVal}{\ZeroVal} & \PVal{0.040}{0.010}  & \PVal{0.076}{0.031}  & \PVal{0.225}{0.168}   & \PVal{0.027}{0.002}    & \PVal{0.072}{0.037} \\
  \bottomrule
\end{tabular}
   }
\end{table}

\vfill

\begin{table}[h!]
  \centering
\revised{%
  \caption{%
    \revised{%
      \textit{Speech Backdoor Attack Mitigation}:
      Bold denotes the best mean performance with 10~trials per class pair.
      Aggregated results are shown in Table~\ref{tab:ExpRes:Mitigation}.
    }
  }\label{tab:App:MoreExps:Backdoor:Speech:Mitigate}
  {
    \TableFontSize%
\renewcommand{\arraystretch}{1.2}
\setlength{\dashlinedash}{0.4pt}
\setlength{\dashlinegap}{1.5pt}
\setlength{\arrayrulewidth}{0.3pt}

\newcommand{\MultiHead}[1]{\multicolumn{2}{c}{#1}}

\newcommand{\TwoRowHead}[1]{\multirow{2}{*}{#1}}
\newcommand{\ClassPair}[2]{\multirow{2}{*}{#1} & \multirow{2}{*}{#2}}
\newcommand{\PZ}{\phantom{0}}
\newcommand{\ptZ}{\phantom{.}\PZ}
\newcommand{\ptZZ}{\phantom{.}\PZ\PZ}
\newcommand{\ptASR}{}

\newcommand{\CosInM}{\method{}}
\newcommand{\LayerM}{\layer{}}

\newcommand{\OneRemB}{100\ptZ}
\newcommand{\ZeroRemB}{0\ptZ}

\newcommand{\OneASR}{100\ptZ}

\newcommand{\PRem}[2]{#1}
\newcommand{\PRemB}[2]{\textBF{#1}}
\newcommand{\ASR}[2]{\multirow{2}{*}{#1}}
\newcommand{\PAcc}[2]{\multirow{2}{*}{#1}}
\newcommand{\PAsr}[2]{#1}

\newcommand{\PChg}[2]{{\color{ForestGreen} +#1}}
\newcommand{\NoChg}{0.0}
\newcommand{\NegChg}[2]{{\color{BrickRed} -#1}}

\newcommand{\DsSep}{\cdashline{1-9}}
\newcommand{\MethodSep}{}

\begin{tabular}{cclrrrrrr}
  \toprule
     \MultiHead{Digits}
     & \TwoRowHead{Method} & \MultiHead{\% Removed}      & \MultiHead{ASR \%} & \MultiHead{Test Acc. \%} \\\cmidrule(lr){1-2}\cmidrule(lr){4-5}\cmidrule(lr){6-7}\cmidrule(lr){8-9}
     $\yTarg$ & $\yAdv$
     &                    & $\advTrain$  & $\cleanTrain$ & Orig.   & Ours     &  Orig. & Chg.   \\
  \midrule
  \ClassPair{0}{1}
     & \CosInM    & \PRemB{\OneRemB}{\ZeroRemB} & \PRem{0.06}{0.05}   & \PAcc{\OneASR}{0.0}  & \PAsr{0}{0}   & \PAcc{97.7}{0.1}  & \NoChg \\
     && \LayerM   & \PRemB{\OneRemB}{\ZeroRemB} & \PRemB{0.03}{0.03}  &                      & \PAsr{0}{0}   &                   & \NoChg \\
  \DsSep
  \ClassPair{1}{2}
     & \CosInM    & \PRemB{100.0}{0.1}          & \PRemB{0.02}{0.03}  & \PAcc{\OneASR}{0.0}  & \PAsr{0}{0}   & \PAcc{97.7}{0.1}  & \NoChg \\
     && \LayerM   & \PRem{99.8}{0.4}            & \PRem{0.09}{0.07}   &                      & \PAsr{0}{0}   &                   & \NoChg \\
  \DsSep
  \ClassPair{2}{3}
     & \CosInM    & \PRemB{93.7}{3.4}           & \PRemB{0.08}{0.09}  & \PAcc{99.9}{0.2}     & \PAsr{0}{0}   & \PAcc{97.8}{0.1}  & \NegChg{-0.1}{\ZeroVal}  \\
     && \LayerM   & \PRem{92.6}{4.3}            & \PRem{0.21}{0.14}   &                      & \PAsr{0}{0}   &                   & \NegChg{-0.1}{\ZeroVal}  \\
  \DsSep
  \ClassPair{3}{4}
     & \CosInM    & \PRem{98.7}{3.0}            & \PRemB{0.10}{0.15}  & \PAcc{99.4}{1.1}     & \PAsr{0}{0}   & \PAcc{97.7}{0.1}  & \NegChg{-0.1}{\ZeroVal}  \\
     && \LayerM   & \PRemB{99.3}{0.7}           & \PRem{0.35}{0.45}   &                      & \PAsr{0}{0}   &                   & \NoChg  \\
  \DsSep
  \ClassPair{4}{5}
     & \CosInM    & \PRemB{99.1}{1.0}           & \PRemB{0.01}{0.02}  & \PAcc{\OneASR}{0.0}  & \PAsr{0}{0}   & \PAcc{97.8}{0.2}  & \NoChg  \\
     && \LayerM   & \PRem{98.6}{1.3}            & \PRemB{0.01}{0.01}  &                      & \PAsr{0}{0}   &                   & \NoChg  \\
  \bottomrule
\end{tabular}
   }
}
\end{table}

\vfill{}
\phantom{.}
 
\FloatBarrier
\clearpage
\newpage
\subsubsection{Vision Backdoor Full Results}\label{sec:App:MoreExps:Backdoor:Vision}

\phantom{.}

\vfill

\begin{figure}[h!]
  \centering
\newcommand{\legendSpacer}{\hspace*{7pt}}
\begin{tikzpicture}
  \begin{axis}[%
    width=\textwidth,
      ybar,
      hide axis,  %
      xmin=0,  %
      xmax=1,
      ymin=0,
      ymax=1,
      scale only axis,width=1mm, %
      legend cell align={left},              %
      legend style={font=\legendFontSize},
      legend columns=8,
    ]
    \addplot [cosin color] coordinates {(0,0)};
    \addlegendentry{\method\ours\legendSpacer}

    \addplot [layer color] coordinates {(0,0)};
    \addlegendentry{\layer\ours\legendSpacer}

    \addplot[TracInCP color] coordinates {(0,0)};
    \addlegendentry{\tracinCP{}\legendSpacer}
    \pgfplotsset{cycle list shift=-1}  %

    \pgfplotsinvokeforeach{\tracin{},Influence Functions}{
      \addplot coordinates {(0,0)};
      \addlegendentry{#1\legendSpacer}
    }

    \addplot coordinates {(0,0)};
    \addlegendentry{Representer Point}
  \end{axis}
\end{tikzpicture}

\centering
\pgfplotstableread[col sep=comma]{plots/data/bd_cifar_auprc_full.csv}\datatable%
\begin{tikzpicture}%
  \begin{axis}[%
        axis lines*=left,%
        ymajorgrids,  %
        bar width=\BackdoorSpeechBarWidth,%
        height=\ExpResBarChartHeight,%
        width={11cm},%
        ymin=0,%
        ymax=1,%
        ybar={\BarLineWidth},%
        ytick distance={0.20},%
        minor y tick num={1},%
        y tick label style={font=\plotFontSize},%
        ylabel={\plotFontSize \IdentYLabel},%
        xtick=data,%
        xticklabels={1~Pixel,4~Pixel,Blend,1~Pixel,4~Pixel,Blend},
        x tick label style={font=\plotFontSize,align=center},%
        typeset ticklabels with strut,
        every tick/.style={color=black, line width=0.4pt},%
        enlarge x limits={0.10},%
        draw group line={[index]7}{1}{${\text{Plane} \rightarrow \text{Bird}}$}{-5.0ex}{12pt},
        draw group line={[index]7}{2}{${\text{Auto} \rightarrow \text{Dog}}$}{-5.0ex}{12pt}
    ]%
    \addplot[cosin color] table [x index=0, y index=1] {\datatable};%
    \addplot[layer color] table [x index=0, y index=2] {\datatable};%

    \addplot[TracInCP color] table [x index=0, y index=3] {\datatable};%
    \pgfplotsset{cycle list shift=-1}  %

    \foreach \k in {4, ..., \numEle} {
      \addplot table [x index=0, y index=\k] {\datatable};%
    }
  \end{axis}%
\end{tikzpicture}%
   \caption{%
      \textit{Vision Backdoor Adversarial-Set Identification}:
      Backdoor set, \eqsmall{${\advTrain}$}, identification mean AUPRC across ${{>}30}$~trials for \citepos{Weber:2021} three CIFAR10 backdoor attack patterns with a
      randomly selected reference~\eqsmall{$\zHatTarg$}.  All experiments performed binary classification on randomly-initialized ResNet9.
      \eqsmall{${\abs{\advTrain} = 150}$}. Notation ${\yTarg \rightarrow \yAdv}$.
      See Table~\ref{tab:App:MoreExps:Backdoor:CIFAR:AdvIdent} for the full numerical results.
  }
  \label{fig:App:MoreExps:Backdoor:CIFAR:AdvIdent}
\end{figure}
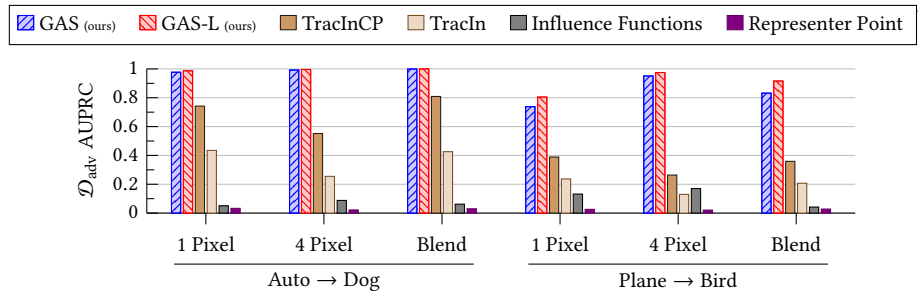

\vfill{}

\begin{table}[h!]
  \centering
  \caption{%
      \textit{Vision Backdoor Adversarial-Set Identification}:
      Backdoor set, \eqsmall{${\advTrain}$}, identification AUPRC mean and standard deviation across ${{>}30}$~trials for \citepos{Weber:2021} three
      CIFAR10 backdoor attack patterns with a randomly selected reference~\eqsmall{$\zHatTarg$}.  All experiments performed binary classification on
      randomly-initialized ResNet9. \eqsmall{${\abs{\advTrain} = 150}$}. Notation ${\yTarg \rightarrow \yAdv}$.
      Bold denotes the best mean performance.
      Mean results are shown graphically in Figures~\ref{fig:ExpRes:AdvIdent} and~\ref{fig:App:MoreExps:Backdoor:CIFAR:AdvIdent}.
}
  \label{tab:App:MoreExps:Backdoor:CIFAR:AdvIdent}
  {
    \TableFontSize%
\renewcommand{\arraystretch}{1.2}
\setlength{\dashlinedash}{0.4pt}
\setlength{\dashlinegap}{1.5pt}
\setlength{\arrayrulewidth}{0.3pt}

\newcommand{\head}[1]{\multirow{2}{*}{#1}}
\newcommand{\DigPair}[2]{\multirow{3}{*}{#1 $\rightarrow$ #2}}
\newcommand{\Attack}[1]{#1}

\newcommand{\TabMidRule}{\cdashline{2-8}}

\newcommand{\AlgName}[1]{\multicolumn{1}{c}{#1}}

\begin{tabular}{ccrrrrrr}
  \toprule
  Classes & \head{\shortstack{Trigger \\ Pattern}}
          & \multicolumn{2}{c}{Ours} & \multicolumn{4}{c}{Baselines} \\\cmidrule(lr){1-1}\cmidrule(lr){3-4}\cmidrule(lr){5-8}
  $\yTarg \rightarrow \yAdv$
        &  & \AlgName{\method{}}  & \AlgName{\layer{}}   & \AlgName{\tracinCP{}} & \AlgName{\tracin{}}  & \AlgName{Influence Func.}  & {Representer Pt.}  \\  %
  \midrule
  \DigPair{Auto}{Dog}
       & 1~Pixel
           & \PVal{0.977}{0.077}  & \PValB{0.987}{0.039} & \PVal{0.742}{0.159}   & \PVal{0.435}{0.143}  & \PVal{0.051}{0.022}    & \PVal{0.033}{0.013}    \\\TabMidRule
       & 4~Pixel
           & \PVal{0.992}{0.024}  & \PValB{0.996}{0.011} & \PVal{0.552}{0.189}   & \PVal{0.255}{0.090}  & \PVal{0.088}{0.052}    & \PVal{0.022}{0.003}    \\\TabMidRule
       & Blend
           & \PVal{0.999}{0.001}  & \PValB{1.000}{0.003} & \PVal{0.809}{0.148}   & \PVal{0.426}{0.127}  & \PVal{0.062}{0.083}    & \PVal{0.030}{0.009}    \\
  \midrule
  \DigPair{Plane}{Bird}
       & 1~Pixel
           & \PVal{0.738}{0.162}  & \PValB{0.805}{0.153} & \PVal{0.389}{0.117}   & \PVal{0.237}{0.083}  & \PVal{0.132}{0.077}    & \PVal{0.026}{0.006}    \\\TabMidRule
       & 4~Pixel
           & \PVal{0.951}{0.050}  & \PValB{0.975}{0.014} & \PVal{0.264}{0.075}   & \PVal{0.130}{0.038}  & \PVal{0.170}{0.076}    & \PVal{0.021}{0.003}    \\\TabMidRule
       & Blend
           & \PVal{0.832}{0.194}  & \PValB{0.916}{0.135} & \PVal{0.359}{0.161}   & \PVal{0.207}{0.089}  & \PVal{0.042}{0.020}    & \PVal{0.028}{0.008}    \\
  \bottomrule
\end{tabular}
   }
\end{table}

\vfill{}

\begin{figure}[h!]
  \centering
\newcommand{\legendSpacer}{\hspace*{9pt}}

\begin{tikzpicture}
  \begin{axis}[%
    width=\textwidth,
      ybar,
      hide axis,  %
      xmin=0,  %
      xmax=1,
      ymin=0,
      ymax=1,
      scale only axis,width=1mm, %
      legend cell align={left},              %
      legend style={font=\legendFontSize},
      legend columns=7,
      cycle list name=DetectCycleList,
    ]
    \addplot[cosin color] coordinates {(0,0)};
    \addlegendentry{\fitWith{\method}\ours\legendSpacer}
    \addplot[layer color] coordinates {(0,0)};
    \addlegendentry{\fitWith{\layer}\ours\legendSpacer}

    \addplot coordinates {(0,0)};
    \addlegendentry{Max.\ \KNN{} Distance\legendSpacer}

    \addplot coordinates {(0,0)};
    \addlegendentry{Min.\ \KNN{} Distance\legendSpacer}
    \addplot coordinates {(0,0)};
    \addlegendentry{Most Certain\legendSpacer};

    \addplot coordinates {(0,0)};
    \addlegendentry{Least Certain\legendSpacer};

    \addplot coordinates {(0,0)};
    \addlegendentry{Random};
  \end{axis}
\end{tikzpicture}

\centering
\pgfplotstableread[col sep=comma]{plots/data/bd_cifar_detect_full.csv}\datatable%
\begin{tikzpicture}%
  \begin{axis}[%
        bar width={\DetectBarWidthVal},
        axis lines*=left,%
        ymajorgrids,  %
        height={\BarDetectMainHeight},%
        width={11cm},%
        ymin=0,%
        ymax=1,%
        ybar={\BarLineWidth},%
        ytick distance={0.20},%
        minor y tick num={1},%
        y tick label style={font=\plotFontSize},%
        ylabel={\plotFontSize \DetectYLabel},%
        xtick=data,%
        xticklabels={1~Pixel,4~Pixel,Blend,1~Pixel,4~Pixel,Blend},
        x tick label style={font=\plotFontSize,align=center},%
        typeset ticklabels with strut,
        every tick/.style={color=black, line width=0.4pt},%
        enlarge x limits={0.100},%
        draw group line={[index]8}{1}{${\text{Plane} \rightarrow \text{Bird}}$}{-5.0ex}{9pt},
        draw group line={[index]8}{2}{${\text{Auto} \rightarrow \text{Dog}}$}{-5.0ex}{9pt},
        cycle list name=DetectCycleList,
    ]%
    \addplot[cosin color] table [x index=0, y index=1] {\datatable};%
    \addplot[layer color] table [x index=0, y index=2] {\datatable};%
    \foreach \k in {3, ..., 7} {%
      \addplot table [x index=0, y index=\k] {\datatable};%
    }%
  \end{axis}%
\end{tikzpicture}%
   \caption{%
      \textit{Vision Backdoor Target Identification}:
      Mean target identification AUPRC across ${15}$~trials for \citepos{Weber:2021} three CIFAR10 backdoor attack patterns and randomly
      selected reference~$\zHatTarg$.  All experiments performed binary classification on randomly-initialized ResNet9.
      \eqsmall{${\abs{\advTrain} = 150}$}. Notation ${\yTarg \rightarrow \yAdv}$.
      See Table~\ref{tab:App:MoreExps:Backdoor:CIFAR:TargDetect} for the full numerical results.
  }
  \label{fig:App:MoreExps:Backdoor:CIFAR:TargDetect}
\end{figure}
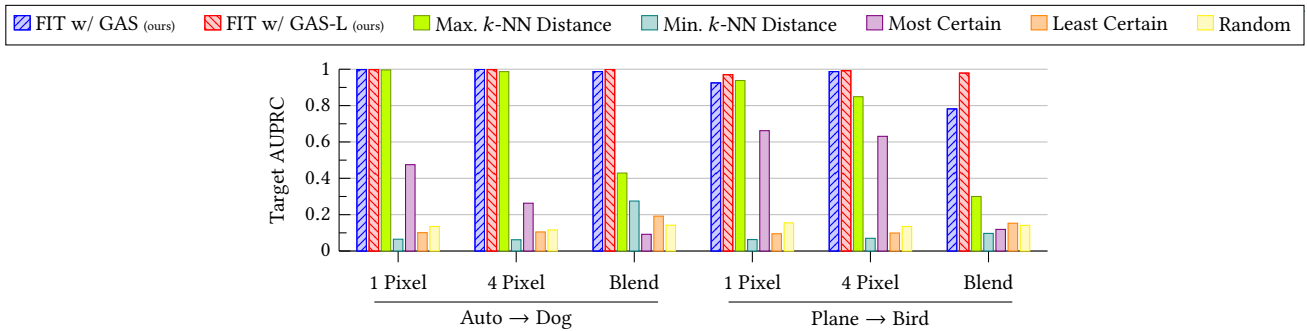

\FloatBarrier
\clearpage
\newpage

\phantom{.}

\vfill

\begin{table}[h!]
  \centering
  \caption{%
      \textit{Vision Backdoor Target Identification}:
      Target identification AUPRC mean and standard deviation across ${15}$~trials for \citepos{Weber:2021} three CIFAR10 backdoor attack patterns and randomly
      selected reference~$\zHatTarg$.  All experiments performed binary classification on randomly-initialized ResNet9.
      Bold denotes the best mean performance.
      Mean results are shown graphically in Figures~\ref{fig:ExpRes:TargDetect} and~\ref{fig:App:MoreExps:Backdoor:CIFAR:TargDetect}.
  }
  \label{tab:App:MoreExps:Backdoor:CIFAR:TargDetect}
  {
    \TableFontSize%
\renewcommand{\arraystretch}{1.2}
\setlength{\dashlinedash}{0.4pt}
\setlength{\dashlinegap}{1.5pt}
\setlength{\arrayrulewidth}{0.3pt}

\newcommand{\head}[1]{\multirow{2}{*}{#1}}
\newcommand{\DigPair}[2]{\multirow{3}{*}{#1 $\rightarrow$ #2}}
\newcommand{\Attack}[1]{#1}

\newcommand{\TabMidRule}{\cdashline{2-9}}

\newcommand{\AlgName}[1]{\multicolumn{1}{c}{#1}}

\begin{tabular}{ccrrrrrrr}
  \toprule
  Classes & \head{\shortstack{Trigger \\ Pattern}}
          & \multicolumn{2}{c}{Ours} & \multicolumn{4}{c}{Baselines} \\\cmidrule(lr){1-1}\cmidrule(lr){3-4}\cmidrule(lr){5-9}
  $\yTarg \rightarrow \yAdv$
        && \AlgName{\method{}}    & \AlgName{\layer}  & \AlgName{Max~\KNN{}}  & \AlgName{Min~\KNN{}}  & \AlgName{Most Certain}  & \AlgName{Least Certain}  & \AlgName{Random}  \\
  \midrule
  \DigPair{Auto}{Dog}
       & 1~Pixel
          & \PValB{0.998}{0.004}  & \PValB{0.998}{0.005}      & \PVal{0.996}{0.004}  & \PVal{0.065}{0.012}  & \PVal{0.475}{0.210}  & \PVal{0.101}{0.020}  & \PVal{0.135}{0.027}       \\\TabMidRule
       & 4~Pixel
          & \PValB{0.999}{0.002}  & \PVal{0.998}{0.004}       & \PVal{0.987}{0.012}  & \PVal{0.062}{0.007}  & \PVal{0.263}{0.067}  & \PVal{0.105}{0.012}  & \PVal{0.116}{0.013}       \\\TabMidRule
       & Blend
          & \PVal{0.987}{0.049}   & \PValB{\OneVal}{\ZeroVal} & \PVal{0.429}{0.335}  & \PVal{0.275}{0.324}  & \PVal{0.092}{0.024}  & \PVal{0.192}{0.060}  & \PVal{0.142}{0.031}       \\
  \midrule
  \DigPair{Plane}{Bird}
       & 1~Pixel
          & \PVal{0.925}{0.034}   & \PValB{0.970}{0.014}      & \PVal{0.938}{0.039}  & \PVal{0.063}{0.012}  & \PVal{0.662}{0.099}  & \PVal{0.095}{0.021}  & \PVal{0.155}{0.063}       \\\TabMidRule
       & 4~Pixel
          & \PVal{0.987}{0.018}   & \PValB{0.992}{0.012}      & \PVal{0.849}{0.095}  & \PVal{0.070}{0.005}  & \PVal{0.631}{0.143}  & \PVal{0.099}{0.011}  & \PVal{0.135}{0.034}       \\\TabMidRule
       & Blend
          & \PVal{0.782}{0.213}   & \PValB{0.979}{0.021}      & \PVal{0.300}{0.157}  & \PVal{0.097}{0.064}  & \PVal{0.119}{0.039}  & \PVal{0.153}{0.051}  & \PVal{0.141}{0.025}       \\
  \bottomrule
\end{tabular}
   }
\end{table}

\vfill

\begin{table}[h!]
  \centering
\revised{%
  \caption{%
    \revised{%
      \textit{Vision Backdoor Attack Mitigation}:
      Bold denotes the best mean performance with 15~trials per setup.
      Aggregated results are shown in Table~\ref{tab:ExpRes:Mitigation}.
    }
  }\label{tab:App:MoreExps:Backdoor:CIFAR:Mitigate}
  {
    \TableFontSize%
\renewcommand{\arraystretch}{1.2}
\setlength{\dashlinedash}{0.4pt}
\setlength{\dashlinegap}{1.5pt}
\setlength{\arrayrulewidth}{0.3pt}

\newcommand{\MultiHead}[1]{\multicolumn{2}{c}{#1}}

\newcommand{\TwoRowHead}[1]{\multirow{2}{*}{#1}}
\newcommand{\ClassPair}[2]{\multirow{6}{*}{$\text{#1} \rightarrow \text{#2}$}}
\newcommand{\Atk}[1]{\multirow{2}{*}{#1}}
\newcommand{\PZ}{\phantom{0}}
\newcommand{\ptZ}{\phantom{.}\PZ}
\newcommand{\ptZZ}{\phantom{.}\PZ\PZ}

\newcommand{\CosInM}{\method{}}
\newcommand{\LayerM}{\layer{}}

\newcommand{\OneRemB}{100\ptZ}
\newcommand{\ZeroRemB}{0\ptZ}

\newcommand{\PRem}[2]{#1}
\newcommand{\PRemB}[2]{\textBF{#1}}
\newcommand{\ASR}[2]{\multirow{2}{*}{#1}}
\newcommand{\PAcc}[2]{\multirow{2}{*}{#1}}
\newcommand{\PAsr}[2]{#1}

\newcommand{\PChg}[2]{{\color{ForestGreen} +#1}}
\newcommand{\NoChg}{0.0}
\newcommand{\NegChg}[2]{{\color{BrickRed} -#1}}

\newcommand{\DsSep}{\cdashline{2-9}}
\newcommand{\MethodSep}{}

\begin{tabular}{cllrrrrrr}
  \toprule
     Classes
     & \TwoRowHead{Attack} & \TwoRowHead{Method} & \MultiHead{\% Removed}      & \MultiHead{ASR \%} & \MultiHead{Test Acc. \%} \\\cmidrule(lr){1-1}\cmidrule(lr){4-5}\cmidrule(lr){6-7}\cmidrule(lr){8-9}
     $\yTarg \rightarrow \yAdv$
     &                     &                     & $\advTrain$  & $\cleanTrain$ & Orig.   & Ours     &  Orig. & Chg.   \\
  \midrule
  \ClassPair{Auto}{Dog}
     & \Atk{1~Pixel}
     & \CosInM    & \PRem{92.6}{4.2}            & \PRem{0.28}{0.09}   & \PAcc{87.7}{0.8}   & \PAsr{0}{0}   & \PAcc{98.8}{0.1}  & \NoChg \\
     && \LayerM   & \PRemB{94.4}{5.5}           & \PRemB{0.08}{0.05}  &                    & \PAsr{0}{0}   &                   & \NoChg \\
  \DsSep
     & \Atk{4~Pixel}
     & \CosInM    & \PRemB{92.8}{10.3}          & \PRem{0.26}{0.13}   & \PAcc{95.0}{0.6}   & \PAsr{0}{0}   & \PAcc{98.9}{0.1}  & \NoChg \\
     && \LayerM   & \PRem{92.6}{10.1}           & \PRemB{0.05}{0.03}  &                    & \PAsr{0}{0}   &                   & \NoChg \\
  \DsSep
     & \Atk{Blend}
     & \CosInM    & \PRem{99.9}{0.3}            & \PRem{0.67}{0.27}   & \PAcc{98.6}{0.1}   & \PAsr{0}{0}   & \PAcc{99.0}{0.1}  & \NegChg{-0.1}{0.1}  \\
     && \LayerM   & \PRemB{\OneRemB}{\ZeroRemB} & \PRemB{0.41}{0.65}  &                    & \PAsr{0}{0}   &                   & \NegChg{-0.1}{\ZeroVal}  \\
  \midrule
  \ClassPair{Plane}{Bird}
     & \Atk{1~Pixel}
     & \CosInM    & \PRem{65.8}{12.3}           & \PRemB{0.66}{0.38}  & \PAcc{80.8}{1.3}   & \PAsr{0}{0}   & \PAcc{93.5}{0.2}  & \NegChg{-0.1}{0.0} \\
     && \LayerM   & \PRemB{75.5}{9.7}           & \PRem{0.84}{0.29}   &                    & \PAsr{0}{0}   &                   & \NoChg \\
  \DsSep
     & \Atk{4~Pixel}
     & \CosInM    & \PRem{92.7}{7.4}            & \PRemB{0.39}{0.42}  & \PAcc{89.0}{0.6}   & \PAsr{0}{0}   & \PAcc{93.5}{0.2}  & \NoChg \\
     && \LayerM   & \PRemB{95.2}{7.4}           & \PRem{0.80}{0.03}   &                    & \PAsr{0}{0}   &                   & \NoChg \\
  \DsSep
     & \Atk{Blend}
     & \CosInM    & \PRem{81.9}{12.2}           & \PRemB{1.39}{2.17}  & \PAcc{92.0}{0.9}   & \PAsr{0}{0}   & \PAcc{93.7}{0.2}  & \NegChg{-0.4}{0.3}  \\
     && \LayerM   & \PRemB{95.9}{4.0}           & \PRem{1.55}{1.26}   &                    & \PAsr{0}{0}   &                   & \NegChg{-0.5}{0.3}  \\
  \bottomrule
\end{tabular}
   }
}
\end{table}

\vfill%

\phantom{.}
 
\FloatBarrier
\clearpage
\newpage
\subsubsection{Natural Language Poisoning Full Results}%
\label{sec:App:MoreExps:Pois:Nlp}
\phantom{.}

\vfill%

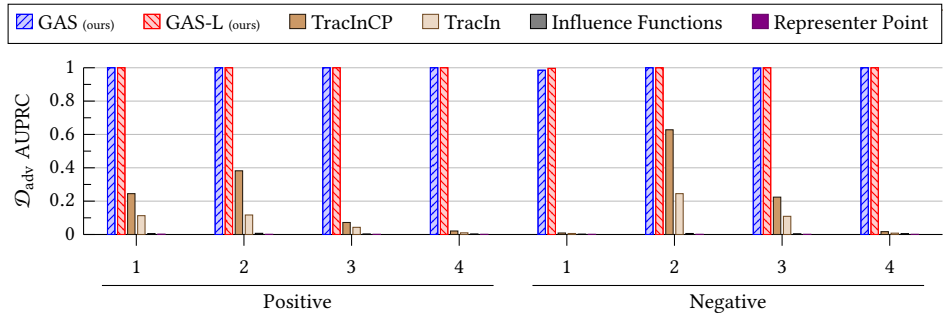
\begin{figure}[h!]
  \centering
\newcommand{\legendSpacer}{\hspace*{9pt}}

\begin{tikzpicture}
  \begin{axis}[%
    width=\textwidth,
      ybar,
      hide axis,  %
      xmin=0,  %
      xmax=1,
      ymin=0,
      ymax=1,
      scale only axis,width=1mm, %
      legend cell align={left},              %
      legend style={font=\legendFontSize},
      legend columns=8,
    ]
    \addplot [cosin color] coordinates {(0,0)};
    \addlegendentry{\method\ours\legendSpacer}

    \addplot [layer color] coordinates {(0,0)};
    \addlegendentry{\layer\ours\legendSpacer}

    \addplot[TracInCP color] coordinates {(0,0)};
    \addlegendentry{\tracinCP{}\legendSpacer}
    \pgfplotsset{cycle list shift=-1}  %

    \pgfplotsinvokeforeach{\tracin{},Influence Functions}{
      \addplot coordinates {(0,0)};
      \addlegendentry{#1\legendSpacer}
    }

    \addplot coordinates {(0,0)};
    \addlegendentry{Representer Point}
  \end{axis}
\end{tikzpicture}

\pgfplotstableread[col sep=comma]{plots/data/pois_nlp_auprc_full.csv}\datatable%
\begin{tikzpicture}%
  \begin{axis}[%
        ybar={\BarLineWidth},%
        width={13cm},%
        height={\NlpIdentificationHeight},%
        axis lines*=left,%
        bar width=\NlpPoisBarWidth,%
        xtick=data,
        xticklabels={${1}$,${2}$,${3}$,${4}$,${1}$,${2}$,${3}$,${4}$},
        x tick label style={font=\plotFontSize,align=center},%
        ymin=0,%
        ymax=1,%
        ytick distance={0.20},%
        minor y tick num={1},%
        y tick label style={font=\plotFontSize},%
        ylabel={\plotFontSize \IdentYLabel},%
        ymajorgrids,  %
        typeset ticklabels with strut,  %
        every tick/.style={color=black, line width=0.4pt},%
        enlarge x limits=0.07,%
        draw group line={[index]1}{1}{Positive}{-5.0ex}{13pt},
        draw group line={[index]1}{2}{Negative}{-5.0ex}{13pt}
    ]%
    \addplot [cosin color] table [x index=0, y index=4] {\datatable};%
    \addplot [layer color] table [x index=0, y index=5] {\datatable};%
    \addplot[TracInCP color] table [x index=0, y index=6] {\datatable};%
    \pgfplotsset{cycle list shift=-1}  %
    \addplot table [x index=0, y index=7] {\datatable};%
    \addplot table [x index=0, y index=8] {\datatable};%
    \addplot table [x index=0, y index=9] {\datatable};%
  \end{axis}%
\end{tikzpicture}%
   \caption{%
      \textit{Natural Language Poisoning Adversarial-Set Identification}:
      See Table~\ref{tab:App:MoreExps:Pois:Nlp:AdvIdent} for the full numerical results.
  }
  \label{fig:App:MoreExps:Pois:NLP:AdvIdent}
\end{figure}

\vfill%

\begin{table}[h!]
  \centering
  \caption{%
      \textit{Natural Language Poisoning Adversarial-Set Identification}:
      Poison identification AUPRC mean and standard deviation across 10~trials for 4~positive and 4~negative
      sentiment SST\=/2 movie reviews \citep{SST2} with \eqsmall{${\abs{\advTrain} = 50}$}.
      \layer{} perfectly identified all poison in all but one trial.
      Bold denotes the best mean performance.
      Mean results are shown graphically in Figures~\ref{fig:ExpRes:AdvIdent} and~\ref{fig:App:MoreExps:Pois:NLP:AdvIdent}.
  }\label{tab:App:MoreExps:Pois:Nlp:AdvIdent}
  {%
    \TableFontSize%
\renewcommand{\arraystretch}{1.2}
\setlength{\dashlinedash}{0.4pt}
\setlength{\dashlinegap}{1.5pt}
\setlength{\arrayrulewidth}{0.3pt}

\newcommand{\RevType}[1]{\multirow{4}{*}{\shortstack{{\Large $\uparrow$} \\ #1 \\ {\Large $\downarrow$}}}}
\newcommand{\RevNum}[1]{#1}

\newcommand{\TabMidRule}{\cdashline{2-8}}

\newcommand{\AlgName}[1]{\multicolumn{1}{c}{#1}}

\begin{tabular}{ccrrrrrr}
  \toprule
  \multicolumn{2}{c}{Review} & \multicolumn{2}{c}{Ours} & \multicolumn{4}{c}{Baselines} \\\cmidrule(r){1-2}\cmidrule(lr){3-4}\cmidrule(l){5-8}
  Sentiment & No.  & \AlgName{\method{}}  & \AlgName{\layer}  & \AlgName{\tracinCP{}} & \AlgName{\tracin{}}  & \AlgName{Influence Func.}  &  {Representer Pt.}  \\
  \midrule
  \RevType{Positive} & \RevNum{1}
          & \PValB{\OneValB}{\ZeroValB}  & \PValB{\OneValB}{\ZeroValB}  & \PVal{0.245}{0.156}  & \PVal{0.113}{0.078}  & \PVal{0.005}{0.005}  & \PVal{0.002}{0.000} \\\TabMidRule
                 & \RevNum{2}
          & \PValB{\OneValB}{\ZeroValB}  & \PValB{\OneValB}{\ZeroValB}  & \PVal{0.382}{0.297}  & \PVal{0.117}{0.084}  & \PVal{0.007}{0.003}  & \PVal{0.001}{0.000} \\\TabMidRule
                 & \RevNum{3}
          & \PValB{\OneValB}{\ZeroValB}  & \PValB{\OneValB}{\ZeroValB}  & \PVal{0.072}{0.048}  & \PVal{0.043}{0.020}  & \PVal{0.003}{0.001}  & \PVal{0.001}{0.000} \\\TabMidRule
                 & \RevNum{4}
          & \PValB{\OneValB}{\ZeroValB}  & \PValB{\OneValB}{\ZeroValB}  & \PVal{0.021}{0.006}  & \PVal{0.010}{0.002}  & \PVal{0.003}{0.002}  & \PVal{0.001}{0.000} \\\midrule
  \RevType{Negative} & \RevNum{1}
          & \PVal{0.985}{0.046}          & \PValB{0.996}{0.012}         & \PVal{0.009}{0.003}  & \PVal{0.006}{0.001}  & \PVal{0.002}{0.001}  & \PVal{0.001}{0.000} \\\TabMidRule
                 & \RevNum{2}
          & \PValB{\OneValB}{\ZeroValB}  & \PValB{\OneValB}{\ZeroValB}  & \PVal{0.628}{0.165}  & \PVal{0.245}{0.069}  & \PVal{0.006}{0.004}  & \PVal{0.001}{0.000} \\\TabMidRule
                 & \RevNum{3}
          & \PVal{0.998}{0.005}          & \PValB{\OneValB}{\ZeroValB}  & \PVal{0.224}{0.112}  & \PVal{0.109}{0.051}  & \PVal{0.004}{0.003}  & \PVal{0.001}{0.002}    \\\TabMidRule
                 & \RevNum{4}
          & \PValB{\OneValB}{\ZeroValB}  & \PValB{\OneValB}{\ZeroValB}  & \PVal{0.017}{0.003}  & \PVal{0.008}{0.001}  & \PVal{0.005}{0.002}  & \PVal{0.001}{0.000} \\
  \bottomrule
\end{tabular}
   }
\end{table}

\vfill%

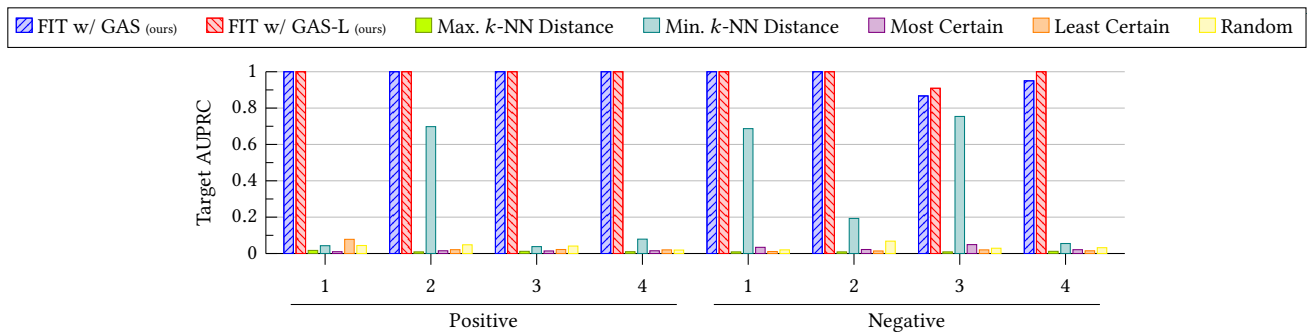
\begin{figure}[h!]
  \centering
\newcommand{\legendSpacer}{\hspace*{9pt}}

\begin{tikzpicture}
  \begin{axis}[%
    width=\textwidth,
      ybar,
      hide axis,  %
      xmin=0,  %
      xmax=1,
      ymin=0,
      ymax=1,
      scale only axis,width=1mm, %
      legend cell align={left},              %
      legend style={font=\legendFontSize},
      legend columns=7,
      cycle list name=DetectCycleList,
    ]
    \addplot[cosin color] coordinates {(0,0)};
    \addlegendentry{\fitWith{\method}\ours\legendSpacer}
    \addplot[layer color] coordinates {(0,0)};
    \addlegendentry{\fitWith{\layer}\ours\legendSpacer}

    \addplot coordinates {(0,0)};
    \addlegendentry{Max.\ \KNN{} Distance\legendSpacer}

    \addplot coordinates {(0,0)};
    \addlegendentry{Min.\ \KNN{} Distance\legendSpacer}
    \addplot coordinates {(0,0)};
    \addlegendentry{Most Certain\legendSpacer};

    \addplot coordinates {(0,0)};
    \addlegendentry{Least Certain\legendSpacer};

    \addplot coordinates {(0,0)};
    \addlegendentry{Random};
  \end{axis}
\end{tikzpicture}

\pgfplotstableread[col sep=comma]{plots/data/pois_nlp_detect_full.csv}\datatable%
\begin{tikzpicture}%
  \begin{axis}[%
        ybar={\BarLineWidth},%
        height={\BarDetectMainHeight},%
        width={13cm},%
        axis lines*=left,%
        bar width={\DetectBarWidthVal},%
        xtick=data,%
        xticklabels={${1}$,${2}$,${3}$,${4}$,${1}$,${2}$,${3}$,${4}$},
        x tick label style={font=\plotFontSize,align=center},%
        ymin=0,%
        ymax=1,%
        ytick distance={0.20},%
        minor y tick num={1},%
        y tick label style={font=\plotFontSize},%
        ylabel={\plotFontSize \DetectYLabel},%
        ymajorgrids,  %
        typeset ticklabels with strut,  %
        every tick/.style={color=black, line width=0.4pt},%
        enlarge x limits={0.080},%
        draw group line={[index]8}{1}{Positive}{-5.0ex}{13pt},
        draw group line={[index]8}{2}{Negative}{-5.0ex}{13pt},
        cycle list name=DetectCycleList,
    ]%
    \addplot[cosin color] table [x index=0, y index=1] {\datatable};%
    \addplot[layer color] table [x index=0, y index=2] {\datatable};%
    \foreach \k in {3, ..., 7} {%
      \addplot table [x index=0, y index=\k] {\datatable};%
    }%
  \end{axis}%
\end{tikzpicture}%
   \caption{%
      \textit{Natural Language Poisoning Target Identification}:
      See Table~\ref{tab:App:MoreExps:Pois:NLP:TargDetect} for the full numerical results.
  }
  \label{fig:App:MoreExps:Pois:NLP:TargDetect}
\end{figure}

\FloatBarrier
\clearpage
\newpage

\phantom{.}
\vfill{}

\begin{table}[h!]
  \centering
  \caption{%
      \textit{Natural Language Poisoning Target Identification}:
      Bold denotes the best mean performance with ${10}$~trials per review.
      Mean results are shown graphically in Figures~\ref{fig:ExpRes:TargDetect} and~\ref{fig:App:MoreExps:Pois:NLP:TargDetect}.
  }
  \label{tab:App:MoreExps:Pois:NLP:TargDetect}
  {%
    \TableFontSize%
\renewcommand{\arraystretch}{1.2}
\setlength{\dashlinedash}{0.4pt}
\setlength{\dashlinegap}{1.5pt}
\setlength{\arrayrulewidth}{0.3pt}

\newcommand{\RevType}[1]{\multirow{4}{*}{\shortstack{{\Large $\uparrow$} \\ #1 \\ {\Large $\downarrow$}}}}
\newcommand{\RevNum}[1]{#1}

\newcommand{\TabMidRule}{\cdashline{2-9}}

\newcommand{\AlgName}[1]{\multicolumn{1}{c}{#1}}

\begin{tabular}{ccrrrrrrr}
  \toprule
  \multicolumn{2}{c}{Review} & \multicolumn{2}{c}{Ours} & \multicolumn{5}{c}{Baselines} \\\cmidrule(r){1-2}\cmidrule(lr){3-4}\cmidrule(l){5-9}
  Sentiment & No.  & \AlgName{\method{}}  & \AlgName{\layer}  & \AlgName{Max~\KNN{}}  & \AlgName{Min~\KNN{}}  & \AlgName{Most Certain}  & \AlgName{Least Certain} & \AlgName{Random}   \\
  \midrule
  \RevType{Positive} & \RevNum{1}
      & \PValB{\OneVal}{\ZeroVal} & \PValB{\OneVal}{\ZeroVal} & \PVal{0.017}{0.011}  & \PVal{0.043}{0.044}  & \PVal{0.010}{0.001}  & \PVal{0.078}{0.030}   & \PVal{0.044}{0.062}  \\\TabMidRule
                  & \RevNum{2}
      & \PValB{\OneVal}{\ZeroVal} & \PValB{\OneVal}{\ZeroVal} & \PVal{0.009}{0.000}  & \PVal{0.698}{0.404}  & \PVal{0.015}{0.002}  & \PVal{0.021}{0.003}   & \PVal{0.048}{0.060}  \\\TabMidRule
                  & \RevNum{3}
      & \PValB{\OneVal}{\ZeroVal} & \PValB{\OneVal}{\ZeroVal} & \PVal{0.012}{0.002}  & \PVal{0.038}{0.017}  & \PVal{0.014}{0.001}  & \PVal{0.022}{0.004}   & \PVal{0.041}{0.075}  \\\TabMidRule
                  & \RevNum{4}
      & \PValB{\OneVal}{\ZeroVal} & \PValB{\OneVal}{\ZeroVal} & \PVal{0.010}{0.001}  & \PVal{0.079}{0.046}  & \PVal{0.015}{0.002}  & \PVal{0.020}{0.003}   & \PVal{0.019}{0.011}  \\\midrule
  \RevType{Negative} & \RevNum{1}
      & \PValB{\OneVal}{\ZeroVal} & \PValB{\OneVal}{\ZeroVal} & \PVal{0.009}{0.000}  & \PVal{0.687}{0.350}  & \PVal{0.034}{0.008}  & \PVal{0.011}{0.001}   & \PVal{0.020}{0.014}  \\\TabMidRule
                  & \RevNum{2}
      & \PValB{\OneVal}{\ZeroVal} & \PValB{\OneVal}{\ZeroVal} & \PVal{0.009}{0.001}  & \PVal{0.193}{0.286}  & \PVal{0.022}{0.004}  & \PVal{0.014}{0.002}   & \PVal{0.068}{0.069}  \\\TabMidRule
                 & \RevNum{3}
      & \PVal{0.867}{0.292}       & \PValB{0.909}{0.287}      & \PVal{0.009}{0.000}  & \PVal{0.754}{0.401}  & \PVal{0.049}{0.039}  & \PVal{0.020}{0.028}   & \PVal{0.029}{0.022}  \\\TabMidRule
                 & \RevNum{4}
      & \PVal{0.950}{0.158}       & \PValB{\OneVal}{\ZeroVal} & \PVal{0.012}{0.003}  & \PVal{0.055}{0.037}  & \PVal{0.021}{0.005}  & \PVal{0.015}{0.002}   & \PVal{0.032}{0.020}  \\
  \bottomrule
\end{tabular}
   }
\end{table}

\vfill%

\begin{table}[h!]
  \centering
\revised{%
  \caption{%
    \revised{%
      \textit{Natural Language Poisoning Attack Mitigation}:
      Bold denotes the best mean performance with ${10}$~trials per review.
      Aggregated results are shown in Table~\ref{tab:ExpRes:Mitigation}.
    }
  }\label{tab:App:MoreExps:Pois:NLP:Mitigate}
  {
    \TableFontSize%
\renewcommand{\arraystretch}{1.2}
\setlength{\dashlinedash}{0.4pt}
\setlength{\dashlinegap}{1.5pt}
\setlength{\arrayrulewidth}{0.3pt}

\newcommand{\MultiHead}[1]{\multicolumn{2}{c}{#1}}

\newcommand{\TwoRowHead}[1]{\multirow{2}{*}{#1}}
\newcommand{\RevType}[1]{\multirow{8}{*}{\shortstack{{\LARGE $\uparrow$} \\~\\~\\ #1 \\~\\~\\ {\LARGE $\downarrow$}}}}
\newcommand{\RevID}[1]{\multirow{2}{*}{#1}}
\newcommand{\PZ}{\phantom{0}}
\newcommand{\ptZ}{\phantom{.}\PZ}
\newcommand{\ptZZ}{\phantom{.}\PZ\PZ}

\newcommand{\CosInM}{\method{}}
\newcommand{\LayerM}{\layer{}}

\newcommand{\OneRemB}{100\ptZ}
\newcommand{\ZeroRemB}{0\ptZ}
\newcommand{\ZeroRemBB}{0\ptZZ}

\newcommand{\PRem}[2]{#1}
\newcommand{\PRemB}[2]{\textBF{#1}}
\newcommand{\ASR}[2]{\multirow{2}{*}{#1}}
\newcommand{\PAcc}[2]{\multirow{2}{*}{#1}}
\newcommand{\PAsr}[2]{#1}

\newcommand{\PChg}[2]{{\color{ForestGreen} +#1}}
\newcommand{\NoChg}{0.0}
\newcommand{\NegChg}[2]{{\color{BrickRed} -#1}}

\newcommand{\DsSep}{\cdashline{2-9}}
\newcommand{\MethodSep}{}

\begin{tabular}{cllrrrrrr}
  \toprule
    \MultiHead{Review}
     & \TwoRowHead{Method} & \MultiHead{\% Removed}      & \MultiHead{ASR \%} & \MultiHead{Test Acc. \%} \\\cmidrule(lr){1-2}\cmidrule(lr){4-5}\cmidrule(lr){6-7}\cmidrule(lr){8-9}
     Sentiment
     & No.                 &                     & $\advTrain$  & $\cleanTrain$ & Orig.   & Ours     &  Orig. & Chg.   \\
  \midrule
  \RevType{Positive}
     & \RevID{1}
     & \CosInM    & \PRemB{\OneRemB}{\ZeroRemB} & \PRemB{0.01}{0.01}             & \PAcc{100}{}       & \PAsr{0}{0}   & \PAcc{94.1}{0.3}  & \NoChg \\
     && \LayerM   & \PRemB{\OneRemB}{\ZeroRemB} & \PRemB{0.01}{0.01}             &                    & \PAsr{0}{0}   &                   & \PChg{0.3}{0.1} \\
  \DsSep
     & \RevID{2}
     & \CosInM    & \PRemB{\OneRemB}{\ZeroRemB} & \PRemB{0.01}{0.02}             & \PAcc{100}{}       & \PAsr{0}{0}   & \PAcc{94.2}{0.1}  & \PChg{0.1}{0.0} \\
     && \LayerM   & \PRemB{\OneRemB}{\ZeroRemB} & \PRem{0.02}{0.03}              &                    & \PAsr{0}{0}   &                   & \PChg{0.2}{0.0} \\
  \DsSep
     & \RevID{3}
     & \CosInM    & \PRemB{\OneRemB}{\ZeroRemB} & \PRem{0.01}{0.01}              & \PAcc{100}{}       & \PAsr{0}{0}   & \PAcc{94.2}{0.1}  & \PChg{0.1}{0.1} \\
     && \LayerM   & \PRemB{\OneRemB}{\ZeroRemB} & \PRemB{0.00}{0.01}             &                    & \PAsr{0}{0}   &                   & \NoChg \\
  \DsSep
     & \RevID{4}
     & \CosInM    & \PRem{99.9}{0.3}            & \PRemB{\ZeroRemBB}{\ZeroRemBB} & \PAcc{100}{}       & \PAsr{0}{0}   & \PAcc{94.3}{0.1}  & \NegChg{-0.1}{0.1} \\
     && \LayerM   & \PRemB{\OneRemB}{\ZeroRemB} & \PRem{0.02}{0.01}              &                    & \PAsr{0}{0}   &                   & \NoChg \\
  \midrule
  \RevType{Negative}
     & \RevID{1}
     & \CosInM    & \PRem{97.1}{6.5}            & \PRemB{0.01}{0.01}             & \PAcc{100}{}       & \PAsr{0}{0}   & \PAcc{94.3}{0.2}  & \PChg{0.3}{0.1} \\
     && \LayerM   & \PRemB{99.5}{1.7}           & \PRem{0.02}{0.05}              &                    & \PAsr{0}{0}   &                   & \PChg{0.1}{0.1} \\
  \DsSep
     & \RevID{2}
     & \CosInM    & \PRemB{\OneRemB}{\ZeroRemB} & \PRem{0.17}{0.30}              & \PAcc{100}{}       & \PAsr{0}{0}   & \PAcc{94.3}{0.2}  & \PChg{0.1}{0.2} \\
     && \LayerM   & \PRemB{\OneRemB}{\ZeroRemB} & \PRemB{0.04}{0.03}             &                    & \PAsr{0}{0}   &                   & \PChg{0.3}{0.2} \\
  \DsSep
     & \RevID{3}
     & \CosInM    & \PRem{99.5}{0.0}            & \PRem{0.05}{0.15}              & \PAcc{90}{}        & \PAsr{0}{0}   & \PAcc{94.1}{0.2}  & \PChg{0.3}{0.1} \\
     && \LayerM   & \PRemB{\OneRemB}{\ZeroRemB} & \PRemB{0.01}{0.01}             &                    & \PAsr{0}{0}   &                   & \PChg{0.4}{0.1} \\
  \DsSep
     & \RevID{4}
     & \CosInM    & \PRemB{\OneRemB}{\ZeroRemB} & \PRemB{\ZeroRemBB}{\ZeroRemBB} & \PAcc{100}{}       & \PAsr{0}{0}   & \PAcc{94.2}{0.2}  & \PChg{0.1}{0.0} \\
     && \LayerM   & \PRemB{\OneRemB}{\ZeroRemB} & \PRemB{\ZeroRemBB}{\ZeroRemBB} &                    & \PAsr{0}{0}   &                   & \PChg{0.1}{0.1} \\
  \bottomrule
\end{tabular}
   }
}
\end{table}

\vfill{}
\phantom{.}
 
\FloatBarrier
\clearpage
\newpage
\subsubsection{Vision Poisoning Full Results}%
\label{sec:App:MoreExps:Poison:CIFAR}
\phantom{.}

\vspace{8pt}
Section~\ref{sec:ExpRes:AdvIdent} considers \citepos{Peri:2020} dedicated, clean\=/label poison defense \deepknn{} as an additional baseline.  By default, nearest neighbor algorithms yield a label, not a score.  To be compatible with~AUPRC, we modified \deepknn{} to rank each training example by the difference between the size of the neighborhood's plurality class and the number of neighborhood instances that share the corresponding example's label.

\vfill%

\begin{figure}[h!]
  \centering
\newcommand{\legendSpacer}{\hspace*{8pt}}
\begin{tikzpicture}
  \begin{axis}[%
      width=\textwidth,
      ybar,
      hide axis,  %
      xmin=0,  %
      xmax=1,
      ymin=0,
      ymax=1,
      scale only axis,width=1mm, %
      legend cell align={left},              %
      legend style={font=\legendFontSize},
      legend columns=9,
    ]
    \addplot [cosin0 color] coordinates {(0,0)};
    \addlegendentry{\methodZero\ours\legendSpacer}

    \addplot [layer0 color] coordinates {(0,0)};
    \addlegendentry{\layerZero\ours\legendSpacer}

    \pgfplotsset{cycle list shift=-2}  %
    \addplot[cosin color] coordinates {(0,0)};
    \addlegendentry{\method\ours\legendSpacer}

    \addplot[layer color] coordinates {(0,0)};
    \addlegendentry{\layer\ours\legendSpacer}

    \addplot[TracInCP color] coordinates {(0,0)};
    \addlegendentry{\tracinCP{}\legendSpacer}
    \pgfplotsset{cycle list shift=-3}  %

    \pgfplotsinvokeforeach{\tracin{},Influence Func.,Representer Pt.}{
      \addplot coordinates {(0,0)};
      \addlegendentry{#1\legendSpacer}
    }

    \addplot[deep knn color] coordinates {(0,0)};
    \addlegendentry{\deepknn}
  \end{axis}
\end{tikzpicture}

\pgfplotstableread[col sep=comma]{plots/data/pois_cifar_auprc_full.csv}\datatable%
\begin{tikzpicture}%
  \begin{axis}[%
        axis lines*=left,%
        ymajorgrids,  %
        bar width={\CifarPoisBarWidth},%
        height={\CifarIdentificationHeight},%
        width={9.5cm},%
        cycle list/Set2,
        xtick=data,%
        xticklabels={Bird $\rightarrow$ Dog, Dog $\rightarrow$ Bird, Frog $\rightarrow$ Deer, Deer $\rightarrow$ Frog},
        x tick label style={font=\plotFontSize,align=center},%
        ymin=0,%
        ymax=1.0,%
        ybar={\BarLineWidth},%
        ytick distance={0.2},%
        minor y tick num={1},%
        y tick label style={font=\plotFontSize},%
        ylabel={\plotFontSize \IdentYLabel},%
        typeset ticklabels with strut,
        every tick/.style={color=black, line width=0.4pt},%
        enlarge x limits=0.17,%
        draw empty group line={[index]10}{1}{Positive}{-4.0ex}{4pt},
        draw empty group line={[index]10}{2}{Negative}{-4.0ex}{4pt}
    ]%
    \addplot [cosin0 color] table [x index=0, y index=1] {\datatable};%
    \addplot [layer0 color] table [x index=0, y index=2] {\datatable};%
    \pgfplotsset{cycle list shift=-2}  %
    \addplot [cosin color] table [x index=0, y index=3] {\datatable};%
    \addplot [layer color] table [x index=0, y index=4] {\datatable};%
    \addplot[TracInCP color] table [x index=0, y index=5] {\datatable};%
    \pgfplotsset{cycle list shift=-3}  %
    \addplot table [x index=0, y index=6] {\datatable};%
    \addplot table [x index=0, y index=7] {\datatable};%
    \addplot table [x index=0, y index=8] {\datatable};%
    \addplot[deep knn color] table [x index=0, y index=9] {\datatable};%
  \end{axis}%
\end{tikzpicture}%
   \caption{%
      \textit{Vision Poisoning Adversarial-Set Identification}:
      Adversarial set~($\advTrain$) identification AUPRC mean and standard deviation across ${{>}15}$~trials for four CIFAR10 class pairs
      with ${\abs{\advTrain} = 50}$. Our renormalized influence estimators, \method{} and \layer{}, using just initial parameters~$\wZero$
      and with 5~subepoch checkpointing outperformed all baselines for all class pairs.
  }
  \label{fig:App:MoreExps:Pois:CIFAR:AdvIdent}
\end{figure}
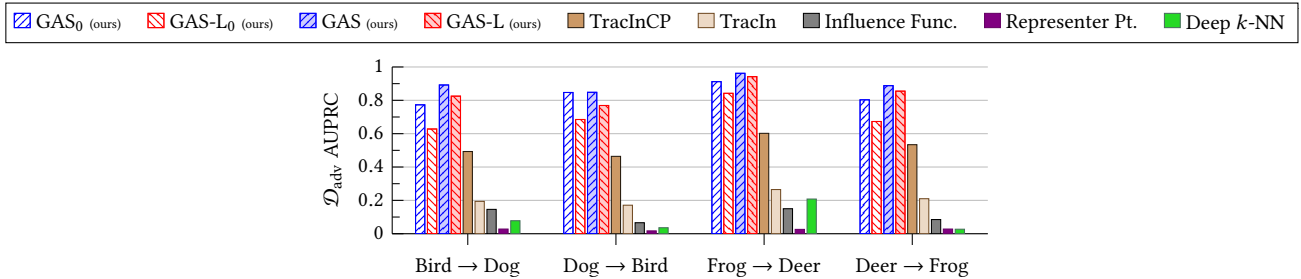

\vfill%

\begin{table}[h!]
  \centering
  \caption{
      \textit{Vision Poisoning Adversarial-Set Identification}:
      Adversarial set~($\advTrain$) identification AUPRC mean and standard deviation across ${{>}15}$~trials for four CIFAR10 class pairs
      with ${\abs{\advTrain} = 50}$. Our renormalized influence estimators, \method{} and \layer{}, using just initial parameters~$\wZero$
      and with 5~subepoch checkpointing outperformed all baselines for all class pairs.
      Bold denotes the best mean performance.
      Mean results are shown graphically in Figure~\ref{fig:ExpRes:AdvIdent} and~\ref{fig:App:MoreExps:Pois:CIFAR:AdvIdent}.
  }
  \label{tab:App:MoreExps:Pois:CIFAR:AdvIdent}
  {%
    \TableFontSize%
\renewcommand{\arraystretch}{1.2}
\setlength{\dashlinedash}{0.4pt}
\setlength{\dashlinegap}{1.5pt}
\setlength{\arrayrulewidth}{0.3pt}

\newcommand{\TabMidRule}{\hdashline}

\newcommand{\AlgName}[1]{\multicolumn{1}{c}{#1}}

\begin{tabular}{@{}crrrrrrrrr@{}}
  \toprule
  \multicolumn{1}{c}{Classes} & \multicolumn{4}{c}{Ours} & \multicolumn{5}{c}{Baselines} \\\cmidrule(r){1-1}\cmidrule(lr){2-5}\cmidrule(l){6-10}
  $\yTarg \rightarrow \yAdv$  & \AlgName{\methodZero} & \AlgName{\layerZero} & \AlgName{\method}  & \AlgName{\layer}  & \AlgName{\tracinCP{}}  & \AlgName{\tracin{}}  & \AlgName{Influence Func.}  &  {Representer Pt.}  & \AlgName{\deepknn{}}  \\
  \midrule
  Bird $\rightarrow$ Dog
      & \PVal{0.773}{0.208}   & \PVal{0.628}{0.242}   & \PValB{0.892}{0.137}   & \PVal{0.825}{0.206}   & \PVal{0.493}{0.233}   & \PVal{0.194}{0.108}   & \PVal{0.146}{0.188}   & \PVal{0.028}{0.015}   & \PVal{0.078}{0.197}   \\\TabMidRule
  Dog $\rightarrow$ Bird
      & \PVal{0.847}{0.142}   & \PVal{0.685}{0.179}   & \PValB{0.848}{0.115}   & \PVal{0.769}{0.170}   & \PVal{0.464}{0.225}   & \PVal{0.171}{0.090}   & \PVal{0.066}{0.075}   & \PVal{0.017}{0.007}   & \PVal{0.036}{0.027}   \\\TabMidRule
  Frog $\rightarrow$ Deer
      & \PVal{0.912}{0.120}   & \PVal{0.842}{0.173}   & \PValB{0.962}{0.100}   & \PVal{0.942}{0.127}   & \PVal{0.602}{0.203}   & \PVal{0.265}{0.135}   & \PVal{0.150}{0.166}   & \PVal{0.026}{0.016}   & \PVal{0.208}{0.320}   \\\TabMidRule
  Deer $\rightarrow$ Frog
      & \PVal{0.803}{0.188}   & \PVal{0.673}{0.202}   & \PValB{0.888}{0.091}   & \PVal{0.855}{0.113}   & \PVal{0.534}{0.197}   & \PVal{0.210}{0.101}   & \PVal{0.085}{0.107}   & \PVal{0.028}{0.025}   & \PVal{0.027}{0.072}   \\
  \bottomrule
\end{tabular}
   }
\end{table}

\vfill%
\phantom{.}

\FloatBarrier
\clearpage
\newpage

\begin{figure}[h!]
  \centering
\newcommand{\legendSpacer}{\hspace*{9pt}}

\begin{tikzpicture}
  \begin{axis}[%
    width=\textwidth,
      ybar,
      hide axis,  %
      xmin=0,  %
      xmax=1,
      ymin=0,
      ymax=1,
      scale only axis,width=1mm, %
      legend cell align={left},              %
      legend style={font=\legendFontSize},
      legend columns=7,
      cycle list name=DetectCycleList,
    ]
    \addplot[cosin color] coordinates {(0,0)};
    \addlegendentry{\fitWith{\method}\ours\legendSpacer}
    \addplot[layer color] coordinates {(0,0)};
    \addlegendentry{\fitWith{\layer}\ours\legendSpacer}

    \addplot coordinates {(0,0)};
    \addlegendentry{Max.\ \KNN{} Distance\legendSpacer}

    \addplot coordinates {(0,0)};
    \addlegendentry{Min.\ \KNN{} Distance\legendSpacer}
    \addplot coordinates {(0,0)};
    \addlegendentry{Most Certain\legendSpacer};

    \addplot coordinates {(0,0)};
    \addlegendentry{Least Certain\legendSpacer};

    \addplot coordinates {(0,0)};
    \addlegendentry{Random};
  \end{axis}
\end{tikzpicture}

\pgfplotstableread[col sep=comma]{plots/data/pois_cifar_detect_full.csv}\datatable%
\begin{tikzpicture}%
  \begin{axis}[%
        axis lines*=left,%
        ymajorgrids,  %
        bar width={\DetectBarWidthVal},
        height={\BarDetectMainHeight},%
        width={9.0cm},%
        ymin=0,%
        ymax=1,%
        ytick distance={0.20},%
        ybar={\BarLineWidth},%
        minor y tick num={1},%
        y tick label style={font=\plotFontSize},%
        ylabel={\plotFontSize \DetectYLabel},%
        xtick=data,%
        xticklabels={Bird $\rightarrow$ Dog, Dog $\rightarrow$ Bird, Frog $\rightarrow$ Deer, Deer $\rightarrow$ Frog},
        x tick label style={font=\plotFontSize,align=center},%
        typeset ticklabels with strut,
        every tick/.style={color=black, line width=0.4pt},%
        enlarge x limits={0.150},%
        cycle list name=DetectCycleList,
    ]%
    \addplot[cosin color] table [x index=0, y index=1] {\datatable};%
    \addplot[layer color] table [x index=0, y index=2] {\datatable};%
    \foreach \k in {3, ..., 7} {%
      \addplot table [x index=0, y index=\k] {\datatable};%
    }%
  \end{axis}%
\end{tikzpicture}%
   \caption{%
      \textit{Vision Poisoning Target Identification}:
      See Table~\ref{tab:App:MoreExps:Pois:CIFAR:TargDetect} for the full numerical results.
  }
  \label{fig:App:MoreExps:Pois:CIFAR:TargDetect}
\end{figure}
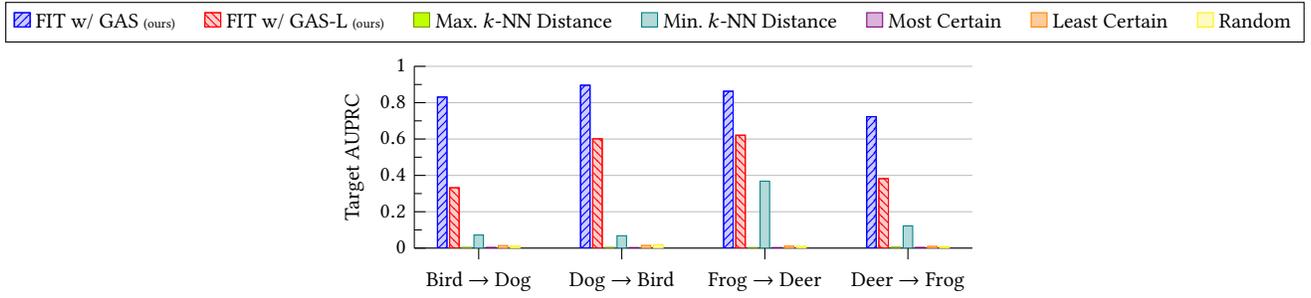

\vfill%

\begin{table}[h!]
  \centering
  \caption{%
      \textit{Vision Poisoning Target Identification}:
      Bold denotes the best mean performance with ${{\geq}15}$~trials per class pair.
      Mean results are shown graphically in Figures~\ref{fig:ExpRes:TargDetect} and~\ref{fig:App:MoreExps:Pois:CIFAR:TargDetect}.
  }
  \label{tab:App:MoreExps:Pois:CIFAR:TargDetect}
  {%
    \TableFontSize%
\renewcommand{\arraystretch}{1.2}
\setlength{\dashlinedash}{0.4pt}
\setlength{\dashlinegap}{1.5pt}
\setlength{\arrayrulewidth}{0.3pt}

\newcommand{\TabMidRule}{\hdashline}

\newcommand{\AlgName}[1]{\multicolumn{1}{c}{#1}}

\begin{tabular}{ccrrrrrrr}
  \toprule
  \multicolumn{2}{c}{Classes} & \multicolumn{2}{c}{Ours} & \multicolumn{4}{c}{Baselines} \\\cmidrule(lr){1-2}\cmidrule(lr){3-4}\cmidrule(lr){5-9}
  $\yTarg$ & $\yAdv$
      & \AlgName{\method{}}   & \AlgName{\layer}    & \AlgName{Max~\KNN{}} & \AlgName{Min~\KNN{}}  & \AlgName{Most Certain} & \AlgName{Least Certain} & \AlgName{Random}   \\
  \midrule
  Bird  & Dog
      & \PValB{0.831}{0.268}  & \PVal{0.332}{0.340}  & \PVal{0.004}{0.006}  & \PVal{0.072}{0.128}  & \PVal{0.004}{0.001}  & \PVal{0.014}{0.034}  & \PVal{0.011}{0.025}  \\\TabMidRule
  Dog   & Bird
      & \PValB{0.896}{0.227}  & \PVal{0.601}{0.390}  & \PVal{0.004}{0.002}  & \PVal{0.068}{0.129}  & \PVal{0.003}{0.001}  & \PVal{0.015}{0.013}  & \PVal{0.008}{0.010}  \\\TabMidRule
  Frog  & Deer
      & \PValB{0.863}{0.253}  & \PVal{0.621}{0.393}  & \PVal{0.003}{0.001}  & \PVal{0.368}{0.434}  & \PVal{0.003}{0.001}  & \PVal{0.011}{0.005}  & \PVal{0.019}{0.043}  \\\TabMidRule
  Deer  & Frog
      & \PValB{0.723}{0.305}  & \PVal{0.382}{0.327}  & \PVal{0.005}{0.007}  & \PVal{0.122}{0.273}  & \PVal{0.004}{0.001}  & \PVal{0.010}{0.008}  & \PVal{0.010}{0.026}  \\
  \bottomrule
\end{tabular}
   }
\end{table}

\vfill%

\begin{table}[h!]
  \centering
\revised{%
  \caption{%
    \revised{%
      \textit{Vision Poisoning Attack Mitigation}:
      Bold denotes the best mean performance with ${{\geq}15}$~trials per class pair.
      Aggregated results are shown in Table~\ref{tab:ExpRes:Mitigation}.
    }
  }\label{tab:App:MoreExps:Pois:CIFAR:Mitigate}
  {
    \TableFontSize%
\renewcommand{\arraystretch}{1.2}
\setlength{\dashlinedash}{0.4pt}
\setlength{\dashlinegap}{1.5pt}
\setlength{\arrayrulewidth}{0.3pt}

\newcommand{\MultiHead}[1]{\multicolumn{2}{c}{#1}}

\newcommand{\TwoRowHead}[1]{\multirow{2}{*}{#1}}
\newcommand{\ClassPair}[2]{\multirow{2}{*}{#1} & \multirow{2}{*}{#2}}
\newcommand{\PZ}{\phantom{0}}
\newcommand{\ptZ}{\phantom{.}\PZ}
\newcommand{\ptZZ}{\phantom{.}\PZ\PZ}
\newcommand{\ptASR}{}

\newcommand{\CosInM}{\method{}}
\newcommand{\LayerM}{\layer{}}

\newcommand{\OneRemB}{100\ptZ}
\newcommand{\ZeroRemB}{0\ptZ}

\newcommand{\PRem}[2]{#1}
\newcommand{\PRemB}[2]{\textBF{#1}}
\newcommand{\ASR}[2]{\multirow{2}{*}{#1}}
\newcommand{\PAcc}[2]{\multirow{2}{*}{#1}}
\newcommand{\PAsr}[2]{#1}

\newcommand{\PChg}[2]{{\color{ForestGreen} +#1}}
\newcommand{\NoChg}{0.0}
\newcommand{\NegChg}[2]{{\color{BrickRed} -#1}}

\newcommand{\DsSep}{\cdashline{1-9}}
\newcommand{\MethodSep}{}

\begin{tabular}{cclrrrrrr}
  \toprule
     \MultiHead{Classes}
     & \TwoRowHead{Method} & \MultiHead{\% Removed}      & \MultiHead{ASR \%} & \MultiHead{Test Acc. \%} \\\cmidrule(lr){1-2}\cmidrule(lr){4-5}\cmidrule(lr){6-7}\cmidrule(lr){8-9}
     $\yTarg$ & $\yAdv$
     &                    & $\advTrain$  & $\cleanTrain$ & Orig.   & Ours     &  Orig. & Chg.   \\
  \midrule
  \ClassPair{Bird}{Dog}
     & \CosInM    & \PRemB{72.1}{27.2} & \PRemB{0.04}{0.09}  & \PAcc{91.4}{}  & \PAsr{0}{0}   & \PAcc{87.0}{0.3}  & \NoChg \\
     && \LayerM   & \PRem{65.2}{32.7}  & \PRemB{0.04}{0.09}  &                & \PAsr{0}{0}   &                   & \PChg{0.1}{0.0} \\
  \DsSep
  \ClassPair{Dog}{Bird}
     & \CosInM    & \PRemB{54.1}{27.8} & \PRemB{0.01}{0.03}  & \PAcc{80.0}{}  & \PAsr{0}{0}   & \PAcc{87.1}{0.3}  & \NegChg{-0.1}{0.0} \\
     && \LayerM   & \PRem{46.6}{26.7}  & \PRemB{0.01}{0.05}  &                & \PAsr{0}{0}   &                   & \NoChg \\
  \DsSep
  \ClassPair{Frog}{Deer}
     & \CosInM    & \PRemB{36.0}{24.4} & \PRemB{0.01}{0.03}  & \PAcc{60.0}{}   & \PAsr{0}{0}   & \PAcc{87.1}{0.4}  & \NoChg \\
     && \LayerM   & \PRem{30.2}{24.6}  & \PRem{0.03}{0.15}   &                 & \PAsr{0}{0}   &                   & \NoChg \\
  \DsSep
  \ClassPair{Deer}{Frog}
     & \CosInM    & \PRemB{88.9}{14.2} & \PRemB{0.03}{0.05}  & \PAcc{80.0}{}   & \PAsr{0}{0}   & \PAcc{87.0}{0.3}  & \PChg{0.1}{0.0}  \\
     && \LayerM   & \PRem{83.5}{19.5}  & \PRemB{0.03}{0.04}  &                 & \PAsr{0}{0}   &                   & \PChg{0.1}{-0.1} \\
  \bottomrule
\end{tabular}
   }
}
\end{table}
  
\FloatBarrier
\clearpage
\newpage
\revTwo{%
\subsection{Jointly-Optimized Adaptive Attacker -- Full Experimental Results}
\label{sec:App:MoreExps:JointOpt:FullResults}
}%

\revTwo{%
  Sections~\ref{sec:ExpRes:AdaptiveAttacks} and~\ref{sec:App:JointOpt:Setup} describe a strong adaptive attack that modifies \citepos{Zhu:2019}'s vision poisoning attack
to simultaneously minimize the adversarial loss and the adversarial set's estimated influence.
Section~\ref{sec:ExpRes:AdaptiveAttacks} summarizes the adversarial-set and target identification results for this jointly-optimized attack.
Sections~\ref{sec:App:MoreExps:JointOpt:FullResults:AdvIdent}
and~\ref{sec:App:MoreExps:JointOpt:FullResults:TargDetect} (resp.) provide more granular versions of those results.

Section~\ref{sec:App:MoreExps:JointOpt:FullResults:Mitigation} provides additional results on target-driven attack
mitigation's effectiveness on this jointly optimized attack.
}

\vspace{16pt}
\subsubsection{Adversarial-Set Identification of the Jointly Optimized Poisoning Attack}%
\label{sec:App:MoreExps:JointOpt:FullResults:AdvIdent}

\phantom{.}
\vfill%

\begin{figure}[h!]
  \centering
  \newcommand{\SubFigureWidth}{0.46\textwidth}
\newcommand{\legendSpacer}{\hspace*{8pt}}
\begin{tikzpicture}
  \begin{axis}[%
      width=\textwidth,
      ybar,
      hide axis,  %
      xmin=0,  %
      xmax=1,
      ymin=0,
      ymax=1,
      scale only axis,width=1mm, %
      legend cell align={left},              %
      legend style={font=\legendFontSize},
      legend columns=9,
    ]
    \addplot [cosin0 color] coordinates {(0,0)};
    \addlegendentry{\methodZero\ours\legendSpacer}

    \addplot [layer0 color] coordinates {(0,0)};
    \addlegendentry{\layerZero\ours\legendSpacer}

    \pgfplotsset{cycle list shift=-2}  %
    \addplot[cosin color] coordinates {(0,0)};
    \addlegendentry{\method\ours\legendSpacer}

    \addplot[layer color] coordinates {(0,0)};
    \addlegendentry{\layer\ours\legendSpacer}

    \addplot[TracInCP color] coordinates {(0,0)};
    \addlegendentry{\tracinCP{}\legendSpacer}
    \pgfplotsset{cycle list shift=-3}  %

    \pgfplotsinvokeforeach{\tracin{},Influence Func.,Representer Pt.}{
      \addplot coordinates {(0,0)};
      \addlegendentry{#1\legendSpacer}
    }

  \end{axis}
\end{tikzpicture}

\newcommand{\JointOptAuprcPlot}[1]{%
\pgfplotstableread[col sep=comma]{plots/data/joint_opt_auprc_#1.csv}\datatable%
\begin{tikzpicture}%
  \begin{axis}[%
        axis lines*=left,%
        ymajorgrids,  %
        bar width={\CifarPoisBarWidth},%
        height={\JointOptSupplementAdvIdentHeight},%
        width={\textwidth},%
        cycle list/Set2,
        xtick=data,%
        xticklabels={Bird $\rightarrow$ Dog, Dog $\rightarrow$ Bird, Frog $\rightarrow$ Deer, Deer $\rightarrow$ Frog},
        x tick label style={font=\plotFontSize,align=center},%
        ymin=0,%
        ymax=1.0,%
        ybar={\BarLineWidth},%
        ytick distance={0.2},%
        minor y tick num={1},%
        y tick label style={font=\plotFontSize},%
        ylabel={\plotFontSize \IdentYLabel},%
        typeset ticklabels with strut,
        every tick/.style={color=black, line width=0.4pt},%
        enlarge x limits=0.17,%
    ]%
    \addplot [cosin0 color] table [x index=0, y index=1] {\datatable};%
    \addplot [layer0 color] table [x index=0, y index=2] {\datatable};%
    \pgfplotsset{cycle list shift=-2}  %
    \addplot [cosin color] table [x index=0, y index=3] {\datatable};%
    \addplot [layer color] table [x index=0, y index=4] {\datatable};%
    \addplot[TracInCP color] table [x index=0, y index=5] {\datatable};%
    \pgfplotsset{cycle list shift=-3}  %
    \addplot table [x index=0, y index=6] {\datatable};%
    \addplot table [x index=0, y index=7] {\datatable};%
    \addplot table [x index=0, y index=8] {\datatable};%
  \end{axis}%
\end{tikzpicture}%
}%
   \begin{subfigure}{\SubFigureWidth}
    \revTwo{%
    \JointOptAuprcPlot{0}
    }%

    \caption{%
      \revTwo{%
      Baseline with ${\surrogateHyper = 0}$
      }%
    }%
  \end{subfigure}
  \hfill
  \begin{subfigure}{\SubFigureWidth}
    \revTwo{%
    \JointOptAuprcPlot{1E-2}
    }%

    \caption{%
      \revTwo{%
      Joint optimization with ${\surrogateHyper = 10^{-2}}$
      }%
    }%
  \end{subfigure}

  \caption{%
    \revTwo{%
      \textit{Adversarial-Set Identification for the Adaptive Vision Poison Attack}:
      Mean AUPRC identifying the adversarial set where \citeauthor{Zhu:2019}'s vision poison attack is jointly optimized with minimizing \method{}
      with ${{\geq}10}$~trials per setup as described in Section~\ref{sec:App:JointOpt:Setup}.
      Section~\ref{sec:ExpRes:AdaptiveAttacks}'s {\color{GreenBar}baseline} results set trade-off hyperparameter ${\surrogateHyper = 0}$, meaning the poison was not jointly optimized.
      The jointly optimized results used ${\surrogateHyper = 10^{-2}}$ as explained in suppl.\ Section~\ref{sec:App:JointOpt:Setup}.
      This joint optimization reduces the \method{} similarity by~7\% at the cost of a 19\%~decrease in ASR w.r.t.\ Table~\ref{tab:ExpRes:Mitigation}.
      See Table~\ref{tab:App:MoreExps:JointOpt:Ident} (below) for the full numerical results, including variance.
    }%
  }
  \label{fig:App:MoreExps:JointOpt:Ident}
\end{figure}
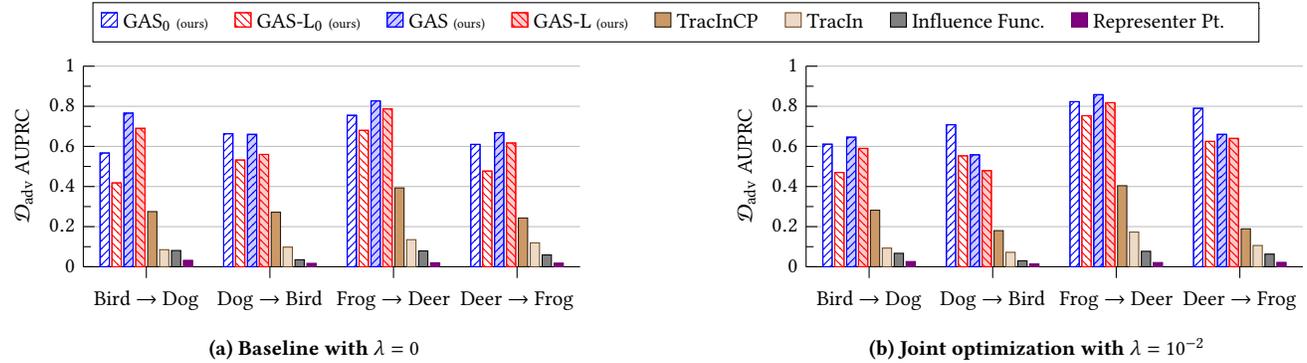

\vfill%

\begin{table}[h!]
  \centering
\revTwo{%
  \caption{
    \revTwo{%
      \textit{Adversarial-Set Identification for the Adaptive Vision Poison Attack}:
      Adversarial-set identification AUPRC mean and standard deviation
      with ${{\geq}10}$~trials per setup as described in Section~\ref{sec:App:JointOpt:Setup}.
      Section~\ref{sec:ExpRes:AdaptiveAttacks}'s {\color{GreenBar}baseline} results set trade-off hyperparameter ${\surrogateHyper = 0}$, meaning the poison was not jointly optimized.
      The jointly optimized results used ${\surrogateHyper = 10^{-2}}$ as explained in suppl.\ Section~\ref{sec:App:JointOpt:Setup}.
      Bold denotes the best mean performance.
      Mean results are shown graphically in Figures~\ref{fig:ExpRes:AdaptiveAttacker:Pois:CIFAR:Joint:Identification} and~\ref{fig:App:MoreExps:JointOpt:Ident}.%
    }%
  }
  \label{tab:App:MoreExps:JointOpt:Ident}
  {%
    \TableFontSize%
\renewcommand{\arraystretch}{1.2}
\setlength{\dashlinedash}{0.4pt}
\setlength{\dashlinegap}{1.5pt}
\setlength{\arrayrulewidth}{0.3pt}

\newcommand{\TabMidRule}{\cdashline{2-10}}

\newcommand{\AlgName}[1]{\multicolumn{1}{c}{#1}}

\newcommand{\hyperVal}[1]{\multirow{4}{*}{$#1$}}

\begin{tabular}{@{}rcrrrrrrrrr@{}}
  \toprule
  \multicolumn{1}{c}{Param.}
  & \multicolumn{1}{c}{Classes} & \multicolumn{4}{c}{Ours} & \multicolumn{5}{c}{Baselines} \\\cmidrule(r){1-1}\cmidrule(r){2-2}\cmidrule(lr){3-6}\cmidrule(l){7-10}
  \multicolumn{1}{c}{$\surrogateHyper$}
  & $\yTarg \rightarrow \yAdv$  & \AlgName{\methodZero} & \AlgName{\layerZero} & \AlgName{\method}  & \AlgName{\layer}  & \AlgName{\tracinCP{}}  & \AlgName{\tracin{}}  & \AlgName{Influence Func.}  &  {Representer Pt.}  \\
  \midrule
  \hyperVal{0\phantom{.00}}
  & Bird $\rightarrow$ Dog
    & \PVal{0.567}{0.370}  & \PVal{0.418}{0.310}  & \PValB{0.766}{0.134} & \PVal{0.690}{0.186}  & \PVal{0.275}{0.163}  & \PVal{0.085}{0.039}  & \PVal{0.081}{0.084}  & \PVal{0.032}{0.027}  \\\TabMidRule
  & Dog $\rightarrow$ Bird
    & \PValB{0.663}{0.392} & \PVal{0.532}{0.337}  & \PVal{0.660}{0.254}  & \PVal{0.560}{0.273}  & \PVal{0.272}{0.199}  & \PVal{0.098}{0.051}  & \PVal{0.035}{0.020}  & \PVal{0.017}{0.006}  \\\TabMidRule
  & Frog $\rightarrow$ Deer
    & \PVal{0.755}{0.378}  & \PVal{0.680}{0.362}  & \PValB{0.827}{0.138} & \PVal{0.787}{0.156}  & \PVal{0.393}{0.214}  & \PVal{0.135}{0.089}  & \PVal{0.079}{0.086}  & \PVal{0.020}{0.009}  \\\TabMidRule
  & Deer $\rightarrow$ Frog
    & \PVal{0.610}{0.358}  & \PVal{0.477}{0.300}  & \PValB{0.669}{0.202} & \PVal{0.617}{0.198}  & \PVal{0.243}{0.150}  & \PVal{0.119}{0.067}  & \PVal{0.059}{0.045}  & \PVal{0.018}{0.006}  \\
  \midrule
  \hyperVal{10^{-2}}
  & Bird $\rightarrow$ Dog
    & \PVal{0.611}{0.336}  & \PVal{0.470}{0.312}  & \PValB{0.646}{0.235} & \PVal{0.590}{0.268}  & \PVal{0.282}{0.159}  & \PVal{0.093}{0.066}  & \PVal{0.067}{0.073}  & \PVal{0.026}{0.018}  \\\TabMidRule
  & Dog $\rightarrow$ Bird
    & \PValB{0.708}{0.319}  & \PVal{0.553}{0.296}  & \PVal{0.558}{0.216} & \PVal{0.479}{0.248}  & \PVal{0.180}{0.112}  & \PVal{0.072}{0.045}  & \PVal{0.030}{0.014}  & \PVal{0.014}{0.003}  \\\TabMidRule
  & Frog $\rightarrow$ Deer
    & \PVal{0.823}{0.320}  & \PVal{0.753}{0.320}  & \PValB{0.858}{0.145} & \PVal{0.818}{0.184}  & \PVal{0.404}{0.177}  & \PVal{0.173}{0.101}  & \PVal{0.077}{0.083}  & \PVal{0.021}{0.012}  \\\TabMidRule
  & Deer $\rightarrow$ Frog
    & \PValB{0.790}{0.159} & \PVal{0.625}{0.175}  & \PVal{0.660}{0.180}  & \PVal{0.640}{0.192}  & \PVal{0.189}{0.170}  & \PVal{0.106}{0.060}  & \PVal{0.063}{0.041}  & \PVal{0.022}{0.010}  \\
  \bottomrule
\end{tabular}
   }
}%
\end{table}

\vfill%
\phantom{.}

\clearpage
\newpage
\subsubsection{Target Identification of the Jointly Optimized Poisoning Attack}%
\label{sec:App:MoreExps:JointOpt:FullResults:TargDetect}

\phantom{.}
\vfill%

\begin{figure}[h!]
  \centering
  \newcommand{\SubFigureWidth}{0.46\textwidth}
\newcommand{\legendSpacer}{\hspace*{9pt}}

\begin{tikzpicture}
  \begin{axis}[%
    width=\textwidth,
      ybar,
      hide axis,  %
      xmin=0,  %
      xmax=1,
      ymin=0,
      ymax=1,
      scale only axis,width=1mm, %
      legend cell align={left},              %
      legend style={font=\legendFontSize},
      legend columns=7,
      cycle list name=DetectCycleList,
    ]
    \addplot[cosin color] coordinates {(0,0)};
    \addlegendentry{\fitWith{\method}\ours\legendSpacer}
    \addplot[layer color] coordinates {(0,0)};
    \addlegendentry{\fitWith{\layer}\ours\legendSpacer}

    \addplot coordinates {(0,0)};
    \addlegendentry{Max.\ \KNN{} Distance\legendSpacer}

    \addplot coordinates {(0,0)};
    \addlegendentry{Min.\ \KNN{} Distance\legendSpacer}
    \addplot coordinates {(0,0)};
    \addlegendentry{Most Certain\legendSpacer};

    \addplot coordinates {(0,0)};
    \addlegendentry{Least Certain\legendSpacer};

    \addplot coordinates {(0,0)};
    \addlegendentry{Random};
  \end{axis}
\end{tikzpicture}

\newcommand{\JointOptDetectPlot}[1]{%
  \pgfplotstableread[col sep=comma]{plots/data/joint_opt_detect_#1.csv}\datatable%
  \begin{tikzpicture}%
    \begin{axis}[%
          axis lines*=left,%
          ymajorgrids,  %
          bar width={\DetectBarWidthVal},
          height={\JointOptSupplementTargIdentHeight},%
          width={\textwidth},%
          ymin=0,%
          ymax=1,%
          ytick distance={0.20},%
          ybar={\BarLineWidth},%
          minor y tick num={1},%
          y tick label style={font=\plotFontSize},%
          ylabel={\plotFontSize \DetectYLabel},%
          xtick=data,%
          xticklabels={Bird $\rightarrow$ Dog, Dog $\rightarrow$ Bird, Frog $\rightarrow$ Deer, Deer $\rightarrow$ Frog},
          x tick label style={font=\plotFontSize,align=center},%
          typeset ticklabels with strut,
          every tick/.style={color=black, line width=0.4pt},%
          enlarge x limits={0.150},%
          cycle list name=DetectCycleList,
      ]%
      \addplot[cosin color] table [x index=0, y index=1] {\datatable};%
      \addplot[layer color] table [x index=0, y index=2] {\datatable};%
      \foreach \k in {3, ..., 7} {%
        \addplot table [x index=0, y index=\k] {\datatable};%
      }%
    \end{axis}%
  \end{tikzpicture}%
}%
   \begin{subfigure}{\SubFigureWidth}
    \revTwo{%
    \JointOptDetectPlot{0}
    }%

    \caption{%
      \revTwo{%
      Baseline with ${\surrogateHyper = 0}$
      }%
    }%
  \end{subfigure}
  \hfill
  \begin{subfigure}{\SubFigureWidth}
    \revTwo{%
    \JointOptDetectPlot{1E-2}
    }%

    \caption{%
      \revTwo{%
      Joint optimization with ${\surrogateHyper = 10^{-2}}$
      }%
    }%
  \end{subfigure}

  \caption{%
    \revTwo{%
      \textit{Target Identification for the Adaptive Vision Poison Attack}:
      Mean target identification AUPRC where \citepos{Zhu:2019} vision poison attack is jointly optimized with minimizing \method{}.
      Section~\ref{sec:ExpRes:AdaptiveAttacks}'s {\color{GreenBar}baseline} results set trade-off hyperparameter ${\surrogateHyper = 0}$, meaning the poison was not jointly optimized.
      The jointly optimized results used ${\surrogateHyper = 10^{-2}}$ as explained in suppl.\ Section~\ref{sec:App:JointOpt:Setup}.
      See Table~\ref{tab:App:MoreExps:JointOpt:Detect} (below) for the full numerical results, including variance.
    }%
  }
  \label{fig:App:MoreExps:JointOpt:Detect}
\end{figure}
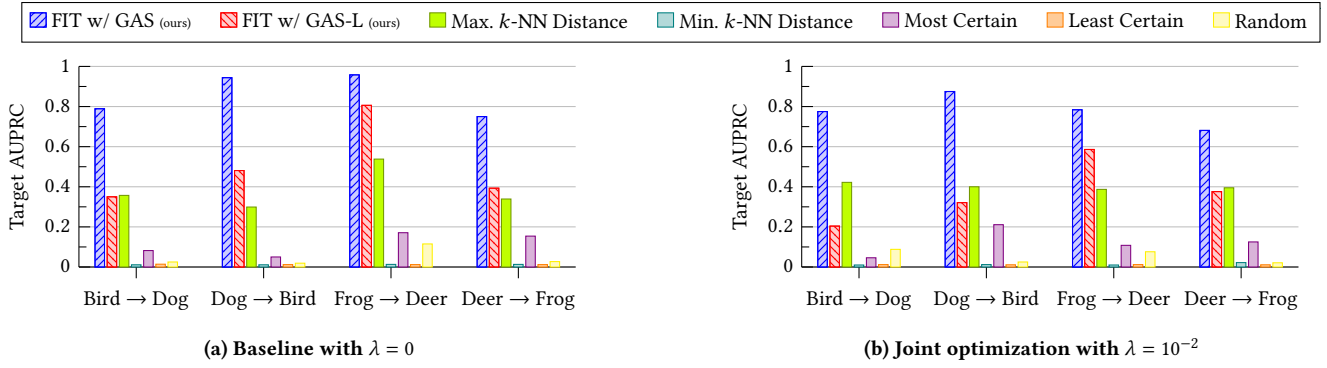

\vfill

\begin{table}[h!]
  \centering
\revTwo{%
  \caption{%
    \revTwo{%
      \textit{Target Identification for the Adaptive Vision Poison Attack}:
      Target identification AUPRC mean and standard deviation where
      \citepos{Zhu:2019} vision poison attack is jointly optimized with minimizing \method{}.
      Section~\ref{sec:ExpRes:AdaptiveAttacks}'s {\color{GreenBar}baseline} results set trade-off hyperparameter ${\surrogateHyper = 0}$, meaning the poison was not jointly optimized.
      The jointly optimized results used ${\surrogateHyper = 10^{-2}}$ as explained in suppl.\ Section~\ref{sec:App:JointOpt:Setup}.
      Bold denotes the best mean performance with ${{\geq}10}$~trials per class pair.
      Mean results are shown graphically in Figures~\ref{fig:ExpRes:AdaptiveAttacker:Pois:CIFAR:Joint:TargDetect} and~\ref{fig:App:MoreExps:JointOpt:Detect}.
    }%
  }
  \label{tab:App:MoreExps:JointOpt:Detect}
  {%
    \TableFontSize%
\renewcommand{\arraystretch}{1.2}
\setlength{\dashlinedash}{0.4pt}
\setlength{\dashlinegap}{1.5pt}
\setlength{\arrayrulewidth}{0.3pt}

\newcommand{\TabMidRule}{\cdashline{2-10}}

\newcommand{\AlgName}[1]{\multicolumn{1}{c}{#1}}

\newcommand{\hyperVal}[1]{\multirow{4}{*}{$#1$}}

\begin{tabular}{rccrrrrrrr}
  \toprule
  \multicolumn{1}{c}{Param.}
  & \multicolumn{2}{c}{Classes} & \multicolumn{2}{c}{Ours} & \multicolumn{4}{c}{Baselines} \\\cmidrule(lr){1-1}\cmidrule(lr){2-3}\cmidrule(lr){4-5}\cmidrule(lr){6-10}
  \multicolumn{1}{c}{$\surrogateHyper$}
  & $\yTarg$ & $\yAdv$
  & \AlgName{\method{}}   & \AlgName{\layer}    & \AlgName{Max~\KNN{}} & \AlgName{Min~\KNN{}}  & \AlgName{Most Certain} & \AlgName{Least Certain} & \AlgName{Random}   \\
  \midrule
  \hyperVal{0\phantom{.00}}
  & Bird  & Dog
       & \PValB{0.789}{0.271}  & \PVal{0.350}{0.372}  & \PVal{0.357}{0.360}  & \PVal{0.011}{0.003}  & \PVal{0.082}{0.091}  & \PVal{0.014}{0.008}  & \PVal{0.025}{0.024}  \\\TabMidRule
  & Dog   & Bird
       & \PValB{0.944}{0.167}  & \PVal{0.481}{0.431}  & \PVal{0.299}{0.325}  & \PVal{0.011}{0.005}  & \PVal{0.050}{0.026}  & \PVal{0.012}{0.002}  & \PVal{0.019}{0.010}  \\\TabMidRule
  & Frog  & Deer
       & \PValB{0.958}{0.144}  & \PVal{0.806}{0.300}  & \PVal{0.538}{0.441}  & \PVal{0.013}{0.007}  & \PVal{0.171}{0.279}  & \PVal{0.012}{0.002}  & \PVal{0.115}{0.280}  \\\TabMidRule
  & Deer  & Frog
       & \PValB{0.750}{0.320}  & \PVal{0.393}{0.329}  & \PVal{0.339}{0.355}  & \PVal{0.013}{0.007}  & \PVal{0.154}{0.148}  & \PVal{0.012}{0.003}  & \PVal{0.027}{0.023}  \\
  \midrule
  \hyperVal{10^{-2}}
  & Bird  & Dog
       & \PValB{0.775}{0.282}  & \PVal{0.204}{0.250}  & \PVal{0.422}{0.380}  & \PVal{0.010}{0.003}  & \PVal{0.046}{0.042}  & \PVal{0.012}{0.003}  & \PVal{0.088}{0.142}  \\\TabMidRule
  & Dog   & Bird
       & \PValB{0.875}{0.231}  & \PVal{0.321}{0.333}  & \PVal{0.400}{0.497}  & \PVal{0.012}{0.004}  & \PVal{0.211}{0.329}  & \PVal{0.011}{0.004}  & \PVal{0.025}{0.025}  \\\TabMidRule
  & Frog  & Deer
       & \PValB{0.784}{0.269}  & \PVal{0.586}{0.344}  & \PVal{0.387}{0.335}  & \PVal{0.010}{0.002}  & \PVal{0.108}{0.150}  & \PVal{0.012}{0.002}  & \PVal{0.076}{0.120}  \\\TabMidRule
  & Deer  & Frog
       & \PValB{0.681}{0.288}  & \PVal{0.376}{0.329}  & \PVal{0.395}{0.456}  & \PVal{0.022}{0.025}  & \PVal{0.125}{0.153}  & \PVal{0.011}{0.001}  & \PVal{0.021}{0.012}  \\
  \bottomrule
\end{tabular}
   }
}%
\end{table}

\vfill%
\phantom{.}

\clearpage
\newpage
\subsubsection{Target-Driven Attack Mitigation of the Jointly Optimized Poisoning Attack}%
\label{sec:App:MoreExps:JointOpt:FullResults:Mitigation}
\phantom{.}

\vspace{8pt}
\revTwo{%
This section examines joint optimization's effect on target-driven mitigation.
Averaging across all class pairs, target-driven mitigation using \method{} and \layer{} removed
0.05\% and 0.03\% (resp.) of the clean training data~(\eqsmall{$\cleanTrain$}).  For comparison, \citepos{Zhu:2019}
baseline attack removed on average 0.02\% and 0.03\% of clean training data for \method{}
and \layer{} respectively (see Table~\ref{tab:ExpRes:Mitigation}).
Moreover, after mitigating this jointly-optimized attack, average test accuracy either improved or stayed the same in all but one case.
}%

\vspace{1.0cm}
\begin{table*}[h!]
  \centering
  \revTwo{%
  \caption{%
    \revTwo{%
      \textit{Target-Driven Attack Mitigation for the Adaptive Vision Poison Attack}:
      Algorithm~\ref{alg:Mitigation}'s target-driven data sanitization
      where \citepos{Zhu:2019} vision poison attack is jointly optimized with minimizing the \method{} influence.
      The results below consider exclusively the jointly-optimized attack with ${\surrogateHyper = 10^{-2}}$.
      Clean-data removal remains low, and test accuracy either improved or stayed the same for in but one setup.
      The performance is comparable to the results with \citepos{Zhu:2019}'s standard vision poisoning attack (see Table~\ref{tab:App:MoreExps:Pois:CIFAR:Mitigate}).
      Bold denotes the best mean performance with ${{\geq}10}$~trials per class pair.
    }%
  }%
  {
    \TableFontSize%
\renewcommand{\arraystretch}{1.2}
\setlength{\dashlinedash}{0.4pt}
\setlength{\dashlinegap}{1.5pt}
\setlength{\arrayrulewidth}{0.3pt}

\newcommand{\MultiHead}[1]{\multicolumn{2}{c}{#1}}

\newcommand{\TwoRowHead}[1]{\multirow{2}{*}{#1}}
\newcommand{\ClassPair}[2]{\multirow{2}{*}{#1} & \multirow{2}{*}{#2}}
\newcommand{\PZ}{\phantom{0}}
\newcommand{\ptZ}{\phantom{.}\PZ}
\newcommand{\ptZZ}{\phantom{.}\PZ\PZ}
\newcommand{\ptASR}{}

\newcommand{\CosInM}{\method{}}
\newcommand{\LayerM}{\layer{}}

\newcommand{\OneRemB}{100\ptZ}
\newcommand{\ZeroRemB}{0\ptZ}

\newcommand{\PRem}[2]{#1}
\newcommand{\PRemB}[2]{\textBF{#1}}
\newcommand{\ASR}[2]{\multirow{2}{*}{#1}}
\newcommand{\PAcc}[2]{\multirow{2}{*}{#1}}
\newcommand{\PAsr}[2]{#1}

\newcommand{\PChg}[2]{{\color{ForestGreen} +#1}}
\newcommand{\NoChg}{0.0}
\newcommand{\NegChg}[2]{{\color{BrickRed} -#1}}

\newcommand{\DsSep}{\cdashline{1-9}}
\newcommand{\MethodSep}{}

\begin{tabular}{cclrrrrrr}
  \toprule
     \MultiHead{Classes}
     & \TwoRowHead{Method} & \MultiHead{\% Removed}      & \MultiHead{ASR \%} & \MultiHead{Test Acc. \%} \\\cmidrule(lr){1-2}\cmidrule(lr){4-5}\cmidrule(lr){6-7}\cmidrule(lr){8-9}
     $\yTarg$ & $\yAdv$
     &                    & $\advTrain$  & $\cleanTrain$ & Orig.   & Ours     &  Orig. & Chg.   \\
  \midrule
  \ClassPair{Bird}{Dog}
     & \CosInM    & \PRemB{36.0}{20.1} & \PRem{0.02}{0.04}   & \PAcc{76.2}{}  & \PAsr{0}{0}   & \PAcc{87.0}{0.3}  & \PChg{0.1}{-0.1} \\
     && \LayerM   & \PRem{30.3}{24.1}  & \PRemB{0.00}{0.01}  &                & \PAsr{0}{0}   &                   & \PChg{0.1}{0.0} \\
  \DsSep
  \ClassPair{Dog}{Bird}
     & \CosInM    & \PRem{21.6}{16.3}  & \PRemB{0.00}{0.00}  & \PAcc{57.1}{}  & \PAsr{0}{0}   & \PAcc{87.1}{0.3}  & \PChg{0.1}{0.1} \\
     && \LayerM   & \PRemB{21.9}{14.8} & \PRemB{0.00}{0.00}  &                & \PAsr{0}{0}   &                   & \NegChg{-0.1}{0.2} \\
  \DsSep
  \ClassPair{Frog}{Deer}
     & \CosInM    & \PRem{17.5}{14.7}  & \PRemB{0.00}{0.01}  & \PAcc{38.1}{}   & \PAsr{0}{0}   & \PAcc{87.1}{0.3}  & \NoChg \\
     && \LayerM   & \PRemB{19.4}{18.8} & \PRemB{0.00}{0.00}  &                 & \PAsr{0}{0}   &                   & \NoChg \\
  \DsSep
  \ClassPair{Deer}{Frog}
     & \CosInM    & \PRemB{85.0}{24.5} & \PRem{0.18}{0.23}   & \PAcc{81.0}{}   & \PAsr{0}{0}   & \PAcc{87.1}{0.2}  & \NoChg  \\
     && \LayerM   & \PRem{82.3}{23.2}  & \PRemB{0.13}{0.14}  &                 & \PAsr{0}{0}   &                   & \PChg{0.1}{0.1} \\
  \bottomrule
\end{tabular}
   }
}%
\end{table*}
 
\FloatBarrier
\clearpage
\newpage
\revTwo{%
\subsection{Adaptive Adversarial-Instance Selection}%
\label{sec:App:MoreExps:AdaptiveAdvSelection}%
}

\revTwo{%
  Sections~\ref{sec:ExpRes:AdaptiveAttacks} and~\ref{sec:App:JointOpt:Setup} consider a joint optimization where the attacker crafts the adversarial set to simultaneously minimize both the target's adversarial loss as well as a surrogate estimate of adversarial set~\eqsmall{$\advTrain$}'s \method{} influence.  However, not all adversarial attacks construct the adversarial set via an optimization.  For example, \citepos{Weber:2021} backdoor attack uses a fixed adversarial trigger that is both simple and highly effective.

Such simple trigger attacks require a different adaptive strategy since the trigger is not optimized.
To that end, this section considers a simpler and more general adaptive attack where the attacker adversarially selects \eqsmall{$\advTrain$}'s~seed instances to appear uninfluential.
The attack achieves this by running our method on a gray-box surrogate model and selecting as the attack instances those that are ranked least influential.
To ensure a strong adversary, this surrogate uses the same model architecture, hyperparameters, initial (fully-random) model parameters, and clean training data as the target model.
}

\revised{%
Here we consider \citepos{Weber:2021} three attack backdoor patterns on \texttt{auto} vs.\ \texttt{dog}.%
\footnote{%
  \revTwo{%
  \citepos{SpeechDataset} speech recognition attack also uses a fixed adversarial trigger.
  We did not use it for this evaluation as \citeauthor{SpeechDataset}'s dataset comes with the clean seed instances pre-selected by the authors.%
  }
}
We analyze this adaptive adversarial-set selection using both clean and backdoored surrogate models.
From a held-out set of 2,000 backdoor candidate instances, the 150~least influential instances then form the adversarial set.
}
\revTwo{%
With the exception of the procedure for selecting the adversarial set's clean instance, the evaluation setup was identical to Section~\ref{sec:ExpRes:AdvIdent}.
}

\revTwo{%
Table~\ref{tab:ExpRes:AdaptiveAttacker:CIFAR:Adaptive} reports the mean adversarial-set identification AUPRC, with the non-adaptive baseline being a u.a.r.\ adversarial set.
Observe that \bothMethod{} remain highly effective against this adaptive attacker (0.925--0.947~AUPRC) -- a 5\==8\% decline versus the baseline.
These results would only improve in practice where adversaries have no knowledge of training's random seed (e.g.,~the fully-random initial model parameters).
}

\vspace{1cm}
\begin{table}[h!]
  \centering
  \caption{%
  \revised{%
  \textit{Vision Backdoor Adaptive Attacker}:
  Mean AUPRC identifying adversarial set~\eqsmall{$\advTrain$} for \citepos{Weber:2021} vision backdoor attacks where the adaptive attacker attempts to conceal~\eqsmall{$\advTrain$} via either a clean or backdoored surrogate model.
  \revTwo{%
  \bothMethod{} still achieve consistently high adversarial-set identification performance under this gray-box adaptive attacker.
  }
  Baseline is selecting~\eqsmall{$\advTrain$} u.a.r.
  Results are averaged across related experimental setups with ${{\geq}5}$~trials per setup.
  }%
  }%
  \label{tab:ExpRes:AdaptiveAttacker:CIFAR:Adaptive}
  \revised{%
{
\small
\begin{tabular}{lrrrr}
  \toprule
  Surrogate Type   & \method{} & \layer{} & \tracinCP{} & \tracin{} \\
  \midrule
  Baseline         & 0.997     & 0.999    & 0.613       & 0.312     \\
  Backdoor         & 0.940     & 0.942    & 0.772       & 0.397     \\
  Clean            & 0.925     & 0.947    & 0.356       & 0.202     \\
  \bottomrule
\end{tabular}
}
   }
\end{table}
 
\FloatBarrier
\clearpage
\newpage
\newcommand{\MitigationMinipageWidth}{0.46\textwidth}

\revTwo{%
\subsection{An Adversarial Attack on Target-Driven Mitigation}\label{sec:App:MoreExps:MitigationTriggering}
Our threat model (defined in Section~\ref{sec:ProblemFormulation}) specifies that the adversary attempts to alter the model prediction on a specific target or set of targets.
Consider a different threat model where an adversary's objective is an \keyword{availability training-set attack} that seeks to lower overall model performance,
i.e.,~the attacker is not focused on a specific target.  Such an adversary could leverage our framework to achieve their objective.  This section describes
one such attack procedure and discusses a simple remedy to insulate against that risk.

Recall from Section~\ref{sec:ExpRes:Attacks} that \citepos{Weber:2021} vision backdoor attack inserts an adversarial trigger pattern (e.g.,~one pixel, four pixel, blend) to cause a misprediction.
Specifically, their attack causes test instances from the \colorYTarg~class to be mislabeled as belonging to the \colorYAdv~class.
\citeauthor{Weber:2021} achieve this by inserting the trigger pattern into \colorYTarg~test instances.
When a target is detected, target-driven attack mitigation (Alg.~\ref{alg:Mitigation}) then iteratively sanitizes the training set until the target instance's prediction changes.

Via one small change, \citepos{Weber:2021} backdoor (targeted) attack can be reformulated into an availability (i.e.,~indiscriminate) attack that hijacks our framework to achieve its objective.
This reformulation inserts the adversarial trigger into \colorYAdv~test instances (\underline{not} \eqsmall{$\yTarg$}~test instances as above).
Note that no changes are made to adversarial set~\eqsmall{$\advTrain$} or the training procedure as defined in Section~\ref{sec:ExpRes:Attacks}.
Nonetheless, this reformulated attack has the \textit{opposite effect} as \citepos{Weber:2021} attack -- instead of inducing a misprediction, our reformulated attack actually \textit{increases the confidence of a correct prediction}.

If this reformulated attack's~\eqsmall{$\advTrain$} instances induce a heavy enough upper tail, then the corresponding \colorYAdv~test instance would be identified as a target and target-driven mitigation initiated.
However, unlike in the standard case where mitigation changes a wrong prediction to a correct one, this reformulated attack triggers mitigation that \textit{switches a correct prediction to a wrong one}.
In addition to causing the target instance to be mispredicted, this obviously has the potential to require sanitization of significantly more clean data, which as a result, may significantly increase the test error.

Below we evaluate this reformulation of \citepos{Weber:2021} backdoor attack for class pair \eqsmall{${\color{\targColor}\yTarg} = \colorAuto$} and \eqsmall{${\color{\advColor}\yAdv} = \colorDog$} (where \eqsmall{${\yTarg \rightarrow \yAdv}$}).
Note that the experimental setup is unchanged from Section~\ref{sec:ExpRes} with one exception-- we expand the evaluation to consider the case where the adversarial trigger is inserted into \colorDog{} \textit{test} instances.

Figure~\ref{fig:App:MoreExps:MitigationTriggering:IdentAUPRC} compares the adversarial-set identification performance for both \citeauthor{Weber:2021}'s original attack (with the trigger inserted into u.a.r.\ \colorAuto{}~test instances) and our reformulated attack (with the trigger inserted into u.a.r.\ \colorDog{} test instances).
Observe that \bothMethod{} identifies the adversarial set as influential on the perturbed test instances irrespective of the instance's class.
There is a very small performance difference between the two attacks for the four-pixel and blend patterns and a larger difference for the one-pixel pattern.

Note that this reformulated attack is less likely to trigger target identification than \citepos{Weber:2021} original version.
For example, \textit{on average} with the blend attack pattern, \colorAuto{} target instances had higher anomaly scores~($\anomScore$) than \colorDog{} instances -- specifically by $1.0\Qn$~for \method{} and $1.3\Qn$~for \layer{}.
However, there were multiple cases where \colorDog{} instances had much higher anomaly scores (by up to~$1.2\Qn$ for \method{} and $5.3\Qn$ for \layer{}) -- in particular for the four-pixel attack pattern.
While relatively uncommon (between 5-20\% of the time based on the attack pattern), false target identification is \textit{definitely possible} here, and a persistent adversary could continue retrying the attack with different backdoored \colorYAdv{}~instances until success.

Table~\ref{tab:App:MoreExps:MitigationTriggering:Mitigation} quantifies the effect of our mitigation procedure when a reformulated attack instance is misclassified as a target.
Observe that significantly more clean data is removed than for \citeauthor{Weber:2021}'s original attack -- by multiple orders of magnitude in some cases.%
\footnote{\revTwo{%
By design, Algorithm~\ref{alg:Mitigation} gradually sanitizes the training set to avoid excessive clean data removal.
As Table~\ref{tab:App:MoreExps:MitigationTriggering:Mitigation} details, the reformulated attack requires the removal of so much clean data that Algorithm~\ref{alg:Mitigation}'s execution time became prohibitive.
As a computationally efficient alternative, Table~\ref{tab:App:MoreExps:MitigationTriggering:Mitigation}'s results modified Alg.~\ref{alg:Mitigation} such that data removal cutoff~$\anomCutoff$ was determined via an exponential (i.e.,~doubling) search through the sorted influence values~$\infVec$.
In other words, influence was not re-estimated between each iteration of Alg.~\ref{alg:Mitigation}'s \texttt{while} loop.
Instead, influence scores were calculated for each training instance once, and an iterative search was performed over that (sorted) vector.
Observe that for \citeauthor{Weber:2021}'s standard attack, the amount of clean data removed was slightly less than for the standard version of Algorithm~\ref{alg:Mitigation} (see Table~\ref{tab:App:MoreExps:Backdoor:CIFAR:Mitigate}).
Therefore, Table~\ref{tab:App:MoreExps:MitigationTriggering:Mitigation} may \textit{underestimate} the reformulated attack's already high severity.%
}
}
Removing a large fraction clean data degrades the model's average clean test accuracy by up to 21\%.

The above experiment is intended as a proof of concept that an attacker operating outside of our original threat model could use our framework to trigger an availability attack.
Of course, the potential effect of such an attack will depend on a variety of factors, including the model's confidence when predicting (\eqsmall{$\yAdv$})~test instances, the model's architecture, the attack paradigm (e.g.,~backdoor vs.\ poison), etc.
In situations where this alternate threat model may apply, Algorithm~\ref{alg:Mitigation} should be tweaked slightly to include a threshold on the maximum amount of sanitization before special intervention/analysis is initiated.
For example, this intervention could include (e.g.,~human, forensic) analysis of the identified target as well as the most influential instances as identified by \bothMethod{}.  The value of this ``intervention threshold'' could be set empirically or based on domain-specific knowledge, e.g.,~the maximum percentage of the training set that may be adversarial.

Note that there are multiple possible approaches for the ``intervention'' or ``analysis'' mentioned above, many of which are quite simple.
For example, the reformulated availability attack above would be thwarted if the identified target's true label were verified (e.g.,~by a human) prior to initiating verification.
Another option is that in cases where excessive sanitization is needed to change a prediction, Algorithm~\ref{alg:Mitigation} can terminate early (i.e.,~before the prediction changes) and only sanitize those training instances with sufficiently anomalous influence estimates.
}

\FloatBarrier%
\clearpage%
\newpage%

\phantom{.}
\vfill%

\begin{figure}[h!]
  \begin{minipage}{\MitigationMinipageWidth}
    \centering

\begin{tikzpicture}
  \newcommand{\legendSpacer}{\hspace*{9pt}}
  \begin{axis}[%
    width=\textwidth,
      ybar,
      hide axis,  %
      xmin=0,  %
      xmax=1,
      ymin=0,
      ymax=1,
      scale only axis,width=1mm, %
      legend cell align={left},              %
      legend style={font=\legendFontSize},
      legend columns=4,
    ]

    \addplot [cosin clean color] coordinates {(0,0)};
    \addlegendentry{\colorMethodAdv{} (\colorDog)\legendSpacer}

    \addplot [cosin color] coordinates {(0,0)};
    \addlegendentry{\colorMethodTarg{} (\colorAuto)\legendSpacer}

    \addplot [layer clean color] coordinates {(0,0)};
    \addlegendentry{\colorLayerAdv{} (\colorDog)\legendSpacer}

    \addplot [layer color] coordinates {(0,0)};
    \addlegendentry{\colorLayerTarg{} (\colorAuto)}
  \end{axis}
\end{tikzpicture}

\centering
\pgfplotstableread[col sep=comma]{plots/data/bd_cifar_autoimmune_auprc.csv}\datatable%
\begin{tikzpicture}%
  \begin{axis}[%
        axis lines*=left,%
        ymajorgrids,  %
        bar width={\BackdoorSpeechBarWidth},%
        height=\ExpResBarChartHeight,%
        width={\textwidth},%
        ymin=0,%
        ymax=1,%
        ybar={\BarLineWidth},%
        ytick distance={0.20},%
        minor y tick num={1},%
        y tick label style={font=\plotFontSize},%
        ylabel={\plotFontSize \IdentYLabel},%
        xtick=data,%
        xticklabels={1~Pixel,4~Pixel,Blend},
        x tick label style={font=\plotFontSize,align=center},%
        typeset ticklabels with strut,
        every tick/.style={color=black, line width=0.4pt},%
        enlarge x limits={0.20},%
        draw group line={[index]5}{1}{${\text{Plane} \rightarrow \text{Bird}}$}{-5.0ex}{12pt},
        draw group line={[index]5}{2}{${\text{Auto} \rightarrow \text{Dog}}$}{-5.0ex}{12pt}
    ]%
    \addplot[cosin clean color] table [x index=0, y index=1] {\datatable};%
    \addplot[cosin color] table [x index=0, y index=2] {\datatable};%

    \addplot[layer clean color] table [x index=0, y index=3] {\datatable};%
    \addplot[layer color] table [x index=0, y index=4] {\datatable};%
  \end{axis}%
\end{tikzpicture}%
 
    \captionof{figure}{%
      \revTwo{%
      \textit{Adversarial-Set Identification for the Availability Backdoor Attack}:
      Mean adversarial-instance identification AUPRC for the vision backdoor attack on CIFAR10
      class pair \eqsmall{${\color{\targColor}\yTarg} = \colorAuto$} and
      \eqsmall{${\color{\advColor}\yAdv} = \colorDog$} with
      150~backdoor instances and ${{>}15}$~trials per setup.
      \bothMethod{} identify adversarial set~\eqsmall{$\advTrain$} as highly influential irrespective of
      whether the adversarial instance is inserted into \colorYTarg~(\colorAuto{} -- \citeauthor{Weber:2021}'s attack)
      or \colorYAdv~(\colorDog{} -- reformulated attack) test instances -- although, as expected,
      the original attack does have better identification.
      Full numerical results (including variance) are in Table~\ref{tab:App:MoreExps:MitigationTriggering:IdentAUPRC}.
      }%
    }%
    \label{fig:App:MoreExps:MitigationTriggering:IdentAUPRC}
  \end{minipage}
  \hfill
  \begin{minipage}{\MitigationMinipageWidth}
    \centering
    \captionof{table}{%
      \revTwo{%
      \textit{Adversarial-Set Identification for the Availability Backdoor Attack}:
      Mean and standard deviation adversarial-instance identification AUPRC for the
      vision backdoor attack on CIFAR10 class pair
      \eqsmall{${\color{\targColor}\yTarg} = \colorAuto$}
      and
      \eqsmall{${\color{\advColor}\yAdv} = \colorDog$}
      with 150~backdoor instances and ${{>}15}$~trials per setup.
      \bothMethod{} identify adversarial set~\eqsmall{$\advTrain$} as highly influential irrespective of
      whether the adversarial instance is inserted into \colorYTarg~(\colorAuto{} -- \citeauthor{Weber:2021}'s attack)
      or \colorYAdv~(\colorDog{} -- reformulated attack) test instances -- although, as expected,
      the original attack does have better identification.
      Mean results are shown graphically in Figure~\ref{fig:App:MoreExps:MitigationTriggering:IdentAUPRC}.
      }%
    }%
    \label{tab:App:MoreExps:MitigationTriggering:IdentAUPRC}
    {
      \TableFontSize%
      \revTwo{%
\renewcommand{\arraystretch}{1.2}
\setlength{\dashlinedash}{0.4pt}
\setlength{\dashlinegap}{1.5pt}
\setlength{\arrayrulewidth}{0.3pt}

\newcommand{\head}[1]{\multirow{2}{*}{#1}}
\newcommand{\DigPair}[2]{\multirow{3}{*}{#1 $\rightarrow$ #2}}
\newcommand{\Attack}[1]{#1}

\newcommand{\TabMidRule}{\cdashline{2-4}}

\newcommand{\AlgName}[1]{\multicolumn{1}{c}{#1}}

\newcommand{\TargetTypeHead}[2]{\multirow{3}{*}{\shortstack{#1 \\ (\eqsmall{#2})}}}

\begin{tabular}{ccrr}
  \toprule
  \multirow{2}{*}{\shortstack{Test \\ Class}}
        & \multirow{2}{*}{\shortstack{Attack \\ Pattern}}
        & \multicolumn{2}{c}{Ours} \\\cmidrule(lr){3-4}
        &  & \AlgName{\method{}}  & \AlgName{\layer{}}   \\  %
  \midrule
  \TargetTypeHead{\colorAuto}{\colorYTarg}
       & 1~Pixel
           & \PVal{0.981}{0.020}  & \PValB{0.983}{0.022} \\\TabMidRule
       & 4~Pixel
           & \PValB{0.984}{0.029} & \PVal{0.971}{0.062}  \\\TabMidRule
       & Blend
           & \PVal{0.999}{0.003}  & \PValB{1.000}{0.000} \\
  \midrule
  \TargetTypeHead{\colorDog}{\colorYAdv}
       & 1~Pixel
           & \PValB{0.660}{0.294}  & \PVal{0.549}{0.355}  \\\TabMidRule
       & 4~Pixel
           & \PValB{0.937}{0.215} & \PVal{0.935}{0.192}  \\\TabMidRule
       & Blend
           & \PVal{0.987}{0.033}  & \PValB{0.995}{0.015} \\
  \bottomrule
\end{tabular}
       }
    }
  \end{minipage}
\end{figure}

\vfill%

\begin{table}[h!]
  \centering
  \caption{
    \revTwo{%
    \textit{Target-Driven Mitigation for the Availability Backdoor Attack}:
    Data sanitization applied to the vision backdoor attack on CIFAR10 class pair
    \eqsmall{${\color{\targColor}\yTarg} = \colorAuto$}
    (\citepos{Weber:2021} attack) and
    \eqsmall{${\color{\advColor}\yAdv} = \colorDog$}
    (our reformulated attack) with 150~backdoor instances and ${{\geq}10}$~trials per setup.
    These results demonstrate that \citeauthor{Weber:2021}'s backdoor attack can be
    reformulated as an availability attack that hijacks our framework to remove
    significant clean training data.  The risk of such an attack can be mitigated
    by modifying Algorithm~\ref{alg:Mitigation} to include a threshold on mitigation's
    maximum effect before additional intervention is initiated.
    Attack success rate~(ASR) is w.r.t.\ the analyzed target.
    }%
  }
  \label{tab:App:MoreExps:MitigationTriggering:Mitigation}
  {%
    \TableFontSize%
    \revTwo{%
\renewcommand{\arraystretch}{1.2}
\setlength{\dashlinedash}{0.4pt}
\setlength{\dashlinegap}{1.5pt}
\setlength{\arrayrulewidth}{0.3pt}

\newcommand{\MultiHead}[1]{\multicolumn{2}{c}{#1}}

\newcommand{\TwoRowHead}[1]{\multirow{2}{*}{#1}}
\newcommand{\ClassPair}[2]{\multirow{6}{*}{$\text{#1} \rightarrow \text{#2}$}}
\newcommand{\Atk}[1]{\multirow{2}{*}{#1}}
\newcommand{\PZ}{\phantom{0}}
\newcommand{\ptZ}{\phantom{.}\PZ}
\newcommand{\ptZZ}{\phantom{.}\PZ\PZ}

\newcommand{\CosInM}{\method{}}
\newcommand{\LayerM}{\layer{}}

\newcommand{\OneRemB}{100\ptZ}
\newcommand{\ZeroRemB}{0\ptZ}
\newcommand{\ZeroRemZZB}{0\ptZZ}

\newcommand{\PRem}[2]{#1}
\newcommand{\PRemB}[2]{\textBF{#1}}
\newcommand{\ASR}[2]{\multirow{2}{*}{#1}}
\newcommand{\PAcc}[2]{\multirow{2}{*}{#1}}
\newcommand{\PAsr}[2]{#1}

\newcommand{\PChg}[2]{{\color{ForestGreen} +#1}}
\newcommand{\NoChg}{0.0}
\newcommand{\NegChg}[2]{{\color{BrickRed} -#1}}

\newcommand{\DsSep}{\cdashline{2-7}}
\newcommand{\MethodSep}{}

\newcommand{\TargetTypeHead}[2]{\multirow{6}{*}{\shortstack{$\bigg\uparrow$  \\~\\ #1 \\ (\eqsmall{#2}) \\~\\ $\bigg\downarrow$ }}}

\begin{tabular}{cllrrrr}
  \toprule
  \multirow{2}{*}{\shortstack{Test \\ Class}}
     & \TwoRowHead{Attack} & \TwoRowHead{Method} & \MultiHead{\% Removed} & \MultiHead{Te.\ Acc. \%} \\\cmidrule(lr){4-5}\cmidrule(lr){6-7}
     &   &        & $\advTrain$         & $\cleanTrain$                &  Orig. & Chg.   \\
  \midrule
  \TargetTypeHead{\colorAuto}{\colorYTarg}
     & \Atk{1~Pixel}
     & \CosInM    & \PRem{47.2}{14.8}   & \PRemB{\ZeroRemZZB}{\ZeroRemB}  & \PAcc{98.8}{0.1}  & \NoChg \\
     && \LayerM   & \PRemB{48.2}{23.0}  & \PRemB{\ZeroRemZZB}{\ZeroRemB}  &                   & \NoChg \\
  \DsSep
     & \Atk{4~Pixel}
     & \CosInM    & \PRem{67.6}{21.2}   & \PRemB{\ZeroRemZZB}{\ZeroRemB}  & \PAcc{98.9}{0.1}  & \NoChg \\
     && \LayerM   & \PRemB{69.1}{29.6}  & \PRem{0.00}{0.00}               &                   & \NoChg \\
  \DsSep
     & \Atk{Blend}
     & \CosInM    & \PRem{67.9}{23.1}   & \PRem{0.00}{0.00}               & \PAcc{99.0}{0.1}  & \NegChg{-0.1}{0.1}  \\
     && \LayerM   & \PRemB{69.4}{33.4}  & \PRemB{\ZeroRemZZB}{\ZeroRemB}  &                   & \NegChg{-0.1}{\ZeroVal}  \\
  \midrule
  \TargetTypeHead{\colorDog}{\colorYAdv}
     & \Atk{1~Pixel}
     & \CosInM    & \PRemB{94.0}{22.5}  & \PRem{45.77}{20.39}             & \PAcc{98.8}{0.1}  & \NegChg{-14.8}{-18.1} \\
     && \LayerM   & \PRem{89.0}{34.9}   & \PRemB{44.90}{22.19}            &                   & \NegChg{15.4}{-16.0} \\
  \DsSep
     & \Atk{4~Pixel}
     & \CosInM    & \PRem{99.3}{3.1}    & \PRem{51.25}{24.27}             & \PAcc{98.9}{0.1}  & \NegChg{-20.9}{-21.0} \\
     && \LayerM   & \PRemB{99.7}{1.0}   & \PRemB{36.97}{18.86}            &                   & \NegChg{-9.5}{-13.7} \\
  \DsSep
     & \Atk{Blend}
     & \CosInM    & \PRem{98.6}{5.2}    & \PRem{34.42}{26.43}             & \PAcc{99.0}{0.1}  & \NegChg{-11.8}{-15.6} \\
     && \LayerM   & \PRemB{99.2}{4.9}   & \PRemB{28.29}{20.02}            &                   & \NegChg{-6.2}{-9.5} \\
  \bottomrule
\end{tabular}
    }%
  }
\end{table}

\vfill%
\phantom{.}
 
\FloatBarrier
\clearpage
\newpage
\newcommand{\AblationPlotWidth}{0.48\textwidth}

\revTwo{%
\subsection{Poisoning-Rate Ablation Study}%
\label{sec:App:MoreExps:Ablation:PoisRate}
}%

\revTwo{%
This section analyzes the effect poisoning rate (i.e.,~fraction of the training set that is adversarial) has on our method's performance.
Target identification and target-driven attack mitigation rely on successfully identifying the adversarial set~\eqsmall{$\advTrain$}.
As such, the ablation study focuses on the effect of poisoning rate on adversarial-instance identification -- providing results for all four attacks in Section~\ref{sec:ExpRes:Attacks}.

Recall from Section~\ref{sec:ExpRes} that our method had the worst performance on \citepos{Zhu:2019} vision poisoning attack.
Therefore, we focus on that attack and study the effect of poisoning rate on the attack's target identification and target-driven mitigation performance.
For completeness, our ablation study also includes target identification and target-driven mitigation results for one backdoor attack.
We selected \citepos{Liu:2018} speech recognition backdoor attack since it is a different data modality and dataset than \citeauthor{Zhu:2019} -- unlike \citet{Weber:2021} which also uses CIFAR10.

Overall, our method was remarkably stable across poisoning rates in all tested cases.  Evaluation was limited to those poisoning rates that had ${{\geq}50\%}$~ASR on related experiment setups.

As detailed in Section~\ref{sec:App:ExpSetup:Hyperparams:Crafting}, \citet{SpeechDataset} provide a speech recognition dataset that comes bundled with 300~backdoor training examples, where
\citeauthor{SpeechDataset}'s adversarial trigger was white noise inserted at the beginning of the speech recording.
Our speech recognition experiments used a \textit{fixed} validation set selected u.a.r.  Table~\ref{tab:App:ExpSetup:NumSpeechBackdoor}
details the number of backdoor \textit{training} instances for each speech class pair, with the remaining instances (out of 30)
being part of said validation set.  The ablation study was limited to ${{\geq}10}$~backdoor instances to ensure the attack succeeded for all class pairs.

For \citepos{Wallace:2021} natural-language poisoning attack, adversarial sets with fewer than 10~instances did not consistently succeed and are excluded in the analysis below.

\citepos{Zhu:2019} vision poisoning attack was limited to \eqsmall{${\abs{\advTrain} \leq 400}$} since larger quantities exceeded the Nvidia K80's GPU VRAM capacity of 11.5GB.

The following three subsections visualize our method's performance across adversarial-set identification, target identification, and attack mitigation, respectively.
To improve this section's readability, tables with the full numerical adversarial-set identification and target identification results (including variance) are deferred to Section~\ref{sec:App:MoreExps:Ablation:PoisRate:Tables}.
}%

\clearpage
\newpage
\subsubsection{Adversarial-Set Identification Ablation Study}%
\label{sec:App:MoreExps:Ablation:PoisRate:AdvIdent}
\phantom{.}

\vspace{8pt}
\revTwo{%
To identify the target (Alg.~\ref{alg:TargDetect}) and mitigate the attack (Alg.~\ref{alg:Mitigation}), we must be able to identify the adversarial set,~\eqsmall{$\advTrain$}. Our approach uses (renormalized) influence estimators, which should rank the adversarial set as more influential on the target than clean instances~(\eqsmall{$\cleanTrain$}).
This section evaluates how well different influence estimators perform this ranking for a u.a.r.\ target across various poisoning rates.

Across all four attacks (Figures~\ref{fig:ExpRes:Ablation:Backdoor:Speech:CntSweep:AdvIdent}\==\ref{fig:ExpRes:Ablation:Pois:Cifar:CntSweep:AdvIdent}),
these experiments highlight \bothMethod{}'s stability identifying adversarial set~\eqsmall{$\advTrain$} across the poisoning rate spectrum -- even outperforming Section~\ref{sec:ExpRes}'s results in many cases.%
}

\vfill
\begin{figure}[h!]
  \centering
\newcommand{\legendSpacer}{\hspace*{9pt}}

\begin{tikzpicture}
  \begin{axis}[
      hide axis,  %
      xmin=0,  %
      xmax=1,
      ymin=0,
      ymax=1,
      scale only axis,width=1mm, %
      mark size=0pt,
      line width=\CdfLineWidth,
      legend cell align={left},              %
      legend style={font=\legendFontSize},
      legend columns=6,
      legend image post style={xscale=0.6},  %
    ]
    \addplot [blue] coordinates {(0,0)};
    \addlegendentry{\method\ours\legendSpacer}
    \addplot [red] coordinates {(0,0)};
    \addlegendentry{\layer\ours\legendSpacer}

    \addplot [brown!20!black, dashed] coordinates {(0,0)};
    \addlegendentry{\tracinCP{}\legendSpacer}
    \pgfplotsset{cycle list shift=-1}  %

    \addplot coordinates {(0,0)};
    \addlegendentry{\tracin{}\legendSpacer}

    \addplot [gray] coordinates {(0,0)};
    \addlegendentry{Influence Func.\legendSpacer}

    \addplot [rep pt] coordinates {(0,0)};
    \addlegendentry{Representer Pt.}
  \end{axis}
\end{tikzpicture}
 
  \vspace{12pt}
  \begin{minipage}{\AblationPlotWidth}
    \centering
\pgfplotstableread[col sep=comma] {plots/data/bd_speech_cnt_sweep_auprc.csv}\thedata%
\begin{tikzpicture}
  \begin{axis}[
    smooth,
    width={8cm},%
    height={\AdvIdentAblationHeight},%
    xmin=10,%
    xmax=30,%
    xtick distance={5},
    x tick label style={font=\plotFontSize,align=center},%
    xlabel={\plotFontSize \# Backdoors ($\abs{\advTrain}$)},
    xmajorgrids,
    ymin=0,
    ymax=1,
    ytick distance={0.20},
    minor y tick num={3},
    ymajorgrids,
    y tick label style={font=\plotFontSize,align=center},%
    ylabel={\plotFontSize AUPRC},
    mark size=0pt,
    line width=\CdfLineWidth,
    ]
    \addplot [blue] table [x index=0, y index=1] \thedata;
    \addplot [red] table [x index=0, y index=2] \thedata;

    \addplot [brown!20!black, dashed] table [x index=0, y index=3] \thedata;
    \pgfplotsset{cycle list shift=-1}  %

    \addplot table [x index=0, y index=4] \thedata;

    \addplot [gray] table [x index=0, y index=5] \thedata;

    \addplot [violet] table [x index=0, y index=6] \thedata;
  \end{axis}
\end{tikzpicture}
     \captionof{figure}{%
      \revTwo{%
      \textit{Speech Backdoor Rate Adversarial-Set Identification Ablation Study}:
      Effect of the number of backdoor instances on mean adversarial-set identification AUPRC (\eqsmall{${\abs{\cleanTrain} = 3,000}$}).
      \method{}'s performance improved slightly with fewer adversarial instances, while \layer{}'s performance increased with larger~\eqsmall{$\advTrain$}.
      Results are averaged across related experimental setups with ${{\geq}10}$~trials per setup.
      Table~\ref{tab:ExpRes:Ablation:Backdoor:Speech:CntSweep:AdvIdent} provides the full numerical results, including variance.
      }%
    }
    \label{fig:ExpRes:Ablation:Backdoor:Speech:CntSweep:AdvIdent}
  \end{minipage}
  \hfill
  \begin{minipage}{\AblationPlotWidth}
\pgfplotstableread[col sep=comma] {plots/data/bd_cifar_cnt_sweep.csv}\thedata%
\begin{tikzpicture}
  \begin{axis}[
    smooth,
    width={8cm},%
    height={\AdvIdentAblationHeight},%
    xmin=1,%
    xmax=5,%
    xtick distance={1},
    minor x tick num={1},
    x tick label style={font=\plotFontSize,align=center},%
    xlabel={\plotFontSize \% Backdoors},
    xmajorgrids,
    ymin=0,
    ymax=1,
    ytick distance={0.20},
    minor y tick num={3},
    ymajorgrids,
    y tick label style={font=\plotFontSize,align=center},%
    ylabel={\plotFontSize AUPRC},
    mark size=0pt,
    line width=\CdfLineWidth,
    ]
    \addplot [blue] table [x index=0, y index=1] \thedata;
    \addplot [red] table [x index=0, y index=2] \thedata;

    \addplot [brown!20!black, dashed] table [x index=0, y index=3] \thedata;
    \pgfplotsset{cycle list shift=-1}  %

    \addplot table [x index=0, y index=4] \thedata;

    \addplot [gray] table [x index=0, y index=5] \thedata;

    \addplot [violet] table [x index=0, y index=6] \thedata;
  \end{axis}
\end{tikzpicture}
     \caption{%
      \revTwo{%
      \textit{Vision Backdoor Rate Adversarial-Set Identification Ablation Study}:
      Effect of the fraction of the training set that is backdoors on mean adversarial-set identification AUPRC (\eqsmall{${\abs{\fullTrain} = 10,000}$}).
      Only attacks with a minimum success rate of 50\% were considered.
      Results are averaged across related experimental setups with ${{\geq}10}$~trials per setup.
      Table~\ref{tab:ExpRes:Ablation:Backdoor:CIFAR:CntSweep:AdvIdent} provides the full numerical results, including variance.\\~
    }
    }
    \label{fig:ExpRes:Ablation:Backdoor:CIFAR:CntSweep:AdvIdent}
  \end{minipage}

  \vspace{1.0cm}
  \begin{minipage}{\AblationPlotWidth}
    \centering
\pgfplotstableread[col sep=comma] {plots/data/pois_nlp_cnt_sweep.csv}\thedata%
\begin{tikzpicture}
  \begin{axis}[
    smooth,
    width={8cm},%
    height={\AdvIdentAblationHeight},%
    xmin=10,%
    xmax=50,%
    xtick distance={10},
    minor x tick num={1},
    x tick label style={font=\plotFontSize,align=center},%
    xlabel={\plotFontSize \# Poison},
    xmajorgrids,
    ymin=0,
    ymax=1,
    ytick distance={0.20},
    minor y tick num={3},
    ymajorgrids,
    y tick label style={font=\plotFontSize,align=center},%
    ylabel={\plotFontSize AUPRC},
    mark size=0pt,
    line width=\CdfLineWidth,
    ]
    \addplot [blue, line width=1.5pt] table [x index=0, y index=1] \thedata;
    \addplot [red, line width=1.5pt] table [x index=0, y index=2] \thedata;

    \addplot [brown!20!black, dashed] table [x index=0, y index=3] \thedata;
    \pgfplotsset{cycle list shift=-1}  %

    \addplot table [x index=0, y index=4] \thedata;

    \addplot [gray] table [x index=0, y index=5] \thedata;

    \addplot [violet] table [x index=0, y index=6] \thedata;
  \end{axis}
\end{tikzpicture}
     \caption{%
      \revTwo{%
      \textit{Natural-Language Poisoning Rate Adversarial-Set Identification Ablation Study}:
      Effect of the number of adversarial instances on mean adversarial-set identification AUPRC for \citepos{Wallace:2021} natural-language poisoning attack on \SST{} (\eqsmall{${\abs{\cleanTrain} = 67,349}$}).
      While \tracinCP{}'s performance changes significantly across the entire poisoning rate range, \method{} and \layer{}'s performance is essentially perfect.
      Results are averaged across related experimental setups with ${{\geq}5}$~trials per setup for each of the first four reviews.
      Table~\ref{tab:ExpRes:Ablation:Pois:NLP:CntSweep:AdvIdent} provides the full numerical results, including variance.
    }
    }
    \label{fig:ExpRes:Ablation:Pois:NLP:CntSweep:AdvIdent}
  \end{minipage}
  \hfill
  \begin{minipage}{\AblationPlotWidth}
\pgfplotstableread[col sep=comma] {plots/data/pois_cifar_cnt_sweep_auprc.csv}\thedata%
\begin{tikzpicture}
  \begin{axis}[
    smooth,
    width={8cm},%
    height={\AdvIdentAblationHeight},%
    xmin=40,%
    xmax=400,%
    xtick distance={50},
    minor x tick num={0},
    x tick label style={font=\plotFontSize,align=center},%
    xlabel={\plotFontSize \# Poison},
    xmajorgrids,
    ymin=0,
    ymax=1,
    ytick distance={0.20},
    minor y tick num={3},
    ymajorgrids,
    y tick label style={font=\plotFontSize,align=center},%
    ylabel={\plotFontSize AUPRC},
    mark size=0pt,
    line width=\CdfLineWidth,
    ]
    \addplot [blue] table [x index=0, y index=1] \thedata;
    \addplot [red] table [x index=0, y index=2] \thedata;

    \addplot [brown!20!black, dashed] table [x index=0, y index=3] \thedata;
    \pgfplotsset{cycle list shift=-1}  %

    \addplot table [x index=0, y index=4] \thedata;

    \addplot [gray] table [x index=0, y index=5] \thedata;

    \addplot [violet] table [x index=0, y index=6] \thedata;
  \end{axis}
\end{tikzpicture}
     \caption{%
      \revTwo{%
      \textit{Vision Poisoning Rate Adversarial-Set Identification Ablation Study}:
      Effect of the number of adversarial instances on mean adversarial-set identification AUPRC for \citepos{Zhu:2019} vision poisoning attack on CIFAR10 (\eqsmall{${\abs{\fullTrain} = 25,000}$}).
      Results are averaged across related experimental setups with ${{\geq}10}$~trials per setup.
      Only attacks that successfully changed the target's label are considered.
      The adversarial set was limited to ${{\leq}400}$~instances by the Nvidia K80's GPU VRAM capacity of 11.5GB.
      Table~\ref{tab:ExpRes:Ablation:Pois:Cifar:CntSweep:AdvIdent} provides the full numerical results, including variance.
    }
    }
    \label{fig:ExpRes:Ablation:Pois:Cifar:CntSweep:AdvIdent}
  \end{minipage}
\end{figure}
\vfill

\phantom{.}

\clearpage
\newpage
\subsubsection{Target-Identification Ablation Study}%
\label{sec:App:MoreExps:Ablation:PoisRate:TargDetect}
\phantom{.}

\vspace{8pt}
\revTwo{%
  This section considers poisoning rate's effects on target identification.
  First, for \citepos{SpeechDataset} speech backdoor attack, we consider two class pairs, one of which is ${1 \rightarrow 2}$ for which we observed the worst performance among those speech class pairs tested (see Table~\ref{tab:App:MoreExps:Backdoor:Speech:AdvIdent}).
  Note that in Section~\ref{sec:ExpRes}, ${\anomCount = 10}$ was used to determine the speech attack's upper-tail heaviness.
  However, these experiments consider the case where adversarial set~\eqsmall{$\advTrain$} has as few as~6 instances.
  As discussed in Section~\ref{sec:Method:TargetDetection}, if \eqsmall{${\abs{\advTrain} < \anomCount}$}, target identification performance degrades severely.
  As such, Figure~\ref{fig:ExpRes:Ablation:Backdoor:Speech:CntSweep:TargDetect} uses an upper-tail count of \eqsmall{${\anomCount = \min\set{10,\abs{\advTrain}}}$}.

  The other attack considered below is \citepos{Zhu:2019} vision poisoning attack on CIFAR10.  We selected
  this attack since it was the one on which we observed the worst target identification performance (see
  Figure~\ref{sec:ExpRes:TargetIdent}).  Note that as poisoning rate increased, a larger upper-tail count~($\anomCount$)
  improved target identification performance.  Below we report target identification performance on vision poisoning with
  ${\anomCount = 2}$ (like in Section~\ref{sec:ExpRes:TargetIdent}) as well as with ${\anomCount = 10}$.

  Overall, target identification was stable across all tested poisoning rates for both the backdoor and poisoning attacks.
}

\vfill
\begin{figure}[h!]
  \centering
\newcommand{\legendSpacer}{\hspace*{9pt}}

\begin{tikzpicture}
  \begin{axis}[%
    width=\textwidth,
      hide axis,  %
      xmin=0,  %
      xmax=1,
      ymin=0,
      ymax=1,
      scale only axis,width=1mm, %
      legend cell align={left},              %
      legend style={font=\legendFontSize},
      legend columns=7,
      mark size=0pt,
      line width=1pt,
      cycle list name=DetectCycleList,
    ]
    \addplot[blue] coordinates {(0,0)};
    \addlegendentry{\fitWith{\method}\ours\legendSpacer}
    \addplot[red] coordinates {(0,0)};
    \addlegendentry{\fitWith{\layer}\ours\legendSpacer}

    \addplot [Max KNN] coordinates {(0,0)};
    \addlegendentry{Max.\ \KNN{} Distance\legendSpacer}

    \addplot [Min KNN] coordinates {(0,0)};
    \addlegendentry{Min.\ \KNN{} Distance\legendSpacer}
    \addplot [Most Certain] coordinates {(0,0)};
    \addlegendentry{Most Certain\legendSpacer};

    \addplot [Least Certain] coordinates {(0,0)};
    \addlegendentry{Least Certain};
  \end{axis}
\end{tikzpicture}
 
  \vspace{6pt}
  \begin{minipage}{\AblationPlotWidth}
    \centering
\pgfplotstableread[col sep=comma] {plots/data/bd_speech_cnt_sweep_detect.csv}\thedata%
\begin{tikzpicture}
  \begin{axis}[
    smooth,
    width={8cm},%
    height={\TargIdentAblationHeight},%
    xmin=10,%
    xmax=30,%
    xtick distance={10},
    minor x tick num={1},
    x tick label style={font=\plotFontSize,align=center},%
    xlabel={\plotFontSize \# Backdoor},
    xmajorgrids,
    ymin=0,
    ymax=1,
    ytick distance={0.20},
    minor y tick num={3},
    ymajorgrids,
    y tick label style={font=\plotFontSize,align=center},%
    ylabel={\plotFontSize AUPRC},
    mark size=0pt,
    line width=\CdfLineWidth,
    ]
    \addplot [blue] table [x index=0, y index=1] \thedata;
    \addplot [red] table [x index=0, y index=2] \thedata;

    \addplot [Max KNN] table [x index=0, y index=3] \thedata;
    \addplot [Min KNN] table [x index=0, y index=4] \thedata;

    \addplot [Most Certain] table [x index=0, y index=5] \thedata;
    \addplot [Least Certain] table [x index=0, y index=6] \thedata;
  \end{axis}
\end{tikzpicture}
     \caption{%
      \revTwo{%
      \textit{Speech Backdoor Rate Target Identification Ablation Study}:
      Effect of the number of backdoor instances on mean target identification AUPRC for class pairs ${0 \rightarrow 1}$ and ${1 \rightarrow 2}$ (\eqsmall{${\abs{\cleanTrain} = 3,000}$}).
      \citepos{SpeechDataset} speech dataset includes 30~backdoor instances per class, which were downsampled uniformly at random.
      Results are averaged across related experimental setups with ${{\geq}5}$~trials per setup.
      Table~\ref{tab:ExpRes:Ablation:Backdoor:Speech:CntSweep:TargDetect} provides the full numerical results, including variance.
      }%
    }
    \label{fig:ExpRes:Ablation:Backdoor:Speech:CntSweep:TargDetect}
  \end{minipage}
  \hfill
  \begin{minipage}{\AblationPlotWidth}
    \mbox{\hfill}
  \end{minipage}

  \vspace{1.0cm}
\newcommand{\PlotCifarPoisDetectCntSweep}[1]{%
  \pgfplotstableread[col sep=comma] {plots/data/pois_cifar_cnt_sweep_detect_k=#1.csv}\thedata%
  \begin{tikzpicture}
    \begin{axis}[
      smooth,
      width={8cm},%
      height={\TargIdentAblationHeight},%
      xmin=50,%
      xmax=400,%
      xtick distance={50},
      minor x tick num={0},
      x tick label style={font=\plotFontSize,align=center},%
      xlabel={\plotFontSize \# Poison},
      xmajorgrids,
      ymin=0,
      ymax=1,
      ytick distance={0.20},
      minor y tick num={3},
      ymajorgrids,
      y tick label style={font=\plotFontSize,align=center},%
      ylabel={\plotFontSize AUPRC},
      mark size=0pt,
      line width=\CdfLineWidth,
      ]
      \addplot [blue] table [x index=0, y index=1] \thedata;
      \addplot [red] table [x index=0, y index=2] \thedata;

      \addplot [Max KNN] table [x index=0, y index=3] \thedata;
      \addplot [Min KNN] table [x index=0, y index=4] \thedata;

      \addplot [Most Certain] table [x index=0, y index=5] \thedata;
      \addplot [Least Certain] table [x index=0, y index=6] \thedata;
    \end{axis}
  \end{tikzpicture}
}
   \begin{minipage}{\AblationPlotWidth}
    \begin{subfigure}{\textwidth}
      \centering
      \PlotCifarPoisDetectCntSweep{2}

      \caption{${\anomCount = 2}$}
    \end{subfigure}
  \end{minipage}
  \hfill
  \begin{minipage}{\AblationPlotWidth}
    \begin{subfigure}{\textwidth}
      \centering
      \PlotCifarPoisDetectCntSweep{10}

      \caption{${\anomCount = 10}$}
    \end{subfigure}
  \end{minipage}
  \caption{%
    \revTwo{%
    \textit{Vision Poisoning Rate Target Identification Ablation Study}:
    Effect of the number of adversarial instances on mean target identification AUPRC for \citepos{Zhu:2019} vision poisoning attack on CIFAR10 (\eqsmall{${\abs{\fullTrain} = 25,000}$}).
    Results are averaged across all four class pairs with ${{\geq}10}$~trials per setup.
    Table~\ref{tab:ExpRes:Ablation:Pois:Cifar:CntSweep:TargDetect} provides the full numerical results, including variance.
    }
  }
  \label{fig:ExpRes:Ablation:Pois:Cifar:CntSweep:TargDetect}
\end{figure}

\vfill
\phantom{.}

\clearpage
\newpage
\subsubsection{Target-Driven Attack Mitigation Ablation Study}%
\label{sec:App:MoreExps:Ablation:PoisRate:Mitigation}
\phantom{.}

\vspace{8pt}
\revTwo{%
This section considers poisoning rate's effects on target-driven attack mitigation -- specifically for the
speech backdoor and vision poisoning attacks.  In both cases, the fraction of clean data removed was
slightly larger when the adversarial set was small.

For \citepos{Zhu:2019} vision poisoning attack, the average fraction of clean data removed was ${{\leq}0.26\%}$ across the poisoning rate range -- a very small fraction.
Similarly, the test accuracy post-sanitization never lagged the baseline test accuracy by more than 0.2\%;
see Table~\ref{tab:ExpRes:Ablation:Pois:Cifar:CntSweep:TargDetect} for the baseline results.
}

\phantom{.}
\vfill

\begin{table}[h!]
  \centering
  \caption{%
    \revTwo{%
    \textit{Speech Backdoor Rate Target-Driven Attack Mitigation Ablation Study}:
    Algorithm~\ref{alg:Mitigation}'s target-driven, iterative data sanitization applied to
    \citepos{SpeechDataset} backdoored speech recognition dataset for
    class pairs ${0 \rightarrow 1}$ and ${1 \rightarrow 2}$ across different backdoor quantities (\eqsmall{${\abs{\cleanTrain} = 3,000}$}).
    The attacks were neutralized with few clean instances removed and little change in test accuracy.
    Attack success rate~(ASR) is w.r.t.\ specifically the analyzed target.
    Bold denotes the best mean performance with 10~trials per class pair.
    }%
  }%
  \label{tab:ExpRes:Ablation:Backdoor:Speech:CntSweep:Mitigation}
  {
    \TableFontSize%
    \revTwo{%
\renewcommand{\arraystretch}{1.2}
\setlength{\dashlinedash}{0.4pt}
\setlength{\dashlinegap}{1.5pt}
\setlength{\arrayrulewidth}{0.3pt}

\newcommand{\MultiHead}[1]{\multicolumn{2}{c}{#1}}

\newcommand{\TwoRowHead}[1]{\multirow{2}{*}{#1}}
\newcommand{\ClassPair}[2]{\multirow{2}{*}{#1} & \multirow{2}{*}{#2}}
\newcommand{\PZ}{\phantom{0}}
\newcommand{\ptZ}{\phantom{.}\PZ}
\newcommand{\ptZZ}{\phantom{.}\PZ\PZ}
\newcommand{\ptASR}{}

\newcommand{\CosInM}{\method{}}
\newcommand{\LayerM}{\layer{}}

\newcommand{\OneRemB}{100\ptZ}
\newcommand{\ZeroRemB}{0\ptZ}

\newcommand{\OneASR}{100\ptZ}

\newcommand{\PRem}[2]{#1}
\newcommand{\PRemB}[2]{\textBF{#1}}
\newcommand{\ASR}[2]{\multirow{2}{*}{#1}}
\newcommand{\PAcc}[2]{\multirow{2}{*}{#1}}
\newcommand{\PAsr}[2]{#1}

\newcommand{\PChg}[2]{{\color{ForestGreen} +#1}}
\newcommand{\NoChg}{0.0}
\newcommand{\NegChg}[2]{{\color{BrickRed} -#1}}

\newcommand{\DsSep}{\cdashline{2-10}}
\newcommand{\HoldoutSep}{\cmidrule{1-10}}
\newcommand{\MethodSep}{}

\newcommand{\NumBackdoor}[1]{\multirow{4}{*}{#1}}

\begin{tabular}{ccclrrrrrr}
  \toprule
  \TwoRowHead{\#~BD}
     & \MultiHead{Digits}
     & \TwoRowHead{Method} & \MultiHead{\% Removed}      & \MultiHead{ASR \%} & \MultiHead{Test Acc. \%} \\\cmidrule(lr){2-3}\cmidrule(lr){5-6}\cmidrule(lr){7-8}\cmidrule(lr){9-10}
     & $\yTarg$ & $\yAdv$
     &                    & $\advTrain$  & $\cleanTrain$ & Orig.   & Ours     &  Orig. & Chg.   \\
  \midrule
  \NumBackdoor{10}
  & \ClassPair{0}{1}
     & \CosInM     & \PRemB{\OneRemB}{\ZeroRemB} & \PRem{0.17}{0.11}   & \PAcc{97.2}{3.2}    & \PAsr{0}{0}   & \PAcc{97.7}{0.1}  & \NoChg \\
     &&& \LayerM   & \PRemB{\OneRemB}{\ZeroRemB} & \PRemB{0.08}{0.05}  &                     & \PAsr{0}{0}   &                   & \NoChg \\
  \DsSep
  & \ClassPair{1}{2}
     &  \CosInM    & \PRemB{99.0}{2.1}           & \PRemB{0.03}{0.05}  & \PAcc{\OneASR}{0.0} & \PAsr{0}{0}   & \PAcc{97.8}{0.2}  & \NoChg \\
     &&& \LayerM   & \PRem{98.7}{2.8}            & \PRem{0.07}{0.06}   &                     & \PAsr{0}{0}   &                   & \NoChg \\
  \HoldoutSep
  \NumBackdoor{15}
  & \ClassPair{0}{1}
     & \CosInM     & \PRemB{\OneRemB}{\ZeroRemB} & \PRem{0.12}{0.13}   & \PAcc{99.8}{0.4}    & \PAsr{0}{0}   & \PAcc{97.8}{0.1}  & \NegChg{-0.1}{0.0} \\
     &&& \LayerM   & \PRemB{\OneRemB}{\ZeroRemB} & \PRemB{0.08}{0.07}  &                     & \PAsr{0}{0}   &                   & \NegChg{-0.1}{0.1} \\
  \DsSep
  & \ClassPair{1}{2}
     &  \CosInM    & \PRemB{99.4}{1.1}           & \PRemB{0.04}{0.03}  & \PAcc{\OneASR}{0.0} & \PAsr{0}{0}   & \PAcc{97.8}{0.1}  & \NegChg{-0.1}{0.0} \\
     &&& \LayerM   & \PRem{98.8}{1.4}            & \PRem{0.07}{0.07}   &                     & \PAsr{0}{0}   &                   & \NoChg \\
  \HoldoutSep
  \NumBackdoor{20}
  & \ClassPair{0}{1}
     & \CosInM     & \PRemB{\OneRemB}{\ZeroRemB} & \PRem{0.04}{0.04}   & \PAcc{\OneASR}{0.0} & \PAsr{0}{0}   & \PAcc{97.7}{0.1}  & \NoChg \\
     &&& \LayerM   & \PRemB{\OneRemB}{\ZeroRemB} & \PRemB{0.02}{0.03}  &                     & \PAsr{0}{0}   &                   & \NegChg{-0.1}{0.0} \\
  \DsSep
  & \ClassPair{1}{2}
     &  \CosInM    & \PRem{99.3}{0.7}            & \PRemB{0.00}{0.01}  & \PAcc{\OneASR}{0.0} & \PAsr{0}{0}   & \PAcc{97.7}{0.1}  & \NoChg \\
     &&& \LayerM   & \PRemB{\OneRemB}{\ZeroRemB} & \PRem{0.04}{0.01}   &                     & \PAsr{0}{0}   &                   & \NegChg{-0.1}{-0.1} \\
  \HoldoutSep
  \NumBackdoor{25}
  & \ClassPair{0}{1}
     & \CosInM     & \PRemB{\OneRemB}{\ZeroRemB} & \PRem{0.07}{0.07}   & \PAcc{\OneASR}{0.0} & \PAsr{0}{0}   & \PAcc{97.8}{0.2}  & \NoChg \\
     &&& \LayerM   & \PRemB{\OneRemB}{\ZeroRemB} & \PRemB{0.02}{0.04}  &                     & \PAsr{0}{0}   &                   & \NoChg \\
  \DsSep
  & \ClassPair{1}{2}
     &  \CosInM    & \PRemB{99.7}{0.6}           & \PRemB{0.00}{0.01}  & \PAcc{\OneASR}{0.0} & \PAsr{0}{0}   & \PAcc{97.8}{0.1}  & \NoChg \\
     &&& \LayerM   & \PRem{99.3}{1.2}            & \PRem{0.13}{0.16}   &                     & \PAsr{0}{0}   &                   & \NegChg{-0.1}{0.0} \\
  \HoldoutSep
  \NumBackdoor{30}
  & \ClassPair{0}{1}
     & \CosInM     & \PRemB{\OneRemB}{\ZeroRemB} & \PRem{0.06}{0.05}   & \PAcc{\OneASR}{0.0} & \PAsr{0}{0}   & \PAcc{97.7}{0.1}  & \NoChg \\
     &&& \LayerM   & \PRemB{\OneRemB}{\ZeroRemB} & \PRemB{0.03}{0.03}  &                     & \PAsr{0}{0}   &                   & \NoChg \\
  \DsSep
  & \ClassPair{1}{2}
     &  \CosInM    & \PRemB{100.0}{0.1}          & \PRemB{0.02}{0.02}  & \PAcc{\OneASR}{0.0} & \PAsr{0}{0}   & \PAcc{97.7}{0.1}  & \NoChg \\
     &&& \LayerM   & \PRem{99.8}{0.4}            & \PRem{0.09}{0.07}   &                     & \PAsr{0}{0}   &                   & \NoChg \\
  \bottomrule
\end{tabular}
     }%
  }
\end{table}

\vfill
\phantom{.}

\clearpage
\newpage

\phantom{.}
\vfill

\begin{table}[h!]
  \centering
  \caption{%
    \revTwo{%
    \textit{Vision Poisoning Rate Target-Driven Attack Mitigation Ablation Study}:
    Algorithm~\ref{alg:Mitigation}'s target-driven, iterative data sanitization applied to
    \citepos{Zhu:2019} vision poisoning attack on CIFAR10 for
    class pairs ${\texttt{bird} \rightarrow \texttt{dog}}$ and ${\texttt{dog} \rightarrow \texttt{bird}}$
    (\eqsmall{${\abs{\fullTrain} = 25,000}$}).
    The attacks were neutralized with few clean instances removed and little change in test accuracy.
    Attack success rate~(ASR) is w.r.t.\ specifically the analyzed target.
    Bold denotes the best mean performance with ${{\geq}5}$~trials per class pair.
    }%
  }%
  \label{tab:ExpRes:Ablation:Pois:Vision:CntSweep:Mitigation}
  {
    \TableFontSize%
    \revTwo{%
\renewcommand{\arraystretch}{1.2}
\setlength{\dashlinedash}{0.4pt}
\setlength{\dashlinegap}{1.5pt}
\setlength{\arrayrulewidth}{0.3pt}

\newcommand{\MultiHead}[1]{\multicolumn{2}{c}{#1}}

\newcommand{\TwoRowHead}[1]{\multirow{2}{*}{#1}}
\newcommand{\NPois}[1]{\multirow{4}{*}{#1}}
\newcommand{\ClassPair}[2]{\multirow{2}{*}{#1} & \multirow{2}{*}{#2}}
\newcommand{\PZ}{\phantom{0}}
\newcommand{\ptZ}{\phantom{.}\PZ}
\newcommand{\ptZZ}{\phantom{.}\PZ\PZ}
\newcommand{\ptASR}{}

\newcommand{\CosInM}{\method{}}
\newcommand{\LayerM}{\layer{}}

\newcommand{\OneRemB}{100\ptZ}
\newcommand{\ZeroRemB}{0\ptZ}

\newcommand{\PRem}[2]{#1}
\newcommand{\PRemB}[2]{\textBF{#1}}
\newcommand{\ASR}[2]{\multirow{2}{*}{#1}}
\newcommand{\PAcc}[2]{\multirow{2}{*}{#1}}
\newcommand{\PAsr}[2]{#1}

\newcommand{\PChg}[2]{{\color{ForestGreen} +#1}}
\newcommand{\NoChg}{0.0}
\newcommand{\NegChg}[2]{{\color{BrickRed} -#1}}

\newcommand{\DsSep}{\cdashline{2-10}}
\newcommand{\MethodSep}{}

\begin{tabular}{lcclrrrrrr}
  \toprule
     \TwoRowHead{\#~Pois.} & \MultiHead{Classes}
         & \TwoRowHead{Method} & \MultiHead{\% Removed}      & \MultiHead{ASR \%} & \MultiHead{Test Acc. \%} \\\cmidrule(lr){2-3}\cmidrule(lr){5-6}\cmidrule(lr){7-8}\cmidrule(lr){9-10}
       & $\yTarg$ & $\yAdv$
         &                    & $\advTrain$  & $\cleanTrain$ & Orig.   & Ours     &  Orig. & Chg.   \\
  \midrule
      \NPois{50}  &   \ClassPair{Bird}{Dog}
                  &   \CosInM &   \PRemB{38.0}{19.9}  &   \PRem{0.08}{0.15}   &   \PAcc{93.8}{}   &   \PAsr{0}{0} &   \PAcc{87.0}{0.3}    &   \PChg{0.2}{0.0} \\
          &       &&  \LayerM &   \PRem{26.4}{21.5}   &   \PRemB{0.07}{0.11}  &                   &   \PAsr{0}{0} &                       &   \PChg{0.4}{0.0} \\\DsSep
          &   \ClassPair{Dog}{Bird}
                  &   \CosInM &   \PRemB{43.2}{19.0}  &   \PRemB{0.00}{0.01}  &   \PAcc{68.8}{}   &   \PAsr{0}{0} &   \PAcc{87.1}{0.3}    &   \NoChg  \\
          &       &&  \LayerM &   \PRem{31.2}{24.1}   &   \PRemB{0.00}{0.01}  &                   &   \PAsr{0}{0} &                       &   \NoChg  \\\midrule
      \NPois{150} &   \ClassPair{Bird}{Dog}
                  &   \CosInM &   \PRem{51.6}{35.5}   &   \PRemB{0.07}{0.13}  &   \PAcc{75.0}{}   &   \PAsr{0}{0} &   \PAcc{86.9}{0.4}    &   \PChg{0.1}{0.1} \\
          &       &&  \LayerM &   \PRemB{55.3}{26.8}  &   \PRem{0.13}{0.18}   &                   &   \PAsr{0}{0} &                       &   \PChg{0.1}{0.2} \\\DsSep
          &   \ClassPair{Dog}{Bird}
                  &   \CosInM &   \PRemB{45.8}{39.6}  &   \PRem{0.03}{0.05}   &   \PAcc{62.5}{}   &   \PAsr{0}{0} &   \PAcc{87.1}{0.5}    &   \PChg{0.1}{0.1} \\
          &       &&  \LayerM &   \PRem{37.0}{30.5}   &   \PRemB{0.00}{0.00}  &                   &   \PAsr{0}{0} &                       &   \NegChg{-0.1}{0.0}  \\\midrule
      \NPois{200} &   \ClassPair{Bird}{Dog}
                  &   \CosInM &   \PRem{67.7}{23.7}   &   \PRemB{0.07}{0.14}  &   \PAcc{68.8}{}   &   \PAsr{0}{0} &   \PAcc{87.0}{0.2}    &   \NoChg  \\
          &       &&  \LayerM &   \PRemB{72.1}{21.3}  &   \PRem{0.11}{0.22}   &                   &   \PAsr{0}{0} &                       &   \NoChg  \\\DsSep
          &   \ClassPair{Dog}{Bird}
                  &   \CosInM &   \PRemB{50.1}{35.0}  &   \PRem{0.02}{0.05}   &   \PAcc{62.5}{}   &   \PAsr{0}{0} &   \PAcc{87.3}{0.3}    &   \NegChg{-0.2}{0.1}  \\
          &       &&  \LayerM &   \PRem{37.2}{33.6}   &   \PRemB{0.00}{0.00}  &                   &   \PAsr{0}{0} &                       &   \NegChg{-0.2}{-0.1} \\\midrule
      \NPois{250} &   \ClassPair{Bird}{Dog}
                  &   \CosInM &   \PRemB{60.3}{35.0}  &   \PRem{0.08}{0.12}   &   \PAcc{87.5}{}   &   \PAsr{0}{0} &   \PAcc{87.0}{0.1}    &   \NegChg{-0.1}{-0.2} \\
          &       &&  \LayerM &   \PRem{46.6}{40.0}   &   \PRemB{0.02}{0.02}  &                   &   \PAsr{0}{0} &                       &   \NegChg{-0.2}{-0.3} \\\DsSep
          &   \ClassPair{Dog}{Bird}
                  &   \CosInM &   \PRemB{28.8}{28.3}  &   \PRemB{0.00}{0.01}  &   \PAcc{81.3}{}   &   \PAsr{0}{0} &   \PAcc{87.2}{0.3}    &   \NegChg{-0.2}{0.1}  \\
          &       &&  \LayerM &   \PRem{26.4}{33.1}   &   \PRemB{0.00}{0.01}  &                   &   \PAsr{0}{0} &                       &   \NoChg  \\\midrule
      \NPois{400} &   \ClassPair{Bird}{Dog}
                  &   \CosInM &   \PRemB{68.1}{27.7}  &   \PRemB{0.18}{0.26}  &   \PAcc{87.5}{}   &   \PAsr{0}{0} &   \PAcc{86.9}{0.3}    &   \PChg{0.1}{0.0} \\
          &       &&  \LayerM &   \PRem{63.1}{28.3}   &   \PRem{0.26}{0.42}   &                   &   \PAsr{0}{0} &                       &   \NoChg  \\\DsSep
          &   \ClassPair{Dog}{Bird}
                  &   \CosInM &   \PRemB{64.3}{34.9}  &   \PRem{0.06}{0.10}   &   \PAcc{43.8}{}   &   \PAsr{0}{0} &   \PAcc{87.1}{0.3}    &   \NoChg  \\
          &       &&  \LayerM &   \PRem{42.5}{45.0}   &   \PRemB{0.04}{0.06}  &                   &   \PAsr{0}{0} &                       &   \PChg{0.3}{0.0} \\
  \bottomrule
\end{tabular}
     }%
  }
\end{table}

\vfill
\phantom{.}

\clearpage
\newpage
\subsubsection{Ablation Study Reference Tables}%
\label{sec:App:MoreExps:Ablation:PoisRate:Tables}
\phantom{.}

\vspace{8pt}
\revTwo{%
As explained above, the ablation study's result tables were deferred until this section to improve overall readability.
}

\vfill
\begin{table}[h!]
  \centering
  \caption{%
  \revTwo{%
    \textit{Speech Backdoor Rate Adversarial-Set Identification Ablation Study}:
    Mean and standard deviation adversarial-instance identification AUPRC for \citepos{SpeechDataset} speech dataset, which contains 30~backdoor instances per class pair and 3,000~total clean instances.
    Bold denotes the best mean performance with ${{\geq}10}$~trials per setup.
    The backdoor count includes examples held-out as part of the fixed validation set (see Table~\ref{tab:App:ExpSetup:NumSpeechBackdoor}), with results
    limited to ${{\geq}10}$~backdoor instances to ensure the attack succeeds consistently for all class pairs.
    These experiments highlight the stability of \method{} and \layer{} at identifying adversarial set~\eqsmall{$\advTrain$} even at small poisoning rates.
    Results are shown graphically in Figure~\ref{fig:ExpRes:Ablation:Backdoor:Speech:CntSweep:AdvIdent}.
  }%
  }%
  \label{tab:ExpRes:Ablation:Backdoor:Speech:CntSweep:AdvIdent}
  {
    \TableFontSize%
    \revTwo{%
\renewcommand{\arraystretch}{1.2}
\setlength{\dashlinedash}{0.4pt}
\setlength{\dashlinegap}{1.5pt}
\setlength{\arrayrulewidth}{0.3pt}

\newcommand{\RevType}[1]{\multirow{4}{*}{\shortstack{{\Large $\uparrow$} \\ #1 \\ {\Large $\downarrow$}}}}
\newcommand{\RevNum}[1]{#1}

\newcommand{\TabMidRule}{\cdashline{2-7}}

\newcommand{\AlgName}[1]{\multicolumn{1}{c}{#1}}

\begin{tabular}{crrrrrr}
  \toprule
  \multirow{2}{*}{\#~Backdoors} & \multicolumn{2}{c}{Ours} & \multicolumn{4}{c}{Baselines} \\\cmidrule(lr){2-3}\cmidrule(l){4-7}
                                & \AlgName{\method{}}  & \AlgName{\layer}  & \AlgName{\tracinCP{}} & \AlgName{\tracin{}}  & \AlgName{Influence Func.}  &  {Representer Pt.}  \\
  \midrule
  10  & \PValB{0.968}{0.059}  & \PVal{0.922}{0.090}  & \PVal{0.769}{0.117}  & \PVal{0.701}{0.136}  & \PVal{0.569}{0.150}  & \PVal{0.140}{0.082}  \\\TabMidRule
  12  & \PValB{0.974}{0.047}  & \PVal{0.939}{0.081}  & \PVal{0.772}{0.142}  & \PVal{0.674}{0.154}  & \PVal{0.507}{0.170}  & \PVal{0.116}{0.069}  \\\TabMidRule
  15  & \PValB{0.986}{0.026}  & \PVal{0.954}{0.055}  & \PVal{0.768}{0.175}  & \PVal{0.657}{0.199}  & \PVal{0.547}{0.149}  & \PVal{0.097}{0.038}  \\\TabMidRule
  17  & \PValB{0.985}{0.021}  & \PVal{0.960}{0.039}  & \PVal{0.694}{0.224}  & \PVal{0.568}{0.208}  & \PVal{0.522}{0.167}  & \PVal{0.100}{0.074}  \\\TabMidRule
  20  & \PValB{0.980}{0.038}  & \PVal{0.957}{0.059}  & \PVal{0.685}{0.186}  & \PVal{0.545}{0.162}  & \PVal{0.532}{0.168}  & \PVal{0.099}{0.087}  \\\TabMidRule
  22  & \PValB{0.986}{0.022}  & \PVal{0.956}{0.044}  & \PVal{0.701}{0.212}  & \PVal{0.532}{0.191}  & \PVal{0.623}{0.129}  & \PVal{0.078}{0.086}  \\\TabMidRule
  25  & \PValB{0.990}{0.019}  & \PVal{0.964}{0.044}  & \PVal{0.682}{0.224}  & \PVal{0.506}{0.178}  & \PVal{0.642}{0.135}  & \PVal{0.090}{0.063}  \\\TabMidRule
  27  & \PValB{0.988}{0.025}  & \PVal{0.972}{0.040}  & \PVal{0.677}{0.185}  & \PVal{0.487}{0.155}  & \PVal{0.693}{0.125}  & \PVal{0.082}{0.072}  \\\TabMidRule
  30  & \PValB{0.977}{0.037}  & \PVal{0.974}{0.033}  & \PVal{0.643}{0.211}  & \PVal{0.467}{0.149}  & \PVal{0.700}{0.135}  & \PVal{0.087}{0.030}  \\
  \bottomrule
\end{tabular}
     }%
  }
\end{table}

\vfill

\begin{table}[h!]
  \centering
  \caption{%
  \revTwo{%
    \textit{Vision Backdoor Rate Adversarial-Set Identification Ablation Study}:
    Mean and standard deviation adversarial-instance identification AUPRC for \citepos{Weber:2021} vision backdoor attack with ${{\geq}7}$~trials per experimental setup (${\abs{\fullTrain} = 10,000}$).
    Bold denotes the best mean performance.
    These experiments highlight the stability of \method{} and \layer{} at identifying adversarial set~\eqsmall{$\advTrain$}.
    Attacks were limited to ${{\geq}100}$~backdoor instances to ensure ${{\geq}50}$\% attack success rate.
    Results are shown graphically in Figure~\ref{fig:ExpRes:Ablation:Backdoor:CIFAR:CntSweep:AdvIdent}.
  }%
  }%
  \label{tab:ExpRes:Ablation:Backdoor:CIFAR:CntSweep:AdvIdent}
  {
    \TableFontSize%
    \revTwo{%
\renewcommand{\arraystretch}{1.2}
\setlength{\dashlinedash}{0.4pt}
\setlength{\dashlinegap}{1.5pt}
\setlength{\arrayrulewidth}{0.3pt}

\newcommand{\RevType}[1]{\multirow{4}{*}{\shortstack{{\Large $\uparrow$} \\ #1 \\ {\Large $\downarrow$}}}}
\newcommand{\RevNum}[1]{#1}

\newcommand{\TabMidRule}{\cdashline{2-7}}

\newcommand{\AlgName}[1]{\multicolumn{1}{c}{#1}}

\begin{tabular}{crrrrrr}
  \toprule
  \multirow{2}{*}{\%~Backdoor} & \multicolumn{2}{c}{Ours} & \multicolumn{4}{c}{Baselines} \\\cmidrule(lr){2-3}\cmidrule(l){4-7}
                               & \AlgName{\method{}}  & \AlgName{\layer}     & \AlgName{\tracinCP{}} & \AlgName{\tracin{}}  & \AlgName{Influence Func.}  &  {Representer Pt.}  \\
  \midrule
  1      & \PVal{0.848}{0.208} & \PValB{0.866}{0.189} & \PVal{0.517}{0.229}  & \PVal{0.288}{0.137}   & \PVal{0.067}{0.041}  & \PVal{0.027}{0.011}  \\\TabMidRule
  1.5    & \PVal{0.915}{0.106} & \PValB{0.946}{0.076} & \PVal{0.519}{0.220}  & \PVal{0.282}{0.123}   & \PVal{0.091}{0.051}  & \PVal{0.027}{0.005}  \\\TabMidRule
  2.5    & \PVal{0.948}{0.066} & \PValB{0.973}{0.034} & \PVal{0.479}{0.198}  & \PVal{0.280}{0.118}   & \PVal{0.156}{0.115}  & \PVal{0.040}{0.009}  \\\TabMidRule
  5      & \PVal{0.984}{0.022} & \PValB{0.992}{0.012} & \PVal{0.540}{0.140}  & \PVal{0.322}{0.087}   & \PVal{0.238}{0.104}  & \PVal{0.071}{0.015}  \\
  \bottomrule
\end{tabular}
     }%
  }
\end{table}

\vfill
\phantom{.}

\clearpage
\newpage

\phantom{.}
\vfill

\begin{table}[h!]
  \centering
  \caption{%
  \revTwo{%
    \textit{Natural-Language Poisoning Rate Adversarial-Set Identification Ablation Study}:
    Mean and standard deviation adversarial-instance identification AUPRC for \citepos{Wallace:2021} natural-language poisoning attack on \SST{} for different poisoning rates (\eqsmall{${\abs{\cleanTrain} = 67,349}$}).
    Results are averaged across related experimental setups with ${{\geq}5}$~trials per setup for each of the first four reviews with the best mean performance in bold.
    Only attacks that successfully changed the target's label are considered.
    For attacks with adversarial sets smaller than~10, the target's label did not consistently change and are excluded.
    Results are shown graphically in Figure~\ref{fig:ExpRes:Ablation:Pois:NLP:CntSweep:AdvIdent}.
  }%
  }%
  \label{tab:ExpRes:Ablation:Pois:NLP:CntSweep:AdvIdent}
  {
    \TableFontSize%
    \revTwo{%
\renewcommand{\arraystretch}{1.2}
\setlength{\dashlinedash}{0.4pt}
\setlength{\dashlinegap}{1.5pt}
\setlength{\arrayrulewidth}{0.3pt}

\newcommand{\RevType}[1]{\multirow{4}{*}{\shortstack{{\Large $\uparrow$} \\ #1 \\ {\Large $\downarrow$}}}}
\newcommand{\RevNum}[1]{#1}

\newcommand{\TabMidRule}{\cdashline{2-7}}

\newcommand{\AlgName}[1]{\multicolumn{1}{c}{#1}}

\begin{tabular}{crrrrrr}
  \toprule
  \multirow{2}{*}{\#~Poison} & \multicolumn{2}{c}{Ours} & \multicolumn{4}{c}{Baselines} \\\cmidrule(lr){2-3}\cmidrule(l){4-7}
                             & \AlgName{\method{}}  & \AlgName{\layer}  & \AlgName{\tracinCP{}} & \AlgName{\tracin{}}  & \AlgName{Influence Func.}  &  {Representer Pt.}  \\
  \midrule
  10 & \PValB{0.999}{0.004}       & \PValB{0.999}{0.002}       & \PVal{0.575}{0.512}       & \PVal{0.048}{0.041}       & \PVal{0.002}{0.001}       & \PVal{0.001}{0.001}       \\\TabMidRule
  25 & \PVal{\OneValB}{\ZeroValB} & \PVal{\OneValB}{\ZeroValB} & \PVal{0.407}{0.290}       & \PVal{0.127}{0.084}       & \PVal{0.019}{0.033}       & \PVal{0.001}{0.001}       \\\TabMidRule
  40 & \PValB{1.000}{0.001}       & \PValB{1.000}{0.000}       & \PVal{0.363}{0.325}       & \PVal{0.118}{0.110}       & \PVal{0.014}{0.020}       & \PVal{0.001}{0.001}       \\\TabMidRule
  45 & \PValB{1.000}{0.001}       & \PValB{1.000}{0.000}       & \PVal{0.298}{0.281}       & \PVal{0.114}{0.096}       & \PVal{0.010}{0.011}       & \PVal{0.001}{0.001}       \\\TabMidRule
  50 & \PVal{0.996}{0.023}        & \PValB{0.999}{0.002}       & \PVal{0.316}{0.259}       & \PVal{0.120}{0.098}       & \PVal{0.005}{0.002}       & \PVal{0.001}{0.001}       \\
  \bottomrule
\end{tabular}
     }%
  }
\end{table}

\vfill

\begin{table}[h!]
  \centering
  \caption{%
  \revTwo{%
    \textit{Vision Poisoning Rate Adversarial-Set Identification Ablation Study}:
    Mean and standard deviation adversarial-instance identification AUPRC for \citepos{Zhu:2019} vision poisoning attack on CIFAR10 across different poisoning rates (${\abs{\fullTrain} = 25,000}$).
    These experiments highlight the stability of \method{} and \layer{} at identifying adversarial set~\eqsmall{$\advTrain$} even at small poisoning rates.
    In contrast, \tracinCP{}'s performance changes significantly across the entire poisoning rate range.
    Results are averaged across related experimental setups with ${{\geq}10}$~trials per setup with the best mean performance in bold.
    Only attacks that successfully changed the target's label are considered.
    Results are shown graphically in Figure~\ref{fig:ExpRes:Ablation:Pois:Cifar:CntSweep:AdvIdent}.
  }%
  }%
  \label{tab:ExpRes:Ablation:Pois:Cifar:CntSweep:AdvIdent}
  {
    \TableFontSize%
    \revTwo{%
\renewcommand{\arraystretch}{1.2}
\setlength{\dashlinedash}{0.4pt}
\setlength{\dashlinegap}{1.5pt}
\setlength{\arrayrulewidth}{0.3pt}

\newcommand{\RevType}[1]{\multirow{4}{*}{\shortstack{{\Large $\uparrow$} \\ #1 \\ {\Large $\downarrow$}}}}
\newcommand{\RevNum}[1]{#1}

\newcommand{\TabMidRule}{\cdashline{1-7}}

\newcommand{\AlgName}[1]{\multicolumn{1}{c}{#1}}

\begin{tabular}{crrrrrr}
  \toprule
  \multirow{2}{*}{\#~Poison} & \multicolumn{2}{c}{Ours} & \multicolumn{4}{c}{Baselines} \\\cmidrule(lr){2-3}\cmidrule(l){4-7}
                             & \AlgName{\method{}}  & \AlgName{\layer}  & \AlgName{\tracinCP{}} & \AlgName{\tracin{}}  & \AlgName{Influence Func.}  &  {Representer Pt.}  \\
  \midrule
  40   & \PValB{0.726}{0.085}  & \PVal{0.692}{0.092}  & \PVal{0.270}{0.030}  & \PVal{0.107}{0.011}  & \PVal{0.066}{0.020}  & \PVal{0.025}{0.007}  \\\TabMidRule
  50   & \PValB{0.906}{0.054}  & \PVal{0.857}{0.102}  & \PVal{0.495}{0.063}  & \PVal{0.192}{0.023}  & \PVal{0.096}{0.036}  & \PVal{0.023}{0.006}  \\\TabMidRule
  100  & \PValB{0.835}{0.061}  & \PVal{0.805}{0.055}  & \PVal{0.275}{0.068}  & \PVal{0.122}{0.017}  & \PVal{0.067}{0.022}  & \PVal{0.031}{0.004}  \\\TabMidRule
  150  & \PValB{0.790}{0.074}  & \PVal{0.761}{0.082}  & \PVal{0.241}{0.013}  & \PVal{0.107}{0.006}  & \PVal{0.060}{0.015}  & \PVal{0.038}{0.003}  \\\TabMidRule
  200  & \PValB{0.857}{0.074}  & \PVal{0.839}{0.063}  & \PVal{0.255}{0.042}  & \PVal{0.123}{0.017}  & \PVal{0.074}{0.021}  & \PVal{0.048}{0.005}  \\\TabMidRule
  250  & \PValB{0.837}{0.046}  & \PVal{0.821}{0.050}  & \PVal{0.242}{0.043}  & \PVal{0.129}{0.016}  & \PVal{0.069}{0.015}  & \PVal{0.054}{0.004}  \\\TabMidRule
  400  & \PValB{0.869}{0.046}  & \PVal{0.866}{0.059}  & \PVal{0.238}{0.017}  & \PVal{0.139}{0.013}  & \PVal{0.090}{0.014}  & \PVal{0.080}{0.005}  \\
  \bottomrule
\end{tabular}
     }%
  }
\end{table}

\vfill
\phantom{.}

\clearpage
\newpage
\phantom{.}
\vfill

\begin{table}[h!]
  \centering
  \caption{%
    \revTwo{%
    \textit{Speech Backdoor Rate Target Identification Ablation Study}:
    Mean and standard deviation target identification AUPRC for \citepos{SpeechDataset} speech dataset across different backdoor quantities (\eqsmall{${\abs{\cleanTrain} = 3,000}$}).
    We consider class pairs ${0 \rightarrow 1}$ and ${1 \rightarrow 2}$ with ${{\geq}5}$ trials per setup.
    \citepos{SpeechDataset} speech dataset includes 30~backdoor instances per class which were downsampled uniformly at random.
    Results are averaged across related experimental setups with ${{\geq}5}$~trials per setup with the best mean performance in bold.
    Results are shown graphically in Figure~\ref{fig:ExpRes:Ablation:Backdoor:Speech:CntSweep:TargDetect}.
    }%
  }
  \label{tab:ExpRes:Ablation:Backdoor:Speech:CntSweep:TargDetect}
  {
    \TableFontSize%
    \revTwo{%
\renewcommand{\arraystretch}{1.2}
\setlength{\dashlinedash}{0.4pt}
\setlength{\dashlinegap}{1.5pt}
\setlength{\arrayrulewidth}{0.3pt}

\newcommand{\RevType}[1]{\multirow{4}{*}{\shortstack{{\Large $\uparrow$} \\ #1 \\ {\Large $\downarrow$}}}}
\newcommand{\RevNum}[1]{#1}

\newcommand{\TabMidRule}{\cdashline{2-7}}

\newcommand{\AlgName}[1]{\multicolumn{1}{c}{#1}}

\begin{tabular}{crrrrrrr}
  \toprule
  \multirow{2}{*}{\#~Backdoors} & \multicolumn{2}{c}{Ours} & \multicolumn{5}{c}{Baselines} \\\cmidrule(lr){2-3}\cmidrule(l){4-8}
      & \AlgName{\method{}}   & \AlgName{\layer}     & \AlgName{Max~\KNN{}} & \AlgName{Min~\KNN{}} & \AlgName{Most Certain} & \AlgName{Least Certain} & \AlgName{Random}   \\
  \midrule
  10 & \PValB{0.801}{0.158}  & \PVal{0.742}{0.155}  & \PVal{0.023}{0.001}  & \PVal{0.440}{0.052}  & \PVal{0.080}{0.072}  & \PVal{0.121}{0.070}  & \PVal{0.070}{0.045}  \\\TabMidRule
  15 & \PValB{0.778}{0.103}  & \PVal{0.719}{0.147}  & \PVal{0.023}{0.001}  & \PVal{0.323}{0.039}  & \PVal{0.047}{0.022}  & \PVal{0.160}{0.082}  & \PVal{0.050}{0.016}  \\\TabMidRule
  20 & \PValB{0.885}{0.150}  & \PVal{0.828}{0.151}  & \PVal{0.025}{0.002}  & \PVal{0.195}{0.089}  & \PVal{0.032}{0.008}  & \PVal{0.331}{0.207}  & \PVal{0.064}{0.034}  \\\TabMidRule
  25 & \PValB{0.873}{0.095}  & \PVal{0.800}{0.067}  & \PVal{0.033}{0.010}  & \PVal{0.243}{0.067}  & \PVal{0.031}{0.010}  & \PVal{0.265}{0.156}  & \PVal{0.069}{0.043}  \\\TabMidRule
  30 & \PValB{0.961}{0.038}  & \PVal{0.897}{0.086}  & \PVal{0.094}{0.057}  & \PVal{0.095}{0.032}  & \PVal{0.034}{0.013}  & \PVal{0.222}{0.224}  & \PVal{0.063}{0.027}  \\
  \bottomrule
\end{tabular}
     }%
  }
\end{table}

\vfill

\begin{table}[h!]
  \centering
  \caption{%
    \revTwo{%
    \textit{Vision Poisoning Rate Target Identification Ablation Study}:
    Mean and standard deviation target identification AUPRC for \citepos{Zhu:2019} vision poisoning attack on CIFAR10 across different poisoning rates (\eqsmall{${\abs{\fullTrain} = 25,000}$}).
    Results are averaged across all four class pairs with ${{\geq}10}$~trials per setup with the best mean performance in bold.
    Results are shown graphically in Figure~\ref{fig:ExpRes:Ablation:Pois:Cifar:CntSweep:TargDetect}.
    }%
  }%
  \label{tab:ExpRes:Ablation:Pois:Cifar:CntSweep:TargDetect}
  \begin{subtable}{\textwidth}
    \centering
    \caption{%
      \revTwo{%
      ${\anomCount = 2}$
      }%
    }
    {
      \TableFontSize%
      \revTwo{%
\renewcommand{\arraystretch}{1.2}
\setlength{\dashlinedash}{0.4pt}
\setlength{\dashlinegap}{1.5pt}
\setlength{\arrayrulewidth}{0.3pt}

\newcommand{\RevType}[1]{\multirow{4}{*}{\shortstack{{\Large $\uparrow$} \\ #1 \\ {\Large $\downarrow$}}}}
\newcommand{\RevNum}[1]{#1}

\newcommand{\TabMidRule}{\cdashline{1-8}}

\newcommand{\AlgName}[1]{\multicolumn{1}{c}{#1}}

\begin{tabular}{crrrrrrr}
  \toprule
  \multirow{2}{*}{\#~Poison} & \multicolumn{2}{c}{Ours} & \multicolumn{5}{c}{Baselines} \\\cmidrule(lr){2-3}\cmidrule(l){4-8}
       & \AlgName{\method{}}   & \AlgName{\layer}     & \AlgName{Max~\KNN{}} & \AlgName{Min~\KNN{}} & \AlgName{Most Certain} & \AlgName{Least Certain} & \AlgName{Random}   \\
  \midrule
  50
      & \PValB{0.758}{0.251}  & \PVal{0.402}{0.364}  & \PVal{0.400}{0.411}  & \PVal{0.016}{0.013}  & \PVal{0.054}{0.047}  & \PVal{0.014}{0.004}  & \PVal{0.039}{0.052}  \\\TabMidRule
  150
      & \PValB{0.840}{0.239}  & \PVal{0.499}{0.347}  & \PVal{0.467}{0.427}  & \PVal{0.012}{0.005}  & \PVal{0.055}{0.047}  & \PVal{0.013}{0.003}  & \PVal{0.053}{0.091}  \\\TabMidRule
  200
      & \PValB{0.721}{0.309}  & \PVal{0.463}{0.342}  & \PVal{0.474}{0.398}  & \PVal{0.011}{0.003}  & \PVal{0.055}{0.040}  & \PVal{0.013}{0.003}  & \PVal{0.046}{0.071}  \\\TabMidRule
  250
      & \PValB{0.737}{0.309}  & \PVal{0.512}{0.394}  & \PVal{0.492}{0.415}  & \PVal{0.018}{0.018}  & \PVal{0.088}{0.130}  & \PVal{0.022}{0.017}  & \PVal{0.087}{0.135}  \\\TabMidRule
  400
      & \PValB{0.717}{0.301}  & \PVal{0.418}{0.371}  & \PVal{0.460}{0.430}  & \PVal{0.016}{0.016}  & \PVal{0.085}{0.139}  & \PVal{0.017}{0.010}  & \PVal{0.069}{0.111}  \\
  \bottomrule
\end{tabular}
       }%
    }
  \end{subtable}

  \vspace{0.25in}
  \begin{subtable}{\textwidth}
    \centering
    \caption{%
      \revTwo{%
      ${\anomCount = 10}$
      }%
    }
    {
      \TableFontSize%
      \revTwo{%
\renewcommand{\arraystretch}{1.2}
\setlength{\dashlinedash}{0.4pt}
\setlength{\dashlinegap}{1.5pt}
\setlength{\arrayrulewidth}{0.3pt}

\newcommand{\RevType}[1]{\multirow{4}{*}{\shortstack{{\Large $\uparrow$} \\ #1 \\ {\Large $\downarrow$}}}}
\newcommand{\RevNum}[1]{#1}

\newcommand{\TabMidRule}{\cdashline{1-8}}

\newcommand{\AlgName}[1]{\multicolumn{1}{c}{#1}}

\begin{tabular}{crrrrrrr}
  \toprule
  \multirow{2}{*}{\#~Poison} & \multicolumn{2}{c}{Ours} & \multicolumn{5}{c}{Baselines} \\\cmidrule(lr){2-3}\cmidrule(l){4-8}
       & \AlgName{\method{}}   & \AlgName{\layer}     & \AlgName{Max~\KNN{}} & \AlgName{Min~\KNN{}} & \AlgName{Most Certain} & \AlgName{Least Certain} & \AlgName{Random}   \\
  \midrule
  50
      & \PValB{0.750}{0.310} & \PVal{0.193}{0.225}  & \PVal{0.181}{0.319}  & \PVal{0.016}{0.013}  & \PVal{0.054}{0.047}  & \PVal{0.014}{0.004}  & \PVal{0.039}{0.052}  \\\TabMidRule
  150
      & \PValB{0.866}{0.162} & \PVal{0.352}{0.298}  & \PVal{0.364}{0.418}  & \PVal{0.012}{0.005}  & \PVal{0.055}{0.047}  & \PVal{0.013}{0.003}  & \PVal{0.053}{0.091}  \\\TabMidRule
  200
      & \PValB{0.884}{0.190} & \PVal{0.427}{0.350}  & \PVal{0.415}{0.445}  & \PVal{0.011}{0.003}  & \PVal{0.055}{0.040}  & \PVal{0.013}{0.003}  & \PVal{0.046}{0.071}  \\\TabMidRule
  250
      & \PValB{0.853}{0.265} & \PVal{0.389}{0.384}  & \PVal{0.342}{0.399}  & \PVal{0.018}{0.018}  & \PVal{0.088}{0.130}  & \PVal{0.022}{0.017}  & \PVal{0.087}{0.135}  \\\TabMidRule
  400
      & \PValB{0.808}{0.309} & \PVal{0.436}{0.384}  & \PVal{0.401}{0.458}  & \PVal{0.016}{0.016}  & \PVal{0.085}{0.139}  & \PVal{0.017}{0.010}  & \PVal{0.069}{0.111}  \\
  \bottomrule
\end{tabular}
       }%
    }
  \end{subtable}
\end{table}

\vfill
\phantom{.}
 
\newpage
\clearpage
\subsection{Upper-Tail Heaviness Hyperparameter\texorpdfstring{~$\anomCount$}{} Sensitivity Analysis}\label{sec:App:MoreExps:EffectAnomCount}

Section~\ref{sec:Method:TargetDetection} defines the upper-tail heaviness of influence vector~$\infVec$ as the $\anomCount$\=/th largest anomaly score in vector~$\anomScoreVec$.
Test examples in set~$\targAnalysisSet$ are then ranked by their respective heaviness with those highest ranked more commonly targets.

Figure~\ref{fig:MoreExps:Hyperparam:Cutoff} reports \method{}'s and \layer{}'s target identification AUPRC for the vision backdoor~\citep{Weber:2021} and natural language poisoning~\citep{Wallace:2021} attacks across a range of~$\anomCount$ values.  In all cases, our performance is remarkably stable.  For example, \layer{}'s and \method{}'s natural language target identification AUPRC varied only 0.2\% and 2.1\% respectively for all tested ${\anomCount \in \sbrack{1,~25}}$. The vision backdoor target identification AUPRC varied only 1.8\% and 5.8\% respectively for ${\anomCount \in \sbrack{2,~50}}$.

Recall from Section~\ref{sec:ExpRes:Attacks} that the vision backdoor adversarial set is three times larger than that of natural language poisoning (150~vs.~50).  That is why vision backdoor target identification's performance is stable over a wider range of~$\anomCount$ values than natural language poisoning.

\vspace{1.0cm}
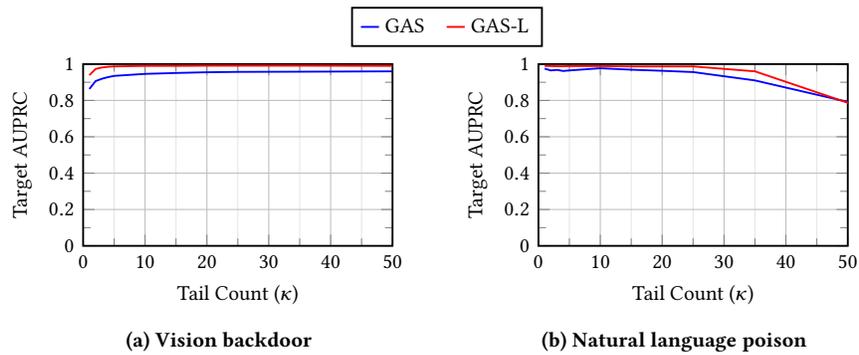
\begin{figure}[h]
  \centering
\newcommand{\legendSpacer}{\hspace*{6pt}}

\begin{tikzpicture}
  \begin{axis}[%
      hide axis,  %
      xmin=0,  %
      xmax=1,
      ymin=0,
      ymax=1,
      scale only axis,width=1mm, %
      line width=\CdfLineWidth,
      legend cell align={left},              %
      legend style={font=\legendFontSize},
      legend columns=3,
      legend image post style={xscale=0.4},  %
    ]
    \addplot [CosIn Line] coordinates {(0,0)};
    \addlegendentry{\method{}\legendSpacer}
    \addplot [LayIn Line] coordinates {(0,0)};
    \addlegendentry{\layer{}}
\end{axis}
\end{tikzpicture}

\newcommand{\HyperparamCutoffTrendPlot}[1]{
  \pgfplotstableread[col sep=comma] {plots/data/hyperparam-cutoff_trend_#1.csv}\thedata%
  \begin{tikzpicture}
    \begin{axis}[
        width=\columnwidth,%
        height={\AdvIdentAblationHeight},%
        xmin=0,%
        xmax=50,%
        xtick distance={10},
        minor x tick num={1},
        x tick label style={font=\plotFontSize,align=center},%
        xlabel={\plotFontSize Tail Count~($\anomCount$)},
        xmajorgrids,
        xminorgrids,
        ymin=0,
        ymax=1,
        ytick distance={0.2},
        minor y tick num={1},
        ymajorgrids,
        y tick label style={font=\plotFontSize,align=center},%
        ylabel={\plotFontSize Target AUPRC},
        mark size=0pt,
        line width=\CdfLineWidth,
      ]
      \addplot [CosIn Line] table [x index=0, y index=1] \thedata;
      \addplot [LayIn Line] table [x index=0, y index=2] \thedata;
    \end{axis}
  \end{tikzpicture}
}
  \newcommand{\subfigureWidth}{0.32\textwidth}
  \begin{subfigure}[t]{\subfigureWidth}
    \HyperparamCutoffTrendPlot{bd_cifar}
    \caption{Vision backdoor}%
    \label{fig:MoreExps:Hyperparam:Cutoff:Backdoor:CIFAR}
  \end{subfigure}
  \hspace{0.2cm}
  \begin{subfigure}[t]{\subfigureWidth}
    \HyperparamCutoffTrendPlot{pois_nlp}
    \caption{Natural language poison}%
    \label{fig:MoreExps:Hyperparam:Cutoff:Pois:NLP}
  \end{subfigure}
  \caption{%
    \textit{Upper-Tail Heaviness Hyperparameter~($\anomCount$) Sensitivity Analysis}:
    Mean target identification AUPRC across a range of tail-heaviness values~($\anomCount$) for \method{} and \layer{}. Target identification performance fluctuates very little even when $\anomCount$~changes by more than an order of magnitude.
    Results are averaged across all experimental setups/class pairs with ${{\geq}10}$~trials per setup.%
  }
  \label{fig:MoreExps:Hyperparam:Cutoff}
\end{figure}
 
\clearpage
\newpage
\subsection{Alternatives to Renormalization}\label{sec:App:MoreExps:LossOnly}

Renormalization directly addresses the low-loss penalty of \citepos{Yeh:2018:Representer} representer point estimator by normalizing by~\LossOnlyNorm{}.
For the other influence estimators (\tracin{}, \tracinCP{}, and influence functions), a slightly different approach is taken.

For influence estimator~$\influence$, Section~\ref{sec:RenormInfluence:Measures} defines that renormalized influence~$\renormInf$ replaces each gradient vector~$\gradLetter$ in~$\influence$ with unit vector~\eqsmall{$\frac{\gradLetter}{\norm{\gradLetter}}$}. This modified approach does not correct solely for the low-loss penalty.  This choice was made for a few reasons.  First, loss and gradient magnitude are generally very tightly coupled as shown in Figure~\ref{fig:RenormInfluence:CifarVsMnist:Trend}, making the two often interchangeable.  In addition, most automatic differentiation frameworks (e.g.,~\texttt{torch}) do not directly return just the loss function's gradient; rather, they provide the gradient for each parameter.  It is therefore an easier implementation to normalize by the full gradient's magnitude.  Lastly, using the full gradient vector's norm also corrects for variance in the other parts of the gradient's magnitude -- specifically w.r.t.\ the parameter values.

This section examines the change in \method{}'s performance had \LossOnlyNorm{} been used to renormalize \tracinCP{} instead of $\norm{\gradLetter}$.
Figure~\ref{fig:MoreExps:LossNormOnly} compares those two approaches against the \tracinCP{} baseline for the CIFAR10 \& MNIST joint training experiment in Section~\ref{sec:RenormInfluence:SimpleExperiment} as well as Section~\ref{sec:ExpRes:AdvIdent}'s adversarial-set identification experiments.\footnote{At the time of submission, data collection for \citepos{Zhu:2019} vision poisoning attack had not yet completed.}
As in the original experiments, performance was measured using AUPRC.

For backdoor vision and poison, performance was comparable irrespective of which of the two renormalization schemes was used.
Renormalizing \tracinCP{} using \LossOnlyNorm{} actually performed better on the CIFAR10 \& MNIST joint training experiment, achieving near-perfect (0.998) mean AUPRC.  In contrast, \method{} was the top performer for the natural language poisoning experiments.  The performance difference on that baseline is primarily due to \robertaBase{}'s very large model size and the large variance that can induce on the magnitude of vector~\eqsmall{$\ActsGradient$}.

\begin{figure}[h]
  \centering
\newcommand{\legendSpacer}{\hspace*{9pt}}
\newcommand{\LossOnly}{\=/Deriv}

\begin{tikzpicture}
  \begin{axis}[%
    width=\textwidth,
      ybar,
      hide axis,  %
      xmin=0,  %
      xmax=1,
      ymin=0,
      ymax=1,
      scale only axis,width=1mm, %
      legend cell align={left},              %
      legend style={font=\legendFontSize},
      legend columns=8,
    ]
    \addplot [cosin color] coordinates {(0,0)};
    \addlegendentry{\method\legendSpacer}

    \addplot [cosin loss deriv color] coordinates {(0,0)};
    \addlegendentry{\method\LossOnly\legendSpacer}

    \addplot [TracInCP color] coordinates {(0,0)};
    \addlegendentry{\tracinCP{}}
  \end{axis}
\end{tikzpicture}

\pgfplotstableread[col sep=comma]{plots/data/loss_norm_only.csv}\datatable%
\begin{tikzpicture}%
  \begin{axis}[%
        ybar={\BarLineWidth},%
        height={\BarLossNormOnlyHeight},%
        width={\BarLossNormOnlyWidth},%
        axis lines*=left,%
        bar width={\LossDerivBarWidth},%
        xtick=data,%
        xticklabels={Joint,Speech,Vision,NLP},%
        x tick label style={font=\plotFontSize,align=center},%
        ymin=0,%
        ymax=1,%
        ytick distance={0.20},%
        minor y tick num={1},%
        y tick label style={font=\plotFontSize},%
        ylabel={\plotFontSize \IdentYLabel},%
        ymajorgrids,  %
        typeset ticklabels with strut,  %
        every tick/.style={color=black, line width=0.4pt},%
        enlarge x limits={0.175},%
        draw group line={[index]7}{1}{Backdoor}{-5.0ex}{12pt},
        draw bottom line,
    ]%
    \addplot[cosin color] table [x index=0, y index=1] {\datatable};%

    \addplot[cosin loss deriv color] table [x index=0, y index=2] {\datatable};%

    \addplot[TracInCP color] table [x index=0, y index=3] {\datatable};%

  \end{axis}%
\end{tikzpicture}%
   \caption{%
    \textit{Alternate \method{} Renormalization Using~\LossOnlyNorm}:
    Mean \eqsmall{$\advTrain$} identification AUPRC for the CIFAR10 \& MNIST joint training and adversarial attack experiments.  For the two backdoor attacks, both renormalization schemes performed similarly.  Renormalization using only the loss function's norm performed best for the joint training experiments while \method{} was the top performer for natural language poisoning.
    Results are averaged across all experimental setups/class pairs with ${{\geq}10}$~trials per setup.%
  }
  \label{fig:MoreExps:LossNormOnly}
\end{figure}

\revised{%
  \citet{Barshan:2020:RelatIF} propose \keyword{relative influence}, which normalizes specifically  influence functions by Hessian-vector product (HVP), {${\invHess \gradI}$}.  They theoretically motivate their method via a global-vs.-local influence tradeoff.  However, as \citeauthor{Barshan:2020:RelatIF} acknowledge, their method is computationally intractable.  Estimating a single HVP can take up to several hours.  Since Hessian-vector product {${\invHess \gradI}$} must be calculated for \textit{each training instance}~$\zI$, \citeauthor{Barshan:2020:RelatIF} only evaluate their method on training sets of a few hundred instances.
Even if it were computationally feasible, existing HVP-estimation methods are fragile and generally perform poorly on large models \citep{Basu:2021:Influence}.

Moreover, relative influence is specific to influence functions while our renormalized influence in Section~\ref{sec:RenormInfluence} applies generally to all loss-based influence estimators (e.g.,~\tracin{}, \hydra{}, representer point).  An additional consequence of that is \citeauthor{Barshan:2020:RelatIF} do not consider normalizing by test gradient magnitude~\eqsmall{$\norm{\gradHatTe}$} to prevent penalizing low-loss training \textit{iterations}.
}
 
\clearpage
\newpage
\subsection{Why Aggregate Training Gradients?}\label{sec:App:MoreExps:Aggregation}

Recall from Section~\ref{sec:RelatedWork:Influence} that \keyword{static influence estimators} (e.g., influence functions \citep{Koh:2017:Understanding}, representer point \citep{Yeh:2018:Representer}) quantify influence using only final model parameters~$\wFin$.  In contrast, \keyword{dynamic influence estimators} (e.g., \tracin{} \& \tracinCP{}~\citep{Pruthi:2020}, \hydra{} \citep{Chen:2021:Hydra}) quantify influence by aggregating changes across (a subset of) intermediate training parameters, ${\wZero,\ldots,\wFin}$.  These experiments examine the benefit of gradient aggregation for our renormalized influence estimator, \method{}, which is based on \tracinCP{}.

Similar to Section~\ref{sec:ExpRes:AdvIdent}'s experimental setup,
Figure~\ref{fig:MoreExps:Aggregate} compares the mean adversarial set~(\eqsmall{$\advTrain$}) identification AUPRC for \method{} against two custom variants \methodZero{} and \methodFin{}, where \baseMethod{\method}{$\tau$} denotes that the corresponding iteration subset is ${\subsetItr = \set{\tau}}$.\footnote{For example, \methodZero{} denotes ${\subsetItr = \set{0}}$, i.e.,~only initial parameters~$\wZero$ are analyzed.}
In all cases, \method{} aggregation outperformed the single best checkpoint on average.  Therefore, even if the optimal checkpoint could be known via an oracle, it still may not beat aggregation in many cases.
In particular for poisoning but also for \citepos{Weber:2021} vision backdoor attack where the adversarial trigger may be clipped, each adversarial instance in~\eqsmall{$\advTrain$} may not have the same attack pattern/data.
Therefore, instances in \eqsmall{$\advTrain$} may be most influential at different points in training depending on when the corresponding adversarial feature(s) best aligns with the target.  Averaging across multiple checkpoints enables dynamic methods to potentially detect all such points in training.

Note also that $\wFin$ was better at identifying~\eqsmall{$\advTrain$} for only two of three attacks, namely natural-language poisoning and vision backdoor.  For vision poisoning, $\wZero$~yielded better identification. This performance difference is \textit{not} due to whether~$\wZero$ is pre\=/trained as both poisoning attacks used pre\=/trained models.

\vspace{1.0cm}
\begin{figure}[h]
  \centering
\newcommand{\legendSpacer}{\hspace*{9pt}}

\begin{tikzpicture}
  \begin{axis}[%
      width=\textwidth,
      ybar,
      hide axis,  %
      xmin=0,  %
      xmax=1,
      ymin=0,
      ymax=1,
      scale only axis,width=1mm, %
      legend cell align={left},              %
      legend style={font=\legendFontSize},
      legend columns=6,
      cycle list name=AggregateCycleList,
    ]
    \addplot [fill=white, draw=blue, pattern=north east lines, pattern color=blue] coordinates {(0,0)};
    \addlegendentry{\methodZero\legendSpacer}
    \addplot [cosin color] coordinates {(0,0)};
    \addlegendentry{\method\legendSpacer}
    \addplot [fill=gray, draw=teal, pattern=north west lines, pattern color=teal] coordinates {(0,0)};
    \addlegendentry{\methodFin}
  \end{axis}
\end{tikzpicture}
 
  \newcommand{\subfigureWidth}{0.32\textwidth}
  \begin{subfigure}[t]{\subfigureWidth}
\pgfplotstableread[col sep=comma]{plots/data/bd_cifar_aggregate.csv}\datatable%
\begin{tikzpicture}%
  \begin{axis}[%
        ybar={\BarLineWidth},%
        width={1.06\textwidth},%
        height={\NlpIdentificationHeight},%
        axis lines*=left,%
        bar width=\NlpPoisBarWidth,%
        xtick=data,
        xticklabels={1~Pixel,4~Pixel,Blend},
        x tick label style={font=\plotFontSize,align=center},%
        ymin=0,%
        ymax=1,%
        ytick distance={0.20},%
        minor y tick num={1},%
        y tick label style={font=\plotFontSize},%
        ylabel={\plotFontSize \IdentYLabel},%
        ymajorgrids,  %
        typeset ticklabels with strut,  %
        every tick/.style={color=black, line width=0.4pt},%
        enlarge x limits=0.25,%
        draw group line={[index]1}{1}{${\text{Auto} \rightarrow \text{Dog}}$}{-5.0ex}{09pt},%
        cycle list name=AggregateCycleList
    ]%
    \addplot [fill=white, draw=blue, pattern=north east lines, pattern color=blue] table [x index=0, y index=2] {\datatable};%
    \addplot [cosin color] table [x index=0, y index=3] {\datatable};%
    \addplot [fill=gray, draw=teal, pattern=crosshatch, pattern color=teal] table [x index=0, y index=4] {\datatable};%
  \end{axis}%
\end{tikzpicture}%
     \caption{%
        \textit{Vision Backdoor}:
        Mean adversarial set~(\eqsmall{$\advTrain$}) identification AUPRC averaged over ${\geq}$15~trials for CIFAR10 label pair Auto~$\rightarrow$~Dog (${\yTarg \rightarrow \yAdv}$) using \citepos{Weber:2021} three attack patterns.
        See Table~\ref{tab:MoreExps:Aggregate:Backdoor:CIFAR} for the full numerical results.
    }
    \label{fig:MoreExps:Aggregate:Backdoor:CIFAR}
  \end{subfigure}
  \hfill
  \begin{subfigure}[t]{\subfigureWidth}
\pgfplotstableread[col sep=comma]{plots/data/pois_nlp_aggregate.csv}\datatable%
\begin{tikzpicture}%
  \begin{axis}[%
        ybar={\BarLineWidth},%
        width={1.06\textwidth},%
        height={\NlpIdentificationHeight},%
        axis lines*=left,%
        bar width=\NlpPoisBarWidth,%
        xtick=data,
        xticklabels={${3}$,${4}$,${3}$,${4}$},
        x tick label style={font=\plotFontSize,align=center},%
        ymin=0,%
        ymax=1,%
        ytick distance={0.20},%
        minor y tick num={1},%
        y tick label style={font=\plotFontSize},%
        ylabel={\plotFontSize \IdentYLabel},%
        ymajorgrids,  %
        typeset ticklabels with strut,  %
        every tick/.style={color=black, line width=0.4pt},%
        enlarge x limits=0.15,%
        draw group line={[index]1}{1}{Positive}{-5.0ex}{4pt},%
        draw group line={[index]1}{2}{Negative}{-5.0ex}{4pt},%
        cycle list name=AggregateCycleList
    ]%
    \addplot [fill=white, draw=blue, pattern=north east lines, pattern color=blue] table [x index=0, y index=2] {\datatable};%
    \addplot [cosin color] table [x index=0, y index=3] {\datatable};%
    \addplot [fill=red, draw=teal, pattern=crosshatch, pattern color=teal] table [x index=0, y index=4] {\datatable};%
  \end{axis}%
\end{tikzpicture}%
     \caption{%
        \textit{Natural Language Poisoning}:
        Mean adversarial set~(\eqsmall{$\advTrain$}) identification AUPRC averaged over 10~trials for two positive and two negative \SST{} movie reviews.
        See Table~\ref{tab:MoreExps:Aggregate:Pois:Nlp} for the full numerical results including variance.
    }
    \label{fig:MoreExps:Aggregate:Pois:NLP}
  \end{subfigure}
  \hfill
  \begin{subfigure}[t]{\subfigureWidth}
\pgfplotstableread[col sep=comma]{plots/data/pois_cifar_aggregate.csv}\datatable%
\begin{tikzpicture}%
  \begin{axis}[%
        ybar={\BarLineWidth},%
        width={1.06\textwidth},%
        height={\NlpIdentificationHeight},%
        axis lines*=left,%
        bar width=\NlpPoisBarWidth,%
        xtick={1,3},  %
        xticklabels={${\text{Bird} \rightarrow \text{Dog}}$,${\text{Deer} \rightarrow \text{Frog}}$},
        x tick label style={font=\plotFontSize,align=center},%
        extra x ticks={2},  %
        extra x tick labels={${\text{Frog} \rightarrow \text{Deer}}$},
        extra x tick style={major tick length=18pt},
        every extra x tick/.append style={yshift=0.5mm},
        ymin=0,%
        ymax=1,%
        ytick distance={0.20},%
        minor y tick num={1},%
        y tick label style={font=\plotFontSize},%
        ylabel={\plotFontSize \IdentYLabel},%
        ymajorgrids,  %
        typeset ticklabels with strut,  %
        every tick/.style={color=black, line width=0.4pt},%
        enlarge x limits=0.25,%
        cycle list name=AggregateCycleList
    ]%
    \addplot [fill=white, draw=blue, pattern=north east lines, pattern color=blue] table [x index=0, y index=2] {\datatable};%
    \addplot [cosin color] table [x index=0, y index=3] {\datatable};%
    \addplot [fill=gray, draw=teal, pattern=crosshatch, pattern color=teal] table [x index=0, y index=4] {\datatable};%
  \end{axis}%
\end{tikzpicture}%
     \caption{\textit{Vision Poisoning}:
      Mean adversarial set~(\eqsmall{$\advTrain$}) identification AUPRC averaged over 30~trials for three CIFAR10 class pairs (${\yTarg \rightarrow \yAdv}$) with \citepos{Zhu:2019} convex polytope poisoning attack.
        See Table~\ref{tab:MoreExps:Aggregate:Pois:CIFAR} for the full numerical results.
    }
    \label{fig:MoreExps:Aggregate:Pois:CIFAR}
  \end{subfigure}
  \caption{\textit{Gradient Aggregation}: Mean \method{} adversarial set~(\eqsmall{$\advTrain$}) identification AUPRC when considering the initial (${\itr = 0}$) and final (${\itr = \nItr}$) checkpoints individually versus aggregating across multiple checkpoints.  Vision and natural language poisoning used pre\=/trained initial parameters~$\wZero$ while vision backdoor's~$\wZero$ was randomly initialized.  See Section~\ref{sec:App:ExpSetup:Hyperparams} for each of these experiments' hyperparameter settings.}
  \label{fig:MoreExps:Aggregate}
\end{figure}
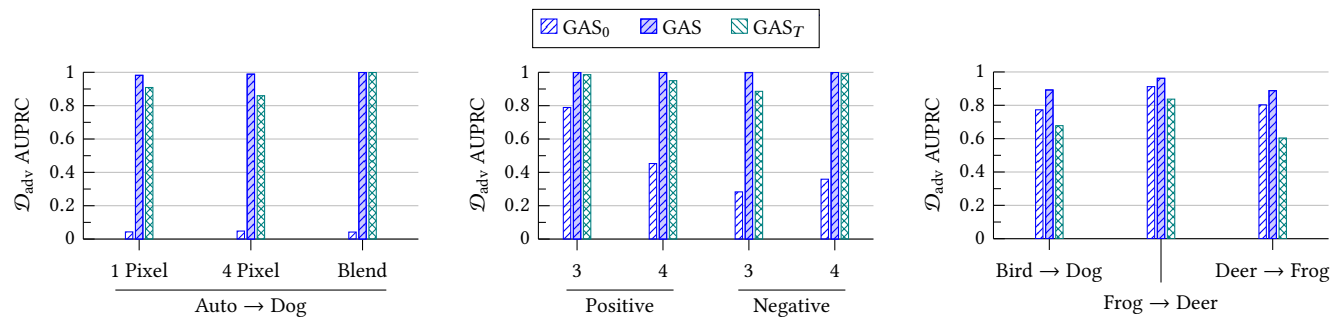

\FloatBarrier
\clearpage
\newpage

\phantom{.}
\vfill

\begin{table}[h]
  \centering
  \caption{%
    \textit{Vision Backdoor Gradient Aggregation}:
    Adversarial set~(\eqsmall{$\advTrain$}) identification AUPRC mean and standard deviation averaged over ${\geq}$15~trials for CIFAR10 label pair Auto~$\rightarrow$~Dog (${\yTarg \rightarrow \yAdv}$) using \citepos{Weber:2021} three attack patterns.
    Following \citepos{Weber:2021} experimental setup, $\wZero$~was randomly initialized causing $\methodZero$ as expected to perform poorly.
    Bold denotes the best mean performance.
    Mean values are shown graphically in Figure~\ref{fig:MoreExps:Aggregate:Backdoor:CIFAR}.%
  }
  \label{tab:MoreExps:Aggregate:Backdoor:CIFAR}
  {
    \TableFontSize%
\renewcommand{\arraystretch}{1.2}
\setlength{\dashlinedash}{0.4pt}
\setlength{\dashlinegap}{1.5pt}
\setlength{\arrayrulewidth}{0.3pt}

\newcommand{\head}[1]{#1}
\newcommand{\DigPair}[2]{\multirow{3}{*}{#1 $\rightarrow$ #2}}
\newcommand{\Attack}[1]{#1}

\newcommand{\TabMidRule}{\cdashline{2-5}}

\newcommand{\AlgName}[1]{\multicolumn{1}{c}{#1}}

\begin{tabular}{ccrrr}
  \toprule
  $\yTarg \rightarrow \yAdv$ & \head{Attack}
        &  \AlgName{\methodZero{}}  & \AlgName{\method{}}    & \AlgName{\methodFin{}}  \\  %
  \midrule
  \DigPair{Auto}{Dog}
       & 1~Pixel
           & \PVal{0.043}{0.015}  & \PValB{0.982}{0.015}      & \PVal{0.909}{0.069}  \\\TabMidRule
       & 4~Pixel
           & \PVal{0.048}{0.021}  & \PValB{0.989}{0.033}      & \PVal{0.860}{0.137}  \\\TabMidRule
       & Blend
           & \PVal{0.042}{0.021}  & \PValB{\OneVal}{\ZeroVal} & \PVal{0.999}{0.002}  \\
  \bottomrule
\end{tabular}
   }
\end{table}

\vfill

\begin{table}[h]
  \centering
  \caption{%
    \textit{Natural Language Poisoning Gradient Aggregation}:
    Adversarial set~(\eqsmall{$\advTrain$}) identification AUPRC mean and standard deviation averaged over 10~trials for two positive and two negative \SST{} movie reviews.
    Bold denotes the best mean performance.
    Mean values are shown graphically in Figure~\ref{fig:MoreExps:Aggregate:Pois:NLP}.%
  }
  \label{tab:MoreExps:Aggregate:Pois:Nlp}
  {
    \TableFontSize%
\renewcommand{\arraystretch}{1.2}
\setlength{\dashlinedash}{0.4pt}
\setlength{\dashlinegap}{1.5pt}
\setlength{\arrayrulewidth}{0.3pt}

\newcommand{\RevType}[1]{\multirow{2}{*}{#1}}
\newcommand{\RevNum}[1]{#1}

\newcommand{\TabMidRule}{\cdashline{2-9}}

\newcommand{\AlgName}[1]{\multicolumn{1}{c}{#1}}

\begin{tabular}{ccrrrrrrr}
  \toprule
  Sentiment & No.                & \AlgName{\methodZero} & \AlgName{\method{}}  & \AlgName{\methodFin{}} \\
  \midrule
  \RevType{Pos.} & \RevNum{3}
          & \PVal{0.789}{0.144}  & \PValB{\OneValB}{\ZeroValB}  & \PVal{0.986}{0.037}  \\\TabMidRule
                 & \RevNum{4}
          & \PVal{0.453}{0.066}  & \PValB{\OneValB}{\ZeroValB}  & \PVal{0.950}{0.083}  \\
          \midrule
  \RevType{Neg.} & \RevNum{3}
          & \PVal{0.283}{0.087}  & \PValB{0.998}{0.005}         & \PVal{0.886}{0.213}  \\\TabMidRule
                 & \RevNum{4}
          & \PVal{0.359}{0.108}  & \PValB{\OneValB}{\ZeroValB}  & \PVal{0.992}{0.025}  \\
  \bottomrule
\end{tabular}
   }
\end{table}

\vfill

\begin{table}[h]
  \centering
  \caption{%
    \textit{Vision Poisoning Gradient Aggregation}:
    Adversarial set~(\eqsmall{$\advTrain$}) identification AUPRC mean and standard deviation averaged over 30~trials for three CIFAR10 class pairs (${\yTarg \rightarrow \yAdv}$) with \citepos{Zhu:2019} convex polytope poisoning attack.
    Bold denotes the best mean performance.
    Mean values are shown graphically in Figure~\ref{fig:MoreExps:Aggregate:Pois:CIFAR}.%
  }
  \label{tab:MoreExps:Aggregate:Pois:CIFAR}
  {
    \TableFontSize%
\renewcommand{\arraystretch}{1.2}
\setlength{\dashlinedash}{0.4pt}
\setlength{\dashlinegap}{1.5pt}
\setlength{\arrayrulewidth}{0.3pt}

\newcommand{\TabMidRule}{\hdashline}

\newcommand{\AlgName}[1]{\multicolumn{1}{c}{#1}}

\begin{tabular}{ccrrrrrrrr}
  \toprule
  \multicolumn{2}{c}{$\yTarg \rightarrow \yAdv$}
      & \AlgName{\methodZero} & \AlgName{\method}  & \AlgName{\methodFin} \\
  \midrule
  Bird  & Dog
      & \PVal{0.773}{0.208}  & \PValB{0.892}{0.137}  & \PVal{0.678}{0.265}  \\\TabMidRule
  Frog  & Deer
      & \PVal{0.912}{0.120}  & \PValB{0.962}{0.100}  & \PVal{0.837}{0.160}  \\\TabMidRule
  Deer  & Frog
      & \PVal{0.803}{0.188}  & \PValB{0.888}{0.091}  & \PVal{0.604}{0.234}  \\
  \bottomrule
\end{tabular}
   }
\end{table}

\vfill
\phantom{.}
 
\newpage
\revTwo{%
\subsection{Evaluating on the Tail of the Attack's Influence Distribution}%
\label{sec:App:MoreExps:LeastInfluential}
}

\revTwo{%
When the number of attack instances is larger than necessary, it becomes increasingly important to detect all of the adversarial set (with very few false positives) in order to mitigate the attack and minimize damage.
As a complement to this paper's identification and mitigation experiments, here we evaluate one specific aspect of our method in one specific challenging setting: performance on the 10\%~most difficult instances in the attack.

Figure~\ref{fig:ExpRes:AdaptiveAttacker:Backdoor:CIFAR:BottomFilt} considers this case on \citepos{Weber:2021} three vision backdoor attack patterns.  Each attack's maximum perturbation distance is halved w.r.t.\ the earlier experiments (to reduce their individual influence and increase the challenge) and 1,000~adversarial instances (10\% of~\eqsmall{$\fullTrain$}) are used.  Each method's \textbf{\FullSetBarColor} bar represents the adversarial-instance identification AUPRC when considering the entire adversarial set~\eqsmall{$\advTrain$}. The \BottomFiltColor{} bar shows the adversarial-instance identification AUPRC for just the bottom 10\%~least influential backdoor training instances. This quantity is necessarily less than or equal to the AUPRC for the full~\eqsmall{$\advTrain$}.  This experiment mimics an attacker secretly concealing from our defense the top-90\% of \eqsmall{$\advTrain$}.
Figures~\ref{fig:ExpRes:AdaptiveAttacker:Pois:NLP:BottomFilt} and~\ref{fig:ExpRes:AdaptiveAttacker:Backdoor:Speech:BottomFilt} provide similar results for the \citepos{Wallace:2021} natural-language poisoning and \citepos{SpeechDataset} speech recognition backdoor attacks respectively.
}

\revTwo{%
  In summary, \bothMethod{} were highly adept at detecting these least influential instances for all three attacks.  In contrast, the baseline estimators' performance drops substantially -- generally to nearly zero AUPRC.
}

\begin{figure*}[h!]
  \centering
  \newcommand{\subfigureWidth}{0.48\textwidth}
\newcommand{\legendSpacer}{\hspace*{13pt}}

\begin{tikzpicture}
  \begin{axis}[%
    width=\textwidth,
      ybar,
      hide axis,  %
      xmin=0,  %
      xmax=1,
      ymin=0,
      ymax=1,
      scale only axis,width=1mm, %
      legend cell align={left},              %
      legend style={font=\legendFontSize},
      legend columns=4,
    ]
    \addplot [Bottom Filt] coordinates {(0,0)};
    \addlegendentry{Bottom 10\% of $\advTrain$\legendSpacer}

    \addplot [Full Set] coordinates {(0,0)};
    \addlegendentry{All of $\advTrain$}
  \end{axis}
\end{tikzpicture}
 
  \vspace{6pt}
  \begin{minipage}[t]{\subfigureWidth}
\pgfplotstableread[col sep=comma]{plots/data/bd_cifar_bottom_filt.csv}\datatable%
\begin{tikzpicture}
  \begin{axis}[
    xbar={\BarLineWidth},
    height={\BottomFiltHeight},%
    width={\BottomFiltWidth},%
    bar width={\DetectBarWidthVal},%
    xmin=0,         %
    xmax=1,         %
    xtick distance={0.20},%
    minor x tick num={1},%
    xmajorgrids,
    x tick label style={font=\plotFontSize,align=center},%
    xlabel={\plotFontSize AUPRC},
    ytick=data,     %
    ytick style={draw=none},%
    y tick label style={font=\footnotesize},%
    yticklabels={\method{}\ours{},\layer{}\ours{},\tracinCP{},\tracin{},Inf.\ Func.,Rep.\ Pt.},
    enlarge y limits={0.110},%
  ]
  \addplot [Bottom Filt] table [x index=3, y index=1] {\datatable};
  \addplot [Full Set] table [x index=2, y index=1] {\datatable};
  \end{axis}
\end{tikzpicture}
     \caption{%
      \revTwo{%
      \textit{Vision Backdoor Adversarial-Set Identification of the Least Influential Instances}:
      Mean AUPRC identifying adversarial set~\eqsmall{$\advTrain$} for \citepos{Weber:2021} vision backdoor with 1,000~backdoor instances.
      Even if the attacker can conceal the top~90\% most influential instances from our defense (\BottomFiltColor{}), \bothMethod{} still remain highly effective at identifying the adversarial set -- unlike the baselines whose performance severely degrades.
      Results are averaged across related experimental setups with ${{\geq}10}$~trials per setup.
    }%
    }
    \label{fig:ExpRes:AdaptiveAttacker:Backdoor:CIFAR:BottomFilt}
  \end{minipage}
  \hfill
  \begin{minipage}[t]{\subfigureWidth}
\pgfplotstableread[col sep=comma]{plots/data/pois_nlp_bottom_filt.csv}\datatable%
\begin{tikzpicture}
  \begin{axis}[
    xbar={\BarLineWidth},
    height={\BottomFiltHeight},%
    width={\BottomFiltWidth},%
    bar width={\DetectBarWidthVal},%
    xmin=0,         %
    xmax=1,         %
    xtick distance={0.20},%
    minor x tick num={1},%
    xmajorgrids,
    x tick label style={font=\plotFontSize,align=center},%
    xlabel={\plotFontSize AUPRC},
    ytick=data,     %
    legend style={at={(axis cs:65,0.2)},anchor=south west},
    ytick style={draw=none},%
    y tick label style={font=\plotFontSize},%
    enlarge y limits={0.100},%
    yticklabels={\method{}\ours{},\layer{}\ours{},\tracinCP{},\tracin{},Inf.\ Func.,Rep.\ Pt.},
  ]
  \addplot [Bottom Filt] table [x index=3, y index=1] {\datatable};
  \addplot [Full Set] table [x index=2, y index=1] {\datatable};
  \end{axis}
\end{tikzpicture}
     \caption{%
      \revTwo{%
      \textit{Natural-Language Poison Adversarial-Set Identification of the Least Influential Instances}:
      Mean AUPRC identifying adversarial set~\eqsmall{$\advTrain$} for \citepos{Wallace:2021} natural-language poisoning attack.
      Even if the attacker can conceal the top~90\% most influential instances (i.e.,~top 45/50 instances) from our defense (\BottomFiltColor{}), \bothMethod{} remain highly effective at identifying~\eqsmall{$\advTrain$} even for these 10\%~least-influential instances (bottom 5/50).
      Results are averaged across related experimental setups with ${10}$~trials per setup.
    }%
    }
    \label{fig:ExpRes:AdaptiveAttacker:Pois:NLP:BottomFilt}
  \end{minipage}

  \null
  \vfill
  \null

  \begin{minipage}[t]{\subfigureWidth}
    \centering
\pgfplotstableread[col sep=comma]{plots/data/bd_speech_bottom_filt.csv}\datatable%
\begin{tikzpicture}
  \begin{axis}[
    xbar={\BarLineWidth},
    height={\BottomFiltHeight},%
    width={\BottomFiltWidth},%
    bar width={\DetectBarWidthVal},%
    xmin=0,         %
    xmax=1,         %
    xtick distance={0.20},%
    minor x tick num={1},%
    xmajorgrids,
    x tick label style={font=\plotFontSize,align=center},%
    xlabel={\plotFontSize AUPRC},
    ytick=data,     %
    legend style={at={(axis cs:65,0.2)},anchor=south west},
    ytick style={draw=none},%
    y tick label style={font=\plotFontSize},%
    enlarge y limits={0.100},%
    yticklabels={\method{}\ours{},\layer{}\ours{},\tracinCP{},\tracin{},Inf.\ Func.,Rep.\ Pt.},
  ]
  \addplot [Bottom Filt] table [x index=3, y index=1] {\datatable};
  \addplot [Full Set] table [x index=2, y index=1] {\datatable};

  \end{axis}
\end{tikzpicture}
     \caption{%
      \revTwo{%
      \textit{Speech Backdoor Adversarial-Set Identification of the Least Influential Instances}:
      Mean AUPRC identifying adversarial set~\eqsmall{$\advTrain$} for \citepos{SpeechDataset} speech backdoor attack.
      Even if the attacker can conceal the top~90\% most influential instances (top 27/30 instances) from our defense (\BottomFiltColor{}), our method still remains highly effective at identifying the adversarial set using just the 10\%~least influential instances (bottom 3/30 instances).
      Results are averaged across related experimental setups with ${{\geq}10}$~trials per setup.
      }%
    }
    \label{fig:ExpRes:AdaptiveAttacker:Backdoor:Speech:BottomFilt}
  \end{minipage}
  \hfill
  \begin{minipage}[t]{\subfigureWidth}
    ~~
  \end{minipage}
\end{figure*}
 
\clearpage
\newpage
\subsection{Adversarial-Set Identification Execution Time}\label{sec:App:MoreExps:Runtime}

Recall from Section~\ref{sec:RenormInfluence:Measures} that \bothMethod{} is a renormalized version of \tracinCP{} that removes the low-loss penalty.  Therefore, \bothMethod{}'s execution time is, in essence, that of \tracinCP{}.  Table~\ref{tab:App:MoreExps:PoisonExecTime} compares the execution time of \tracinCP{}/\bothMethod{} to the other influence estimators.\footnote{Table~\ref{tab:App:MoreExps:PoisonExecTime} does not separately report \method{} and \layer{}'s execution times as they were consistently similar to \tracinCP{}.} All results were collected on an HPC system with 3~Intel E5\=/2690v4 cores, 48GB~of 2400MHz~DDR4 RAM, and \textit{one} NVIDIA Tesla~K80.
The
reported execution times consider calculating vector~$\infVec$ for training set~$\fullTrain$ w.r.t.\ a \textit{single} random test example~$\zHatTe$.   The only exception is the ``Amortized'' results where training gradient~($\gradI$) computation is amortized across multiple test examples allowing for significant speed-up (see Alg.~\ref{alg:CosIn}).   Amortization relies on simultaneously storing each test example's gradient~$\gradHatTe$ in GPU memory.  Therefore, the GPU model constrains amortization's possible benefits.  Table~\ref{tab:App:MoreExps:ExecTime:AmortizedCount} enumerates the number of concurrent test examples considered when calculating amortized results on a~K80.  More modern GPUs with larger onboard memory see considerable single and amortized speed-ups compared to Table~\ref{tab:App:MoreExps:PoisonExecTime}.

Recall from Section~\ref{sec:ExpRes} that natural-language poisoning attacked \robertaBase{} (125M parameters) as specified by \citet{Wallace:2021}. Large models can be slow to analyze -- even with amortization.  That is why natural language target identification uses the two-phase target identification procedure described in Section~\ref{sec:Method:TargetDetection}.  We observed no meaningful performance drop with this streamlined approach (see Figure~\ref{fig:ExpRes:TargDetect}).

Recall from Section~\ref{sec:RelatedWork:Influence} that \citepos{Yeh:2018:Representer} representer point estimator only considers the network's final linear classification layer, which is why it was the fastest. Notably, it also had the worst performance (see Figure~\ref{fig:ExpRes:AdvIdent}).

For completeness, \citepos{Peri:2020} \deepknn{} empirical clean-label data poisoning defense's mean execution time was~242s with standard deviation~1.1s.

Target identification execution times can be extrapolated from Table~\ref{tab:App:MoreExps:PoisonExecTime}'s ``Amortized'' values.

\phantom{.}
\vfill

\begin{table}[h]
  \centering
  \caption{\textit{Adversarial-Set Identification Execution Time}: Mean and standard deviation algorithm execution time (in seconds) to analyze a single test instance across ${{>}50}$~trials for \bothMethod{} and the influence estimator baselines on Section~\ref{sec:ExpRes:Attacks}'s four training-set attacks.  See Table~\ref{tab:App:MoreExps:ExecTime:AmortizedCount} for the number of parallel instances analyzed by \tracinCP{} and \bothMethod{} for each attack.}%
  \label{tab:App:MoreExps:PoisonExecTime}
  {
    \TableFontSize%
\newcommand{\BaseHead}[1]{\multirow{2}{*}{#1}}
\newcommand{\AtkHead}[1]{\BaseHead{#1}}

\renewcommand{\arraystretch}{1.2}
\setlength{\dashlinedash}{0.4pt}
\setlength{\dashlinegap}{1.5pt}
\setlength{\arrayrulewidth}{0.3pt}

\newcommand{\AttackTypeCDash}{\cdashline{1-8}}

\newcommand{\SingleExpCDash}{\cdashline{2-8}}

\newcommand{\phan}{\phantom{0}}
\newcommand{\phanC}{\phantom{0,}}
\newcommand{\NA}{\multicolumn{1}{c}{N/A}}
\newcommand{\TRes}[2]{${#1 \pm #2}$}
\newcommand{\Head}[1]{\multicolumn{1}{c}{#1}}
\begin{tabular}{lllrrrrr}
  \toprule
  \BaseHead{Attack}
  & \BaseHead{Type} & \BaseHead{Dataset}
  & \multicolumn{2}{c}{\tracinCP{} \& \bothMethod{}}  & \multicolumn{3}{c}{Others} \\\cmidrule(lr){4-5}\cmidrule(lr){6-8}
  &&    & \Head{Single}            & \Head{Amortized}         & \Head{\tracin{}}    & \Head{Inf.\ Func.}   & \Head{Rep.\ Pt.}  \\
  \midrule
  \AtkHead{Backdoor~}
    & Speech & \citep{SpeechDataset}
        & \TRes{2,605}{\phanC710}  & \TRes{489}{\phanC155}    & \TRes{894}{279}     & \TRes{4,595}{177}       & \TRes{49}{\phan1}   \\\SingleExpCDash
    & Vision & CIFAR10
        & \TRes{9,252}{2,253}      & \TRes{1,473}{\phanC408}  & \TRes{2,418}{738}   & \TRes{6,316}{250}       & \TRes{128}{\phan3}  \\\AttackTypeCDash
  \AtkHead{Poison}
    & NLP    & SST\=/2
        & \TRes{27,723}{7,933}     & \TRes{8,971}{3,719}      & \TRes{16,667}{187}  & \TRes{21,409}{\phan77}  & \TRes{1,697}{11}    \\\SingleExpCDash
    & Vision & CIFAR10
        & \TRes{24,267}{1,939}     & \TRes{2,088}{\phanC177}  & \TRes{5,910}{600}   & \TRes{15,634}{641}      & \TRes{187}{\phan2}  \\
  \bottomrule
\end{tabular}
   }
\end{table}

\vfill

\begin{table}[h]
  \centering
  \caption{%
    \textit{Single GPU Amortization}: Number of test examples analyzed by \tracinCP{}/\bothMethod{} when amortizing on an Nvidia Tesla K80 computation of training gradients~($\gradI$) in Table~\ref{tab:App:MoreExps:PoisonExecTime}'s ``Amortized'' execution time result.
  }%
  \label{tab:App:MoreExps:ExecTime:AmortizedCount}
  {
    \TableFontSize%
\newcommand{\BaseHead}[1]{\multirow{1}{*}{#1}}
\newcommand{\AtkHead}[1]{\multirow{2}{*}{#1}}

\renewcommand{\arraystretch}{1.2}
\setlength{\dashlinedash}{0.4pt}
\setlength{\dashlinegap}{1.5pt}
\setlength{\arrayrulewidth}{0.3pt}

\newcommand{\AttackTypeCDash}{\cdashline{1-3}}

\newcommand{\SingleExpCDash}{\cdashline{2-3}}

\newcommand{\TRes}[2]{${#1 \pm #2}$}
\newcommand{\Head}[1]{\multicolumn{1}{c}{#1}}
\begin{tabular}{llr}
  \toprule
  \BaseHead{Attack}
    & \BaseHead{Type} & \\
  \midrule
  \AtkHead{Backdoor~}
    & Speech          &   20   \\\SingleExpCDash
    & Vision          &   16   \\\AttackTypeCDash
  \AtkHead{Poison}
    & NLP             &   5    \\\SingleExpCDash
    & Vision          &   135  \\
  \bottomrule
\end{tabular}
   }
\end{table}

\vfill
\phantom{.}
   \stopcontents  %

\end{document}